\documentclass[11pt]{article}

\usepackage{fullpage}

\usepackage{microtype}
\usepackage{graphicx}
\usepackage{booktabs,bookmark} 

\usepackage{hyperref}

\usepackage{natbib}

\usepackage{algorithm}
\usepackage{algorithmic}

\usepackage{amsfonts}       
\usepackage{nicefrac}       

\usepackage{epstopdf}
\usepackage{enumitem}
\usepackage{courier}
\usepackage{mathtools}
\usepackage{color,amsmath,amssymb,multirow}
\usepackage{multicol}

\newcommand{\bulletpoint}{\noindent$\bullet$\;\;}

\def \EE {\mathbb{E}}

\def\mat#1{\mathbf{#1}}
\def\vec#1{\mathbf{#1}}

\def\R{\mathbb{R}}

\def\M{\mathcal{M}}

\def\P{\mathcal{P}}

\def\bzero{{\mathbf 0}}

\def\bg{{\mathbf g}}
\def\bl{{\mathbf l}}
\def\br{{\mathbf r}}

\def\bv{\mathbf v}

\def\bA{\mathbf A}

\def\bG{\mathbf G}
\def\bH{\mathbf H}
\def\bI{\mathbf I}
\def\bL{\mathbf L}

\def\bR{\mathbf R}

\def\bV{{\mathbf V}}

\newcommand{\ie}{\textit{, i.e., }}

\newcommand{\gradf}{{\rm grad} f}

\newcommand\norm[1]{\left\lVert#1\right\rVert}

\usepackage{amsmath}
\usepackage{amsthm}
\newtheorem{Def}{Definition}[section]
\newtheorem{Thm}[Def]{Theorem}
\newtheorem{Lem}[Def]{Lemma}

\newtheorem{assumption}{Assumption}
\newtheorem{Cor}[Def]{Corollary}

\newcommand{\changeHK}[1]{#1}
\newcommand{\changeHKK}[1]{#1}

\newif\iflongversion
 \longversionfalse

\begin{document} 
\title{Riemannian adaptive stochastic gradient algorithms on\\ matrix manifolds}

\author{Hiroyuki Kasai\thanks{The University of Electro-Communications, Japan (e-mail: kasai@is.uec.ac.jp)} \and Pratik Jawanpuria\thanks{Microsoft, India (e-mail: pratik.jawanpuria@microsoft.com)} \and Bamdev Mishra\thanks{Microsoft, India (e-mail: bamdevm@microsoft.com)}}



\maketitle

\begin{abstract}
Adaptive stochastic gradient algorithms in the Euclidean space have attracted much attention lately. Such explorations on Riemannian manifolds, on the other hand, are relatively new, limited, and challenging. This is because of the intrinsic non-linear structure of the underlying manifold and the absence of a canonical coordinate system. In machine learning applications, however, most manifolds of interest are represented as matrices with notions of row and column subspaces. In addition, the implicit manifold-related constraints may also lie on such subspaces. For example, the Grassmann manifold is the set of column subspaces. To this end, such a rich structure should not be lost by transforming matrices to just a stack of vectors while developing optimization algorithms on manifolds. We propose novel stochastic gradient algorithms for problems on Riemannian matrix manifolds by adapting the row and column subspaces of gradients. Our algorithms are provably convergent and they achieve the convergence rate of order $\mathcal{O}(\log (T)/\sqrt{T})$, where $T$ is the number of iterations. Our experiments illustrate the efficacy of the proposed algorithms on several applications. 
\end{abstract}

\section{Introduction}
\label{Sec:Introduction}


Large-scale machine learning applications are predominantly trained using stochastic gradient descent (SGD) \citep{Bottou_SIAMRev_2018} based algorithms today. A (sub)class of such algorithms that has become increasingly common lately adapts the learning rate of each coordinate of the gradient vector based on past iterates \citep{mcmahan10a,Duchi_JMLR_2011_s}. A key motivation for this is to have different learning rates for different coordinates \citep{Zeiler_arXiv_2012,pennington14}, a feature which vanilla SGD lacks. ADAM \citep{Kingma_ICLR_2015}, arguably the most popular adaptive gradient method, additionally employs a momentum term to modify the search direction as well. Adaptive gradient methods have enjoyed varying degrees of success in various applications \citep{pennington14,wilson17a,zaheer18a,shah18a}.  

In this paper, we focus on adaptive stochastic gradient algorithms on Riemannian manifolds. 
Riemannian geometry is a generalization of the Euclidean geometry \citep{Lee03a}. 
It includes several non-Euclidean spaces such as set of symmetric positive-definite matrices, and set of orthogonal matrices, to name a few. Numerous machine learning problems can be cast as an optimization problem on Riemannian manifolds. Examples include principal component analysis (PCA), matrix completion \citep{Vandereycken_SIAMOpt_2013_s,Mishra_ICDC_2014_s,boumal15a_s,jawanpuria18a}, learning taxonomy or latent hierarchies \citep{nickel17a,nickel18a,ganea18a}, deep metric learning \citep{harandi17a}, multi-task learning \citep{Mishra_ML_2019}, applications in computer vision \citep{Kasai_ICML_2016_s,Harandi_PAMI_2017,Nimishakavi_NeurIPS_2018}, bilingual lexicon induction \citep{jawanpuria19a}, among others. Several Euclidean algorithms, e.g., steepest descent, conjugate gradients, and trust-regions, have been generalized to Riemannian manifolds \citep{Absil_OptAlgMatManifold_2008}. 

\citet{Bonnabel_IEEETAC_2013_s} first proposed the Riemannian SGD (RSGD) algorithm, which is a generalization of the (Euclidean) SGD algorithm to Riemannian manifolds. However, such a generalization of Euclidean adaptive gradient algorithm is relatively unexplored. 
A key difficulty in this regard is the absence of a canonical coordinate system on manifolds and the inherent non-linear geometry of the manifold. 
Recent works in this direction compute Euclidean-style adaptive weights in the Riemannian setting, ignoring the geometry of the underlying manifold. \citet{Roy_CVPR_2018} propose a RMSProp-style algorithm for manifolds, viewing the gradient as a vector and computing a corresponding adaptive weight vector. On the other hand, \citet{cho17a,Becigneul_ICLR_2019} adapt the step size by computing a scalar weight instead of directly adapting the gradient in a momentum-based Riemannian AMSGrad/ADAM-style algorithm. 


We develop a novel approach of adapting the Riemannian gradient, which allows to exploit the structure of the underlying manifolds. In particular, we propose to adapt the row and column subspaces of the Riemannian stochastic gradient $\bG$. Euclidean adaptive algorithms compute (positive) adaptive weight vectors for the gradient vectors. Our model computes \textit{left} and \textit{right} adaptive weight matrices for $\bG$ denoted by $\bL$ and $\bR$, respectively.  $\bL$ adapts the row subspace of $\bG$ and $\bR$ adapts the column subspace of $\bG$. Both $\bL$ and $\bR$ are positive definite matrices and are computed using the row covariance and column covariance matrices of $\bG$, respectively. For computational efficiency, we model $\bL$ and $\bR$ as diagonal matrices, taking cue from AdaGrad \citep{Duchi_JMLR_2011_s}. Overall, we propose computationally efficient Riemannian adaptive stochastic gradient algorithms, henceforth collectively termed as RASA. 

Under a set of mild conditions, our analysis guarantees the convergence of our algorithms with a rate of convergence of order $O(\mathrm{log}(T)/\sqrt{T})$, where $T$ is the number of iterations. To the best of our knowledge, ours is the first Riemannian adaptive gradient algorithm to provide convergence analysis for non-convex stochastic optimization setting. Among the existing works, \citet{Becigneul_ICLR_2019} provide convergence analysis only for geodesically convex functions. 

Empirically, we compare our algorithms with the existing Riemannian adaptive gradient algorithms. In applications such as principal components analysis and matrix completion, we observe that our algorithms perform better than the baselines in most experiments both on synthetic and real-world datasets. 

The main contributions of this work are:\newline
\bulletpoint we propose a principled approach for modeling adaptive weights for Riemannian stochastic gradient. We model adaptive weight matrices for row and column subspaces exploiting the geometry of the manifold. \newline
\bulletpoint we develop efficient Riemannian adaptive stochastic gradient algorithms, based on the proposed modeling approach. \newline
\bulletpoint we provide convergence analysis of our algorithm, under a set of mild conditions. Our algorithms achieve a rate of convergence order $O(\mathrm{log}(T)/\sqrt{T})$, where $T$ is the number of iterations, for non-convex stochastic optimization.


The paper is organized as follows. Section~\ref{sec:preliminaries} discusses preliminaries of the Riemannian (stochastic) optimization framework and adaptive stochastic gradient algorithms in the Euclidean space. In  Section~\ref{Sec:ProposdAlg}, we present our Riemannian adaptive stochastic algorithms. Section~\ref{Sec:MainResults} provides the convergence analysis of the proposed algorithms. 
Related works are discussed in Section~\ref{sec:discussion}, while empirical results are presented in Section~\ref{Sec:NumericalEvaluations}. Section~\ref{sec:conclusion} concludes the paper. 
\section{Preliminaries}\label{sec:preliminaries}
In this section, we briefly summarize various concepts of Riemannian geometry and optimization. We refer the interested readers to \citep{Absil_OptAlgMatManifold_2008} for more details.

\subsection*{Notions on Riemannian manifolds}

Informally, a manifold is a generalization of the notion of surface to higher dimensions. A manifold $\M$ of dimension $d$ can be locally approximated by $\R^d$. For example, the Stiefel manifold is the set of $n\times r$ orthonormal matrices and has dimension $d=np-p(p+1)/2$ \citep{Absil_OptAlgMatManifold_2008}. 
The first order approximation of $\M$ around a point $x\in\M$ is a $d$ dimensional vector space and is known as the tangent space $T_x\M$. The tangent space $T_x\M$ contains all tangent vectors to $\M$ at $x$. Corresponding to each tangent space $T_x\M$, we define an inner product $g_x(\cdot,\cdot):T_x\M\times T_x\M\rightarrow\R$, which varies smoothly with $x$. This inner product is termed as the Riemannian metric in differential geometry. 
A Riemannian manifold comprises a manifold $\M$ along with the collection of Riemannian metric $g\coloneqq (g_x)_{x\in\M}$. 

It should be noted that $g_x$ induces a norm on the tangent space $T_x\M$ at $x$: $\norm{\xi}_x=\sqrt{g_x(\xi,\xi)}$, where $\xi\in T_x\M$. This in turn induces a (local) distance function on $\M$, which is employed to compute length of a path between two points on $\M$. A geodesic between two points is a path of shortest length on the manifold, generalizing the notion of straight line in Euclidean space to manifolds. 

Minimizing an objective function $f$ over the manifold $\M$ requires the notion of gradients at every point $x\in\M$. 
Correspondingly, the Riemannian gradient $\gradf(x)$ at $x\in\M$ is defined as an element of the tangent space $T_x\M$.

\subsection*{Riemannian stochastic gradient update}

The \textit{Riemannian stochastic gradient descent} (RSGD) update \citep{Bonnabel_IEEETAC_2013_s} is given by 
\begin{equation}\label{eqn:riemanniansgd}
x_{t+1}  =  R_{x_{t}}(-\alpha_t \gradf_{t}(x_t)),
\end{equation}
where $\alpha_t>0$ is a (decaying) step-size. 
The term $\gradf_{i}(x)$ represents a Riemannian stochastic gradient and is typically computed as the Riemannian gradient of the objective function on a batch of data $y_{i}$\ie $\gradf_i(x)\coloneqq\gradf_i(x;y_{i})$. 
The retraction, $R_{x} :T_{x}\mathcal{M} \rightarrow \mathcal{M}: \zeta \mapsto R_{x} (\zeta)$, maps tangent space $T_x\mathcal{M}$ onto $\mathcal{M}$ with a local rigidity condition that preserves {the} gradients at $x$ \citep[Sec.~4.1]{Absil_OptAlgMatManifold_2008}. The exponential mapping is an instance of the retraction.

It should be noted that when the manifold is the Euclidean space $(\M=\R^d)$ and the Riemannian metric is the standard Euclidean inner product, the RSGD update (\ref{eqn:riemanniansgd}) simplifies to the (vanilla) stochastic gradient descent (SGD) update \citep{Bottou_SIAMRev_2018}: 
\begin{equation}
x_{t+1}=x_t - \alpha_t\nabla f_{t}(x_t),
\end{equation}
where $\nabla f_{i}(x)$ is a stochastic gradient commonly computed as the gradient of the objective function on a batch of data $y_i$\ie $\nabla f_{i}(x)=\nabla f_i(x;y_{i})$. 
The retraction function in the Euclidean space is given by $R_x(z)=x+z$.

\subsection*{Euclidean adaptive stochastic gradient updates}

The adaptive stochastic gradient methods employ the past gradients to compute a local distance measure and subsequently rescale the learning rate of each coordinate of the (Euclidean) model parameter. The AdaGrad algorithm \citep{Duchi_JMLR_2011_s} introduced the following update, when $x\in\R^d$:
\begin{equation}\label{eqn:euclideanAdaptiveUpdate}
x_{t+1}=x_t - \alpha_t\bV_t^{-1/2}\nabla f_{t}(x_t),
\end{equation} 
where $\bV_t$ is a diagonal matrix corresponding to the vector $\bv_t=\sum_{k=1}^t\nabla f_k(x_k)\circ\nabla f_k(x_k)$\ie $\bV_t=\mathrm{Diag}(\bv_t)$, and the symbol `$\circ$' denotes the entry-wise product (Hadamard product). 
RMSProp \citep{Tieleman_Unpulished_2012} proposed an exponentially moving average of the adaptive term $\bv_t$ as follows: $\bv_t=\beta\bv_{t-1} + (1-\beta)\nabla f_t(x_t)\circ\nabla f_t(x_t)$, where $\beta\in(0,1)$ is another hyper-parameter. 

%

Various stochastic gradient algorithms also employ an adaptive momentum term, ADAM \citep{Kingma_ICLR_2015} being the most popular of them. In this paper, we restrict our discussion to algorithms \emph{without} a momentum term because the calculation of the momentum term needs additional \textit{vector transport operation} every iteration \citep[Sec.~8.1]{Absil_OptAlgMatManifold_2008}. 
However, we include existing adaptive momentum based Riemannian stochastic gradient algorithms \citep{Becigneul_ICLR_2019} as baselines in our experiments. 
\section{Riemannian adaptive stochastic gradient}\label{Sec:ProposdAlg}

In this section, we propose novel adaptive stochastic gradient algorithms in Riemannian setting. 
We first introduce a few notations, which will be useful in the rest of the manuscript. 
For given vectors $\vec{a}$ and $\vec{b}$, $\sqrt{\vec{a}}$ or $\vec{a}^{1/2}$ is element-wise square root, $\vec{a}^2$ is element-wise square, $\vec{a}/\vec{b}$ denotes element-wise division, and $\max(\vec{a},\vec{b})$ denotes element-wise maximum. For a given matrix $\bA\in\R^{n\times r}$, $\mathrm{vec}(\bA)\in\R^{nr}$ denotes its vectorial representation. The term $\bI_d$ represents the identity matrix of $d$ dimension.

We consider the following problem 
\begin{equation}\label{Eq:ProblemFormulation}
\min_{x\in\M} f(x), 
\end{equation}
where elements of the manifold $\mathcal{M}$ are represented as matrices of size $n\times r$.
Such a geometry is true for most Riemannian manifolds employed in machine learning applications. Prominent examples of such manifolds include the Stiefel, Grassmann, spherical, symmetric positive definite, spectrahedron, and hyperbolic manifolds, to name a few. 

In this paper, we are interested in a stochastic optimization setting where we assume that at each time step $t$, the algorithm computes a feasible point $x_{t+1}\in\M$ given the current iterate $x_t\in\M$. 
For simplicity, the Riemannian stochastic gradient $\gradf_{t}(x_t)\in\R^{n\times r}$ at a given iterate $x_t$ is denoted by $\bG_t$. 
We exploit the matrix structure of $\bG_t$ by proposing separate adaptive weight \textit{matrices} corresponding to row and column subspaces with  $\bL_t\in \R^{n \times n}$ and $\bR_t\in\R^{r \times r}$, respectively. These weights are computed in an exponentially weighted manner as 
\begin{equation}\label{eqn:LR_updates}
\begin{array}{lll}
\bL_t &=& \beta \bL_{t-1} + (1-\beta)\bG_t \bG_t^\top/r,\\
\bR_t &=& \beta \bR_{t-1} + (1-\beta)\bG_t^\top\bG_t/n, 
\end{array}
\end{equation}
where $\beta \in (0,1)$ is a hyper-parameter. 
It should be noted that $\bG_t \bG_t^\top/r$ and $\bG_t^\top\bG_t/n$ correspond to row and column covariance matrices, respectively. 
We propose to \textit{adapt}\footnote{For numerical stability while computing square root of $\bL_t$ and $\bR_t$ in (\ref{eqn:adaptedRgrad}), a small $\epsilon=10^{-8}$ is added to their diagonal entries.} the Riemannian gradient $\bG_t$ as 
\begin{equation}\label{eqn:adaptedRgrad}
\tilde{\bG}_t=\bL_t^{-1/4} \bG_t \bR_t^{-1/4}. 
\end{equation}
We observe that $\bL_{t}$ and $\bR_{t}$ weight matrices suitably modify the Riemannian gradient $\bG_t$ while respecting the matrix structure of $\bG_t$, rather than plainly viewing $\bG_t$ as a vector in $\R^{nr}$. 
{For instance in the Grassmann manifold, our right adaptive weight matrix $\bR_t$ essentially provides different learning rates to different column subspaces instead of giving different learning rates to the individual entries (coordinates) of $\bG_t$.} 
Our framework allows adapting only the row or the column space of $\bG_t$ as well by simply substituting $\bR_t=\bI_r$ or $\bL_t=\bI_n$, respectively. 
In vectorial representation, the adaptive update (\ref{eqn:adaptedRgrad}) may be viewed as 
\begin{equation}\label{eqn:vecAdaptedRgrad}
\mathrm{vec}(\tilde{\bG}_t)= \bV_t^{-1/2}\mathrm{vec}(\bG_t). 
\end{equation}
where $\bV_t=\bR_t^{1/2}\otimes\bL_t^{1/2}$. 

It should be noted that $\tilde{\bG}_t$ does not lie on the tangent plane $T_{x_t}\M$. Hence, the proposed adaptive Riemannian gradient is $\P_{x_t}(\tilde{\bG}_t)$, where $\P_x$ is a linear operator to project onto the tangent space $T_x\M$. 
%
Overall, the proposed update is 
\begin{equation}\label{eqn:generalUpdate}
x_{t+1} = R_{x_t}(-\alpha_t \mathcal{P}_{x_t}( \tilde{\bG}_t)),
\end{equation}
where $\tilde{\bG}_t$ is defined as in (\ref{eqn:adaptedRgrad}). 

{\bf Remark 1:} The following generalization of the proposed adaptation of the Riemannian gradient (\ref{eqn:adaptedRgrad}) is permissible within our framework: 
\begin{equation}\label{eqn:genAdaptedRgrad}
\tilde{\bG}_t=\bL_t^{-1/p} \bG_t \bR_t^{-1/q}. 
\end{equation}
where $1/p + 1/q = 1/2$ and $p,q>0$. 
For clarity in exposition, we choose $p=q=4$ in the remainder of the paper, including the experiment section. 
The proposed theoretical results (discussed in Section~\ref{Sec:MainResults}) hold for other permissible values of $p$ and $q$ as well. 



%

\subsection*{Diagonal adaptation and the proposed algorithm} 
The proposed \textit{full} matrix update (\ref{eqn:generalUpdate}) is computationally prohibitive for high dimensional data. Hence, taking cue from the Euclidean adaptive gradient algorithms  \citep{Duchi_JMLR_2011_s}, we propose diagonal adaptation of adaptive weight matrices. In particular, we model $\bL_t$ and $\bR_t$ as diagonal matrices corresponding to vectors $\bl_t\in\R^n$ and $\br_t\in\R^r$, which are computed as 
\begin{equation}\label{eqn:lr_updates}
\begin{array}{lll}
\bl_t &=& \beta \bl_{t-1} + (1-\beta)\mathrm{diag}(\bG_t \bG_t^\top),\\
\br_t &=& \beta \br_{t-1} + (1-\beta)\mathrm{diag}(\bG_t^\top\bG_t), 
\end{array}
\end{equation}
where the function $\mathrm{diag}(\cdot)$ returns the diagonal vector of a square matrix. 
The diagonal adaptation significantly reduces the effective dimension of adaptive weights to $n+r$ and the computational complexity of adapting $\bG_t\in\R^{n\times r}$ with (diagonal) weight matrices to $O(nr)$. 

The final algorithm is summarized in Algorithm~\ref{Alg:R-AGD-diag} and is henceforth referred to as RASA (\textbf{R}iemannian \textbf{A}daptive \textbf{S}tochastic gradient \textbf{A}lgorithm on matrix manifolds). 
{RASA generalizes Euclidean AdaGrad-type and RMSProp-type algorithms to Riemannian matrix manifolds}. 

We employ $\hat{\bl}_t =\max(\hat{\bl}_{t-1}, \bl_t)$ and $\hat{\br}_t = \max(\hat{\br}_{t-1}, \br_t)$ in our final update (step~$8$,  Algorithm~\ref{Alg:R-AGD-diag}).  
The non-decreasing sequence of the adaptive weights (along with non-increasing sequence of step size $\alpha_t$) ensures the convergence of RASA. 
Such a sequence for adaptive weights was introduced by \citet{Reddi_ICLR_2018} in ADAM to provide convergence guarantees.  
We present our convergence analysis of Algorithm~\ref{Alg:R-AGD-diag} in Section~\ref{Sec:MainResults}. 


We develop the following variants of the RASA algorithm: 
\begin{itemize}
\item RASA-L, which adapts only the row subspace.
\item RASA-R, which adapts only the column subspace. 
\item RASA-LR, which adapts both the row and column subspaces. 
\end{itemize}
It should be noted that when $\M$ is isomorphic to $\R^n$, the adaptive weight vector $\br_t$ reduces to a scalar.  Hence, in this setting, RASA-LR is equivalent to RASA-L (up to a scalar multiple of the adaptive weight vector $\bl_t$). 

%
%

\begin{algorithm}[t]
\caption{Riemannian adaptive stochastic algorithm}
\label{Alg:R-AGD-diag}
\begin{algorithmic}[1]
\REQUIRE{Step size $\{\alpha_t\}_{t=1}^T$, hyper-parameter $\beta$.}
\STATE{Initialize $x_1\in\M, \bl_{0}=\hat{\bl}_{0}=\bzero_{n}, \br_{0}=\hat{\br}_{0}=\bzero_{r}$.}
\FOR{$t=1,2, \ldots, T$} 
\STATE{Compute Riemannian stochastic gradient $\bG_t = \gradf_{t}(x_{t})$.}
\STATE{Update $\bl_t = \beta \bl_{t-1} +  (1-\beta){\rm diag}(\bG_t\bG_t^T)/r$.}
\STATE{Calculate $\hat{\bl}_t = \max(\hat{\bl}_{t-1}, \bl_t)$.}
\STATE{Update $\br_t = \beta \br_{t-1} + (1-\beta){\rm diag}(\bG_t^T\bG_t)/n$.}
\STATE{Calculate $\hat{\br}_t = \max(\hat{\br}_{t-1}, \br_t)$.}
\STATE{$x_{t+1} = R_{\scriptsize x_{t}}(-\alpha_t \P_{\scriptsize x_t} ({\rm Diag}(\hat{\bl}_t^{-1/4})\bG_t {\rm Diag}(\hat{\br}_t^{-1/4})))$.}
\ENDFOR
\end{algorithmic}
\end{algorithm}

\section{Convergence rate analysis}\label{Sec:MainResults}
For the purpose of analyzing the convergence rate for the proposed Algorithm~\ref{Alg:R-AGD-diag}, we view our proposed update, step~8 in Algorithm~\ref{Alg:R-AGD-diag}, as 
\begin{equation}\label{eq:retraction_RASA}
x_{t+1} = R_{x_t}(-\alpha_t\P_{x_t}(\hat{\bV}_t^{-1/2}\bg_t(x_t))),
\end{equation}
where $t$ represents the iteration number, $\hat{\bV}_t=\mathrm{Diag}(\hat{\bv}_t)$, and $\hat{\bv}_t$ is defined as
\begin{equation}
\label{Eq:vt_definition}
\hat{\bv}_t=\hat{\br}_t^{1/2}\otimes\hat{\bl}_t^{1/2}, 
\end{equation}
where the symbol `$\otimes$' represents the Kronecker product and $\bg_t(x)$ denotes the vectorized representation of $\gradf_{t}(x)$. Furthermore, $\bg(x)$ denotes the vectorized representation of $\gradf(x)$. 
Hence, $\hat{\bv}_t$, $\bg_1(x)$, $\bg_t(x)$, and $\bg(x)$ are $nr$ dimensional vectors in our analysis. 
We first summarize essential assumptions and a lemma before discussing our main results. 


\subsection{Definitions and assumptions}
{For simplicity}, the analysis hereinafter assumes the standard Riemannian metric as $g(\xi, \eta)_x  \coloneqq  \langle \xi, \eta \rangle_2$ (Euclidean inner product) and the standard Riemannian norm as $\| \xi \|_x \coloneqq \| \xi \|_2$ (Euclidean norm), where $\xi,\eta \in T_x\mathcal{M}$.  


\begin{Def}{\rm \textbf{Upper-Hessian bounded} \citep{Absil_OptAlgMatManifold_2008}}. 
\label{Def:UpperBoundHessian}
The function $f$ is said to be upper-Hessian bounded in $\mathcal{U} \subset \mathcal{M}$ with respect to {retraction} $R$ if there exists a constant $L > 0$ such that $\frac{d^2 f(R_x(t\eta))}{dt^2} \leq L$, for all $x \in \mathcal{U}$ and $\eta \in T_x\mathcal{M}$ with $\| \eta \|_x=1$, and all $t$ such that $R_x(\tau \eta) \in \mathcal{U}$ for all $\tau \in [0,t]$.
\end{Def}
The above class of functions in the Riemannian setting corresponds to the set of continuous functions with Lipschitz continuous gradients in the Euclidean space. 

We now state our assumptions on problem (\ref{Eq:ProblemFormulation}).
\begin{assumption}
\label{Assump:1}
For problem (\ref{Eq:ProblemFormulation}), we assume the following:

(A1) The function $f$ is  continuously differentiable and is lower bounded\ie $f(x^*) > -\infty$  where $x^*$ is an optimal solution of (\ref{Eq:ProblemFormulation}). 

(A2) The function $f$ has $H$-bounded Riemannian stochastic gradient\ie $\norm{\gradf_i(x)}_F  \leq H$ or equivalently $\norm{\bg_i(x)}_{2}  \leq H$. 


(A3) The function $f$ is upper-Hessian bounded.

\end{assumption}
Existing works \citep{Zhou_arXiv_2018,Chen_ICLR_2019}, which focus on the convergence of (Euclidean) ADAM \citep{Kingma_ICLR_2015}, use a Euclidean variant of A2. Furthermore, A2 holds when the manifold is compact like the Grassmann manifold \citep{Absil_OptAlgMatManifold_2008}, or through slight modification of the objective function and the algorithm \citep{Kasai_ICML_2018}. 

%
Finally, we define Retraction $L$-smooth functions via the following lemma. 
\begin{Lem}
\label{Lemma:DescentLemmaRetraction}
{\rm \textbf{Retraction $L$-smooth} \citep{Huang_SIOPT_2015,Kasai_ICML_2018}}. Suppose Assumption \ref{Assump:1} holds. Then, for all $x, y \in \mathcal{M}$ and constant $L>0$ in Definition \ref{Def:UpperBoundHessian}, we have
\begin{equation}
\label{Eq:DescentLemmaRetraction}
f(z)  \leq   f(x) + \langle  \gradf (x), \xi \rangle_2 + \frac{1}{2} L \| \xi \|_2^2,
\end{equation}
where $\xi \in T_x\mathcal{M}$ and $R_x(\xi)=z$. In particular, such a function $f$ is called retraction $L$-smooth with respect to $R$.
\end{Lem}
The proof of Lemma~\ref{Lemma:DescentLemmaRetraction} is provided in \citep[Lemma 3.5]{Kasai_ICML_2018}. Refer \citep{Boumal_IMAJ_2018} for other means to ensure Retraction $L$-smoothness of a function. 

\subsection{Convergence analysis of Algorithm \ref{Alg:R-AGD-diag}}
\label{Sec:ConvergenceAnalysis}
Our proof structure extends existing convergence analysis of ADAM-type algorithms in the Euclidean space \citep{Zhou_arXiv_2018,Chen_ICLR_2019} into the Riemannian setting. 
However, such an extension is non-trivial and key challenges include
\begin{itemize}
 \setlength{\itemsep}{0pt}%
    \setlength{\parskip}{0pt}%
\item {the update formula based on $\hat{\br}_t$ and $\hat{\bl}_t$ in Algorithm~\ref{Alg:R-AGD-diag} requires the upper bound of $\hat{\bv}_t$ defined in (\ref{Eq:vt_definition}).}
\item {projection onto the tangent space: the Riemannian gradient adaptively weighted by $\hat{\br}_t$ and $\hat{\bl}_t$ needs to be projected back onto a tangent space. Our analysis needs to additionally take care of this. }
\end{itemize}

To this end, we first provide a lemma for the upper bounds of the full Riemannian gradient as well as the elements of the vector $\hat{\bv}_t$, where $\hat{\bl}_t$ and $\hat{\br}_t$ are defined in steps $4-7$ of Algorithm~\ref{Alg:R-AGD-diag}. 
\begin{Lem}
\label{Lem:uppder_bound_v}
Let $\hat{\bv}_t=\hat{\br}_t^{1/2}\otimes\hat{\bl}_t^{1/2}$, where $\hat{\bl}_t$ and $\hat{\br}_t$ are defined in steps $4-7$ in Algorithm~\ref{Alg:R-AGD-diag} . Then, under Assumption \ref{Assump:1}, we have the following results \newline
\ \ \bulletpoint $\| \gradf(x) \|_{F} \leq H$, and\newline
\ \ \bulletpoint the $j$-th element of $\hat{\vec{v}}_t$  satisfies $(\hat{\vec{v}}_{t})_j  \leq H^2$.
%
\end{Lem}
We now present our main result in Theorem \ref{Thm:main_theorem}. 
%
\begin{Thm}
\label{Thm:main_theorem}
Let $\{x_t\}$ and $\{\hat{\vec{v}}_t\}$ be the sequences obtained from Algorithm \ref{Alg:R-AGD-diag}, where $\hat{\bv}_t=\hat{\br}_t^{1/2}\otimes\hat{\bl}_t^{1/2}$. Then, under Assumption \ref{Assump:1}, we have the following results for Algorithm \ref{Alg:R-AGD-diag} 
\begin{align}
\label{Eq:main_theorem}
& \EE  \Bigg[ \sum_{t=2}^T \alpha_{t-1}  \left\langle \vec{g} (x_{t}), \frac{ \vec{g} (x_{t})}{\sqrt{\hat{\vec{v}}_{t-1}}} \right\rangle_2\Bigg] \nonumber\\
 & \leq   
 \EE  \Bigg[\frac{L}{2}   \sum_{t=1}^T\left\lVert \frac{\alpha_t  \vec{g}_{t}(x_t)}{\sqrt{\hat{\vec{v}}_t}} \right\rVert^2_{2}+H^2  \sum_{t=2}^T \left\lVert\frac{\alpha_{t}}{\sqrt{\hat{\vec{v}}_t}} -  \frac{\alpha_{t-1}}{\sqrt{\hat{\vec{v}}_{t-1}}}  \right\rVert_1 \nonumber \\
 &\hspace*{0.4cm}+ C,
 \end{align}
 where $C$ is a constant term independent of $T$. 
\end{Thm}
\noindent 
{\bf Proof sketch:} We first derive a Retraction $L$-smooth inequality from Lemma \ref{Lemma:DescentLemmaRetraction} with respect to the adaptive gradient $\P_{x_t}(\hat{\mat{V}}^{-1/2}_t \vec{g}_{t}(x_t))$. Exploiting the symmetric property of $\P_{x_t}$, we obtain the upper bound of $-\langle \vec{g}(x_t), \alpha_t\P_{x_t}(\hat{\mat{V}}_{t}^{-1/2}\vec{g}_{t}(x_t))\rangle_2$. Here, to remove the dependency of $\hat{\mat{V}}_t$ on $\vec{g}_t(x_t)$, $\hat{\mat{V}}^{-1/2}_{t-1}\vec{g}_t(x_t)$ is evaluated instead of $\hat{\mat{V}}^{-1/2}_{t}\vec{g}_t(x_t)$. Then, taking the expectation and telescoping the inequality, we obtain the desired result. The complete proof is in Section \ref{app:sec:proofs}.

{\bf Remark 2:} 
The first term of (\ref{Eq:main_theorem}) represents the weighted term of the sum of squared step length, and appears in the standard RSGD \citep{Bonnabel_IEEETAC_2013_s}. From this point of view, the advantage of the adaptive gradient algorithms can be seen that they reduce the effect of the term $\EE  [ \sum_{t=1}^T\left\lVert  \alpha_t  \vec{g}_{t}(x_t)/\sqrt{\hat{\vec{v}}_t} \right\rVert^2_{2}]$ compared with the standard RSGD algorithm.

{\bf Remark 3:} Theorem~\ref{Thm:main_theorem} reproduces the results of \citep[Theorem 3.1]{Chen_ICLR_2019} when $\M$ is defined as the Euclidean space.

In Corollary~\ref{Cor:convergence_analysis1}, we derive a convergence rate of Algorithm~\ref{Alg:R-AGD-diag}. 
\begin{Cor}[Convergence rate of Algorithm~\ref{Alg:R-AGD-diag}]
\label{Cor:convergence_analysis1}
Let $\{x_t\}$ and $\{\hat{\vec{v}}_t\}$ be the sequences obtained from Algorithm \ref{Alg:R-AGD-diag}, where $\hat{\bv}_t=\hat{\br}_t^{1/2}\otimes\hat{\bl}_t^{1/2}$. 
Let $\alpha_t=1/\sqrt{t}$ and $\min_{j \in [d]} \sqrt{(\hat{\vec{v}}_1)_j}$ is lower-bounded by a constant $c > 0$, where $d$ is the dimension of the manifold $\mathcal{M}$. Then, under Assumption \ref{Assump:1}, the output of $x_{t}$ of Algorithm \ref{Alg:R-AGD-diag} satisfies 
\begin{equation}
     \min_{t \in [2,\ldots,T]}  \EE\| \gradf (x_{t})\|_{F}^2 \leq \frac{1}{\sqrt{T-1}}(Q_1 + Q_2 \log (T)),
\end{equation}
where  $Q_2 = LH^3/2c^2$ and 
\begin{equation}
     Q_1 = Q_2  + \frac{2dH^3}{c}+ H\EE [f(x_1) -f(x^*)].
\end{equation}
\end{Cor}
%


\section{Related works}\label{sec:discussion}
The asymptotic convergence of stochastic gradient descent on Riemannian manifolds (RSGD) was proved by \citet{Bonnabel_IEEETAC_2013_s}. 
Among the first works to adapt stochastic gradients on Riemannian matrix manifolds is the cRMSProp algorithm \citep{Roy_CVPR_2018}. 
They effectively perceive adaptive weights as vectors. More concretely, if elements in $\M$ have matrix representation of size $n\times r$, then cRMSProp computes an adaptive weight matrix $\bA\in\R^{n\times r}$ and adapts the gradient matrix $\bH\in\R^{n\times r}$ as $\bA\circ\bH$. It can be observed that this interpretation of adaptive weights ignores the matrix structure of $\bA$, $\bH$, and in general $\M$, and process adaptive weights similar to the Euclidean algorithms  discussed in Section~\ref{sec:preliminaries}. 
 \citet{Roy_CVPR_2018} also constrain the adaptive weight matrix $\bA$ to lie on the tangent plane, thereby requiring computationally expensive vector transport operations in each iteration. 
It should also be noted that \citet{Roy_CVPR_2018} do not provide convergence analysis for cRMSProp. 

Another direction of work by \citet{cho17a,Becigneul_ICLR_2019} derive the ADAM algorithm \citep{Kingma_ICLR_2015} on matrix manifolds. In particular, they employ a momentum term along with the Riemannian gradient to compute the search direction. However, the adaptive weight vector in the ADAM algorithm is now substituted with a scalar weight, which effectively adapts the step size rather than the search direction. While \citet{cho17a} do not discuss convergence analysis, \citet{Becigneul_ICLR_2019} provide convergence guarantees limited to geodesically convex functions \citep{Zhang_COLT_2016}. It should be noted that both require vector transport operations in each iteration due to the momentum term. 



Our approach, on the other hand, preserves the underlying matrix structure of the manifold and compute adaptive weight matrices corresponding to row and column subspaces of the Riemannian gradient. We do  not enforce any constraint on the adaptive weight matrices and avoid parallel transport altogether. Our algorithm can be easily be generalized to product of manifolds. 

Existing works have also explored other directions of improvement to RSGD in specific settings, similar to the Euclidean counterpart. These include variance reduction \citep{Zhang_NIPS_2016,Sato_arXiv_2017}, averaged RSGD \citep{tripuraneni18a}, recursive gradients~\citep{Kasai_ICML_2018}, incorporating second-order information \citep{Kasai_AISTATS_2018,Kasai_NeurIPS_2018}, among others. 


In the Euclidean setting, structure-aware preconditioners for (unconstrained) stochastic optimization over matrices or tensors has been discussed in \citep{martens15,gupta18}.



%
%
%

\begin{figure*}[t]
\begin{center}
	\hspace*{-0.2cm}
	\begin{minipage}[t]{.32\textwidth}
	\begin{center}
		\includegraphics[width=\textwidth]{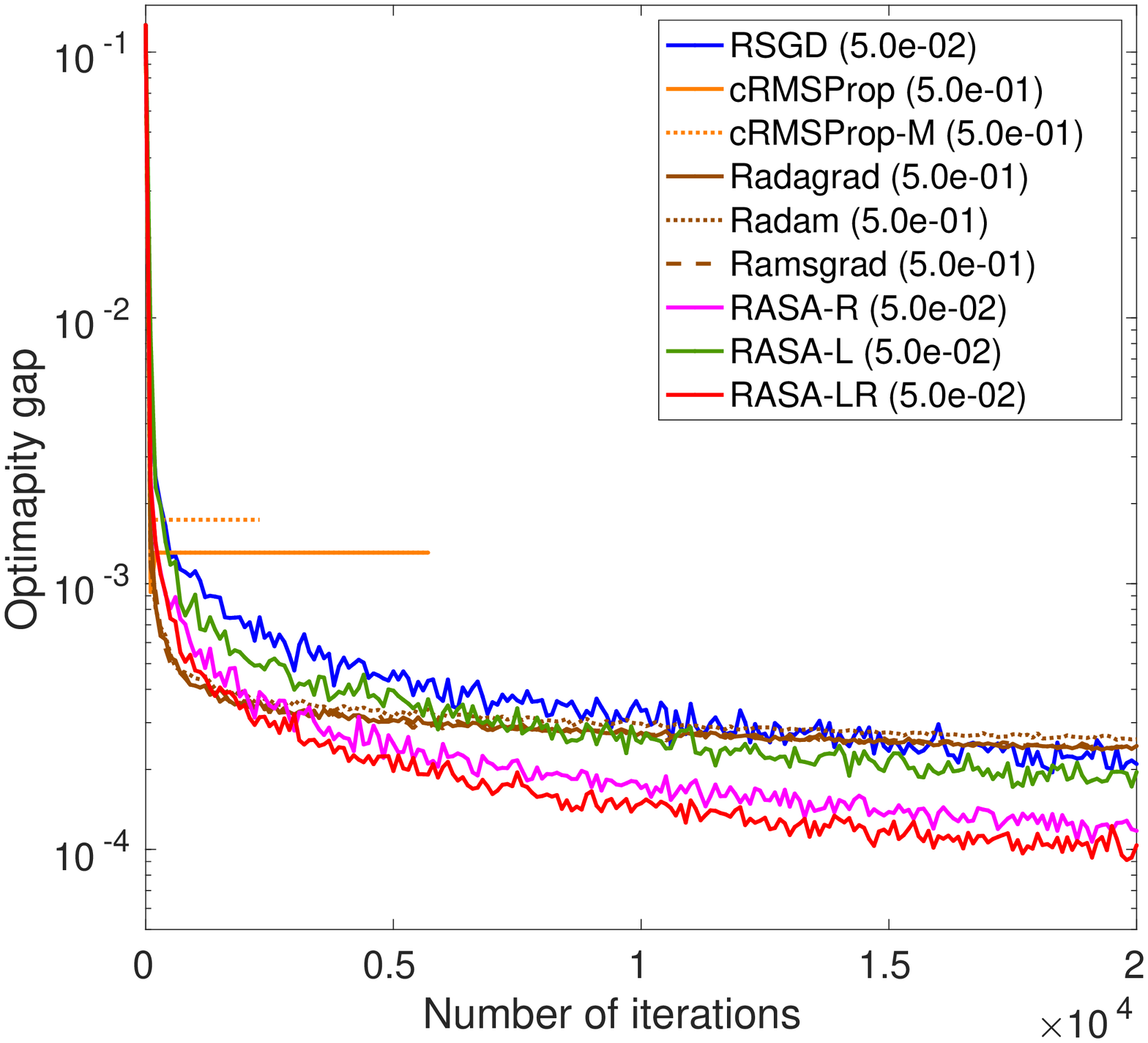}\\
		
		{\scriptsize  (a) {\bf Case P1:} Synthetic dataset.}
		
	\end{center} 
	\end{minipage}
	\hspace*{-0.1cm}
	\begin{minipage}[t]{.32\textwidth}
	\begin{center}
		\includegraphics[width=\textwidth]{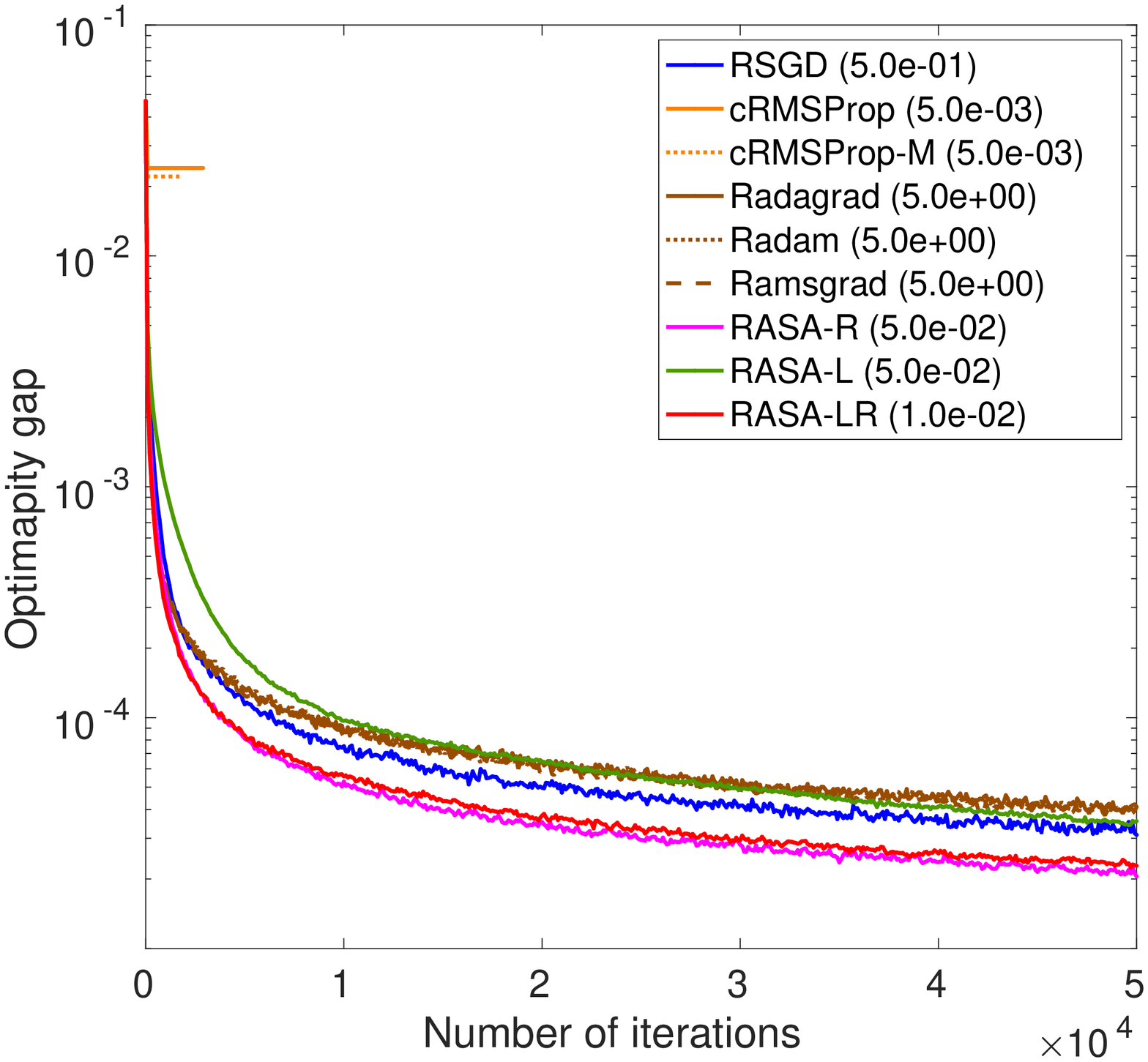}\\
		
		{\scriptsize  (b) {\bf Case P2:} {\tt MNIST} dataset.}
		
	\end{center} 
	\end{minipage}
	\hspace*{-0.1cm}
	\begin{minipage}[t]{.32\textwidth}
	\begin{center}
		\includegraphics[width=\textwidth]{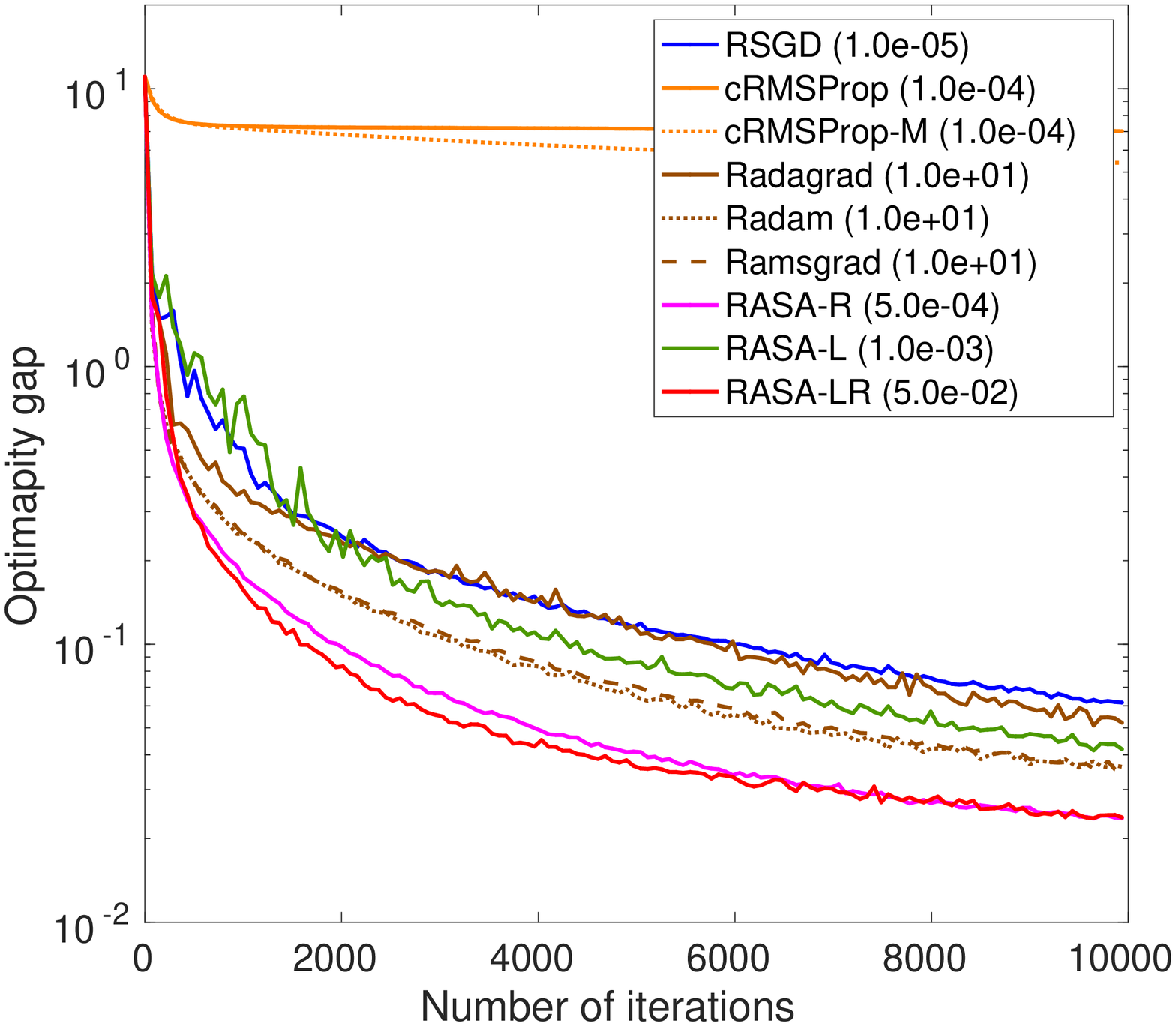}\\
		
		{\scriptsize  (c) {\bf Case P3:} {\tt COIL100} dataset.}
		
	\end{center} 
	\end{minipage}
		
\caption{Performance on the PCA datasets (numbers in the parentheses within legends show best-tuned $\alpha_0$).}
\label{fig:PCA_results}
\end{center}
\end{figure*}

\begin{figure*}[htbp]
\begin{center}

	\begin{minipage}[t]{.24\textwidth}
	\begin{center}
		\includegraphics[width=\textwidth]{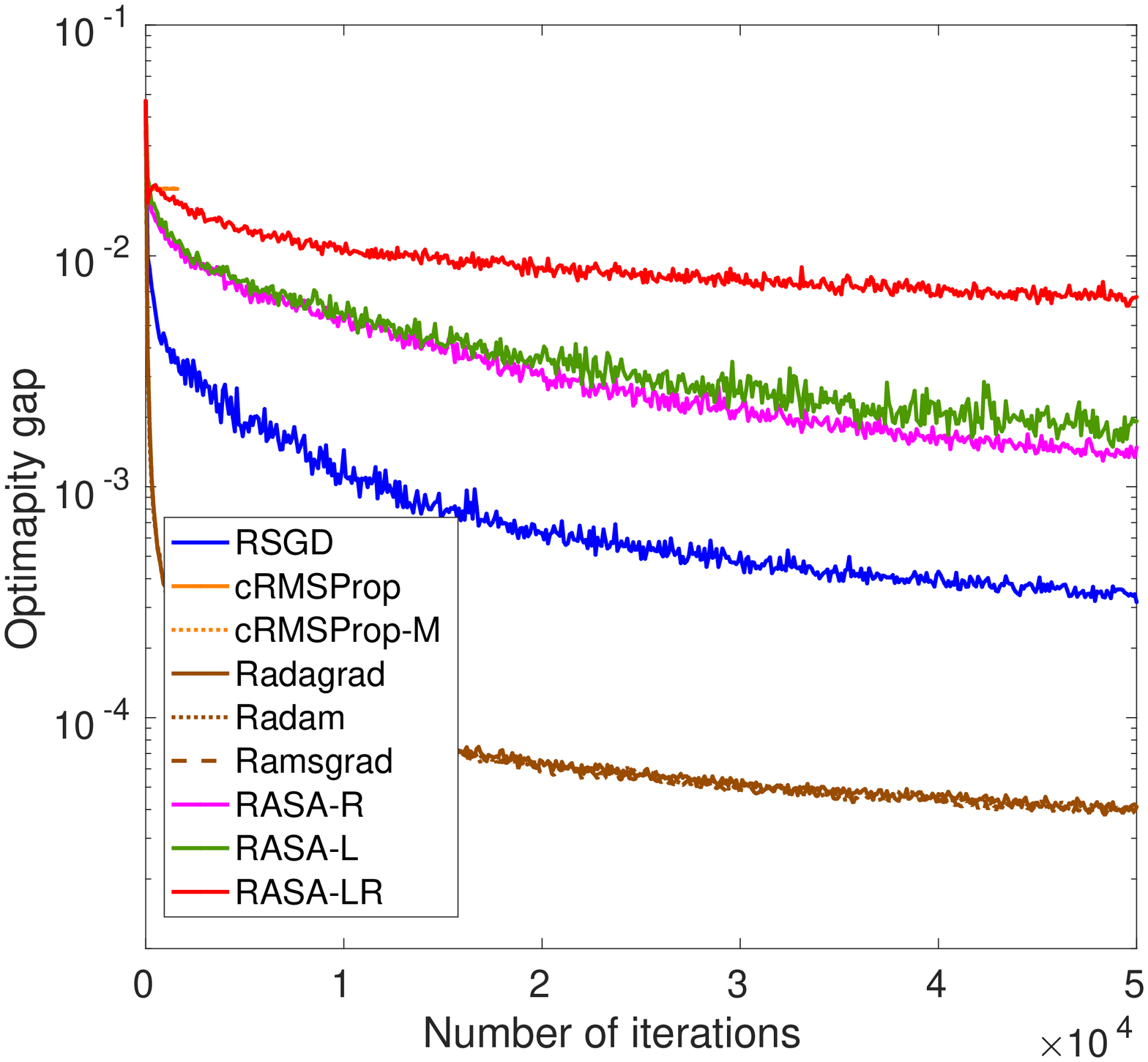}\\
		
		{\scriptsize  (a) $\alpha_0=5$.}
		
	\end{center} 
	\end{minipage}
	\hspace*{-0.1cm}
	\begin{minipage}[t]{.24\textwidth}
	\begin{center}
		\includegraphics[width=\textwidth]{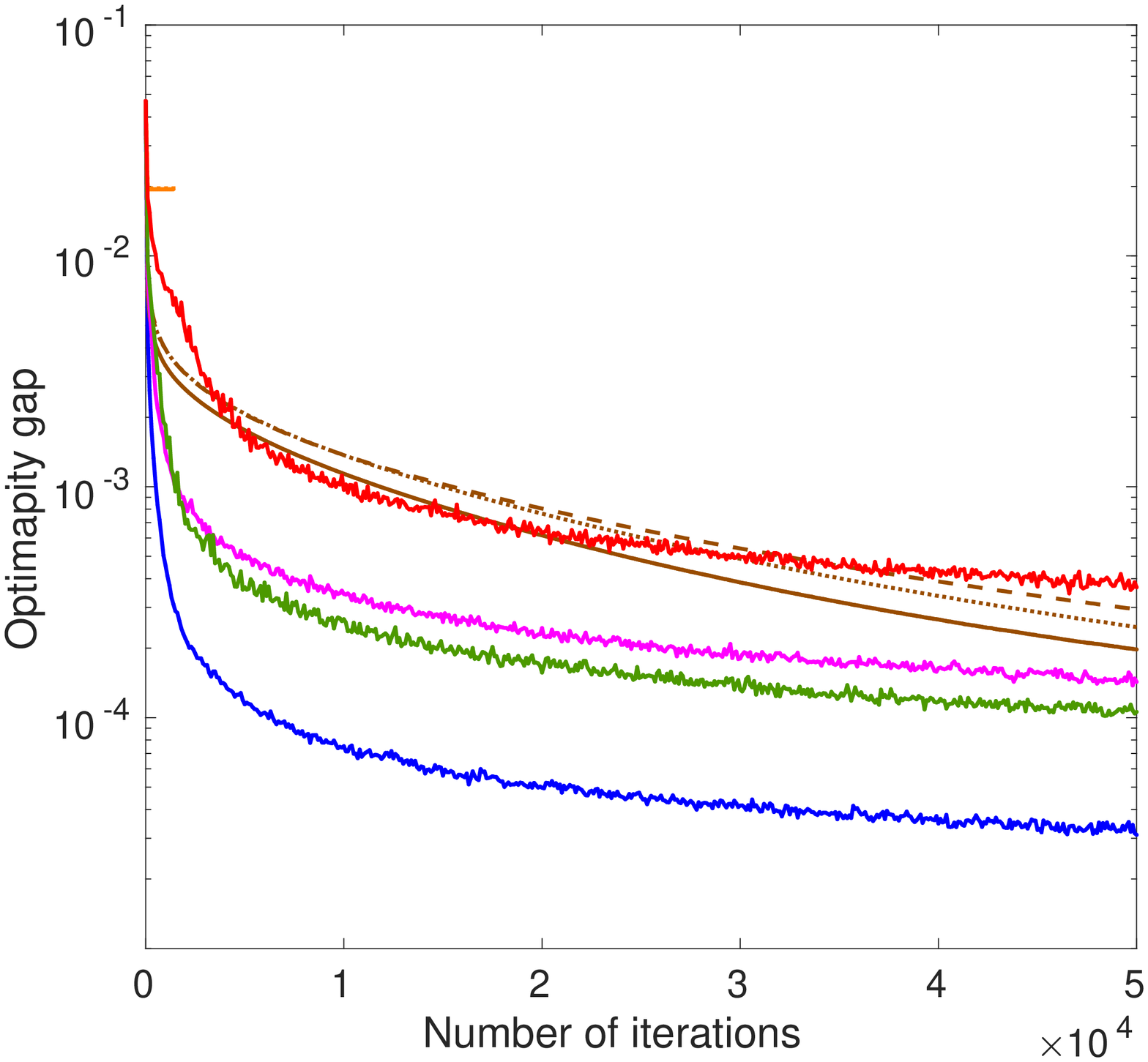}\\
		
		{\scriptsize  (b) $\alpha_0=0.5$.}
		
	\end{center} 
	\end{minipage}
	\hspace*{-0.1cm}
	\begin{minipage}[t]{.24\textwidth}
	\begin{center}
		\includegraphics[width=\textwidth]{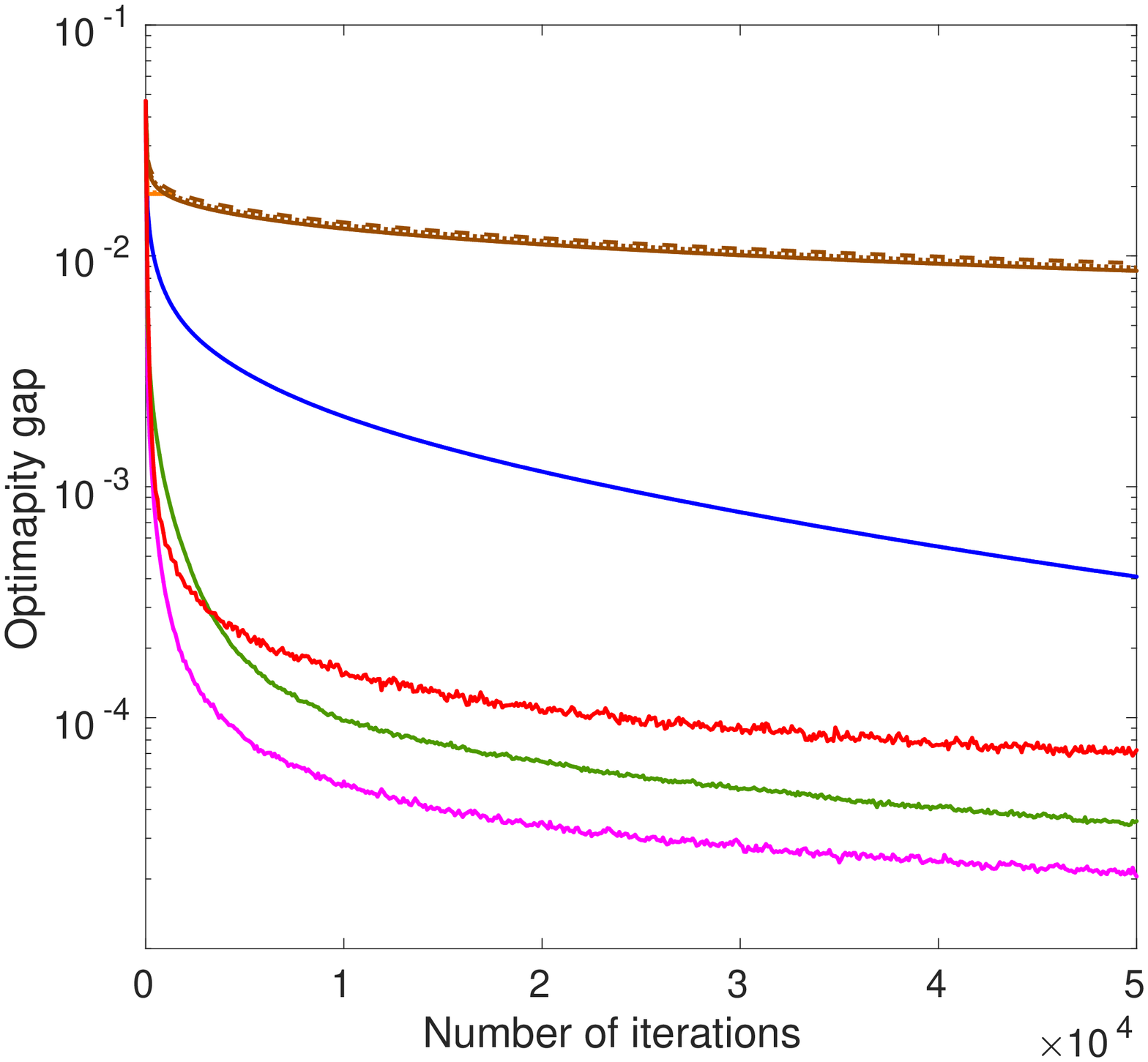}\\
		
		{\scriptsize  (c) $\alpha_0=0.05$.}
		
	\end{center} 
	\end{minipage}
	\hspace*{-0.1cm}
	\begin{minipage}[t]{.24\textwidth}
	\begin{center}
		\includegraphics[width=\textwidth]{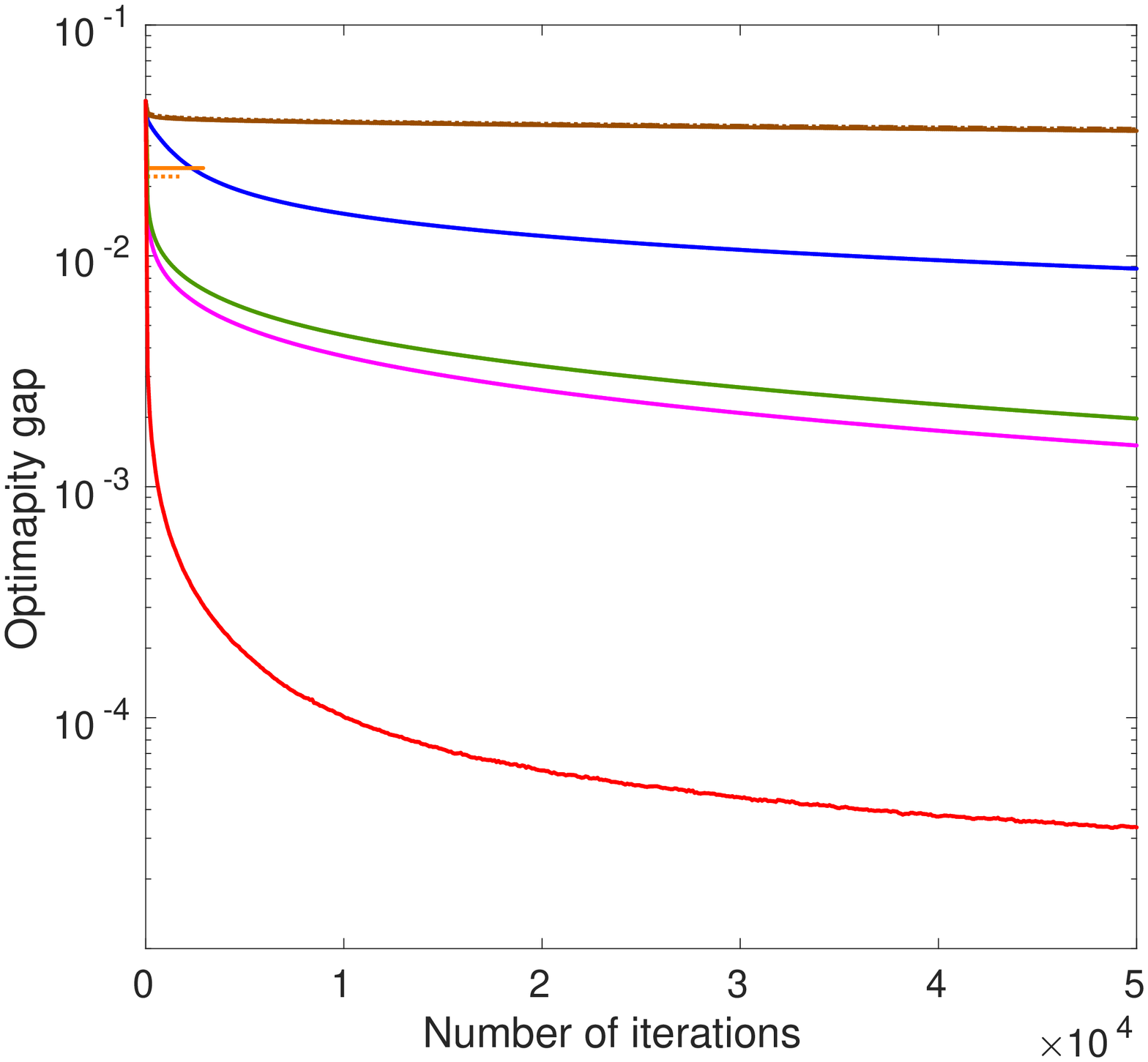}\\
		{\scriptsize  (d) $\alpha_0=0.005$.}
	\end{center}		
	\end{minipage}	

	\caption{Performance on the {\tt MNIST} dataset across different values of the initial step size $\alpha_0$ ({\bf Case P2}).}

\label{fig:PCA_results_MNIST}
\end{center}
\end{figure*}

\section{Experiments}
\label{Sec:NumericalEvaluations}

In this section, we compare our proposed adaptive algorithm, RASA, and its variants with the following baseline Riemannian stochastic algorithms.
\begin{itemize}
 \setlength{\itemsep}{0pt}%
    \setlength{\parskip}{0pt}%
\item RSGD \citep{Bonnabel_IEEETAC_2013_s} is the vanilla Riemannian stochastic gradient algorithm.
\item cRMSProp proposed by \citet{Roy_CVPR_2018}. 
\item cRMSProp-M: a variant of cRMSProp, which uses Riemannian gradients for adaptive weights computations instead of the Euclidean gradients.  
\item Radagrad proposed by \citet{Becigneul_ICLR_2019} considers scalar adaptive weight. Hence, it adapts the step size instead of the gradient.
\item Radam and Ramsgrad by \citet{Becigneul_ICLR_2019}: include momentum terms similar to ADAM and AMSGrad. Additionally, like Radagrad, they adapt only the step size. 
\item RASA-L, RASA-R, and RASA-LR: our proposed variants that either adapt the row (left) subspace, column (right) subspace, or both. 
\end{itemize}
 
All the considered algorithms are implemented in the Matlab toolbox Manopt \citep{Boumal_Manopt_2014_s}. The codes are available at \url{https://github.com/hiroyuki-kasai/RSOpt}. For deep learning applications, RASA can be implemented in Python libraries like McTorch \citep{Meghwanshi_arXiv_2018} and geomstats \citep{miolane18a}.

The algorithms are initialized from the same initialization point and are stopped when the iteration count reaches a predefined value. We fix the batchsize to $10$ (except in the larger MovieLens datasets, where it is set to $100$). The step size sequence $\{ \alpha_t \}$ is generated as $\alpha_t = \alpha_0/\sqrt{t}$ \citep{Chen_ICLR_2019,Reddi_ICLR_2018}, where $t$ is the iteration count. 
We experiment with different values for the initial step size $\alpha_0$. 
The $\beta$ value for adaptive algorithms (all except RSGD) is fixed to $0.99$. The momentum-related $\beta$ term (used only in Radam and Ramsgrad) is set to $0.9$ \citep{Becigneul_ICLR_2019}. 

We address the principal component analysis (PCA) and the independent component analysis (ICA) problems on the {Stiefel} manifold ${\rm St}(r,n)$: the set of orthogonal $r$-frames in $\mathbb{R}^n$ for some $r \leq n$ \citep{Absil_OptAlgMatManifold_2008}. The elements of ${\rm St}(r,n)$ are represented as matrices of size $n\times r$. We also consider the low-rank matrix completion (MC) problem on the Grassmann manifold ${\rm Gr}(r,n)$, which is the set of $r$-dimensional subspaces in $\mathbb{R}^n$ and is a Riemannian quotient manifold of the Stiefel manifold \citep{Absil_OptAlgMatManifold_2008}. 

Apart from the results discussed in this section, Section \ref{app:sec:additional_results} contains additional results including performance of the algorithms across different values of $\alpha_0$, time dependency plots, and decreasing moving average scheme for RASA. 


\begin{figure*}[ht]
\begin{center}
%
%
%
%
	\hspace*{-0.1cm}
	\begin{minipage}[t]{.24\textwidth}
	\begin{center}
		\includegraphics[width=\textwidth]{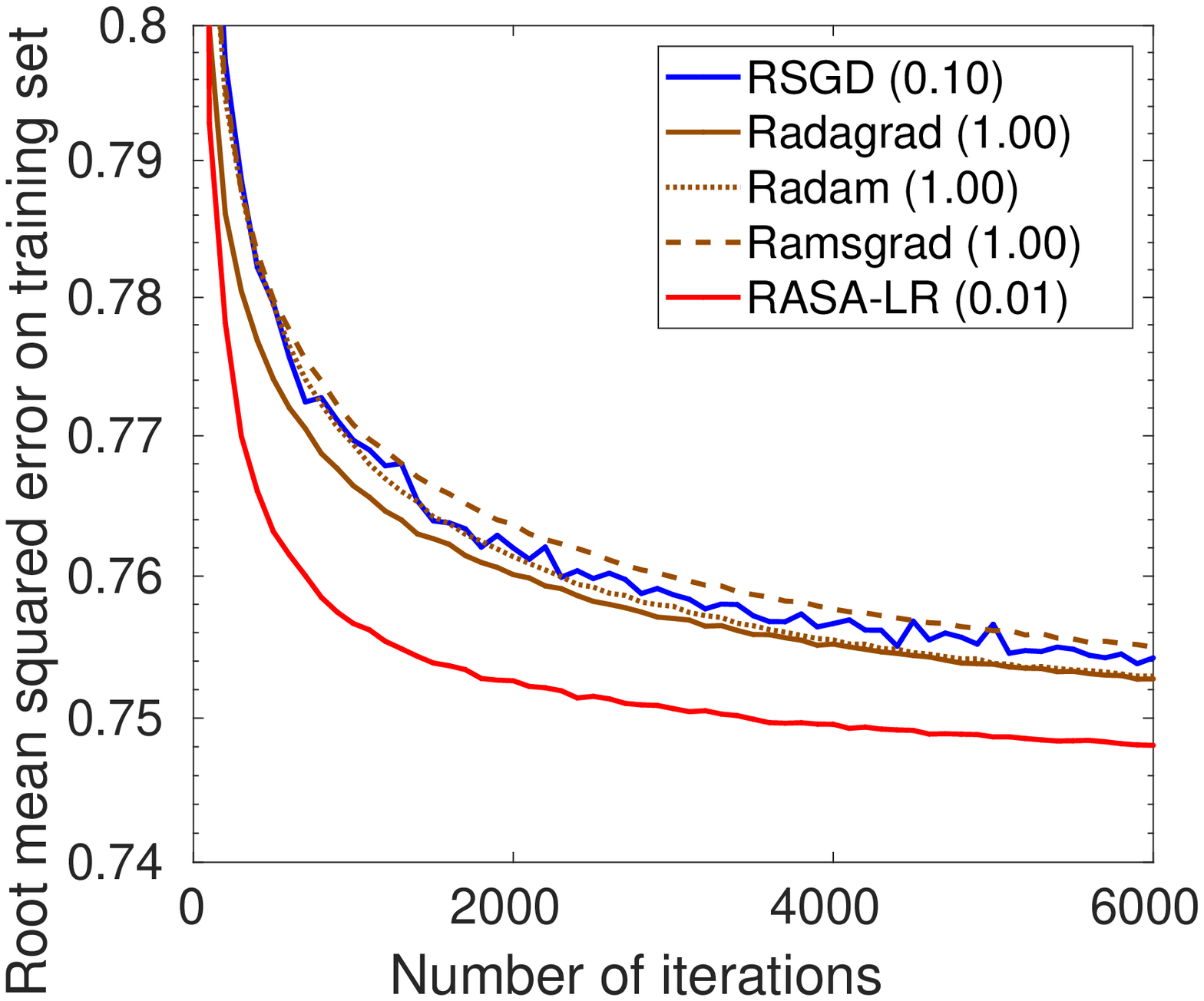}\\
		
		{\scriptsize (a) MovieLens-1M (train).}
		
	\end{center} 
	\end{minipage}
	\hspace*{-0.1cm}
	\begin{minipage}[t]{.24\textwidth}
	\begin{center}
		\includegraphics[width=\textwidth]{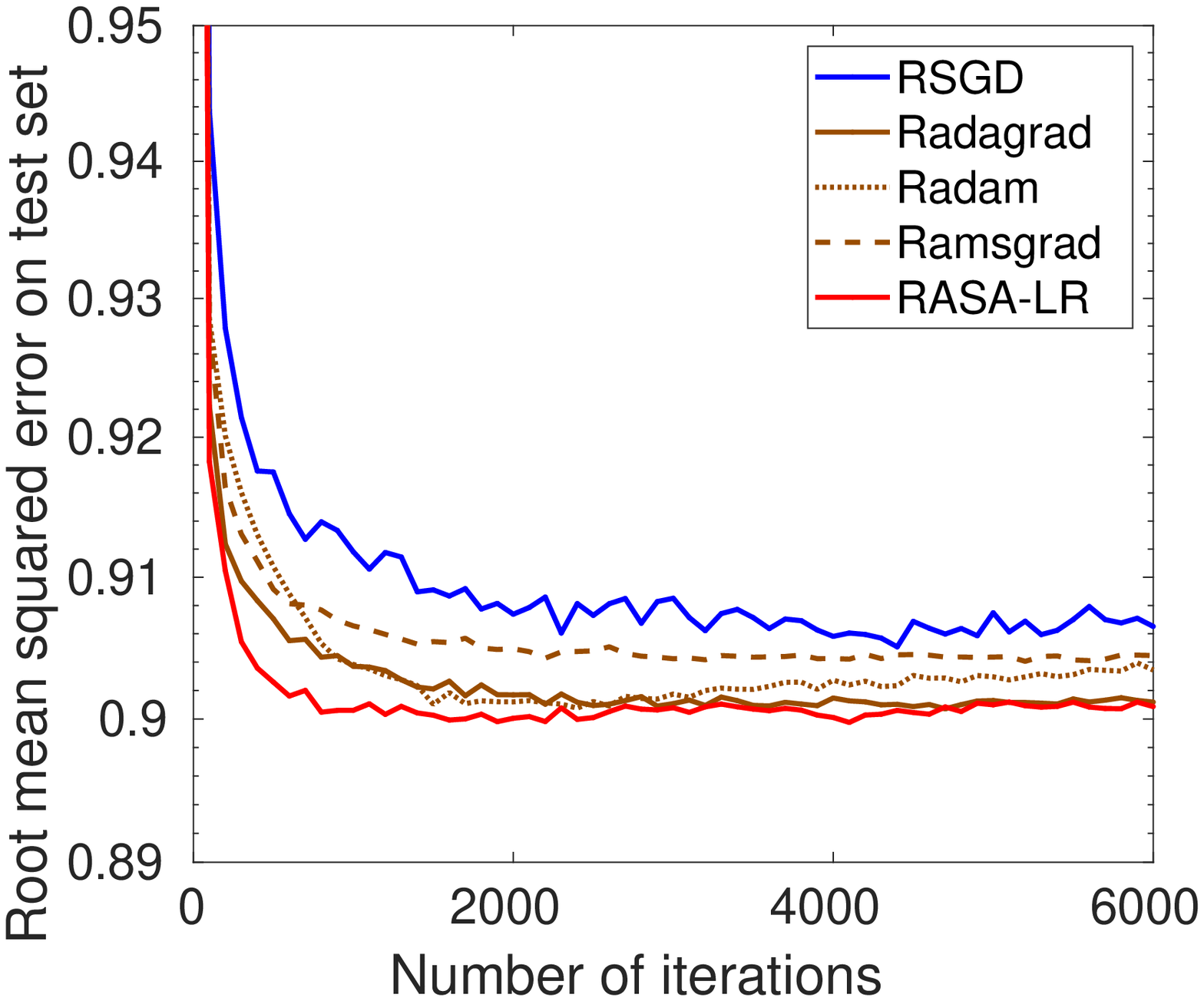}\\
		
		{\scriptsize (b) MovieLens-1M (test).}
		
	\end{center} 
	\end{minipage}
	\hspace*{-0.1cm}
	\begin{minipage}[t]{.24\textwidth}
	\begin{center}
		\includegraphics[width=\textwidth]{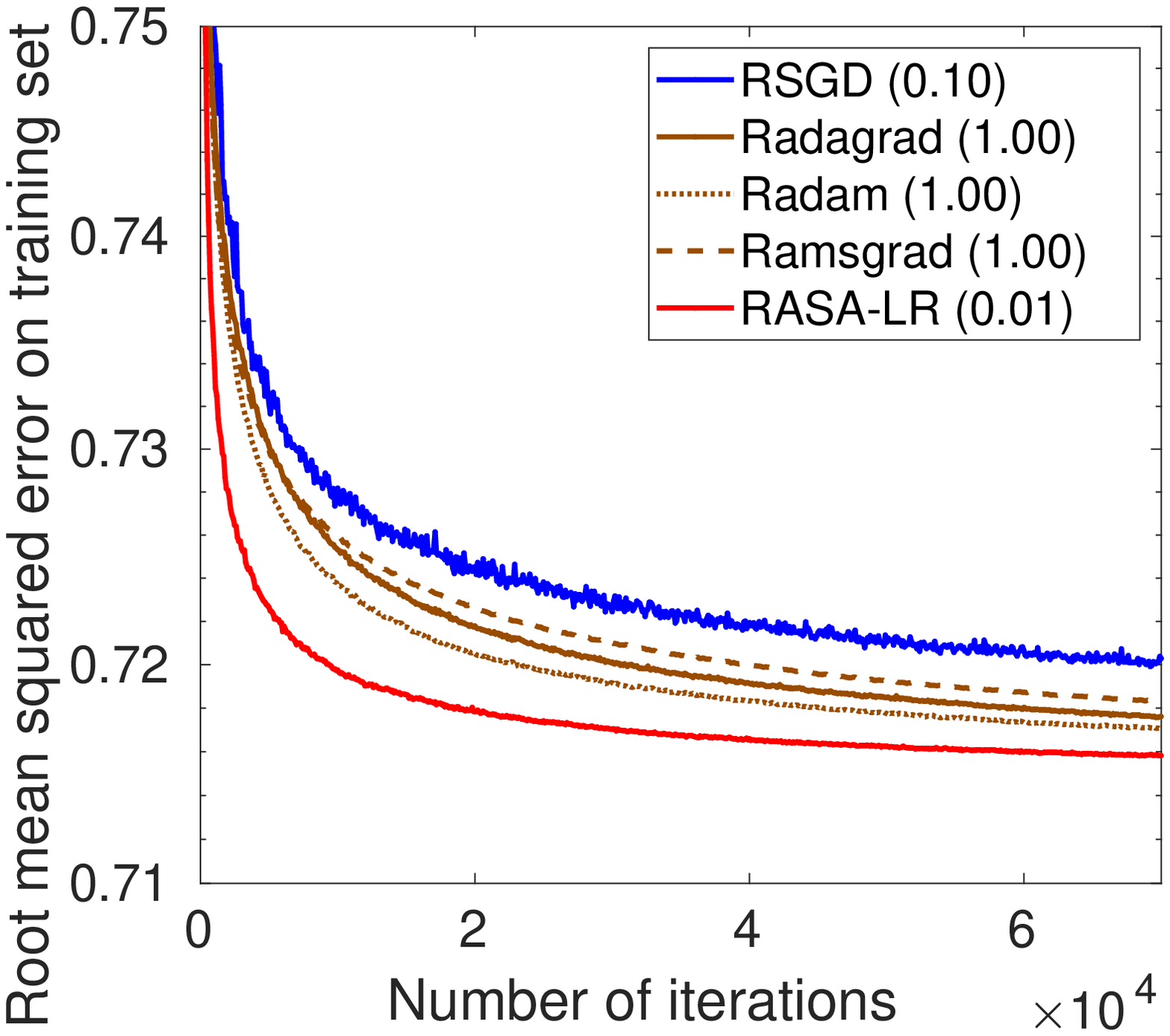}\\
		
		{\scriptsize (c)  MovieLens-10M (train).}
		
	\end{center} 
	\end{minipage}
	\hspace*{-0.1cm}
	\begin{minipage}[t]{.24\textwidth}
	\begin{center}
		\includegraphics[width=\textwidth]{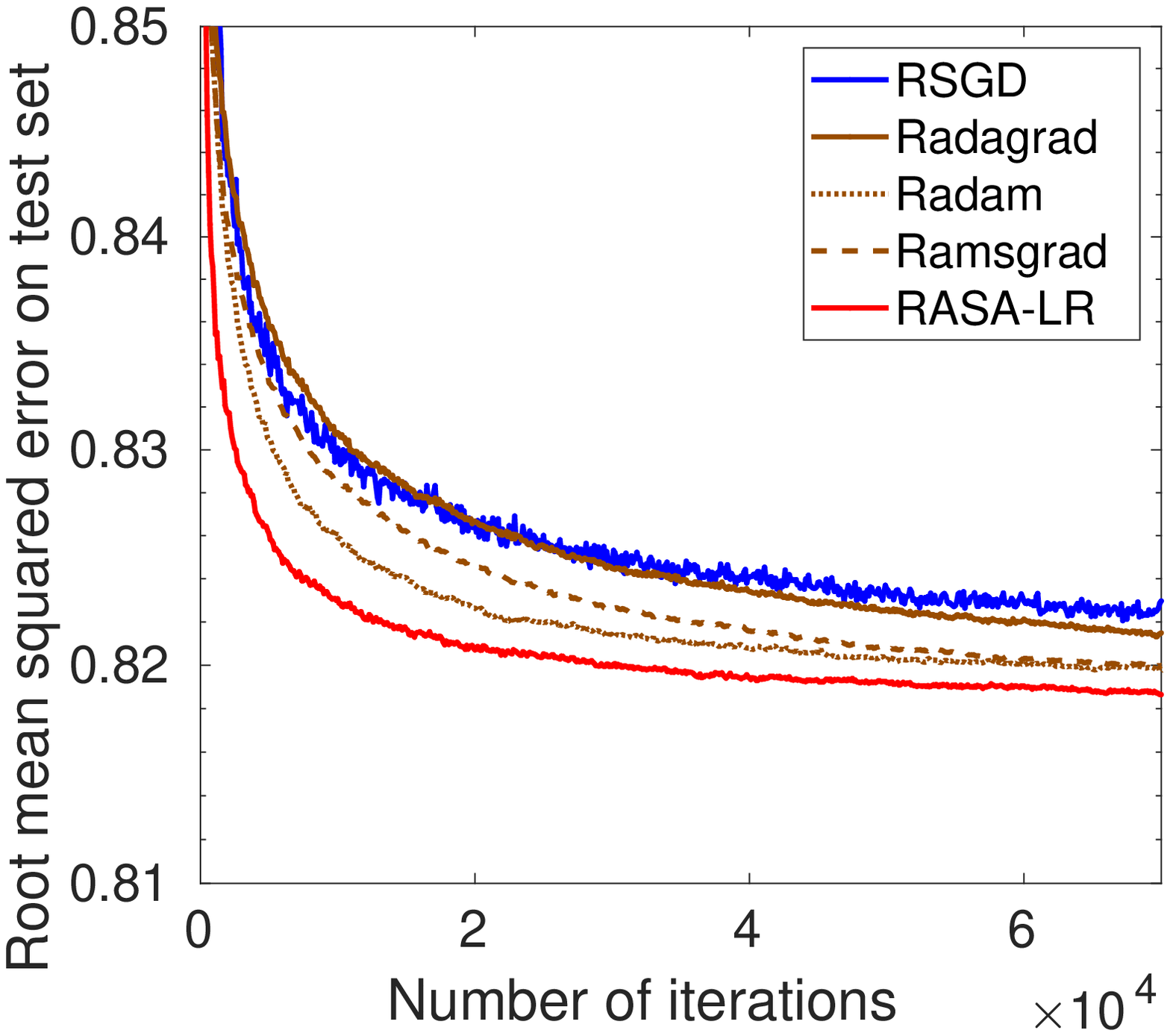}\\
		
		{\scriptsize (d) MovieLens-10M (test).}
		
	\end{center} 
	\end{minipage}
\caption{Performance on the MovieLens datasets (numbers in the parentheses within legends show best-tuned $\alpha_0$). Our proposed algorithm RASA-LR performs competitively and shows faster convergence.}
\label{fig:MC_results}
\end{center}
\end{figure*}

\begin{figure}
\begin{center}
	\hspace*{-0.2cm}
	\begin{minipage}[t]{.4\textwidth}
	\begin{center}
		\includegraphics[width=\textwidth]{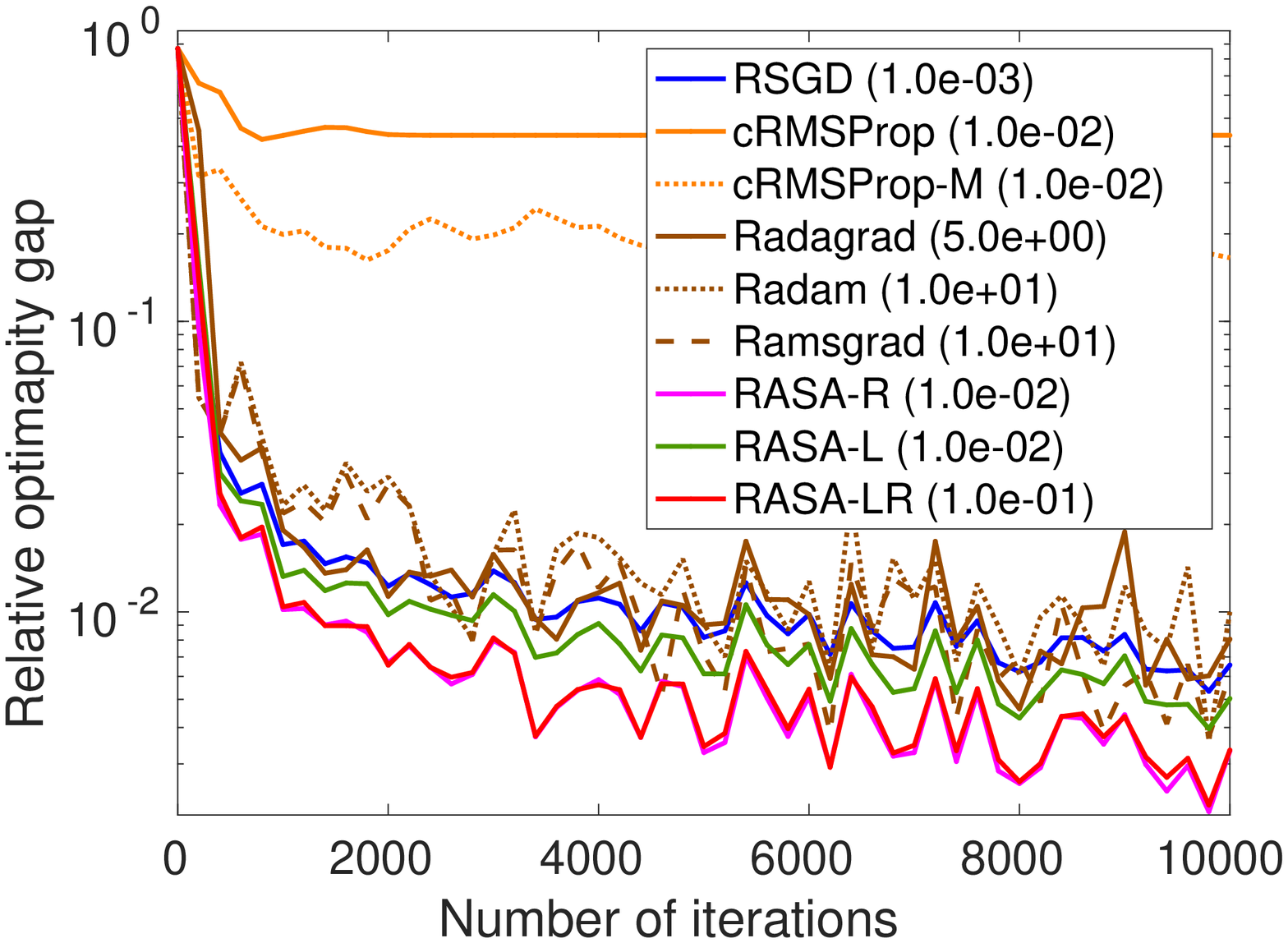}\\
	
		{\scriptsize (a) {\bf Case I1:} {\tt YaleB} dataset.}
		
	\end{center} 
	\end{minipage}
	\hspace*{-0.1cm}
	\begin{minipage}[t]{.4\textwidth}
	\begin{center}
		\includegraphics[width=\textwidth]{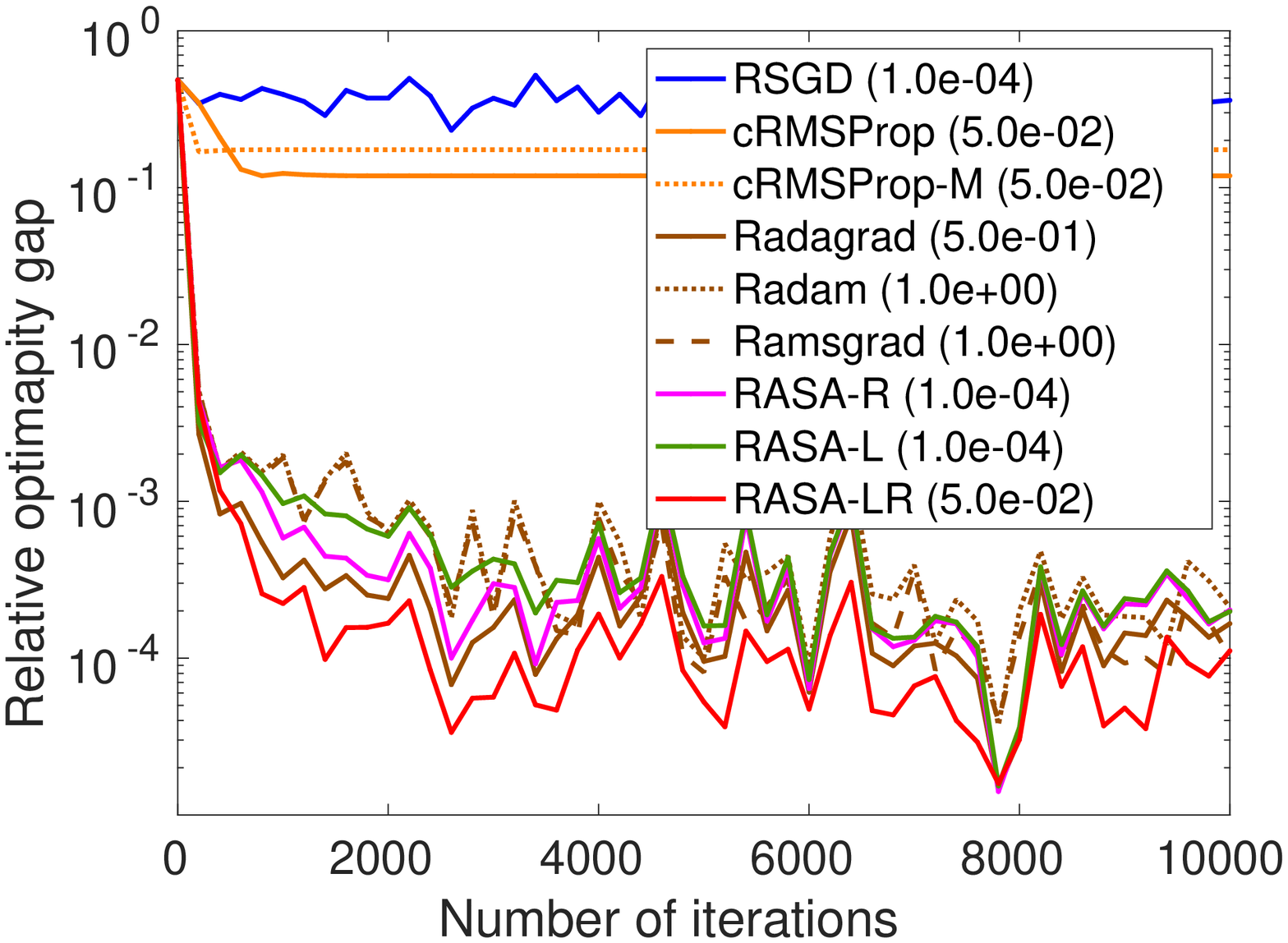}\\
		
		{\scriptsize (b) {\bf Case I2:} {\tt COIL100} dataset.}
		
	\end{center} 
	\end{minipage}
\caption{Performance on the ICA problem (numbers in the parentheses within legends show best-tuned $\alpha_0$).}
\label{fig:ICA_results}
\end{center}
\end{figure}

\subsection{Principal components analysis} 
\label{sec:PCAproblem} 

Given a set of data points $\{\vec{z}_1, \ldots, \vec{z}_N  \} \in \mathbb{R}^{n}$, the PCA problem amounts to learning an orthogonal projector $\mat{U} \in {\rm St}(r,n)$ that minimizes the sum of squared residual errors between projected data points and the original data points. It is formulated as 
$\min_{{\scriptsize \mat{U} \in {\rm St}(r,n)}} \frac{1}{N}  \sum_{i=1}^N \| \vec{z}_i -  \mat{U}\mat{U}^\top \vec{z}_i \|_2^2$. This problem is equivalent to $\min_{{\scriptsize \mat{U} \in {\rm St}(r,n)}}  =-\frac{1}{N} \sum_{i=1}^N \vec{z}_i^\top\mat{U}\mat{U}^\top\vec{z}_i$.


%
We first evaluate the algorithms on a synthetic dataset of size $(N,n,r)=(10^4,10^2,10)$ ({\bf Case P1}). The initial step size $\alpha_0$ is selected from the set $\{0.5,0.1,\ldots,5\times 10^{-3}, 10^{-3}\}$. Figure \ref{fig:PCA_results} (a) shows \emph{optimality gap} with best-tuned step sizes of each algorithm. The minimum loss for the optimality gap computation is obtained by the Matlab command {\tt pca}. 
We observe that our algorithms RASA-LR and RASA-R obtain the best solution. On the other hand, cRMSProp and its modified variant cRMSProp-M converge poorly. 
Radagrad and its momentum based variants (Ramsgrad and Radam) initially proceed well but finally converge to a solution slightly worse than RSGD. 

We additionally evaluate RASA on the {\tt MNIST} ({\bf Case P2}) and {\tt COIL100} ({\bf Case P3}) datasets in the same setting as described in {\bf Case P1}. {\tt MNIST} contains handwritten digits data of $0-9$ \citep{MNIST} and has $60\,000$ images for training and $10\,000$ images for testing with images of size $28 \times 28$. For {\tt MNIST}, $(N,n,r)=(10^4,784,10)$ and the $\alpha_0$ is selected from the set $\{10, 5, \ldots, 10^{-2}, 5\times 10^{-3}\}$. {\tt COIL100} contains normalized $7\,200$ color camera images of the $100$ objects taken from different angles \citep{COIL-100_TR}. We resize them into $32\times 32$ pixels. For {\tt COIL100}, $(N,n,r)=(7200,1024,100)$ and 
the $\alpha_0$ is selected from the set $\{50, 10, \ldots, 10^{-5},  5\times10^{-6}\}$. 
From Figures~\ref{fig:PCA_results} (b)-(c), we observe that RASA-LR and RASA-R perform better than other baselines on both {\tt MNIST} and {\tt COIL100} datasets.
This shows the benefit of adapting the row and column subspaces over individual entries on Riemannian manifolds. 

In order to observe the influence of the choice of $\alpha_0$ on the algorithms, we plot the results of all the algorithms on the {\tt MNIST} dataset for $\alpha_0=\{5,0.5,0.05,0.005\}$ in Figure \ref{fig:PCA_results_MNIST}. 
It can be seen that the proposed algorithms are more robust to changes in $\alpha_0$ than other adaptive algorithms. As discussed before, RASA-LR and RASA-R obtain the best results on this dataset.


%



\subsection{Joint diagonalization in ICA}
\label{sec:ICAproblem} 

The ICA or the blind source separation problem refers to separating a signal into components so that the components are as independent as possible \citep{hyvarinen2000independent_s}. A particular preprocessing step is the whitening step that is proposed through joint diagonalization on the Stiefel manifold \citep{theis09a}. To that end, the optimization problem of interest is
$\min_{{\scriptsize \mat{U}} \in \mathbb{R}^{n \times r}} 
- \frac{1}{N}\sum_{i=1}^N \| {\rm diag}(\mat{U}^\top \mat{C}_i \mat{U}) \|^2_F$,
where $\| {\rm diag}(\mat{A}) \|^2_F$ defines the sum of the squared diagonal elements of {\bf A}. The symmetric matrices $\mat{C}_i$s are of size $n\times n$ and can be cumulant matrices or time-lagged covariance matrices of $N$ different signal samples \citep{theis09a}. 

We use the {\tt YaleB} \citep{Georghiades_PAMI_2001} dataset ({\bf Case I1}), which contains human subjects images under different poses and illumination conditions. The image size of the original images is $168\times 192$. From this dataset, we create a region covariance matrix (RCM) descriptors \citep{Porikli_ICIP_2006_s,Tuzel_ECCV_2006,Pang_CSVT_2008} for $2\,015$ images, and the resultant RCMs are of size $40 \times 40$. The initial step size $\alpha_0$ is selected from {$\{10,5,\ldots,5\times10^{-4},10^{-4}\}$} and we set $r=30$. For {\tt YaleB}, $(N,n,r)=(2015,40,30)$. 
We also use the {\tt COIL100} dataset ({\bf Case I2}). The RCM descriptors of size $7\times 7$ from $7\, 200$ images are created.The initial step size $\alpha_0$ is selected from {$\{10,5,\ldots,5\times10^{-4},10^{-4}\}$} and we set $r=7$. Thus, $(N,n,r)=(7\,200, 7,7)$. 
We show plots for \emph{relative} optimality gap (ratio of optimality gap to the optimal objective value) for all the algorithms in Figure \ref{fig:ICA_results}. The optimal solution value for optimality gap computation is obtained from the algorithm proposed by \citet{theis09a}. We observe that adapting both left and right subspaces (RASA-LR) leads to the best results on both the datasets.

%
%
%
%

\subsection{Low-rank matrix completion}
\label{sec:MCproblem} 

The low-rank matrix completion problem amounts to completing an incomplete matrix $\mat{Z}$, say of size $n \times N$, from a small number of observed entries by assuming a low-rank model for $\mat{Z}$. If $\Omega$ is the set of the indices for which we know the entries in $\mat{Z}$, then the rank-$r$ matrix completion problem amounts to finding matrices $\mat{U}  \in \mathbb{R}^{n \times r}$ and $\mat{A} \in \mathbb{R}^{r\times N}$ that minimizes the error $\|(\mat{UA})_{\Omega} - {\mat Z}_{\Omega} \|_F^2$. Partitioning $\mat{Z} = [\vec{z}_1, \vec{z}_2, \ldots, \vec{z}_N] $ and exploiting the least squares structure \citep{boumal15a_s}, the problem is equivalent to  $\min_{{\scriptsize \mat{U}} \in {\rm Gr}(r,n)} 
\frac{1}{N} \sum_{i=1}^N \| (\mat{U} \vec{a}_i)_{\Omega_i} - {\vec{z}_i}_{\Omega_i} \|_2^2$, where $\vec{z}_i \in \mathbb{R}^n$ and ${\Omega_i}$ is the set of observed indices for $i$-th column.

We show the results on the MovieLens datasets \citep{harper15a_s} which consist of ratings of movies given by different users. The MovieLens-10M dataset consists of $10\, 000\, 054$ ratings by $71\,567$ users for $10\,677$ movies. The  MovieLens-1M dataset consists of $1\,000\,209$ ratings by $6\,040$ users for $3\,900$ movies. We randomly split the data into $80$/$20$ train/test partitions. We run different algorithms for rank $r=10$ and regularization parameter $0.01$ as suggested by \citet{boumal15a_s}. 
Figures \ref{fig:MC_results} (a)-(d) show that our algorithm RASA-LR achieves faster convergence during the training phase and the best generalization performance on the test set. 
cRMSProp and cRMSProp-M did not run successfully on MovieLens datasets for different choices of $\alpha_0$.  


\subsection{RASA-vec: a vectorized variant of RASA}

We end this section noting that the effective dimension of our adaptive weights (vectors $\hat{\bl}_t$ and $\hat{\br}_t$) in Algorithm \ref{Alg:R-AGD-diag} is $n+r$, which is less than the ambient dimension $nr$ of the Riemannian gradient $\bG_t$. In Euclidean algorithms such as ADAM, AMSGrad \citep{Reddi_ICLR_2018}, AdaGrad, or RMSProp, the effective dimension of the adaptive weights is same as the Euclidean gradient. To observe the effect of the proposed reduction in dimension of adaptive weights in the Riemannian setting, we develop a variant of RASA, called RASA-vec, which directly updates all the entries of the Riemannian gradient. That is, for RASA-vec, the Riemannian gradient is adapted as 
\begin{equation*}
\mathrm{vec}(\tilde{\bG}_t)= \hat{\bV}_t^{-1/2}\mathrm{vec}(\bG_t),
\end{equation*}
where the adaptive weights matrix is $\hat{\bV}_t=\mathrm{Diag}(\hat{\bv}_t)$, $\hat{\bv}_t=\mathrm{max}(\bv_t,\hat{\bv}_{t-1})$, and $\bv_{t}=\beta\bv_{t-1} + (1-\beta)\mathrm{vec}(\bG_{t})\circ\mathrm{vec}(\bG_{t})$. Here, the symbol `$\circ$' denotes the entry-wise product (Hadamard product) of vectors. The update strategy of RASA-vec remains same as that of RASA in (\ref{eq:retraction_RASA}), i.e., $x_{t+1} = R_{x}(-\alpha_t \mathcal{P}_{x_t}(\tilde{\bG}_t))$. The convergence guarantees of RASA-vec also follow naturally from Section \ref{Sec:MainResults}.

Figure~\ref{fig:RASA_vec_results} compares the performance of the RASA-LR with RASA-vec on the three problems discussed earlier. The results illustrate the benefits of the proposed modeling approach that we RASA-LR achieves similar performance to RASA-vec even though it stores $n+r$ adaptive weights instead of $nr$ weights.

\begin{figure}[t]
\begin{center}
	\hspace*{-0.2cm}
	\begin{minipage}[t]{.40\textwidth}
	\begin{center}
		\includegraphics[width=\textwidth]{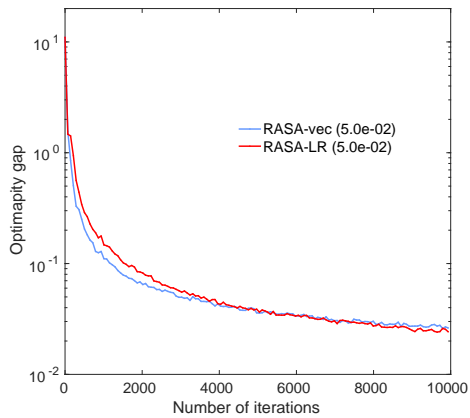}\\
	
		{\scriptsize (a) PCA on {\tt COIL100} dataset.}
		
	\end{center} 
	\end{minipage}
	\hspace*{-0.1cm}
	\begin{minipage}[t]{.4\textwidth}
	\begin{center}
		\includegraphics[width=\textwidth]{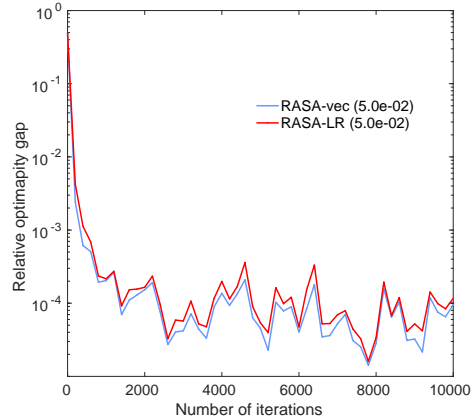}\\
		
		{\scriptsize (b) ICA {\tt COIL100} dataset.}
		
	\end{center} 
	\end{minipage}
	\\
	\hspace*{-0.1cm}
	\begin{minipage}[t]{.4\textwidth}
	\begin{center}
		\includegraphics[width=\textwidth]{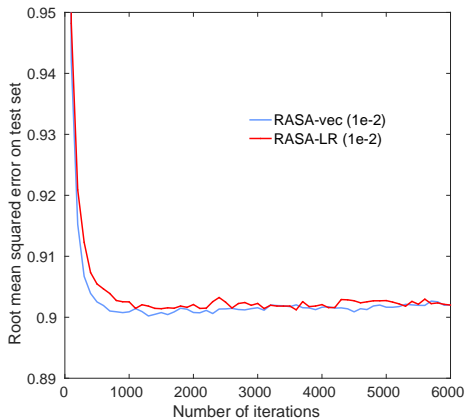}\\
		
		{\scriptsize (c) MovieLens-1M (test).}
		
	\end{center} 
	\end{minipage}
	\hspace*{-0.1cm}
	\begin{minipage}[t]{.4\textwidth}
	\begin{center}
		\includegraphics[width=\textwidth]{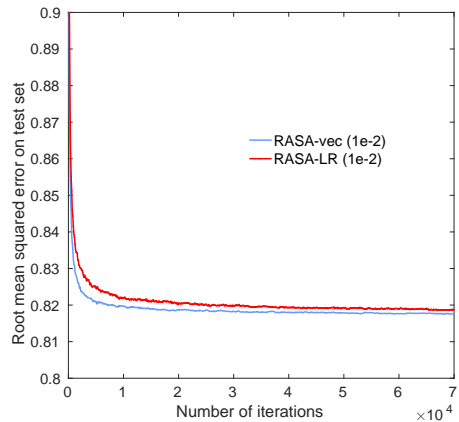}\\
		
		{\scriptsize (d) MovieLens-10M (test).}
		
	\end{center} 
	\end{minipage}
\caption{Similar performance of RASA-vec and RASA-LR on different datasets (numbers in the parentheses within legends show best-tuned $\alpha_0$).}
\label{fig:RASA_vec_results}
\end{center}
\end{figure}

\section{Discussion and conclusion}\label{sec:conclusion}
We observe that the proposed modeling of adaptive weights on the row and column subspaces obtains better empirical performance than existing baselines \citep{Roy_CVPR_2018,cho17a,Becigneul_ICLR_2019}. In particular, we do not observe additional benefit of the adaptive momentum methods (Radam and Ramsgrad) in the Riemannian setting when the adaptive updates exploit the geometry better. 

Overall, we presented a principled approach for modeling adaptive weights for Riemannian stochastic gradient on matrix manifolds. We developed computationally efficient algorithms and discussed the convergence analysis. Our experiments validate the benefits of the proposed algorithms on different applications.

\section*{{Acknowledgements}}
H. Kasai was partially supported by JSPS KAKENHI Grant Numbers JP16K00031 and JP17H01732.

\bibliographystyle{icml2019}
\bibliography{manifold_optimization,matrix_tensor_completion,manifold_computer_vision,optimization_general,stochastic_online_learning,dataset}

\clearpage

\appendix
%
%
%
%
%
%

\renewcommand\thefigure{A.\arabic{figure}}  
\setcounter{figure}{0} 

\renewcommand\thetable{A.\arabic{table}}  
\setcounter{table}{0} 

\renewcommand{\theequation}{A.\arabic{equation}}
\setcounter{equation}{0}

\renewcommand\thealgorithm{A.\arabic{algorithm}}  
\setcounter{algorithm}{0} 

\renewcommand\thefigure{A.\arabic{figure}}  
\setcounter{figure}{0} 

\renewcommand\thetable{A.\arabic{table}}  
\setcounter{table}{0} 

\renewcommand{\theequation}{A.\arabic{equation}}
\setcounter{equation}{0}

\renewcommand\thealgorithm{A.\arabic{algorithm}}  
\setcounter{algorithm}{0}

\section{Proofs} \label{app:sec:proofs}

This section provides complete proofs of the lemmas and the convergence analysis. 

\subsection{Proof of Lemma \ref{Lem:uppder_bound_v}}
\begin{proof}
For any $x$, we have
\begin{eqnarray*}
\| \gradf(x) \|_{{F}} \ =\ \| \EE_{i} \gradf_i(x) \|_{{F}} \ \leq\  \EE_{i}  \| \gradf_i(x) \|_{{F}} \ \leq\  H.
\end{eqnarray*}

Next, we derive the upper bound of the $p$-th element ($p \in [r]$) of $\vec{r}_t $, i.e., $(\vec{r}_t)_{p}$. Denoting $(\vec{r}_t)_{p}$ as $r_{t,p}$ for simplicity, we have $r_{0,p}=0$ from the algorithm. Now, we first assume that $(0 \leq)\ r_{t,p} \leq H^2$. Then, addressing $((\mat{G}_t)_{p,q})^2=g^2_{p,q}\leq H^2$ due to $\| \gradf_i(x) \|_F \leq H$, 
we have 
\begin{eqnarray*}
r_{t+1,p}                                                                                                                                                                                                                                                                                                                                                                 
&=& \beta r_{t,p} + \frac{1-\beta}{n}({\rm diag}(\mat{G}_{t+1}^T \mat{G}_{t+1}))_p \nonumber \\
&=& \beta r_{t,p} +  \frac{1-\beta}{n}\mat{G}^T _{t+1}(:,p) \mat{G} _{t+1}(:,p)\nonumber \\
&=& \beta r_{t,p} +  \frac{1-\beta}{n} \sum_{q=1}^n g^2_{p,q}\nonumber \\
&\leq& \beta H^2 +  \frac{1-\beta}{n} \sum_{q=1}^n H^2\nonumber \\
&=& \beta H^2 +  \frac{1-\beta}{n} n H^2\nonumber \\
&\leq & H^2.
\end{eqnarray*}
Consequently, for any $t \geq 0$, we have $r_{t,p} \leq H^2$. 
Similarly, having $l_{0,q}=0$ for $q \in [n]$, and assuming $(0 \leq)\ l_{t,q} \leq H^2$, we have $l_{t+1,q} \leq H^2$. Thus, we have $l_{t,q} \leq H^2$. 

Now, when $j=n(p-1)+q$, $(\hat{\vec{v}}_t)_j=\hat{v}_{t,j}= \hat{r}_{t,p}^{1/2} \times \hat{l}_{t,q}^{1/2}$ due to $\hat{\vec{v}}_t= \hat{\vec{r}}_{t}^{1/2}\otimes \hat{\vec{l}}_{t}^{1/2}$. 
%
Therefore, supposing that $\hat{r}_{0,p}  = \hat{l}_{0,q} = 0$, and also that $\hat{r}_{t,p} \leq H^2$ and $\hat{l}_{t,q} \leq H^2$, we have 
\begin{eqnarray*}
(\hat{\vec{v}}_{t+1})_j \ \ =\ \ \hat{v}_{t+1,j}\ 
& = & \hat{r}_{t+1,p}^{1/2} \times  \hat{l}_{t+1,q}^{1/2}\nonumber \\
& = & {\rm max}(\hat{r}_{t,p},r_{t+1,p})^{1/2} \times  (\hat{l}_{t,q},l_{t+1,q})^{1/2} \nonumber \\
& \leq & H^2.
\end{eqnarray*}
Thus, for any $t \geq 0$, we have $(\hat{\vec{v}}_{t})_j  \leq H^2$.

This completes the proof. 
\end{proof}

\subsection{Proof of Theorem \ref{Thm:main_theorem}}

This section provides the complete proof of Theorem \ref{Thm:main_theorem}. 

For this purpose, we first provide an essential lemma for projection matrix as presented below. 
\begin{Lem}[Projection matrix]
\label{Lem:ProjectionMat}
Suppose that $\mat{A}\xi_x=0$ is the constraint of the tangent space at $x \in \mathcal{M}$ to be projected onto, then the projection matrix $\mat{P}_x$ is derived as 
\begin{eqnarray*}
\mat{P}_x  = -\mat{A}^T( \mat{A}\mat{A}^T)^{-1} \mat{A} + \mat{I}.
\end{eqnarray*}
\end{Lem}
It should be noted that $\mat{P}_x$ is \emph{symmetric}. 

Now, we provide the proof of Theorem \ref{Thm:main_theorem}.

\begin{proof}[Proof of Theorem \ref{Thm:main_theorem}]

For simplicity, we assume that the parameter $x_t$ represents its vectorized form throughout the analysis below. We also assume that the Riemannian stochastic gradients $\gradf_{i_t}(x_t)$ and the Riemannian gradient $\gradf(x_t)$ at $x_t$ represent their vectorized forms of their original matrix forms, which are $\vec{g}_{t}(x_t)$ and $\vec{g}(x_t)$, respectively. Accordingly, we assume a projection matrix $\mat{P}_{x_t}$ as in Lemma \ref{Lem:ProjectionMat} as the projection operator $\mathcal{P}_{x_t}$.

Since $f$ is retraction $L$-smooth in Lemma \ref{Lemma:DescentLemmaRetraction}, we have
\begin{eqnarray}
\label{Eq:OurDescentLemma_1}
f(x_{t+1}) 
& \leq &  f(x_t) + \langle \vec{g}(x_t), -\alpha_t \mat{P}_{x_t} \hat{\mat{V}}_t^{-1/2}\vec{g}_{t}(x_t)\rangle_2+\frac{L}{2}  \| -\alpha_t  \mat{P}_{x_t} \hat{\mat{V}}_t^{-1/2}\vec{g}_{t}(x_t) \|_{2}^2.
\end{eqnarray}

We first bound the second term in the right hand side of above. Here, we notice that 
\begin{eqnarray}
\langle\vec{g}(x_t),  \mat{P}_{x_t} \hat{\mat{V}}_{t}^{-1/2} \vec{g}_{t}(x_t) \rangle_2   
  &=& \langle \mat{P}_{x_t} \vec{g}(x_t), \hat{\mat{V}}_{t}^{-1/2} \vec{g}_{t}(x_t) \rangle_2 \nonumber\\ %
 &=& \langle \vec{g}(x_t), \hat{\mat{V}}_{t}^{-1/2} \vec{g}_{t}(x_t) \rangle_2, \nonumber%
\end{eqnarray}
where the first equality holds because $\mat{P}_{x_t}$ is symmetric as stated at Lemma \ref{Lem:ProjectionMat} and the second equality holds because $\vec{g}_t(x_t)$ is on the tangent space and $\mat{P}_{x_t} \vec{g}_t(x_t)$ is just $\vec{g}_t(x_t)$. 
This observation can be also obtained by noticing $\mat{P}_{x_t}=- \mat{Q}_t+\mat{I}$, where $\mat{Q} =  \mat{A}^T( \mat{A}\mat{A}^T)^{-1} \mat{A}$ in Lemma \ref{Lem:ProjectionMat}, as presented below.
\begin{eqnarray}
\langle\vec{g}(x_t),  (- \mat{Q}_t+\mat{I}) \hat{\mat{V}}_{t}^{-1/2} \vec{g}_{t}(x_t) \rangle_2  
&=& \langle\vec{g}(x_t), - \mat{Q}_t \hat{\mat{V}}_{t}^{-1/2} \vec{g}_{t} (x_{t}) \rangle_2
 + \langle\vec{g}(x_t), \hat{\mat{V}}_{t}^{-1/2} \vec{g}_{t}(x_t) \rangle_2 \nonumber\\ 
 &=& \langle \vec{g}(x_t), \hat{\mat{V}}_{t}^{-1/2} \vec{g}_{t}(x_t) \rangle_2, \nonumber%
\end{eqnarray}
where the second equality holds because $- \mat{Q}_t\hat{\mat{V}}_{t}^{-1/2} \vec{g}_{t} (x_{t})$ belongs to the \emph{normal space} to the tangent space $T_{x_t}\mathcal{M}$, and its inner product with the tangent vector $\vec{g}(x_t) \in T_{x_t}\mathcal{M}$ is zero. We can also bound the \changeHKK{third} term of (\ref{Eq:OurDescentLemma_1}) as
\begin{eqnarray*}
\label{Eq:BoundOfSquareTerm}
\frac{L}{2}  \| -\alpha_t  \mat{P}_{x_t} \hat{\mat{V}}_t^{-1/2}\vec{g}_{t}(x_t) \|_{2}^2 
& \leq & \frac{L}{2}\| \alpha_t \hat{\mat{V}}_t^{-1/2} \vec{g}_{t}(x_t) \|_{2}^2, 
\end{eqnarray*}
where the inequality holds because of the contraction of the projection operator.

Consequently, (\ref{Eq:OurDescentLemma_1}) results in 
\begin{eqnarray}
\label{Eq:OurDescentLemma_1-1}
f(x_{t+1}) 
& \leq &  f(x_t) + \langle \vec{g}(x_t), -\alpha_t \hat{\mat{V}}_t^{-1/2}\vec{g}_{t}(x_t)\rangle_2+\frac{L}{2}  \| \alpha_t  \hat{\mat{V}}_t^{-1/2}\vec{g}_{t}(x_t) \|_{2}^2.
\end{eqnarray}

Now, we consider the case (\ref{Eq:OurDescentLemma_1-1}) with $t = 1$. We have
\begin{eqnarray*}
\label{Eq:OurDescentLemma_1-1_t=1}
f(x_{2}) 
& \leq &  f(x_1) + \langle \vec{g}(x_1), -\alpha_1 \hat{\mat{V}}_1^{-1/2}\vec{g}_{1}(x_1)\rangle_2+\frac{L}{2}  \| \alpha_1  \hat{\mat{V}}_1^{-1/2}\vec{g}_{1}(x_t) \|_{2}^2.
\end{eqnarray*}
Taking expectation and rearranging terms of above yields
\begin{eqnarray}
\label{Eq:OurDescentLemma_1-1_t=1_exp}
\EE_{i_1}[f(x_{2})] 
& \leq &  \EE_{i_1}[f(x_1)] + \EE_{i_1}[\langle \vec{g}(x_1), -\alpha_1 \hat{\mat{V}}_1^{-1/2}\vec{g}_{1}(x_1)\rangle_2]+\EE_{i_1}\left[\frac{L}{2}  \| \alpha_1  \hat{\mat{V}}_1^{-1/2}\vec{g}_{1}(x_1) \|_{2}^2 \right] \nonumber\\
& = &  f(x_1) + 
{\EE_{i_1}}\Bigg[\left\langle \vec{g}(x_1), -\alpha_1 \frac{\vec{g}_{{1}}(x_1)}{\sqrt{\hat{\vec{v}}_1}}\right\rangle_2\Bigg]
+\EE_{i_1}  \Bigg[\frac{L}{2} \left\lVert \frac{\alpha_1  \vec{g}_{1}(x_1)}{\sqrt{\hat{\vec{v}}_1}} \right\rVert^2_{2}\Bigg].
\end{eqnarray}

Next, we consider the case (\ref{Eq:OurDescentLemma_1-1}) with $t \geq 2$. We have
\begin{eqnarray}
\label{Eq:OurDescentLemma_1_hosoku}
     -\langle \vec{g}(x_t), \alpha_t \mat{P}_{x_t}\hat{\mat{V}}_{t}^{-1/2}\vec{g}_{t}(x_t)\rangle_2
    & =& -\langle \vec{g}(x_t), \alpha_t\hat{\mat{V}}_{t}^{-1/2} \vec{g}_{t}(x_t) \rangle_2 \notag\\
    & =&-\langle \vec{g}(x_t), \alpha_{t-1} \hat{\mat{V}}_{t-1}^{-1/2}\vec{g}_{t}(x_t)\rangle_2\notag \\ 
    & &-\langle \vec{g}(x_t), \alpha_{t} \hat{\mat{V}}_{t}^{-1/2}\vec{g}_{t}(x_t)\rangle_2
    +  \langle \vec{g}(x_t), \alpha_{t-1} \hat{\mat{V}}_{t-1}^{-1/2} \vec{g}_{t}(x_t)\rangle_2  \notag \\ 
    & =&-\langle \vec{g}(x_t), \alpha_{t-1}\hat{\mat{V}}_{t-1}^{-1/2}\vec{g}_{t}(x_t)\rangle_2  -  \langle \vec{g}(x_t), (\alpha_{t}\hat{\mat{V}}_{t}^{-1/2} -  \alpha_{t-1} \hat{\mat{V}}_{t-1}^{-1/2})\vec{g}_{t}(x_t)\rangle_2\notag \\     
    & \leq&-\langle \vec{g}(x_t), \alpha_{t-1}\hat{\mat{V}}_{t-1}^{-1/2}\vec{g}_{t}(x_t)\rangle_2   
    + H^2 \Bigg[ \sum_{j=1}^d \changeHKK{\left| \frac{\alpha_{t}}{\sqrt{(\hat{\vec{v}_t})_j}} -  \frac{\alpha_{t-1}}{\sqrt{(\hat{\vec{v}}_{t-1})_j}} \right|} \Bigg],\notag \\
\end{eqnarray}
where the inequality holds because of Assumption \ref{Assump:1} (1.3) and Lemma \ref{Lem:uppder_bound_v}.

Plugging (\ref{Eq:OurDescentLemma_1_hosoku}) into (\ref{Eq:OurDescentLemma_1-1}) yields 
\begin{eqnarray*}
f(x_{t+1}) 
& \leq &  f(x_t) -\langle \vec{g}(x_t), \alpha_{t-1} \hat{\mat{V}}_{t-1}^{-1/2}\vec{g}_{t}(x_t)\rangle_2  \nonumber \\
&&+H^2 \Bigg[ \sum_{j=1}^d \changeHKK{\left| \frac{\alpha_{t}}{\sqrt{(\hat{\vec{v}_t})_j}} -  \frac{\alpha_{t-1}}{\sqrt{(\hat{\vec{v}}_{t-1})_j}} \right|} \Bigg] +  \frac{L}{2} \left\lVert \frac{\alpha_t  \vec{g}_{t}(x_t)}{\sqrt{\hat{\vec{v}}_t}} \right\rVert^2_{2}.
\end{eqnarray*}

Taking expectation and rearranging terms, we have 
\begin{eqnarray*}
\EE_{i_t} [ f(x_{t+1})] 
& \leq &  \EE_{i_t} [f(x_t)]
- \EE_{i_t} \bigg[ \langle \vec{g}(x_t), \alpha_{t-1}\hat{\mat{V}}_{t-1}^{-1/2}\vec{g}_{t}(x_t)\rangle_2  \bigg] \nonumber \\
&&+H^2 \EE_{i_t} \Bigg[ \sum_{j=1}^d \changeHKK{\left| \frac{\alpha_{t}}{\sqrt{(\hat{\vec{v}_t})_j}} - \frac{\alpha_{t-1}}{\sqrt{(\hat{\vec{v}}_{t-1})_j}} \right|} \Bigg]
 +  \EE_{i_t}  \Bigg[\frac{L}{2} \left\lVert \frac{\alpha_t  \vec{g}_{t}(x_t)}{\sqrt{\hat{\vec{v}}_t}} \right\rVert^2_{2}\Bigg],\nonumber \\
 & \leq &  f(x_t)
-  \langle \vec{g}(x_t), \alpha_{t-1}\hat{\mat{V}}_{t-1}^{-1/2} \EE_{i_t} [\vec{g}_{t}(x_t)]\rangle_2  \nonumber \\
&&+H^2 \EE_{i_t} \Bigg[ \sum_{j=1}^d \changeHKK{\left| \frac{\alpha_{t}}{\sqrt{(\hat{\vec{v}_t})_j}} - \frac{\alpha_{t-1}}{\sqrt{(\hat{\vec{v}}_{t-1})_j}} \right|} \Bigg]
 +  \EE_{i_t}  \Bigg[\frac{L}{2} \left\lVert \frac{\alpha_t  \vec{g}_{t}(x_t)}{\sqrt{\hat{\vec{v}}_t}} \right\rVert^2_{2}\Bigg],\nonumber \\
 & = &  f(x_t)
-  \alpha_{t-1}  \left\langle \vec{g}(x_t), \frac{ \vec{g}(x_t)}{\sqrt{\hat{\vec{v}}_{t-1}}} \right\rangle_2  \nonumber \\
&&+H^2 \EE_{i_t} \Bigg[ \sum_{j=1}^d \changeHKK{\left| \frac{\alpha_{t}}{\sqrt{(\hat{\vec{v}_t})_j}} - \frac{\alpha_{t-1}}{\sqrt{(\hat{\vec{v}}_{t-1})_j}} \right|} \Bigg]
 +  \EE_{i_t}  \Bigg[\frac{L}{2} \left\lVert \frac{\alpha_t  \vec{g}_{t}(x_t)}{\sqrt{\hat{\vec{v}}_t}} \right\rVert^2_{2}\Bigg],\nonumber \\ 
\end{eqnarray*}

Rearranging above, we obtain
\begin{align}
\label{Eq:OurDescentLemma_2}
& \alpha_{t-1}  \left\langle \vec{g}(x_t), \frac{ \vec{g}(x_t)}{\sqrt{\hat{\vec{v}}_{t-1}}} \right\rangle_2 \nonumber\\
 & \leq   
  \EE_{i_t}  \Bigg[\frac{L}{2} \left\lVert \frac{\alpha_t  \vec{g}_{t}(x_t)}{\sqrt{\hat{\vec{v}}_t}} \right\rVert^2_{2}\Bigg]
 +H^2 \EE_{i_t} \Bigg[ \sum_{j=1}^d \changeHKK{\left| \frac{\alpha_{t}}{\sqrt{(\hat{\vec{v}_t})_j}} - \frac{\alpha_{t-1}}{\sqrt{(\hat{\vec{v}}_{t-1})_j}} \right|} \Bigg]
  +  f(x_t) - \EE_{i_t} [ f(x_{t+1})].
\end{align}

After taking the expectation of each side with respect to $i_{t-1}, \ldots, i_2, i_1$ and telescoping \eqref{Eq:OurDescentLemma_2} for $t=2$ to $T$ and adding (\ref{Eq:OurDescentLemma_1-1_t=1_exp}), we have
\begin{align}
\label{Eq:OurDescentLemma_3}
& \EE  \Bigg[ \sum_{t=2}^T \alpha_{t-1}  \left\langle \vec{g}(x_t), \frac{ \vec{g}(x_t)}{\sqrt{\hat{\vec{v}}_{t-1}}} \right\rangle_2\Bigg] \nonumber\\
 & \leq   
  \EE  \Bigg[ \frac{L}{2} \sum_{t=2}^T\left\lVert \frac{\alpha_t  \vec{g}_{t}(x_t)}{\sqrt{\hat{\vec{v}}_t}} \right\rVert^2_{2}\Bigg]
 +H^2 \EE \Bigg[ \sum_{t=2}^T \sum_{j=1}^d \changeHKK{\left| \frac{\alpha_{t}}{\sqrt{(\hat{\vec{v}_t})_j}} - \frac{\alpha_{t-1}}{\sqrt{(\hat{\vec{v}}_{t-1})_j}} \right|}\Bigg]
  + \EE \bigg[\sum_{t=2}^T (f(x_t) -  f(x_{t+1}))\bigg]\nonumber \\
  &  \hspace*{0.4cm}+ \EE \Bigg[f(x_1) - \EE_{i_1}[f(x_{2})] + \left\langle \vec{g}(x_1), -\alpha_1 \frac{\vec{g}(x_1)}{\sqrt{\hat{\vec{v}}_1}}\right\rangle_2
+\EE_{i_1}  \Bigg[\frac{L}{2} \left\lVert \frac{\alpha_1  \vec{g}_{1}(x_1)}{\sqrt{\hat{\vec{v}}_1}} \right\rVert^2_{2}\Bigg]\Bigg]\nonumber \\
 & =   
 \EE  \Bigg[\frac{L}{2}   \sum_{t=1}^T\left\lVert \frac{\alpha_t  \vec{g}_{t}(x_t)}{\sqrt{\hat{\vec{v}}_t}} \right\rVert^2_{2}
 +H^2  \sum_{t=2}^T \left\lVert\frac{\alpha_{t}}{\sqrt{\hat{\vec{v}}_t}} -  \frac{\alpha_{t-1}}{\sqrt{\hat{\vec{v}}_{t-1}}}  \right\rVert_1
  + \left\langle \vec{g}(x_1), -\alpha_1 \frac{\vec{g}_{{1}}(x_1)}{\sqrt{\hat{\vec{v}}_1}}\right\rangle_2\Bigg]+  \EE [f(x_1) -f(x^*)],   
\end{align}
where $x^*$ is an optimal of $f$. The terms $ \mathbb{E} \Bigg [\left\langle \vec{g}(x_1), -\alpha_1 \frac{\vec{g}_{{1}}(x_1)}{\sqrt{\hat{\vec{v}}_1}}\right\rangle_2\Bigg]+  \EE [f(x_1) -f(x^*)]$ in (\ref{Eq:OurDescentLemma_3}) do not depend on $T$ and hence are constants, which we call $C$. This completes the proof. 
\end{proof}

\subsection{Proof of Corollary \ref{Cor:convergence_analysis1}}

This section provides the complete proof of Corollary \ref{Cor:convergence_analysis1}. 
\begin{proof}[Proof of Corollary \ref{Cor:convergence_analysis1}]

From the assumption, the first term in the right-hand side of (\ref{Eq:OurDescentLemma_3}) yields
\begin{align}
	\label{Eq:FirstTerm}
     & \EE\bigg[\frac{L}{2}   \sum_{t=1}^T\left\lVert \frac{\alpha_t  \vec{g}_{t}(x_t)}{\sqrt{\hat{\vec{v}}_t}} \right\rVert^2_{2}   \bigg]
    \leq \EE \bigg[\frac{L}{2} \sum_{t=1}^T\left\lVert \frac{\alpha_t \vec{g}_{t}(x_t)}{c} \right\rVert^2_{2}   \bigg]\notag \\        
   & 
   \leq \EE \bigg[\frac{L}{2} \sum_{t=1}^T \left(\frac{1}{c\sqrt{t}} \right)^2 \| \vec{g}_{t}(x_t) \|_{2}^2   \bigg] \ \leq \frac{LH^2}{2c^2} \sum_{t=1}^T \frac{1}{t} \ \leq \frac{LH^2}{2c^2} (1 + \log T),                      
\end{align}
where the last inequality uses $\sum_{t=1}^T \frac{1}{t} \leq 1+ \log T$.

The second term in the right-hand side of (\ref{Eq:OurDescentLemma_3}) yields 
\begin{align}
\label{Eq:SecondTerm}
&H^2  \EE \bigg[\sum_{t=2}^T \left\lVert\frac{\alpha_{t}}{\sqrt{\hat{\vec{v}}_t}} -  \frac{\alpha_{t-1}}{\sqrt{\hat{\vec{v}}_{t-1}}}  \right\rVert_1
  \Bigg] 
  = H^2  \EE \Bigg[\sum_{t=2}^T\sum_{j=1}^d  \left(\frac{\alpha_{t-1}}{\sqrt{(\hat{\vec{v}}_{t-1})_j}}- \frac{\alpha_{t}}{\sqrt{(\hat{\vec{v}}_t)_j}}   \right) \changeHK{\Bigg]}\notag \\
  & = H^2  \EE \Bigg[\sum_{j=1}^d  \left(\frac{\alpha_{1}}{\sqrt{(\hat{\vec{v}}_{1})_j}}- \frac{\alpha_{T}}{\sqrt{(\hat{\vec{v}}_T)_j}}   \right)  \Bigg]
  \ \changeHK{\leq}\ H^2  \EE \Bigg[\sum_{j=1}^d \frac{\alpha_{1}}{\sqrt{(\hat{\vec{v}}_{1})_j}}    \Bigg] \ \leq \frac{dH^2}{c},
\end{align}
where the first equality holds because of $(\hat{\vec{v}}_t)_j  \geq (\hat{\vec{v}}_{t-1})_j$ and $\alpha_t \leq \alpha_{t-1}$. 

The third term in the right-hand side of (\ref{Eq:OurDescentLemma_3}) with $\alpha_1 = 1$ yields 
\begin{align}
\label{Eq:ThirdTerm}
{\EE}\bigg[\left\langle \vec{g}(x_1), -\alpha_1 \frac{\vec{g}_{{1}}(x_1)}{\sqrt{\hat{\vec{v}}_1}}\right\rangle_2\bigg] 
 \leq H^2 \sum_{j=1}^d \frac{1}{\sqrt{(\hat{\vec{v}}_1)_j}} \leq \frac{dH^2}{c}.
\end{align}

Then, plugging (\ref{Eq:FirstTerm}), (\ref{Eq:SecondTerm}) and (\ref{Eq:ThirdTerm}) into (\ref{Eq:OurDescentLemma_3}) yields
\begin{align}
\label{Eq:OurDescentLemma_4}
      \EE  \Bigg[ \sum_{t=2}^T \alpha_{t-1}  \left\langle \vec{g}(x_t), \frac{ \vec{g}(x_t)}{\sqrt{\hat{\vec{v}}_{t-1}}} \right\rangle_2\Bigg] \leq \frac{LH^2}{2c^2} (1 + \log T) + \frac{2dH^2}{c} + \EE [f(x_1) -f(x^*)].
\end{align}

{
Here, from $\alpha_t = 1/\sqrt{t}$ and $(\hat{\vec{v}}_t)_j \leq H^2$ in Lemma, we have 
\begin{align*}
\frac{\alpha_{t-1}}{(\sqrt{\hat{\vec{v}}_{t-1}})_j} \geq \frac{1}{H \sqrt{t-1}}.
\end{align*}
Therefore, we obtain
\begin{eqnarray}
\label{Eq:OurDescentLemma_5}
\EE  \Bigg[ \sum_{t=2}^T \alpha_{t-1}  \left\langle \vec{g}(x_t), \frac{ \vec{g}(x_t)}{\sqrt{\hat{\vec{v}}_{t-1}}} \right\rangle_2\Bigg] 
&\geq& \EE  \Bigg[ \sum_{t=2}^T  \frac{1}{H\sqrt{t-1}} \| \vec{g}(x_t)\|_2^2 \Bigg] \nonumber \\
&\geq& \sum_{t=2}^T  \frac{1}{H\sqrt{t-1}}  \min_{t \in [2,\ldots,T]} \EE  \Bigg[\| \vec{g}(x_t)\|_2^2 \Bigg] \nonumber \\    
&=& \frac{\min_{t \in [2,\ldots,T]} \EE  [\| \vec{g}(x_t)\|_2^2 ] }{H}\sum_{t=2}^T  \frac{1}{\sqrt{t-1}}  \nonumber \\      
&\geq& \frac{\sqrt{T-1}}{H} \min_{t \in [2,\ldots,T]} \EE  [\| \vec{g}(x_t)\|_2^2 ],
\end{eqnarray}
where the third equality holds due to  $\sum_{t=1}^T \frac{1}{\sqrt{t}} \geq \sqrt{T}$.
}

{
Hence, plugging (\ref{Eq:OurDescentLemma_5}) into (\ref{Eq:OurDescentLemma_4}) and arranging it, we have 
\begin{eqnarray}
      \min_{t \in [2,\ldots,T]} \EE  [\| \vec{g}(x_t)\|_2^2 ] \leq 
      \frac{1}{\sqrt{T-1}}\bigg[\frac{LH^3}{2c^2} (1 + \log T) + \frac{2dH^3}{c} + H\EE [f(x_1) -f(x^*)]\bigg].
\end{eqnarray}
}

Rearranging \eqref{Eq:OurDescentLemma_8} and addressing the original definition $\vec{g}(x):={\rm vec}(\gradf (x))$, we finally obtain 
\begin{align*}
    {\min_{t \in [2,\ldots,T]} } \EE\| \gradf (x_{{t}})\|_{{F}}^2 
    \leq \frac{1}{\sqrt{T{-1}}}(Q_1 + Q_2 \log ({T})),
\end{align*}
where $\{Q_i\}_{i=1}^2$ are defined as below:
\begin{eqnarray*}
     Q_1 &=& \frac{LH^3}{2c^2} + \frac{2dH^3}{c} + H\EE [f(x_1) -f(x^*)], \notag\\
     Q_2 &=& \frac{LH^3}{2c^2}.
 \end{eqnarray*}
This completes the proof. 
\end{proof}


\subsection{\changeHK{Variable $\beta$ algorithm and its convergence analysis (not included in the main manuscript)}}
\label{appSec:VariBeataAlgorithm}

\changeHK{We consider a variant of Algorithm \ref{Alg:R-AGD-diag}, which uses  $\beta = 1 - 1/t$, and give its convergence analysis. For this purpose, we additionally modify Algorithm \ref{Alg:R-AGD-diag} to make the proof simpler, where the max operator is directly performed onto ${\bv}_t$ instead of ${\bl}_t$ and ${\br}_t$ individually. This ensures $(\hat{\vec{v}}_t)_j \geq (\hat{\vec{v}}_{t-1})_j$, and enables to directly apply the result of Lemma \ref{Lem:vt_seq_AdaGrad} for (\ref{Eq:BoundOfHatVtBetaVariAlg}) of the proof of Corollary \ref{Cor:convergence_analysis2} as well as similarly in (\ref{Eq:SecondTerm}) of the proof of Corollary \ref{Cor:convergence_analysis1}. As shown later, its convergence rate is slightly better than that of Algorithm \ref{Alg:R-AGD-diag}. The new algorithm is summarized in Algorithm \ref{Alg:R-AGD-diag-variable-beta}.}

\begin{algorithm}[htbp]
\caption{Riemannian adaptive stochastic algorithm \changeHK{with variable $\beta$}}
\label{Alg:R-AGD-diag-variable-beta}
\begin{algorithmic}[1]
\REQUIRE{Step size $\{\alpha_t\}_{t=1}^T$.}
\STATE{Initialize $x_1\in\M, \changeHK{\bv_{0}=\hat{\bv}_{0}=\bzero_{n}}$.}
\FOR{$t=1,2, \ldots, T$} 
\STATE{\changeHK{Set $\beta = 1 - 1/t$.}}
\STATE{Compute Riemannian stochastic gradient $\bG_t = \gradf_{t}(x_{t})$.}
\STATE{\changeHK{Compute $\vec{p}_t = {\rm diag}(\bG_t\bG_t^T)/r$ and $\vec{q}_t={\rm diag}(\bG_t^T\bG_t)/n$.}}
\STATE{\changeHK{Modify $\vec{p}_t$ and $\vec{q}_t$ to $\hat{\vec{p}}_t$ and $\hat{\vec{q}}_t$ to satisfy Assumption \ref{Assump:2}.}}
\STATE{Update $\bl_t = \beta \bl_{t-1} +  (1-\beta)\changeHK{\hat{\vec{p}}_t}$.}
\STATE{Update $\br_t = \beta \br_{t-1} + (1-\beta)\changeHK{\hat{\vec{q}}_t}$.}
\STATE{\changeHK{Calculate ${\bv}_t={\br}_t^{1/2}\otimes{\bl}_t^{1/2}$.}}
\STATE{\changeHK{Calculate $\hat{\bv}_t = \max(\hat{\bv}_{t-1}, \bv_t)$.}}
\STATE{\changeHK{$x_{t+1} = R_{x_t}(-\alpha_t\P_{x_t}({\rm vec}^{-1}({\mathrm{Diag}(\hat{\bv}_t)}^{-1/2}{\rm vec}(\bG_t))))$.}}
\ENDFOR
\end{algorithmic}
\end{algorithm}

\changeHK{It should be noted that the modified algorithm explicitly generates $\hat{\bv}_t$, and this causes computational inefficiency. It should be also noted that we need to modify $\vec{p}_t = {\rm diag}(\bG_t\bG_t^T)/r $ and $\vec{q}_t={\rm diag}(\bG_t^T\bG_t)/n $ into $\hat{\vec{p}}_t$ and $\hat{\vec{q}}_t$ to guarantee Assumption \ref{Assump:2} as described in Step 6 of Algorithm \ref{Alg:R-AGD-diag-variable-beta}.}

An additional assumption is first required.
\begin{assumption}
\label{Assump:2}
We denote $\gradf_{i_t}(x_t)$ as $\mat{G}_t \in \mathbb{R}^{n \times r}$, and the vectorized form of $\gradf_{i_t}(x_t)$ as $\vec{g}_{t}(x_t) \in \mathbb{R}^{nr}$. Then, we assume 
\[
\sqrt{({\rm diag}(\mat{G}_t^T\mat{G}_t))_p ({\rm diag}(\mat{G}_t\mat{G}_t^T))_q{/(nr)}}\geq(\vec{g}_{t}(x_t))^2_{j},
\]
where $j=n(p-1)+q$ with $p \in [r]$ and  $q \in [n]$.
\end{assumption}
Next, we derive the condition of the sequence of $\vec{v}_t$ below.
\begin{Lem}
\label{Lem:vt_seq_AdaGrad}
Let $\{x_t\}$ and $\{\changeHK{\vec{v}}_t\}$ be the sequences obtained from Algorithm \ref{Alg:R-AGD-diag-variable-beta} with $\beta = 1 - 1/t$. Under Assumptions \ref{Assump:1} and \ref{Assump:2}, the sequence $\{\changeHK{\vec{v}}_t\}$  satisfies 
\begin{eqnarray*}
(\vec{v}_t)_{j} 
&\geq& \frac{1}{t} \sum_{i=1}^t(\vec{g}_{i}(x_i))^2_{j},
\end{eqnarray*}
where $(\vec{g}_{i}(x_i))_{j}$ is the $j$-th element of the vectorized stochastic gradient $\vec{g}_{i}(x_i)$.
\end{Lem}

\begin{proof}
From $\beta = 1 - \frac{1}{t}$ in Algorithm \ref{Alg:R-AGD-diag-variable-beta}, we have
\begin{eqnarray*}
\vec{l}_t &=& \left(1 - \frac{1}{t}\right) \vec{l}_{t-1} +  \frac{1}{t}\frac{{\rm diag}(\mat{G}_t\mat{G}_t^T)}{\changeHK{r}},\\
\vec{r}_t &=& \left(1 - \frac{1}{t}\right) \vec{r}_{t-1} + \frac{1}{t}\frac{{\rm diag}(\mat{G}_t^T\mat{G}_t)}{\changeHK{n}}.
\end{eqnarray*}
Addressing the $\changeHK{j}$-th element of $\vec{v}_t = \vec{r}_t^{1/2} \otimes  \vec{l}_t^{1/2} \in \mathbb{R}^{nr}$, denoted as $(\vec{v}_t)_{\changeHK{j}} \in \mathbb{R}$, we have
\begin{eqnarray*}
(\vec{v}_t)_{\changeHK{j}} 
&=&  (\vec{r}_t^{1/2} \otimes  \vec{l}_t^{1/2})_{\changeHK{j}}\\
&=& 
\left[\left(\frac{t-1}{t} \vec{r}_{t-1} + \frac{1}{t}\frac{{\rm diag}(\mat{G}_t^T\mat{G}_t)}{\changeHK{n}}\right)_p\right]^{1/2}
\left[\left(\frac{t-1}{t}\vec{l}_{t-1} +  \frac{1}{t}\frac{{\rm diag}(\mat{G}_t\mat{G}_t^T)}{\changeHK{r}}\right)_q\right]^{1/2}
 \\
 &=& \left[ \left(\frac{t-1}{t} \vec{r}_{t-1} + \frac{1}{t}\frac{{\rm diag}(\mat{G}_t^T\mat{G}_t)}{\changeHK{n}}\right)_p
 \left(\frac{t-1}{t}\vec{l}_{t-1} +  \frac{1}{t}\frac{{\rm diag}(\mat{G}_t\mat{G}_t^T)}{\changeHK{r}}\right)_q
\right]^{1/2}\\
  &=& \left[
  \left( \frac{t-1}{t}\left( \vec{r}_{t-1}\right)_p + \frac{1}{t}\frac{\left({\rm diag}(\mat{G}_t^T\mat{G}_t)\right)_p}{\changeHK{n}}\right)
  \left(\frac{t-1}{t}\left(\vec{l}_{t-1}\right)_q +  \frac{1}{t}\frac{\left({\rm diag}(\mat{G}_t\mat{G}_t^T)\right)_q}{\changeHK{r}}\right)
\right]^{1/2}\\
  &=& \left[
  \left(\frac{t-1}{t}\right)^2(\vec{r}_{t-1})_p (\vec{l}_{t-1})_q 
  +  
  \left(\frac{1}{t}\right)^2 
  \frac{\left({\rm diag}(\mat{G}_t^T\mat{G}_t)\right)_p}{\changeHK{n}} 
  \frac{\left({\rm diag}(\mat{G}_t\mat{G}_t^T)\right)_q}{\changeHK{r}}  \right.\\
&&\left.
+
\frac{1}{t}\left( 1 - \frac{1}{t}\right) (\vec{r}_{t-1})_p\frac{\left({\rm diag}(\mat{G}_t\mat{G}_t^T)\right)_q}{\changeHK{r}} + 
\frac{1}{t}\left( 1 - \frac{1}{t}\right) (\vec{l}_{t-1})_q\frac{\left({\rm diag}(\mat{G}_t^T\mat{G}_t)\right)_p}{\changeHK{n}}
\right]^{1/2}\\
  &=& \left[
  \left(\frac{t-1}{t}\right)^2(\vec{r}_{t-1})_p (\vec{l}_{t-1})_q 
  +  \left(\frac{1}{t}\right)^2 
  \frac{\left({\rm diag}(\mat{G}_t^T\mat{G}_t)\right)_p}{\changeHK{n}} 
  \frac{\left({\rm diag}(\mat{G}_t\mat{G}_t^T)\right)_q}{\changeHK{r}} \right.\\
&&\left.
+\frac{1}{t}\left( 1 - \frac{1}{t}\right) 
\left[
(\vec{r}_{t-1})_p \frac{\left({\rm diag}(\mat{G}_t\mat{G}_t^T)\right)_q}{\changeHK{r}} + 
(\vec{l}_{t-1})_q \frac{\left({\rm diag}(\mat{G}_t^T\mat{G}_t)\right)_p}{\changeHK{n}} \right]
\right]^{1/2}\\
  &\geq& \left[
  \left(\frac{t-1}{t}\right)^2(\vec{r}_{t-1})_p (\vec{l}_{t-1})_q 
  +  \left(\frac{1}{t}\right)^2 
  \frac{\left({\rm diag}(\mat{G}_t^T\mat{G}_t)\right)_p}{\changeHK{n}} 
  \frac{\left({\rm diag}(\mat{G}_t\mat{G}_t^T)\right)_q}{\changeHK{r}} \right.\\
&&\left.
+\frac{2}{t}\left( 1 - \frac{1}{t}\right) 
\sqrt{(\vec{r}_{t-1})_p \frac{\left({\rm diag}(\mat{G}_t^T\mat{G}_t)\right)_p}{\changeHK{n}}} 
\sqrt{ (\vec{l}_{t-1})_q\frac{\left({\rm diag}(\mat{G}_t\mat{G}_t^T)\right)_q}{\changeHK{r}} } \ \  
\right]^{1/2}\\
  &=& 
  \left(\frac{t-1}{t}\right)\sqrt{(\vec{r}_{t-1})_p (\vec{l}_{t-1})_q} 
  +  \frac{1}{t} 
  \sqrt{ 
  \frac{\left({\rm diag}(\mat{G}_t^T\mat{G}_t)\right)_p}{\changeHK{n}} 
  \frac{\left({\rm diag}(\mat{G}_t\mat{G}_t^T)\right)_q}{\changeHK{r}}},
\end{eqnarray*}
where the inequality uses the inequality of arithmetic and geometric means.
Since $(\vec{v}_{t-1})_{\changeHK{j}} = (\vec{r}_{t-1}^{1/2} \otimes  \vec{l}_{t-1}^{1/2})_{\changeHK{j}}=\sqrt{(\vec{r}_{t-1})_p (\vec{l}_{t-1})_q}$, we finally have 
\begin{eqnarray*}
(\vec{v}_t)_{\changeHK{j}} &\geq&   \left(1 - \frac{1}{t}\right)(\vec{v}_{t-1})_{\changeHK{j}} 
  +  \frac{1}{t}   \sqrt{ 
  \frac{\left({\rm diag}(\mat{G}_t^T\mat{G}_t)\right)_p}{\changeHK{n}} 
  \frac{\left({\rm diag}(\mat{G}_t\mat{G}_t^T)\right)_q}{\changeHK{r}}}.\
\end{eqnarray*}
Consequently, \changeHK{from Assumption \ref{Assump:2},
and $(\vec{v}_0)=\vec{0}$, the recursive updates of the inequality above finally lead to}
\begin{eqnarray*}
(\vec{v}_t)_{\changeHK{j}} 
&\geq& \left(1 - \frac{1}{t}\right)(\vec{v}_{t-1})_{\changeHK{j}}   +  \frac{1}{t} (\vec{g}_{t}(x_t))^2_{\changeHK{j}}\\
&\geq& \changeHK{
\left(1 - \frac{1}{t}\right)
\left( 
\left(1 - \frac{1}{t-1}\right)(\vec{v}_{t-2})_{\changeHK{j}}   +  \frac{1}{t-1} (\vec{g}_{t-1}(x_{t-1}))^2_{\changeHK{j}}
\right)   +  \frac{1}{t} (\vec{g}_{t}(x_t))^2_{\changeHK{j}}
}\\
&\geq& \frac{1}{t} \sum_{i=1}^t(\vec{g}_{i}(x_i))^2_{\changeHK{j}}.
\end{eqnarray*}
This yields the desired result, and this completes the proof. 
\end{proof}

Now, we derive an convergence rate of \changeHK{Algorithm \ref{Alg:R-AGD-diag-variable-beta}} below.
\begin{Cor}[Convergence rate analysis of \changeHK{Algorithm \ref{Alg:R-AGD-diag-variable-beta}}]
\label{Cor:convergence_analysis2}
Let $\{x_t\}$ and $\{\hat{\vec{v}}_t\}$ be the sequences obtained from Algorithm \ref{Alg:R-AGD-diag-variable-beta}. 
Suppose $\alpha_t=1/\sqrt{t}$, and, suppose that $\min_{j \in [d]} \sqrt{(\hat{\vec{v}}_1)_j}$ is lower-bounded by a constant $c > 0$, where $d$ is the dimension of the manifold $\mathcal{M}$. Also, suppose that Lemma \ref{Lem:vt_seq_AdaGrad} holds. Then, under Assumptions \ref{Assump:1} and \ref{Assump:2}, the output of $x_{t}$ of Algorithm \ref{Alg:R-AGD-diag-variable-beta} satisfies
%
 %
%
\begin{align}
   {\min_{t \in [2,\ldots,T]}}  \EE\| \gradf (x_{{t}})\|_{{F}}^2 \leq \frac{1}{\sqrt{T{-1}}}(R_1 + R_2 \log ({T})), \nonumber
\end{align}
where
\begin{eqnarray*}
     R_1 &\!\!\!\!=\!\!\!\!& \frac{dLH}{2}(1 + 2 \log H) + \frac{2dH^3}{c} + H \EE [f(x_1) -f(x^*)], \notag\\
     R_2 &\!\!\!\!=\!\!\!\!& \frac{LH^3}{2c^2}.
 \end{eqnarray*}
\end{Cor}

As seen, the constant $R_1$ is less than $Q_1$ in Corollary \ref{Cor:convergence_analysis1}. 

We provide the complete proof of Corollary \ref{Cor:convergence_analysis2}. 
\begin{proof}[Proof of Corollary \ref{Cor:convergence_analysis2}]

From the assumptions $\alpha_t=1/\sqrt{t}$ and $\sqrt{({\rm diag}(\mat{G}_t^T\mat{G}_t))_p ({\rm diag}(\mat{G}_t\mat{G}_t^T))_q\changeHK{/{nr}}} \geq (\vec{g}_{t}(x_t))^2_{\changeHK{j}}$ in Assumption \ref{Assump:2}, and Lemma \ref{Lem:vt_seq_AdaGrad}, we obtain
\begin{align}
\label{Eq:BoundOfHatVtBetaVariAlg}
\frac{\alpha_t}{\sqrt{(\hat{\vec{v}}_t)_j}} 
\changeHK{\ \leq \ 
\frac{1}{\sqrt{t({\vec{v}}_t)_j}}
}
\ \leq \  
\frac{1}{\sqrt{\sum_{i=1}^t \left(\vec{g}_i(x_i)\right)^{\changeHK{2}}_j}}.
\end{align}

Therefore, the first term in the right-hand side of (\ref{Eq:OurDescentLemma_3}) yields 
\begin{eqnarray*}
     \EE\bigg[\frac{L}{2} \sum_{t=1}^T\left\lVert \frac{\alpha_t  \vec{g}_{t}(x_t)}{\sqrt{\hat{\vec{v}}_t}} \right\rVert^2_{2}   \bigg]
   & \leq &\EE\bigg[\frac{L}{2} \sum_{j=1}^d\sum_{t=1}^T \frac{ (\vec{g}_{t}(x_t))_j^2}{\sum_{i=1}^t (\vec{g}_{i}(x_i))^2_j}  \bigg]  \\
  & \leq &   \EE\bigg[\frac{L}{2} \sum_{j=1}^d \left((1 + \log \sum_{t=1}^T (\vec{g}_{t}(x_t))_j^2\right) \bigg] \notag \\  
  & \leq &  \frac{dL}{2}(1 + 2 \log H + \log T),
\end{eqnarray*}
where the second inequality uses Lemma \ref{Lem:sumlog} below.

The second and third terms in the right-hand side of (\ref{Eq:OurDescentLemma_3}) are identical to (\ref{Eq:SecondTerm}) and (\ref{Eq:ThirdTerm}), respectively. Similarly to Corollary \changeHKK{\ref{Cor:convergence_analysis1}}, we obtain
\begin{eqnarray}
\label{Eq:OurDescentLemma_8}
	{\min_{t \in [2,\ldots,T]} \EE  [\| \vec{g}(x_t)\|_2^2 ]}
      & \leq &
      \frac{1}{\sqrt{T-1}} \bigg[ 
  \frac{dL{H}}{2}(1 + 2 \log H + \log T) + \frac{2dH^{{3}}}{c} + {H}\EE [f(x_1) -f(x^*)]
  \bigg],\notag\\
\end{eqnarray}
where \changeHKK{we use} $\sum_{t=2}^T \alpha_{t-1} = \sum_{t=1}^{T-1} \alpha_{t}$, $\alpha_t = \frac{1}{\sqrt{t}}$ and $\sum_{t=1}^T \frac{1}{\sqrt{t}} \geq \sqrt{T}$.

Rearranging \eqref{Eq:OurDescentLemma_8}, we finally obtain
\begin{align*}
    {\min_{t \in [2,\ldots,T]} } \EE\| \gradf (x_{{t}})\|_{{F}}^2 \leq \frac{1}{\sqrt{T{-1}}}(R_1 + R_2 \log ({T})),
\end{align*}
where $\{R_i\}_{i=1}^2$ are defined as below:
\begin{eqnarray*}
     R_1 &=& \frac{dLH}{2}(1 + 2 \log H) + \frac{2dH^3}{c} + H \EE [f(x_1) -f(x^*)], \notag\\
     R_2 &=& \frac{LH^3}{2c^2}.
 \end{eqnarray*}
This completes the proof. 
\end{proof}

\begin{Lem}[(Lemma 6 in \citep{Ward_arXiv_2018})]
\label{Lem:sumlog}
For $a_t \leq 0$ and $\sum_{i=1}^t a_i \neq 0$, we have
\begin{align*}
\sum_{t=1}^T \frac{a_t}{\sum_{i=1}^t a_t} \leq 1 + \log \sum_{t=1}^T  a_i.
\end{align*}
\end{Lem}

\changeHK{
{\bf Implementation details:}
Lastly, it should be noted that, one naive implementation to satisfy Assumption \ref{Assump:2} is to generate $\hat{\vec{p}}_t$ and $\hat{\vec{q}}_t$ is by enforcing
\[
\sqrt{(\hat{\vec{p}}_t)_p (\hat{\vec{q}}_t)_q}=(\vec{g}_{t}(x_t))^2_{j},
\]
 whenever $\sqrt{({\vec{p}}_t)_p ({\vec{q}}_t)_q}$ is less than $(\vec{g}_{t}(x_t))^2_{j}$. For this particular solution, there still exists freedom about how to determine $(\hat{\vec{p}}_t)_p$ and $(\hat{\vec{q}}_t)_q$. One possible strategy is to set $(\hat{\vec{p}}_t)_p= (\hat{\vec{q}}_t)_q=(\vec{g}_{t}(x_t))^2_{j}$. However, since multiple modifications for one particular $p$ or $q$ could happen for different choices of $j$. Therefore, we select the highest $(\vec{g}_{t}(x_t))^2_{j}$ among them for $(\hat{\vec{p}}_t)_p$, and then select $(\vec{g}_{t}(x_t))^2_{j}$ for $(\hat{\vec{q}}_t)_q$, in a heuristic way.}

\section{Additional numerical results} \label{app:sec:additional_results}

\subsection{Results on synthetic datasets for the PCA problem ({\bf Case P1})}

\changeHK{This section first gives the best-tuned results on a synthetic dataset for the PCA problem, well-conditioned case ({\bf Case P1}). We also include an ill-conditioned case. The results are shown in Figures \ref{appfig:PCA_results_Syn_Best_tuned} (a) and (b). Each result shows the optimality gap against the number of iterations as well as run-time. It should be noted that Figure \ref{appfig:PCA_results_Syn_Best_tuned} (a-1) is identical to Figure \ref{fig:PCA_results} (a). From these figures, we see that the proposed algorithms RASA-LR and RASA-R perform better than other algorithms. We also show all the results for each choice of the step size $\alpha_0$ in Figures \ref{appfig:PCA_results_Syn_well_cond} and \ref{appfig:PCA_results_Syn_ill_cond} for the well-conditioned case and the ill-conditioned cases, respectively. From the figures, in both the cases, RASA gives stably good performances, and RASA-LR yields the lowest optimality gap among all the baseline algorithms.}

\begin{figure*}[htbp]
\begin{center}
	\begin{minipage}[t]{.33\textwidth}
	\begin{center}
		\includegraphics[width=\textwidth]{{results/pca/syn/pca_comp-syn-stiefel-10000-100-10-1-best-tuned}.eps}\\
		
		{\small  (a-1) Number of iterations.}
	\end{center}		
	\end{minipage}	
	\hspace*{0.5cm}
	\begin{minipage}[t]{.33\textwidth}
	\begin{center}
		\includegraphics[width=\textwidth]{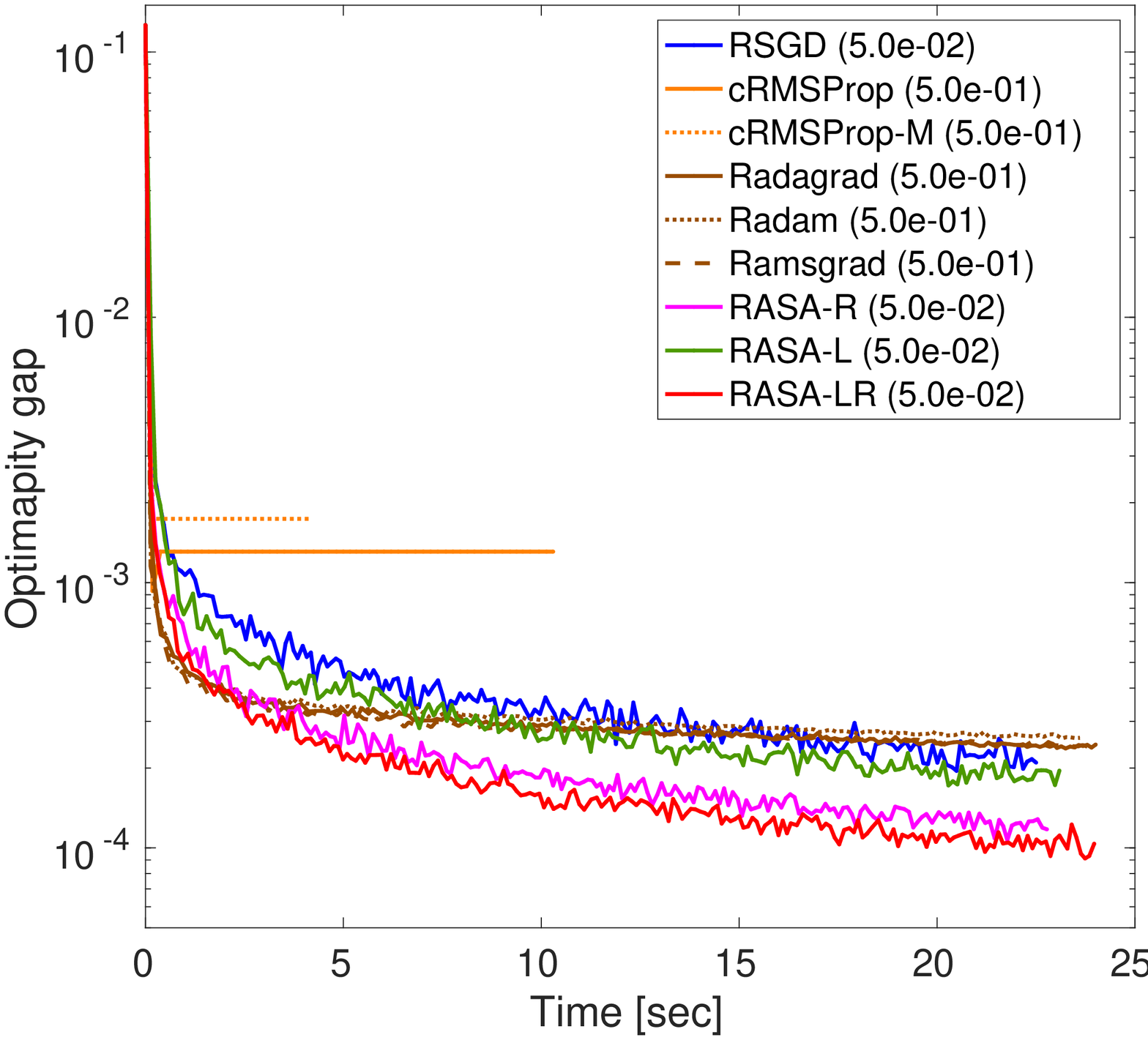}\\
		
		{\small  (a-2) Run-time.}
		
	\end{center} 
	\end{minipage}
	
	\vspace*{0.2cm}	
	
	{\small\bf  (a) well-conditioned case  ({\bf Case P1}).}
	
	\vspace*{0.4cm}
	
	\begin{minipage}[t]{.33\textwidth}
	\begin{center}
		\includegraphics[width=\textwidth]{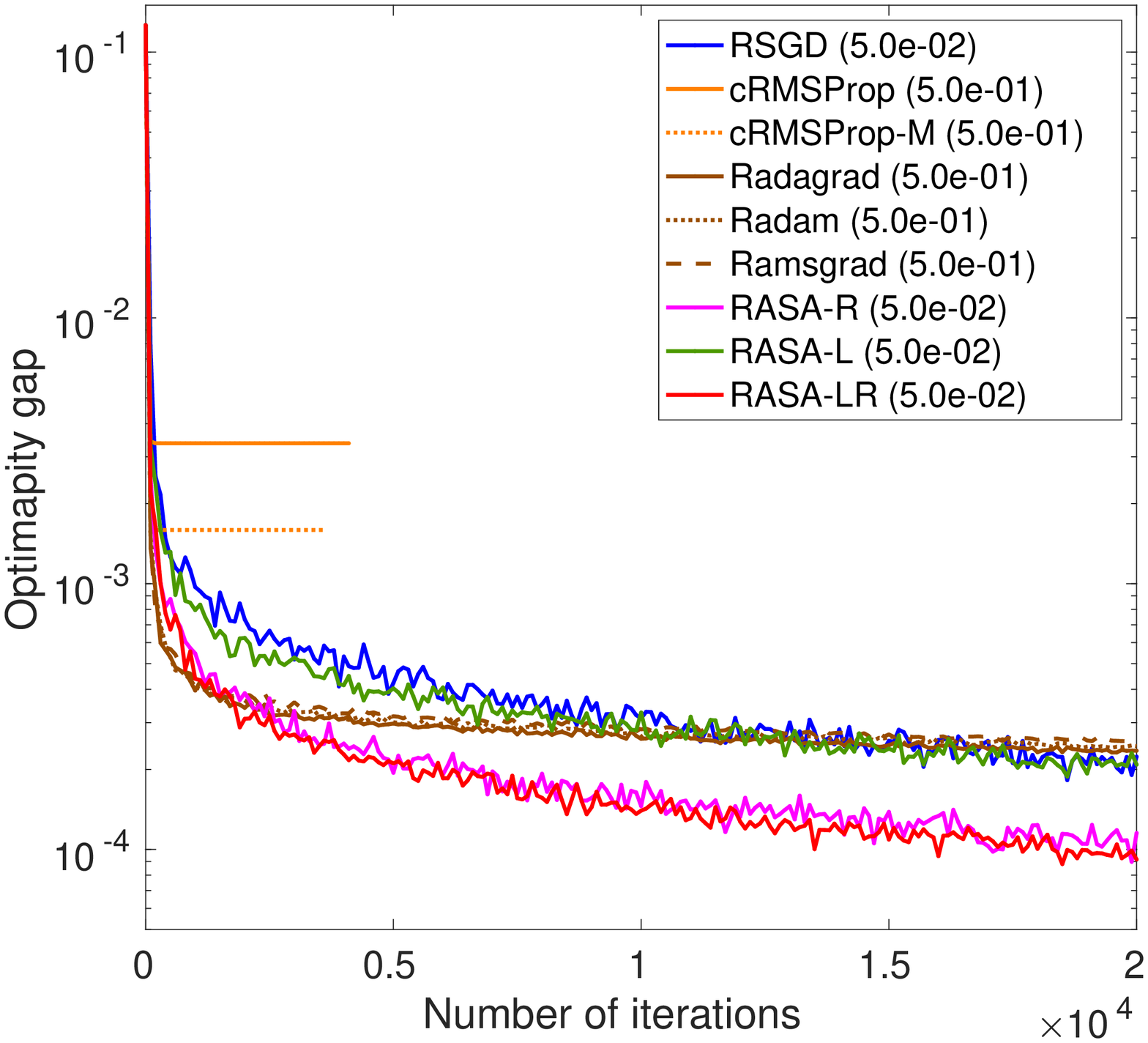}\\
		
		{\small  (b-1) Number of iterations.}
	\end{center}		
	\end{minipage}	
	\hspace*{0.5cm}
	\begin{minipage}[t]{.33\textwidth}
	\begin{center}
		\includegraphics[width=\textwidth]{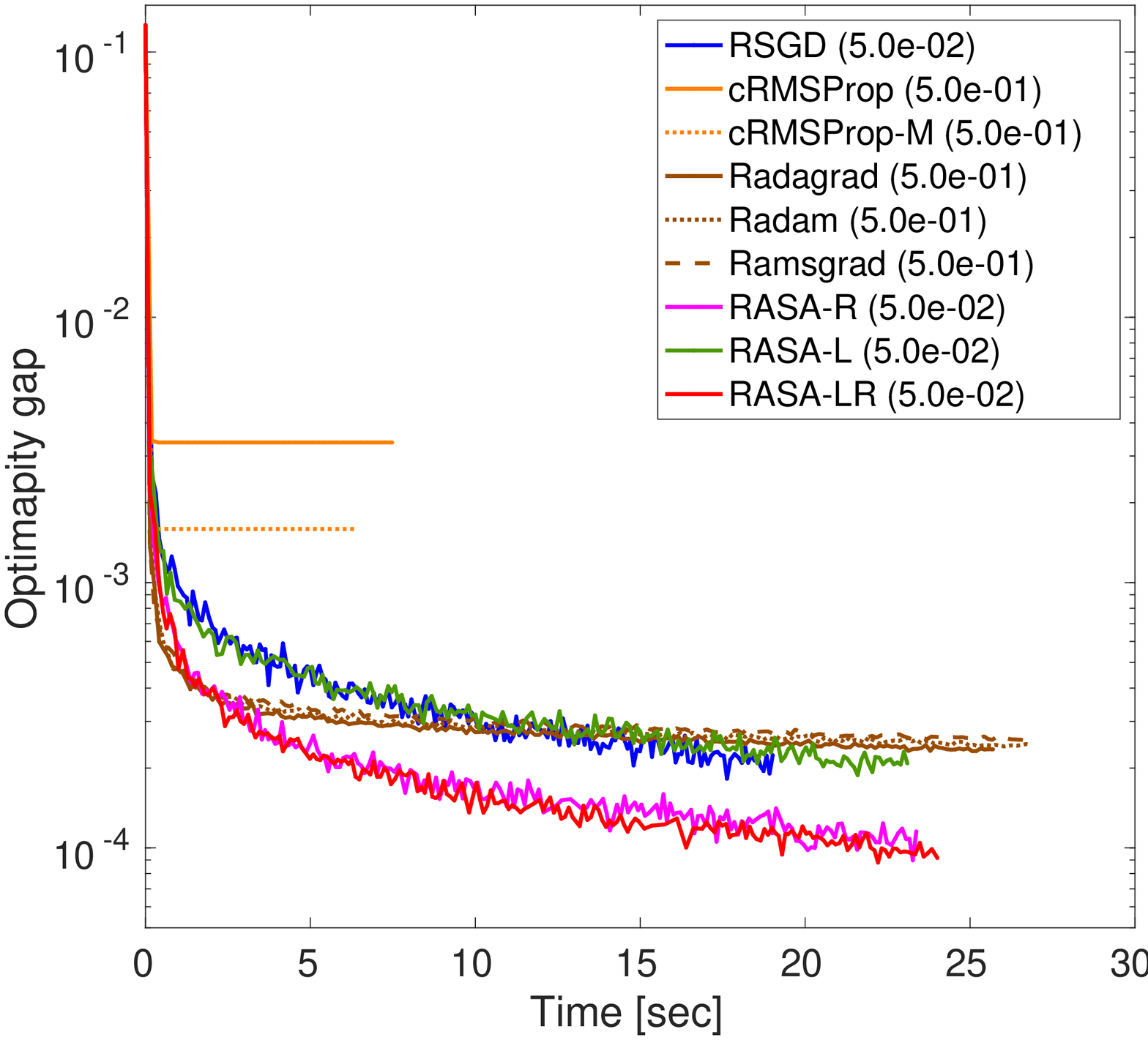}\\
		
		{\small  (b-2) Run-time.}
		
	\end{center} 
	\end{minipage}
	
	\vspace*{0.2cm}	
	
	{\small\bf  (b) ill-conditioned case.}
	
	\vspace*{0.2cm}

	\caption{Best-tuned results on synthetic datasets for the PCA problem.}

\label{appfig:PCA_results_Syn_Best_tuned}
\end{center}
\end{figure*}

\begin{figure*}[htbp]
\begin{center}
	\begin{minipage}[t]{.30\textwidth}
	\begin{center}
		\includegraphics[width=\textwidth]{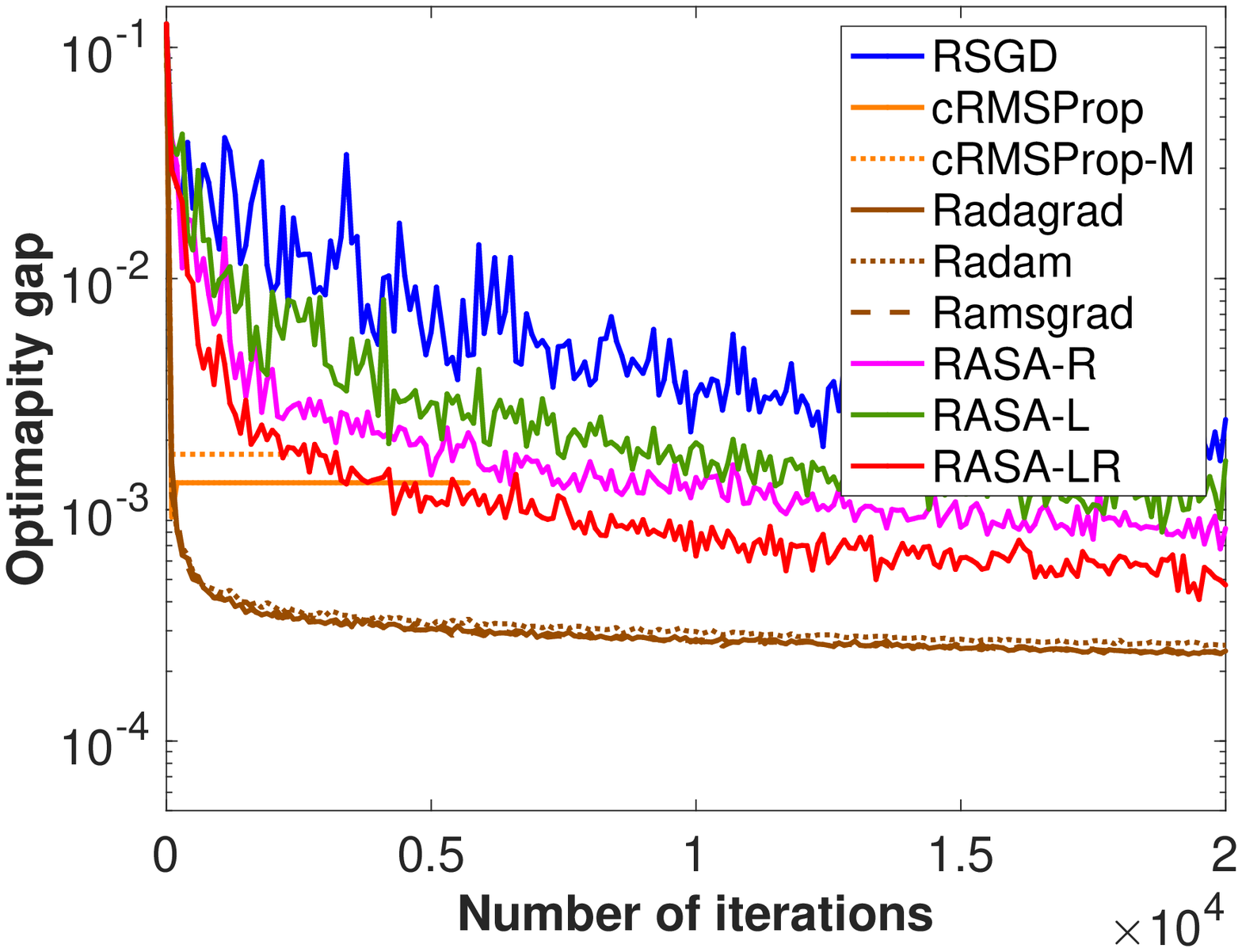}\\
		
		{\small  (a) $\alpha_0=0.5$.}
	\end{center}		
	\end{minipage}	
	\hspace*{-0.1cm}
	\begin{minipage}[t]{.30\textwidth}
	\begin{center}
		\includegraphics[width=\textwidth]{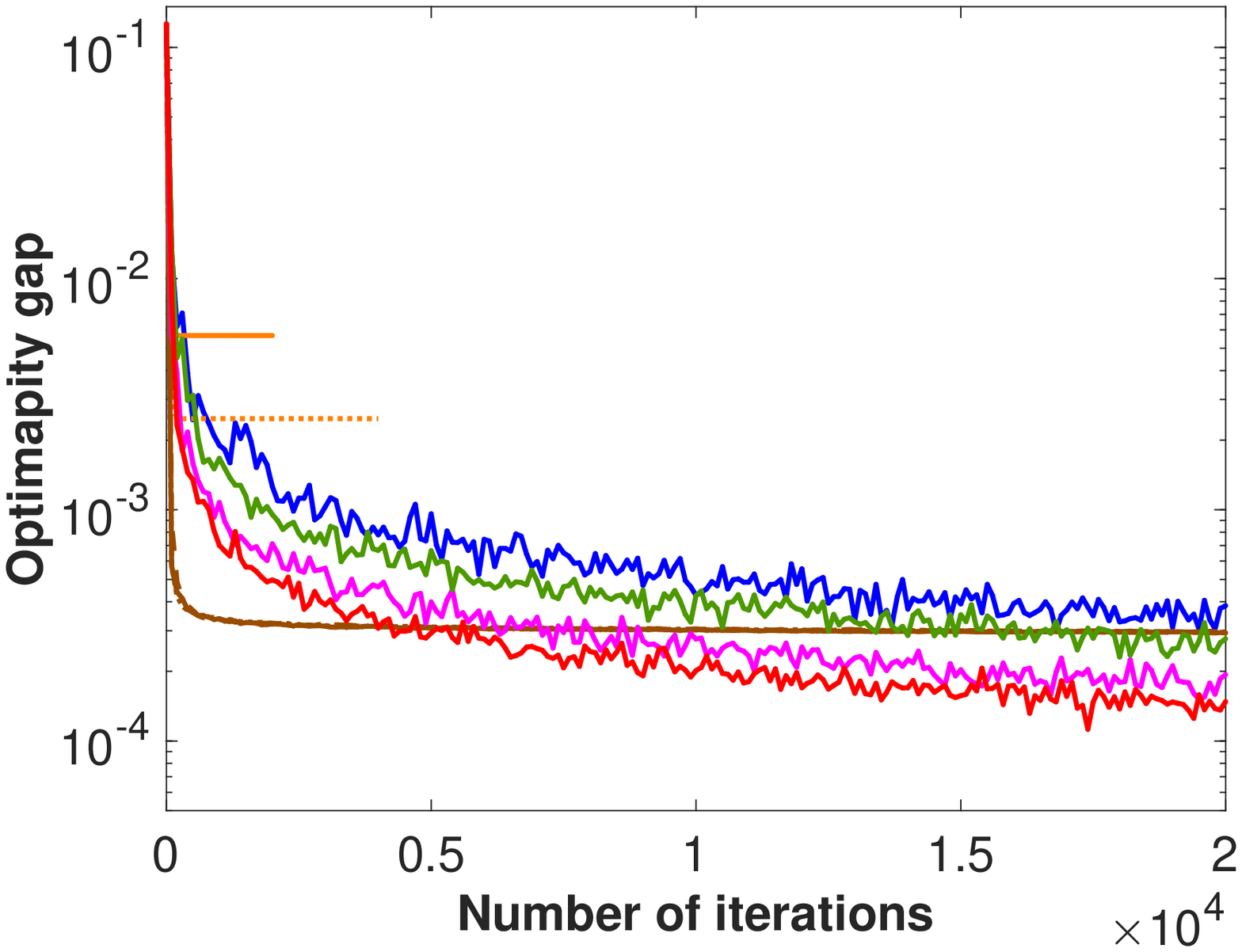}\\
		
		{\small  (b) $\alpha_0=0.1$.}
		
	\end{center} 
	\end{minipage}
	\hspace*{-0.1cm}
	\begin{minipage}[t]{.30\textwidth}
	\begin{center}
		\includegraphics[width=\textwidth]{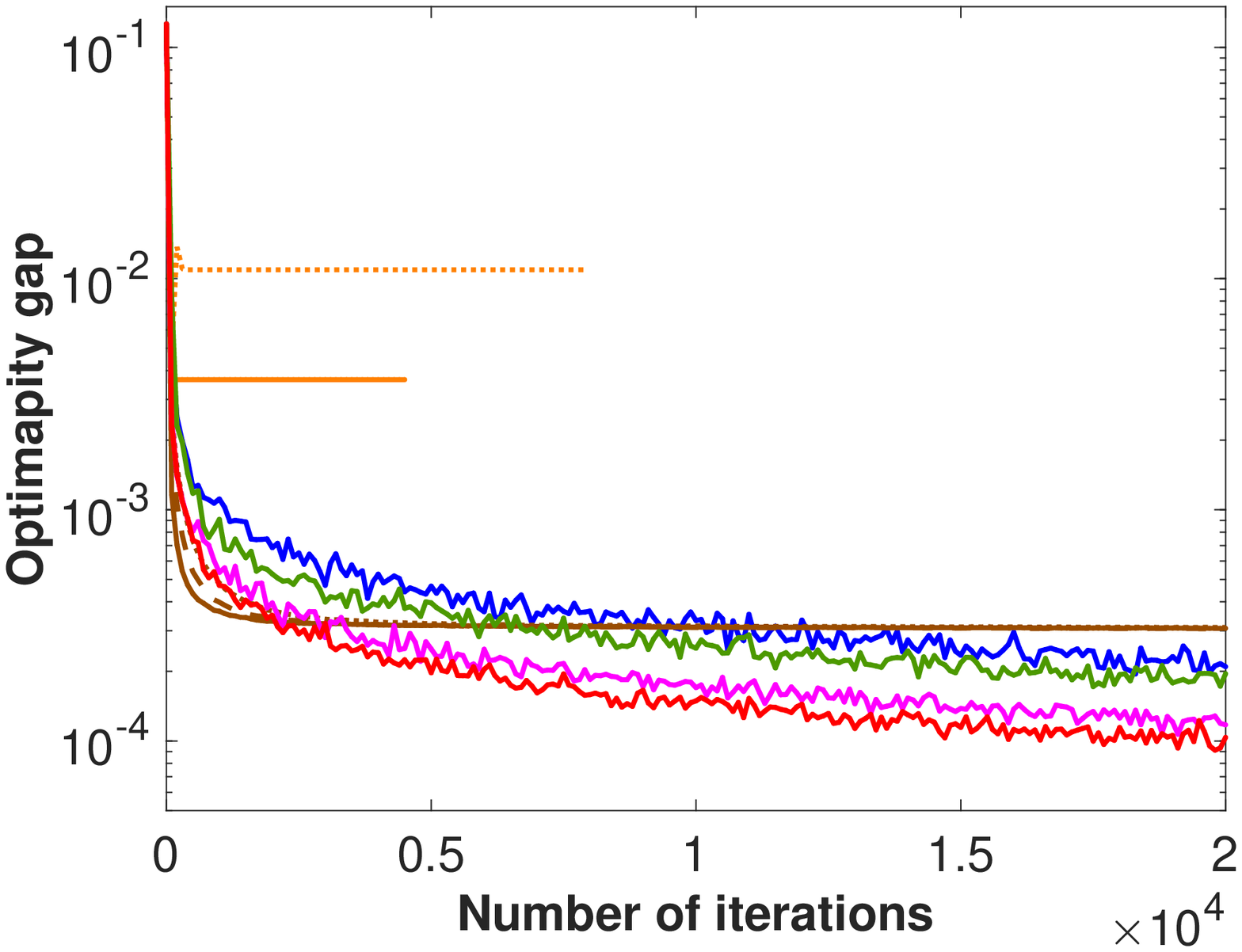}\\
		
		{\small  (c) $\alpha_0=0.05$.}
		
	\end{center} 
	\end{minipage}
	
	\vspace*{0.3cm}
	
	\begin{minipage}[t]{.30\textwidth}
	\begin{center}
		\includegraphics[width=\textwidth]{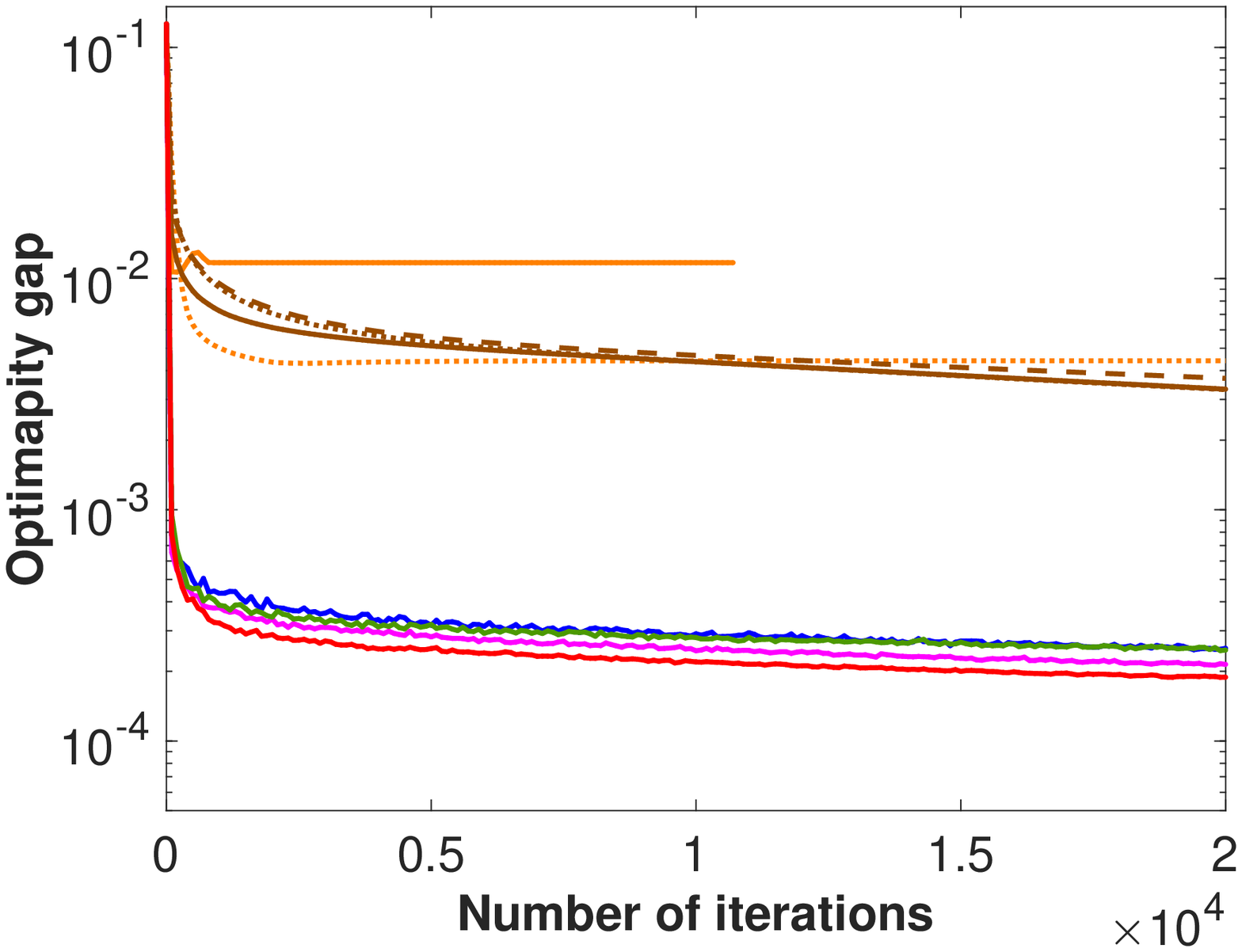}\\
		
		{\small  (d) $\alpha_0=0.01$.}
	\end{center}		
	\end{minipage}	
	\hspace*{-0.1cm}
	\begin{minipage}[t]{.30\textwidth}
	\begin{center}
		\includegraphics[width=\textwidth]{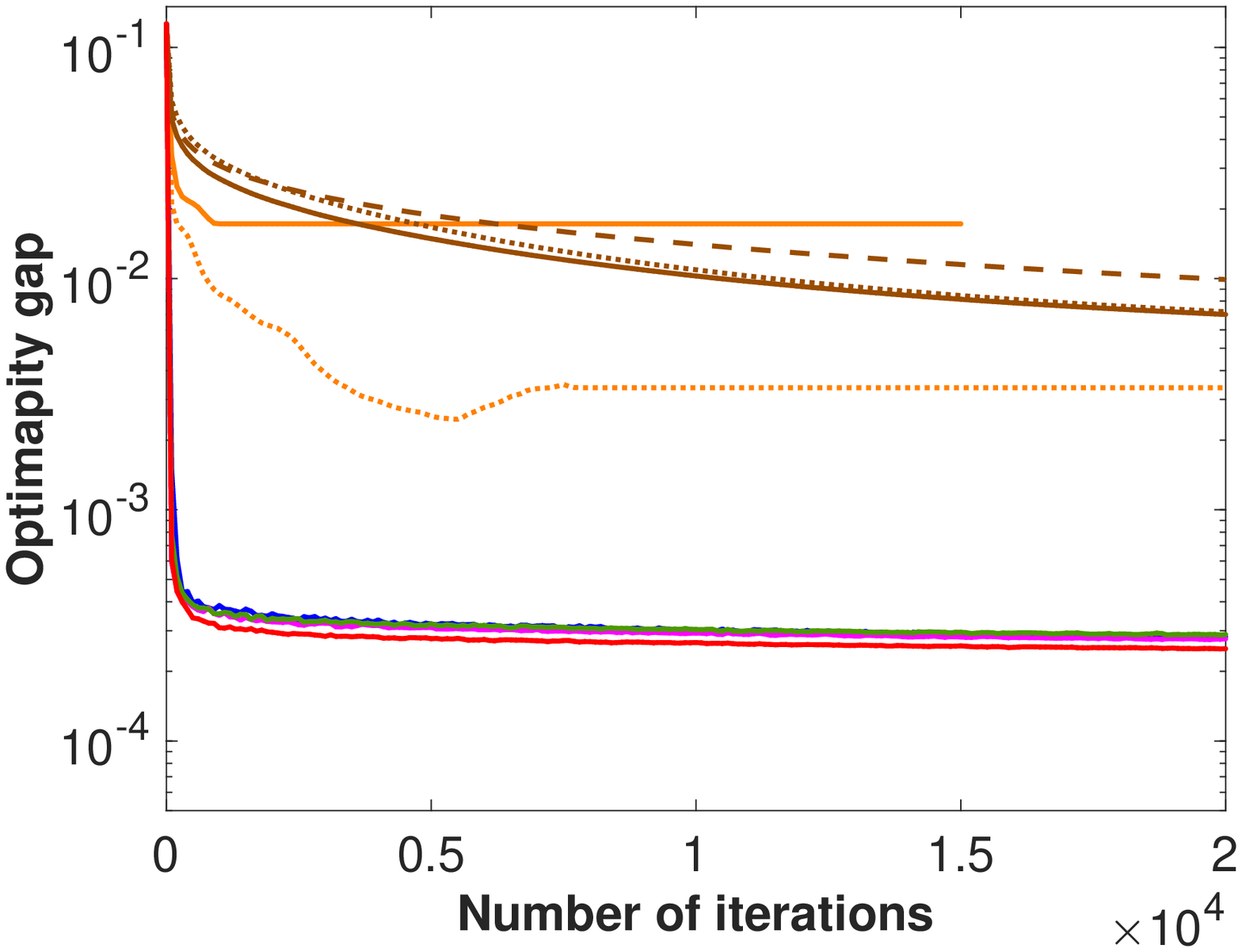}\\
		
		{\small  (e) $\alpha_0=0.005$.}
		
	\end{center} 
	\end{minipage}
	\hspace*{-0.1cm}
	\begin{minipage}[t]{.30\textwidth}
	\begin{center}
		\includegraphics[width=\textwidth]{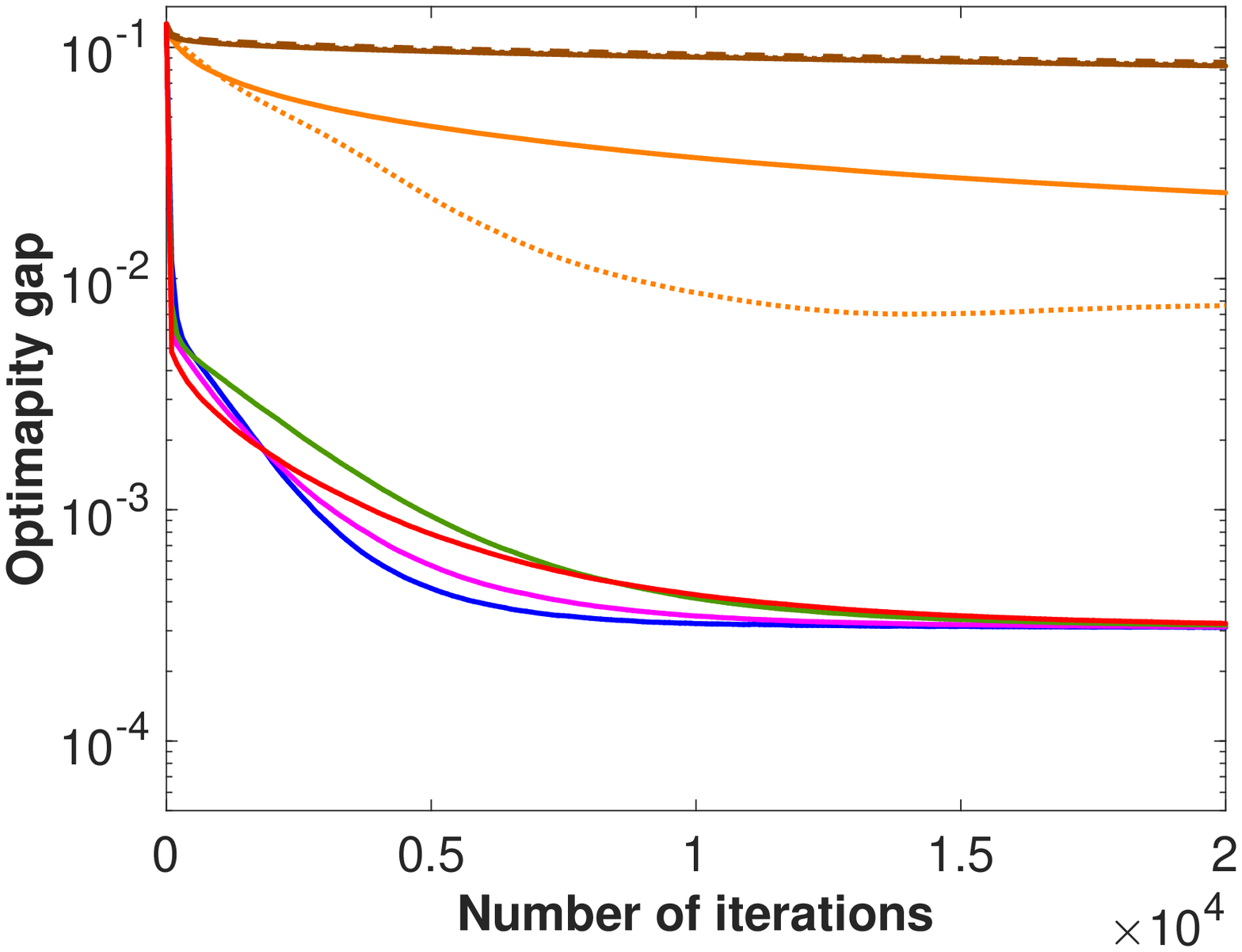}\\
		
		{\small  (f) $\alpha_0=0.001$.}
		
	\end{center} 
	\end{minipage}

	\caption{Synthetic datasets on the PCA problem (well-conditioned case ({\bf Case P1})).}

\label{appfig:PCA_results_Syn_well_cond}
\end{center}
\end{figure*}

\begin{figure*}[htbp]
\begin{center}
	\begin{minipage}[t]{.30\textwidth}
	\begin{center}
		\includegraphics[width=\textwidth]{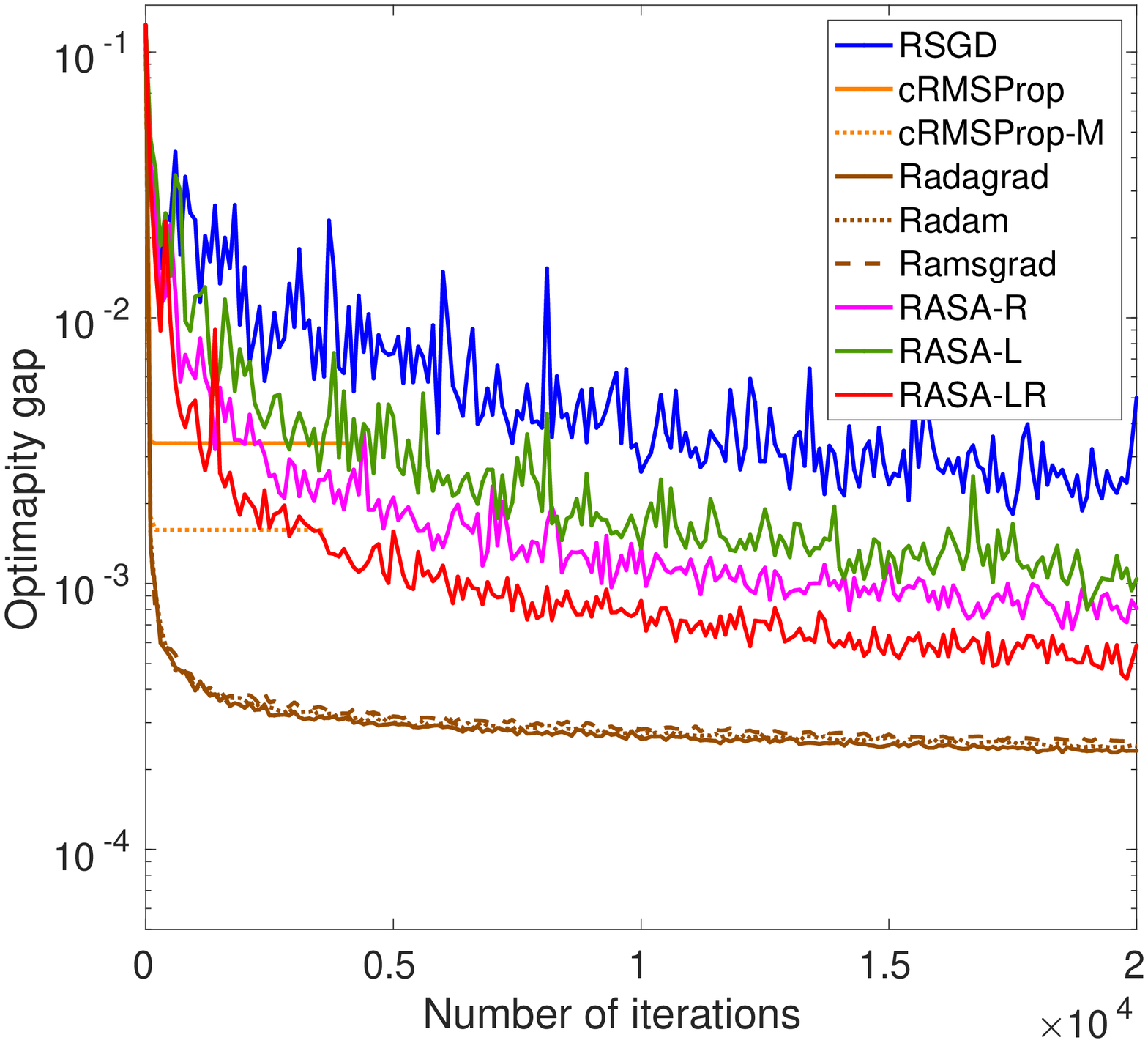}\\
		
		{\small  (a) $\alpha_0=0.5$.}
	\end{center}		
	\end{minipage}	
	\hspace*{-0.1cm}
	\begin{minipage}[t]{.30\textwidth}
	\begin{center}
		\includegraphics[width=\textwidth]{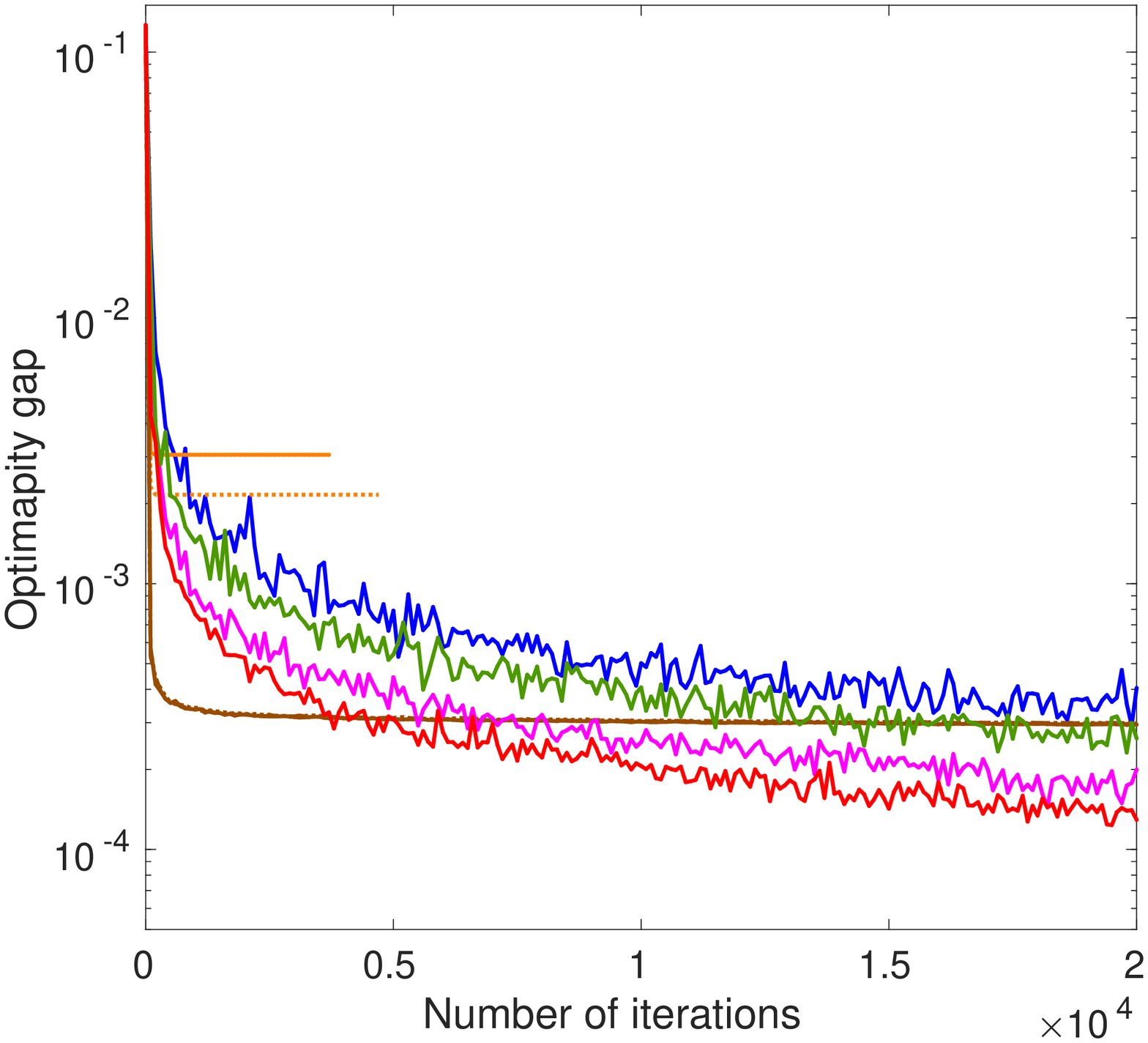}\\
		
		{\small  (b) $\alpha_0=0.1$.}
		
	\end{center} 
	\end{minipage}
	\hspace*{-0.1cm}
	\begin{minipage}[t]{.30\textwidth}
	\begin{center}
		\includegraphics[width=\textwidth]{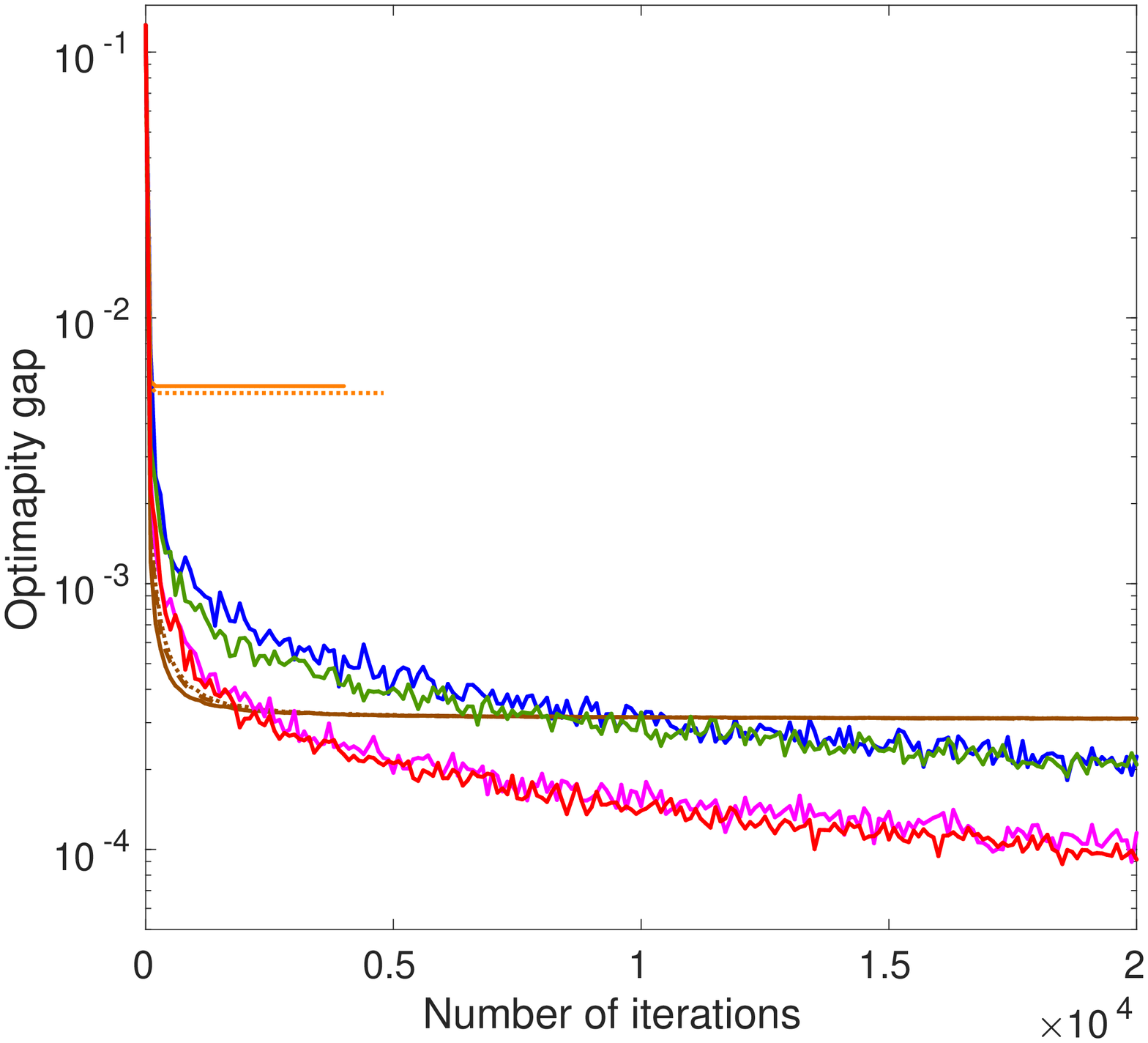}\\
		
		{\small  (c) $\alpha_0=0.05$.}
		
	\end{center} 
	\end{minipage}
	
	\vspace*{0.3cm}
		
	\begin{minipage}[t]{.30\textwidth}
	\begin{center}
		\includegraphics[width=\textwidth]{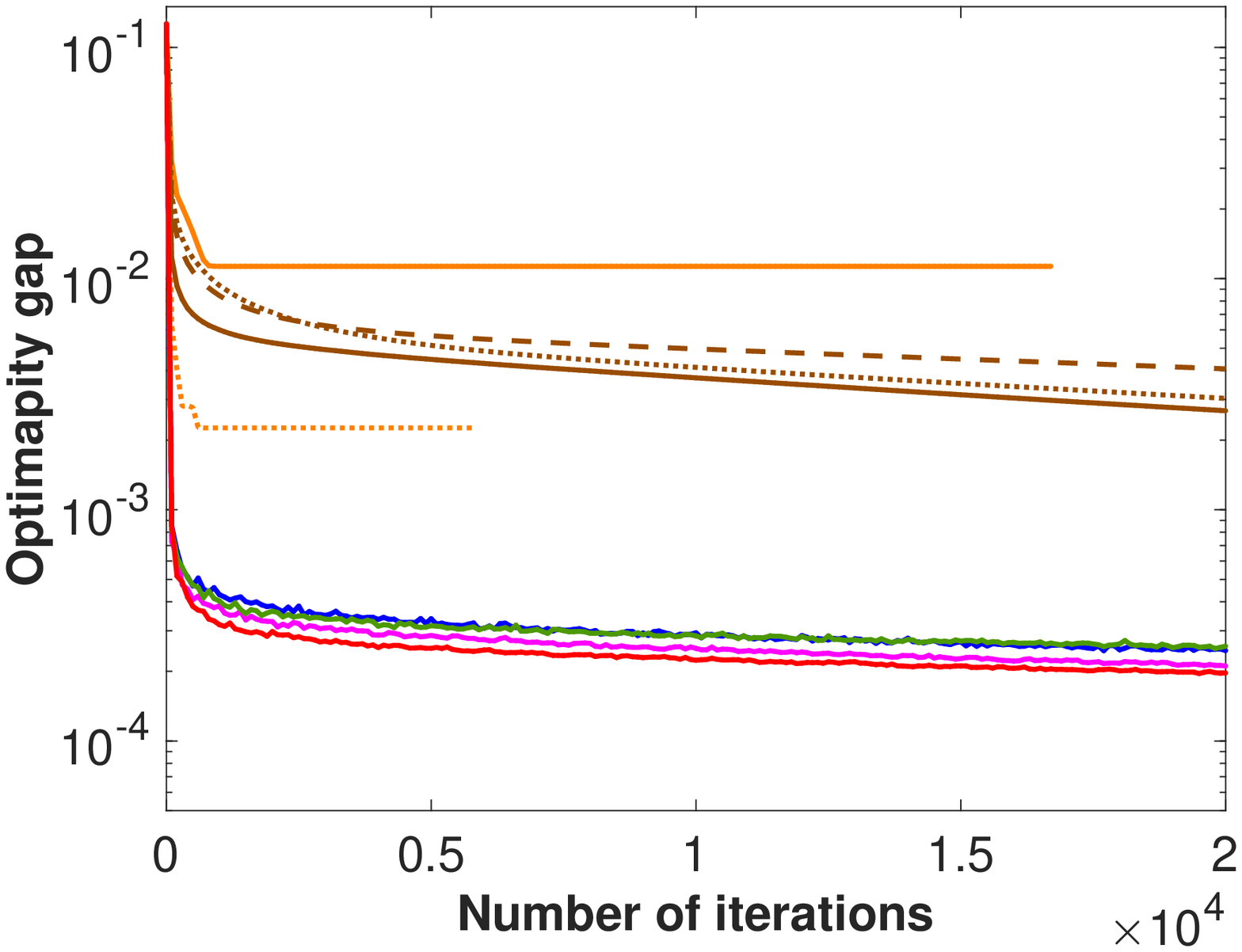}\\
		
		{\small  (d) $\alpha_0=0.01$.}
	\end{center}		
	\end{minipage}	
	\hspace*{-0.1cm}
	\begin{minipage}[t]{.30\textwidth}
	\begin{center}
		\includegraphics[width=\textwidth]{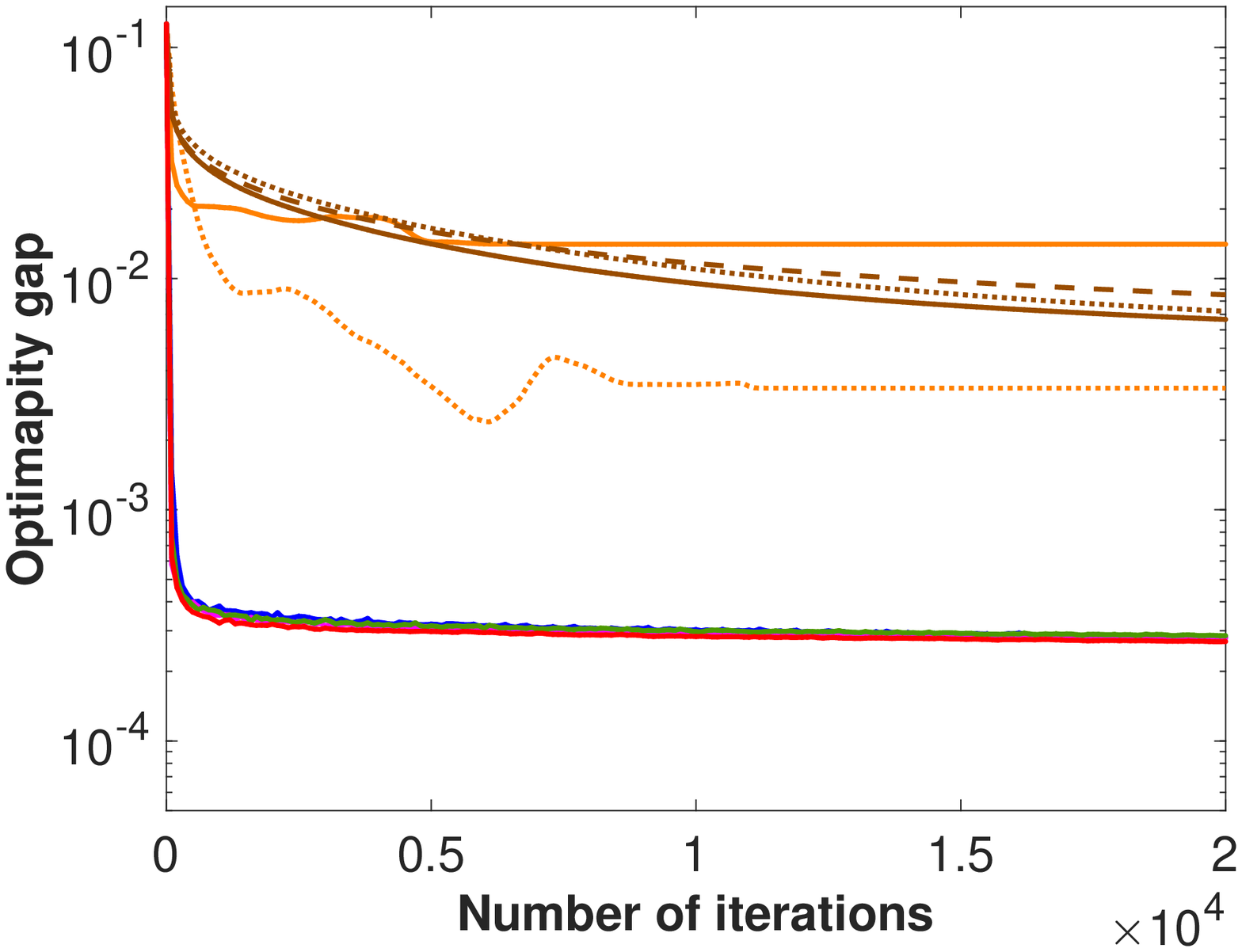}\\
		
		{\small  (e) $\alpha_0=0.005$.}
		
	\end{center} 
	\end{minipage}
	\hspace*{-0.1cm}
	\begin{minipage}[t]{.30\textwidth}
	\begin{center}
		\includegraphics[width=\textwidth]{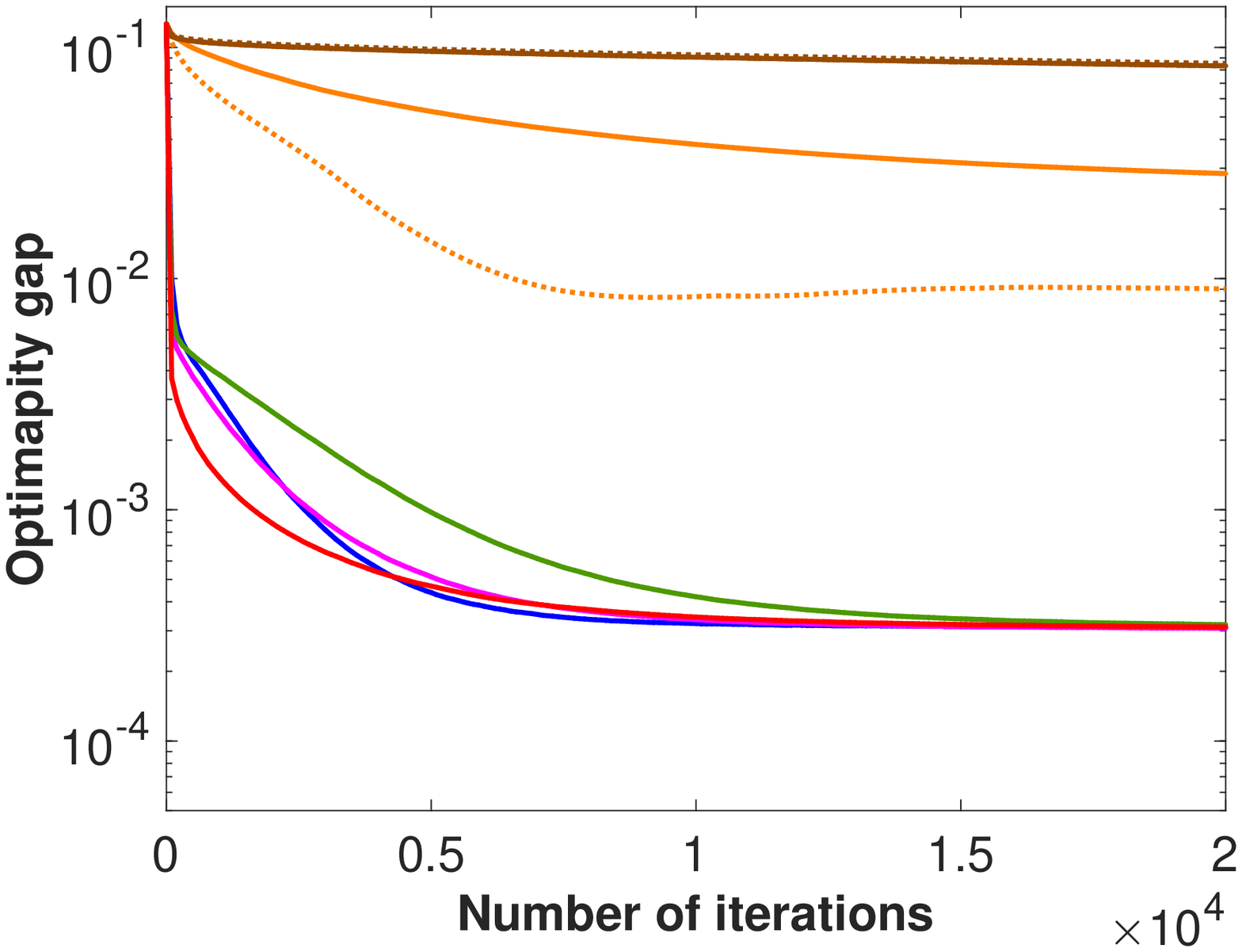}\\
		
		{\small  (f) $\alpha_0=0.001$.}
		
	\end{center} 
	\end{minipage}

	\caption{Synthetic datasets on the PCA problem (ill-conditioned case).}

\label{appfig:PCA_results_Syn_ill_cond}
\end{center}
\end{figure*}

\subsection{Results on the real-world datasets for the PCA problem ({\bf Case P2} and {\bf Case P3})}

\changeHK{This section shows additional results on the real-world datasets: the {\tt MNIST} dataset  ({\bf Case P2}) and the {\tt COIL100} dataset ({\bf Case P3}). We also include the {\tt COIL20} dataset. {\tt COIL20} \citep{COIL-20} contains normalized $1\,440$ camera images of the $20$ objects that were taken from different angles. We use $1\,440$ images that are resized into $32\times 32$ pixels. For {\tt COIL20}, $(N,n,r)=(1440,1024,20)$. 
Figures \ref{appfig:PCA_results_Real_Best_tuned} show the best-tuned results. 
It should be noted that Figures \ref{appfig:PCA_results_Real_Best_tuned} (a-1) and (b-1) are identical to Figures \ref{fig:PCA_results} (b) and (c), respectively.
As for the {\tt MNIST} dataset ({\bf Case P2}), we use $\alpha_0=\{5, \ldots, 0.01, 0.005\}$. As seen in the figures, Radagrad, Radam, and Ramsgrad provide better performance when $\alpha_0=5$. RASA-LR and RASA-R show their best performance when $\alpha_0=\{\changeHK{0.01,0.05}\}$, respectively. 
In the {\tt COIL100} dataset ({\bf Case P3}), we use $\alpha_0=\{10, \ldots, 0.00001\}$. As seen, the overall observations are similar to those in the previous datasets. We see that RASA-LR and RASA-R show their best performance among all the baseline algorithms when $\alpha_0=\{\changeHK{0.05,0.0005}\}$, respectively.
Finally, in the {\tt COIL20} dataset, we use the same range of $\alpha_0$ as the {\tt MNIST} dataset. The overall observations are almost similar to those in the previous two datasets.
In summary, as the same as the synthetic datasets, RASA-LR and RASA-R stably give better performances than all the other baseline algorithms. }

\changeHK{Additionally, we show all the results on each step size on these datasets in Figures \ref{appfig:PCA_results_MNIST}, \ref{appfig:PCA_results_COIL100} and \ref{appfig:PCA_results_COIL20}, respectively. 
} 
Lastly, in all the three datasets, we see poor performance of cRMSProp and cRMSProp-M across difference settings. 
\begin{figure*}[ht]
\begin{center}
	\begin{minipage}[t]{.30\textwidth}
	\begin{center}
		\includegraphics[width=\textwidth]{{results/pca/mnist/pca_comp-mnist-stiefel-10000-784-64-0-best-tuned}.eps}\\
		
		{\small  (a-1) Number of iterations.}
	\end{center}		
	\end{minipage}	
	\hspace*{0.5cm}
	\begin{minipage}[t]{.30\textwidth}
	\begin{center}
		\includegraphics[width=\textwidth]{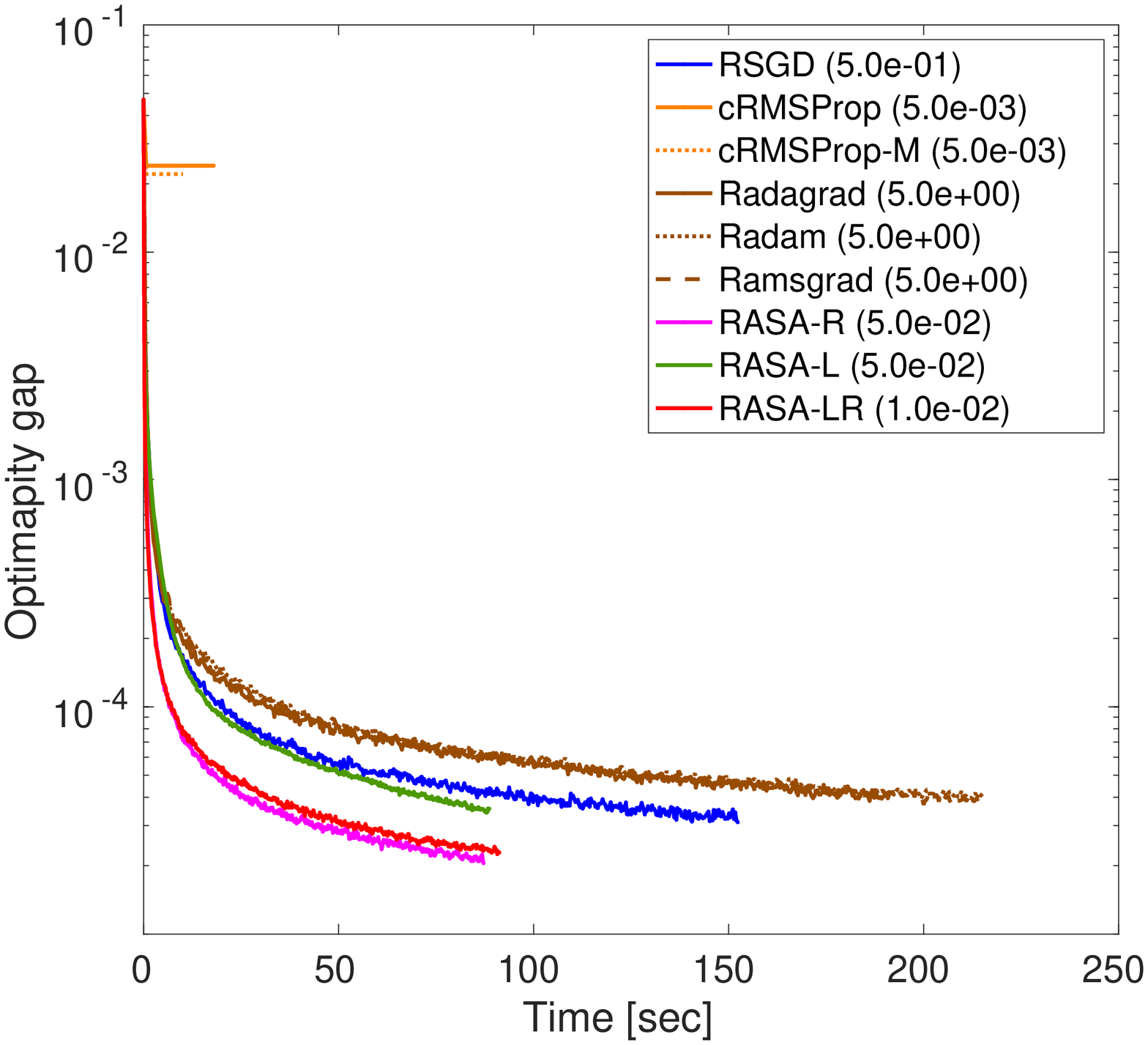}\\
		
		{\small  (a-2) Run-time.}
		
	\end{center} 
	\end{minipage}
	
	
	{\small\bf  (a) {\tt MNIST} dataset  ({\bf Case P2}).}
	
	\vspace*{0.4cm}
	
	\begin{minipage}[t]{.30\textwidth}
	\begin{center}
		\includegraphics[width=\textwidth]{{results/pca/coil100/pca_comp-coil100-stiefel-7200-1024-100-0-best-tuned}.eps}\\
		
		{\small  (b-1) Number of iterations.}
	\end{center}		
	\end{minipage}	
	\hspace*{0.5cm}
	\begin{minipage}[t]{.30\textwidth}
	\begin{center}
		\includegraphics[width=\textwidth]{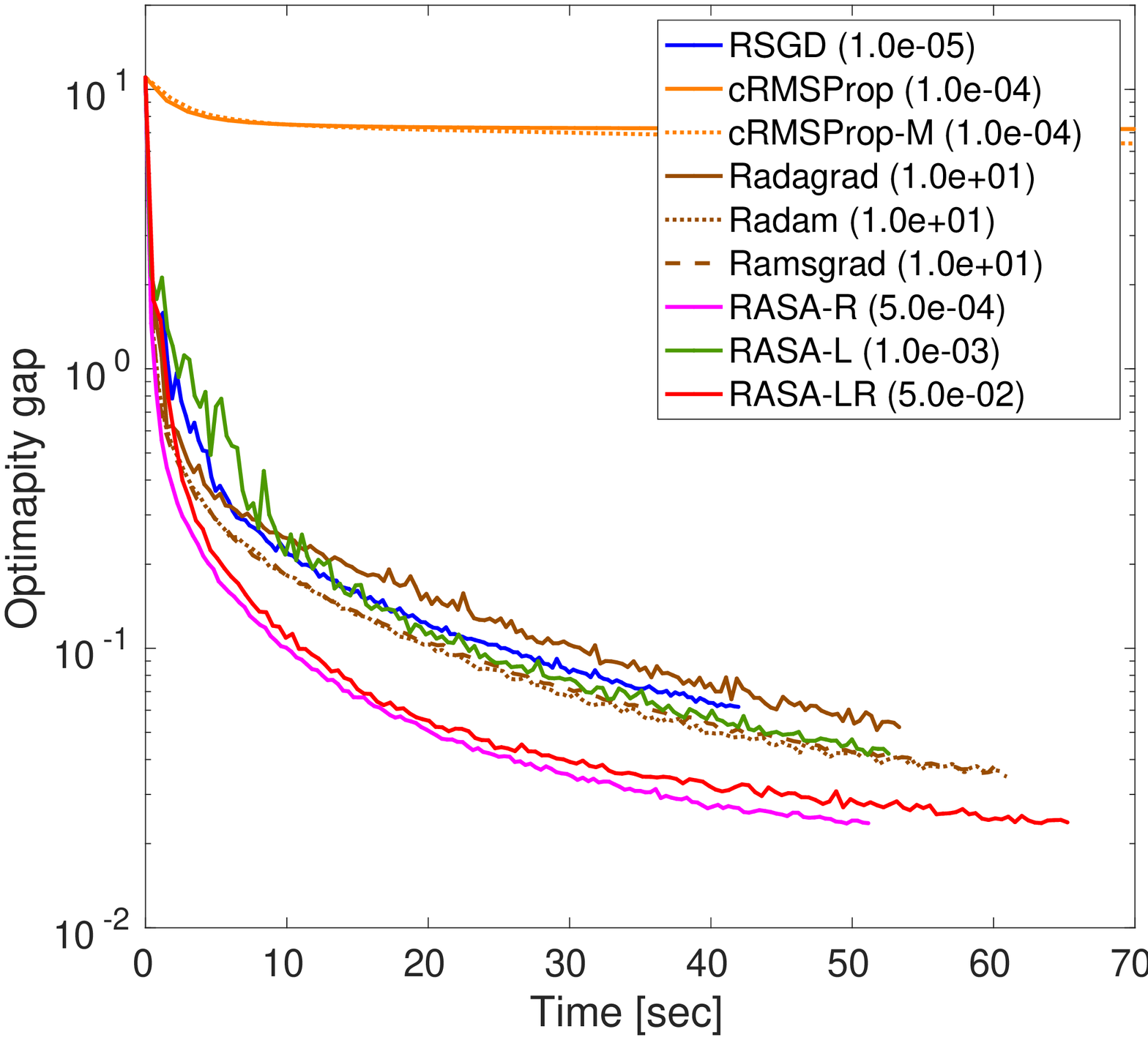}\\
		
		{\small  (b-2) Run-time.}
		
	\end{center} 
	\end{minipage}
	
	\vspace*{0.4cm}
		
	{\small\bf  (b) {\tt COIL100} dataset  ({\bf Case P3}).}	
	
	\vspace*{0.4cm}
	
	\begin{minipage}[t]{.30\textwidth}
	\begin{center}
		\includegraphics[width=\textwidth]{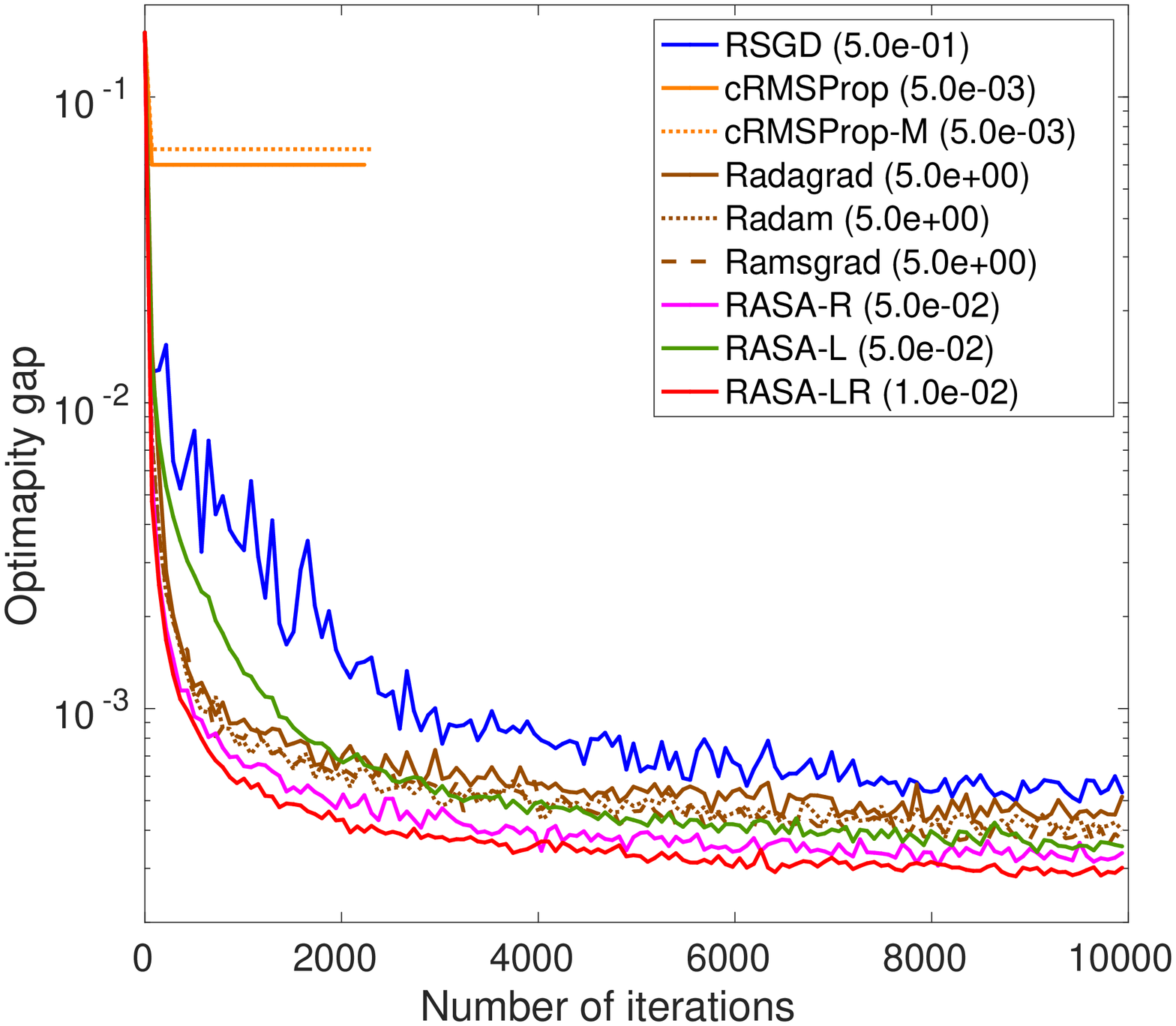}\\
		
		{\small  (c-1) Number of iterations.}
	\end{center}		
	\end{minipage}	
	\hspace*{0.5cm}
	\begin{minipage}[t]{.30\textwidth}
	\begin{center}
		\includegraphics[width=\textwidth]{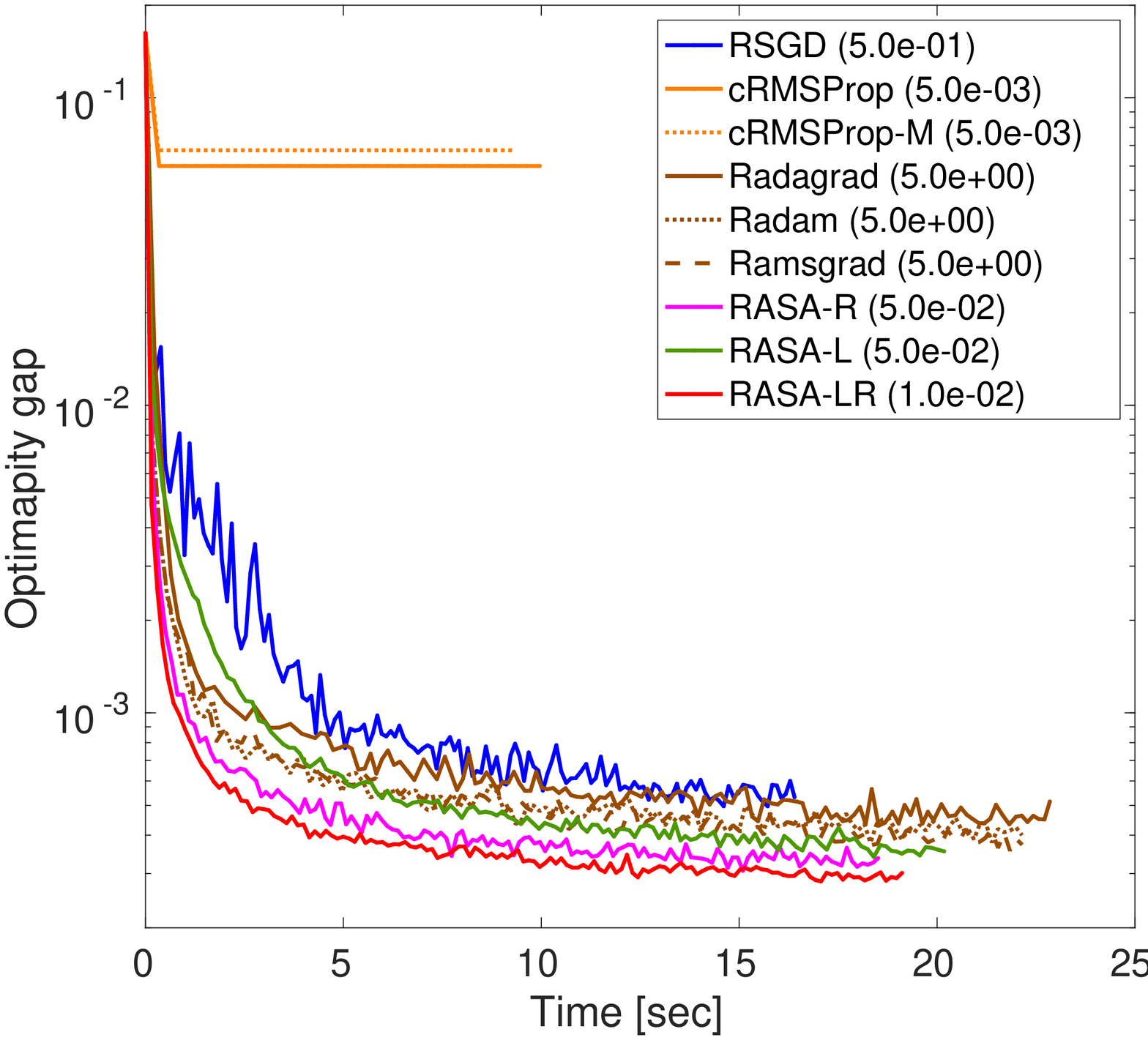}\\
		
		{\small  (c-2) Run-time.}
				
	\end{center} 
	\end{minipage}
	
	\vspace*{0.2cm}	
	
	{\small\bf  (c) {\tt COIL20} dataset.}

	\vspace*{0.2cm}

	\caption{\changeHK{Best-tuned results of real-world datasets on the PCA problem.}}

\label{appfig:PCA_results_Real_Best_tuned}
\end{center}
\end{figure*}

\begin{figure*}[t]
\begin{center}
	\begin{minipage}[t]{.2\textwidth}
	\begin{center}
		\includegraphics[width=\textwidth]{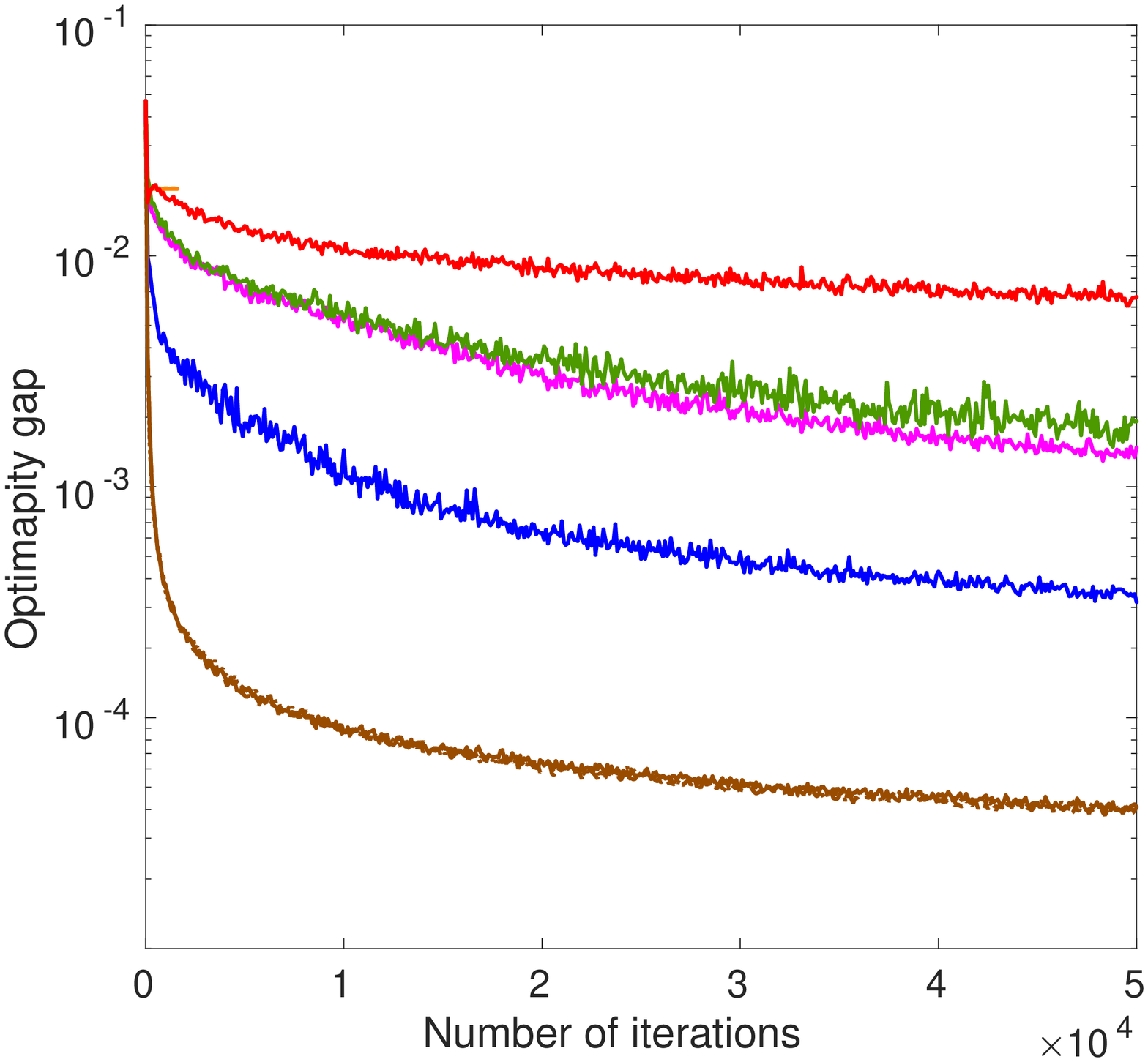}\\
		
		{\small  (a) $\alpha_0=5$.}
		
	\end{center} 
	\end{minipage}
	\hspace*{-0.1cm}
	\begin{minipage}[t]{.2\textwidth}
	\begin{center}
		\includegraphics[width=\textwidth]{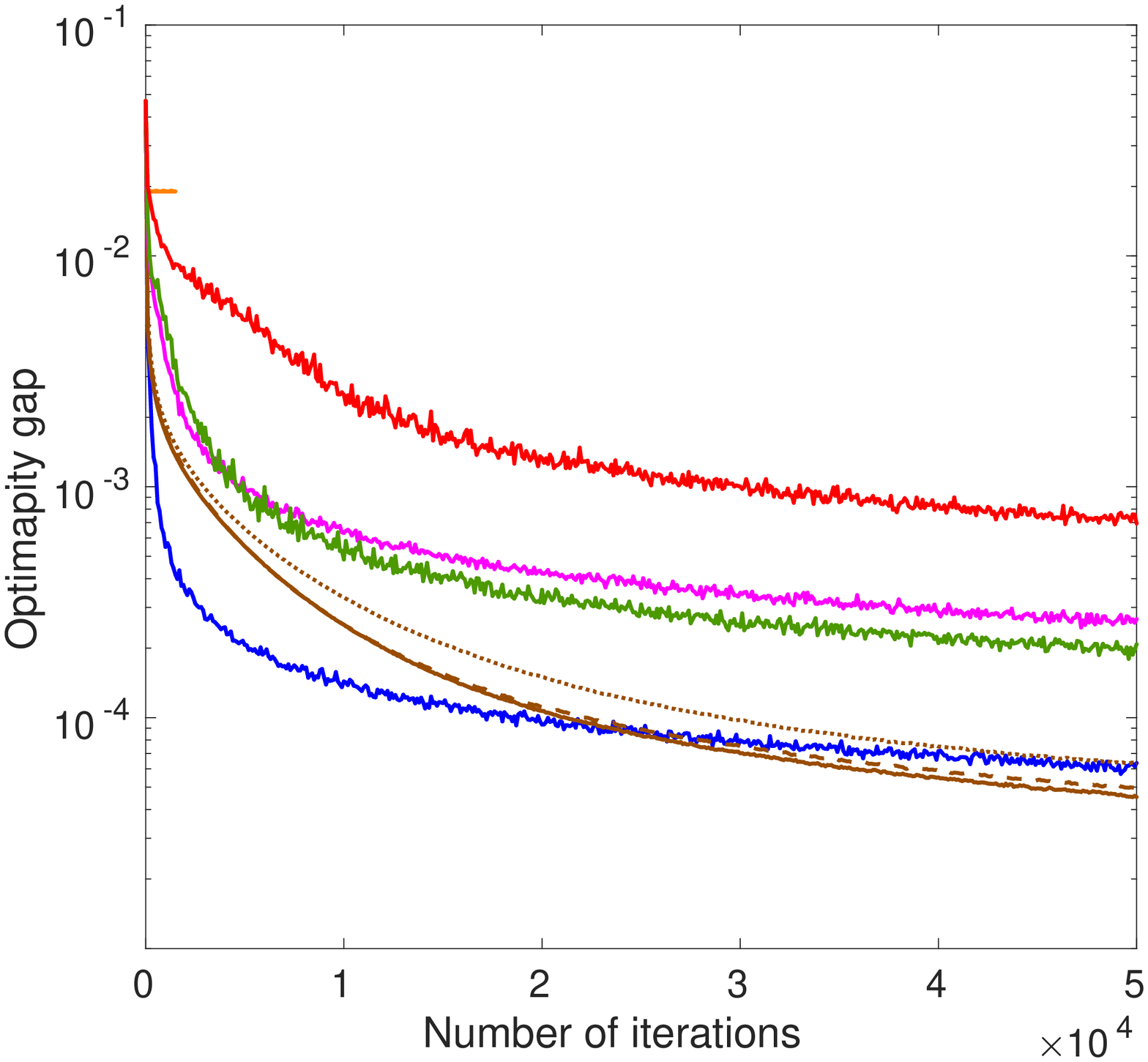}\\
		
		{\small  (b) $\alpha_0=1$.}
		
	\end{center} 
	\end{minipage}
	\hspace*{-0.1cm}
	\begin{minipage}[t]{.2\textwidth}
	\begin{center}
		\includegraphics[width=\textwidth]{{results/pca/mnist/pca_comp-mnist-stiefel-10000-784-64-0-0.5}.eps}\\
		
		{\small  (c) $\alpha_0=0.5$.}
	\end{center}		
	\end{minipage}	
	
	\vspace*{0.2cm}
	
	\begin{minipage}[t]{.2\textwidth}
	\begin{center}
		\includegraphics[width=\textwidth]{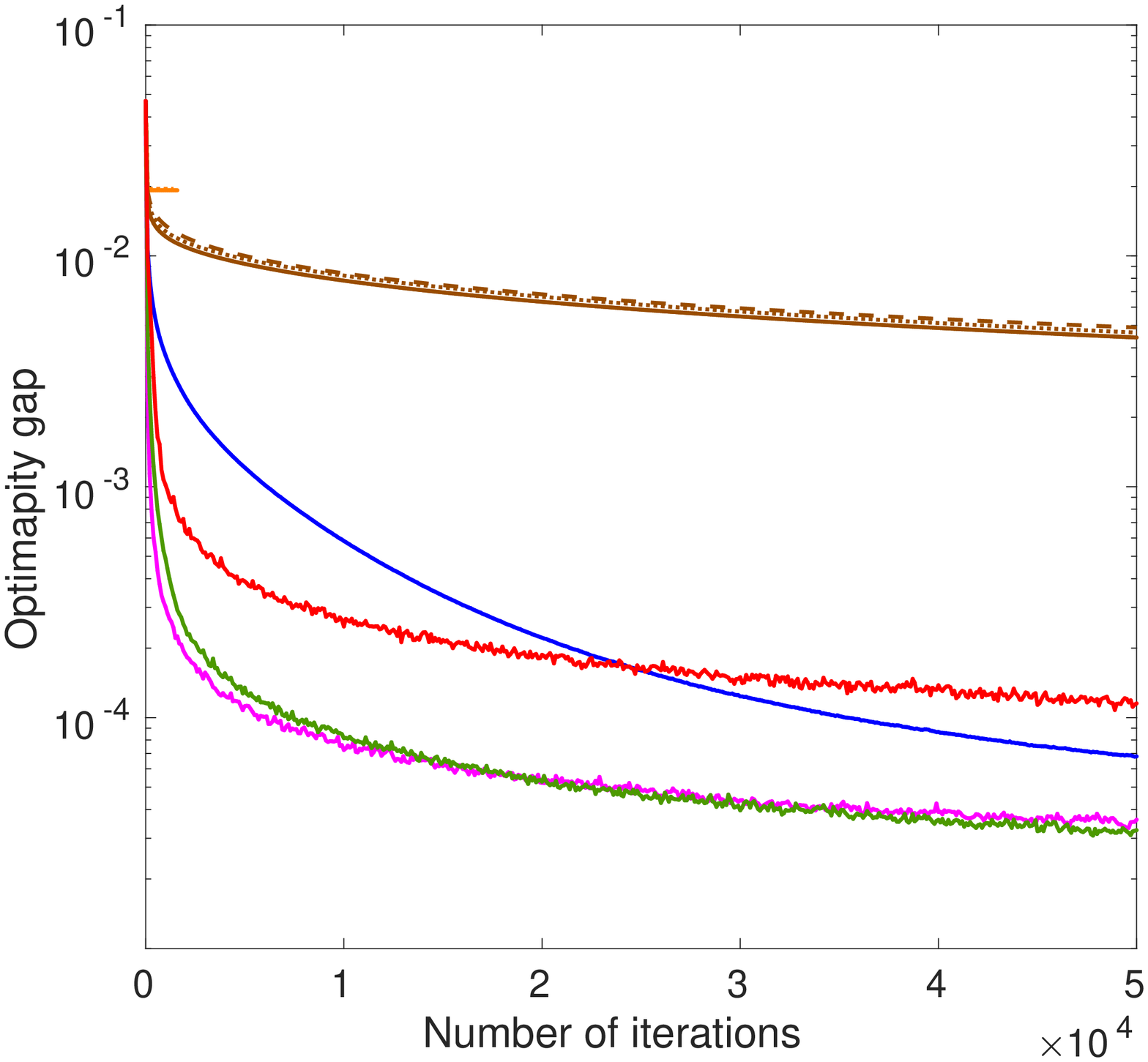}\\
		
		{\small  (d) $\alpha_0=0.1$.}
		
	\end{center} 
	\end{minipage}
	\hspace*{-0.1cm}
	\begin{minipage}[t]{.2\textwidth}
	\begin{center}
		\includegraphics[width=\textwidth]{{results/pca/mnist/pca_comp-mnist-stiefel-10000-784-64-0-0.05}.eps}\\
		
		{\small  (e) $\alpha_0=0.05$.}
		
	\end{center} 
	\end{minipage}
	\hspace*{-0.1cm}
	\begin{minipage}[t]{.2\textwidth}
	\begin{center}
		\includegraphics[width=\textwidth]{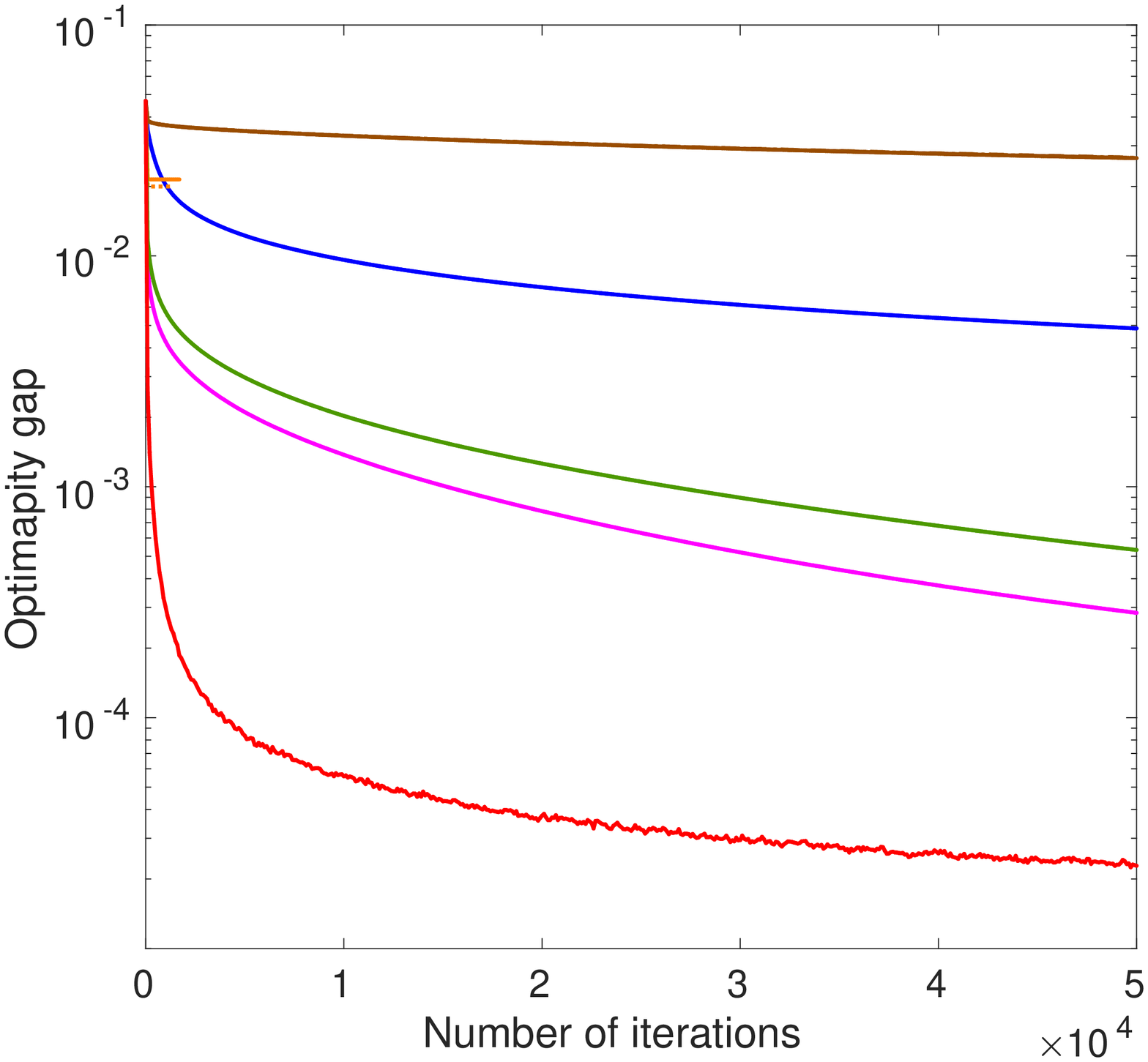}\\
		
		{\small  (f) $\alpha_0=0.01$.}
	\end{center}		
	\end{minipage}	
	\hspace*{-0.1cm}
	\begin{minipage}[t]{.2\textwidth}
	\begin{center}
		\includegraphics[width=\textwidth]{{results/pca/mnist/pca_comp-mnist-stiefel-10000-784-64-0-0.005}.eps}\\
		
		{\small  (g) $\alpha_0=0.005$.}
		
	\end{center} 
	\end{minipage}
%
%

	\vspace{0.2cm}		

	\caption{{\tt MNIST} dataset on the PCA problem ({\bf Case P2}).}

\label{appfig:PCA_results_MNIST}
\end{center}
\end{figure*}

\begin{figure*}[t]
\begin{center}

	\begin{minipage}[t]{.2\textwidth}
	\begin{center}
		\includegraphics[width=\textwidth]{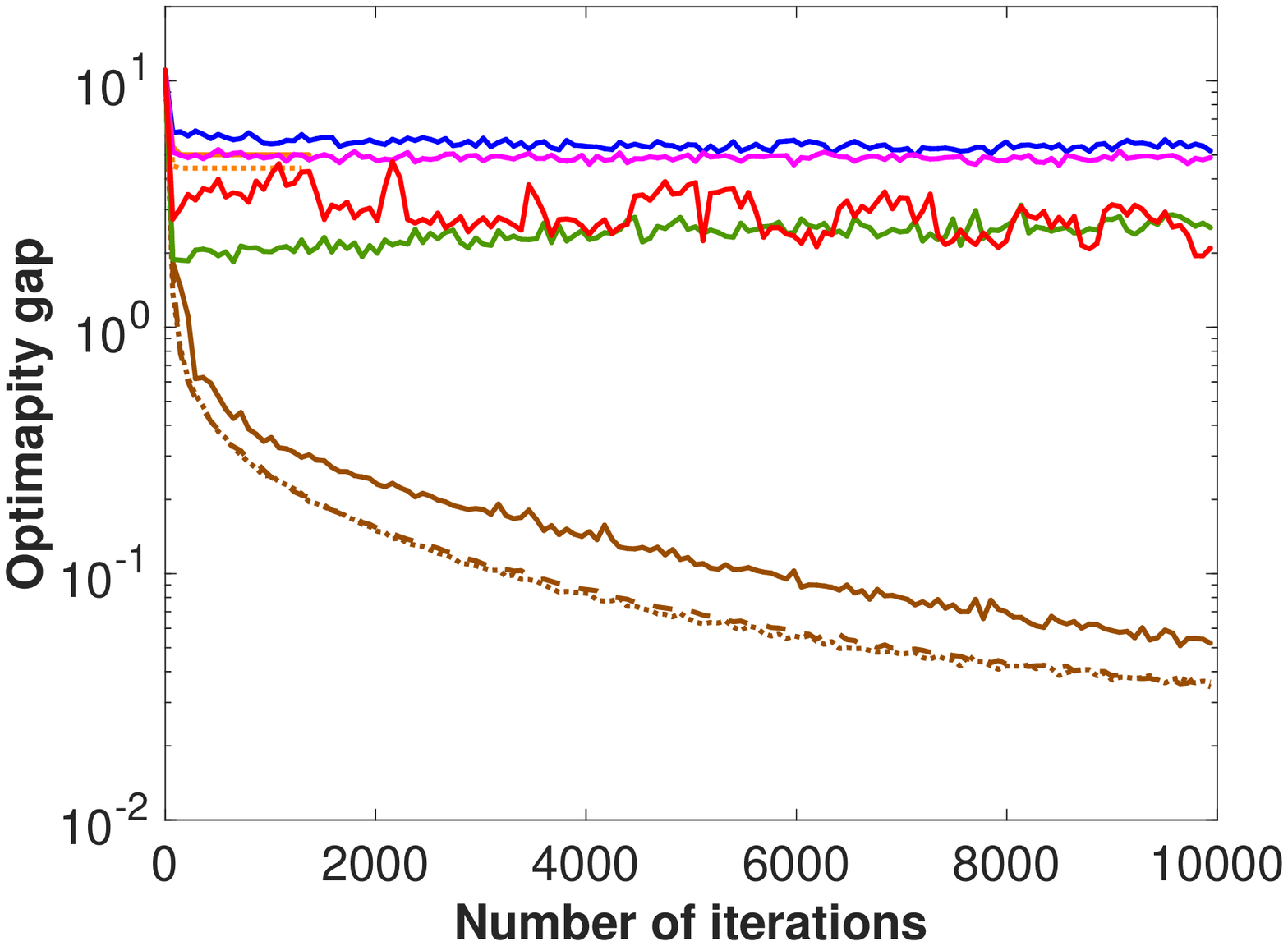}\\
		
		{\small  (a) $\alpha_0=10$.}
		
	\end{center} 
	\end{minipage}
	\hspace*{-0.1cm}
	\begin{minipage}[t]{.2\textwidth}
	\begin{center}
		\includegraphics[width=\textwidth]{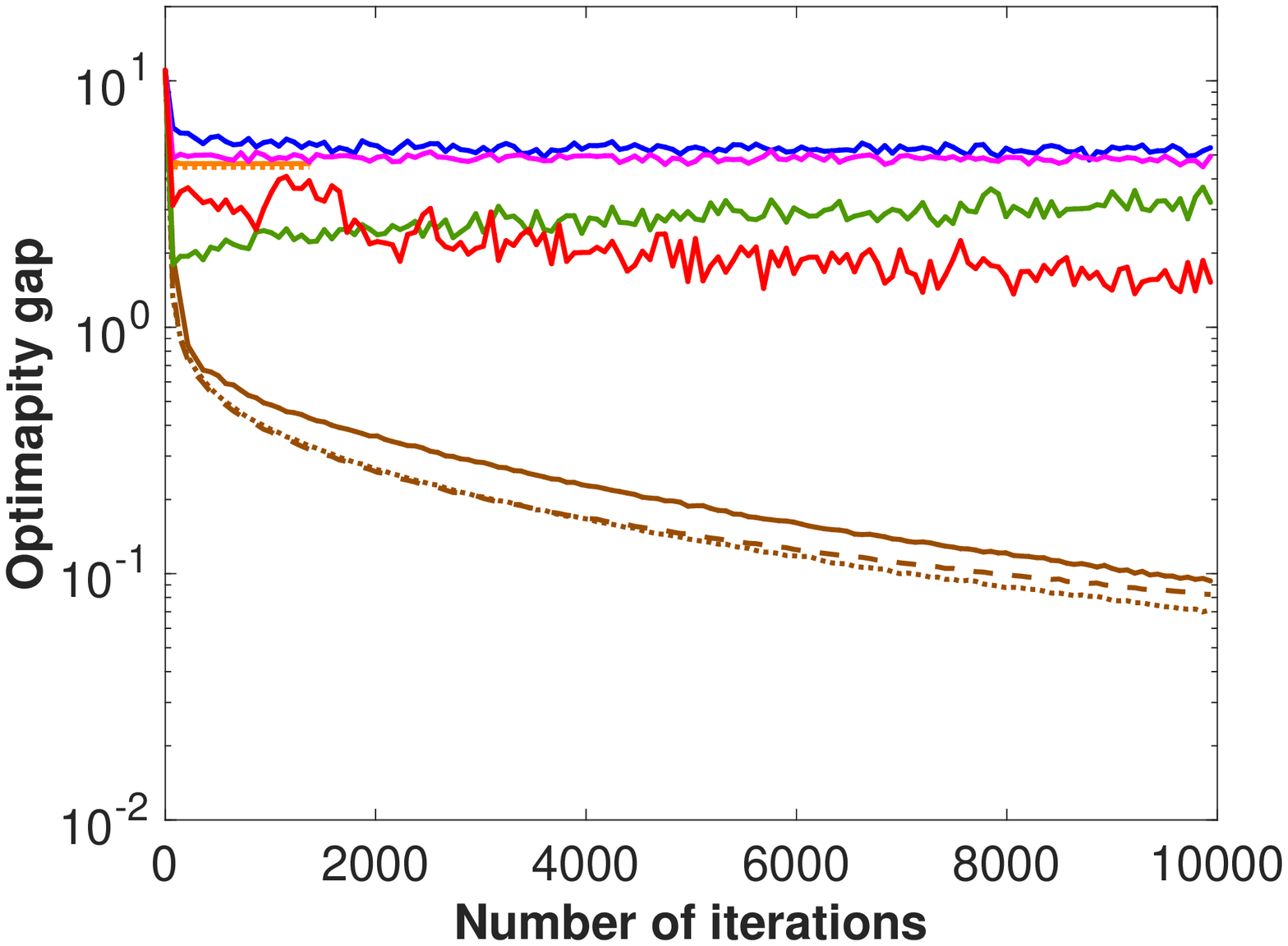}\\
		
		{\small  (b) $\alpha_0=5$.}
		
	\end{center} 
	\end{minipage}
	\hspace*{-0.1cm}
	\begin{minipage}[t]{.2\textwidth}
	\begin{center}
		\includegraphics[width=\textwidth]{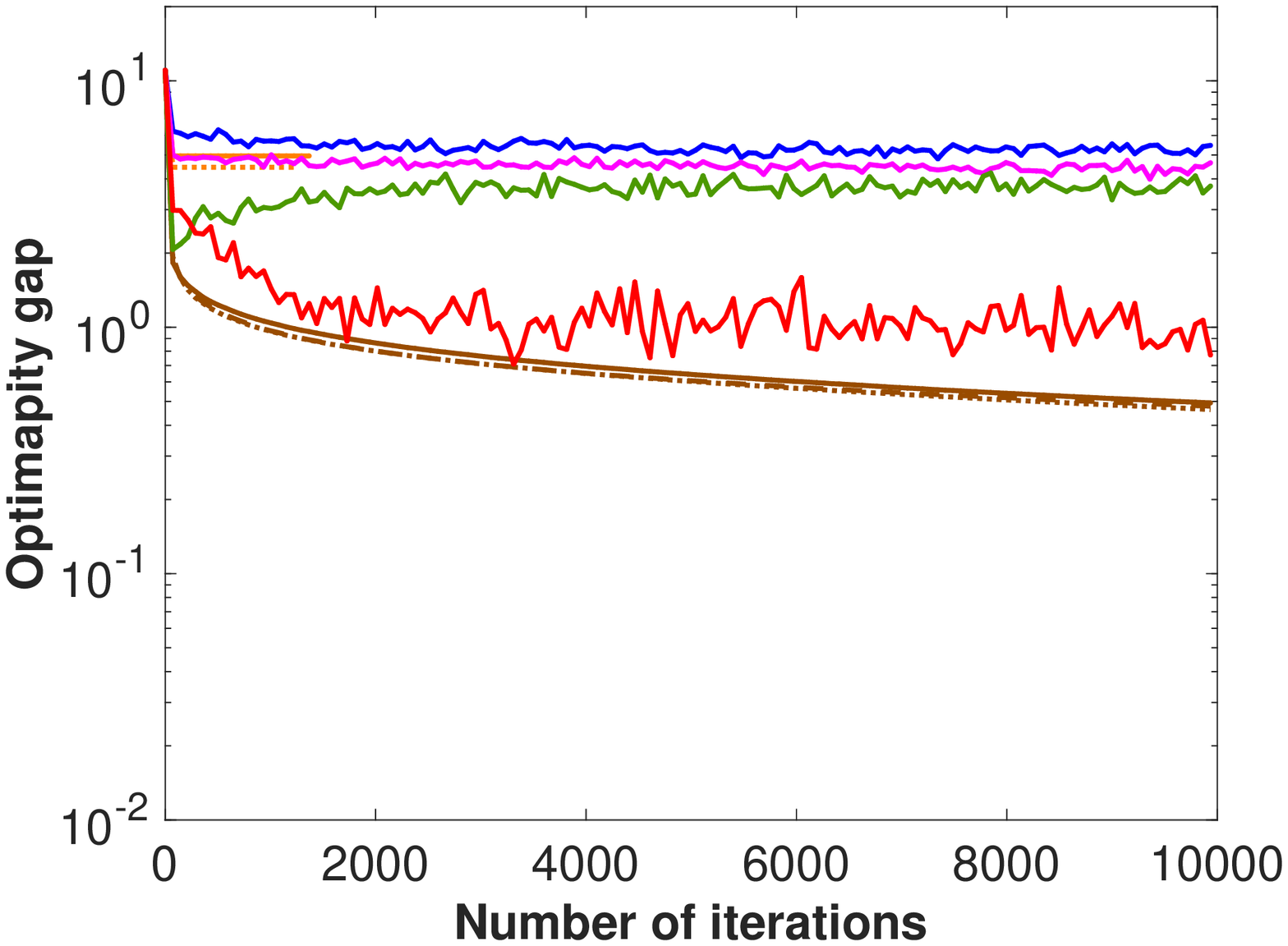}\\
		
		{\small  (c) $\alpha_0=1$.}
		
	\end{center} 
	\end{minipage}

	\vspace*{0.5cm}

	\begin{minipage}[t]{.2\textwidth}
	\begin{center}
		\includegraphics[width=\textwidth]{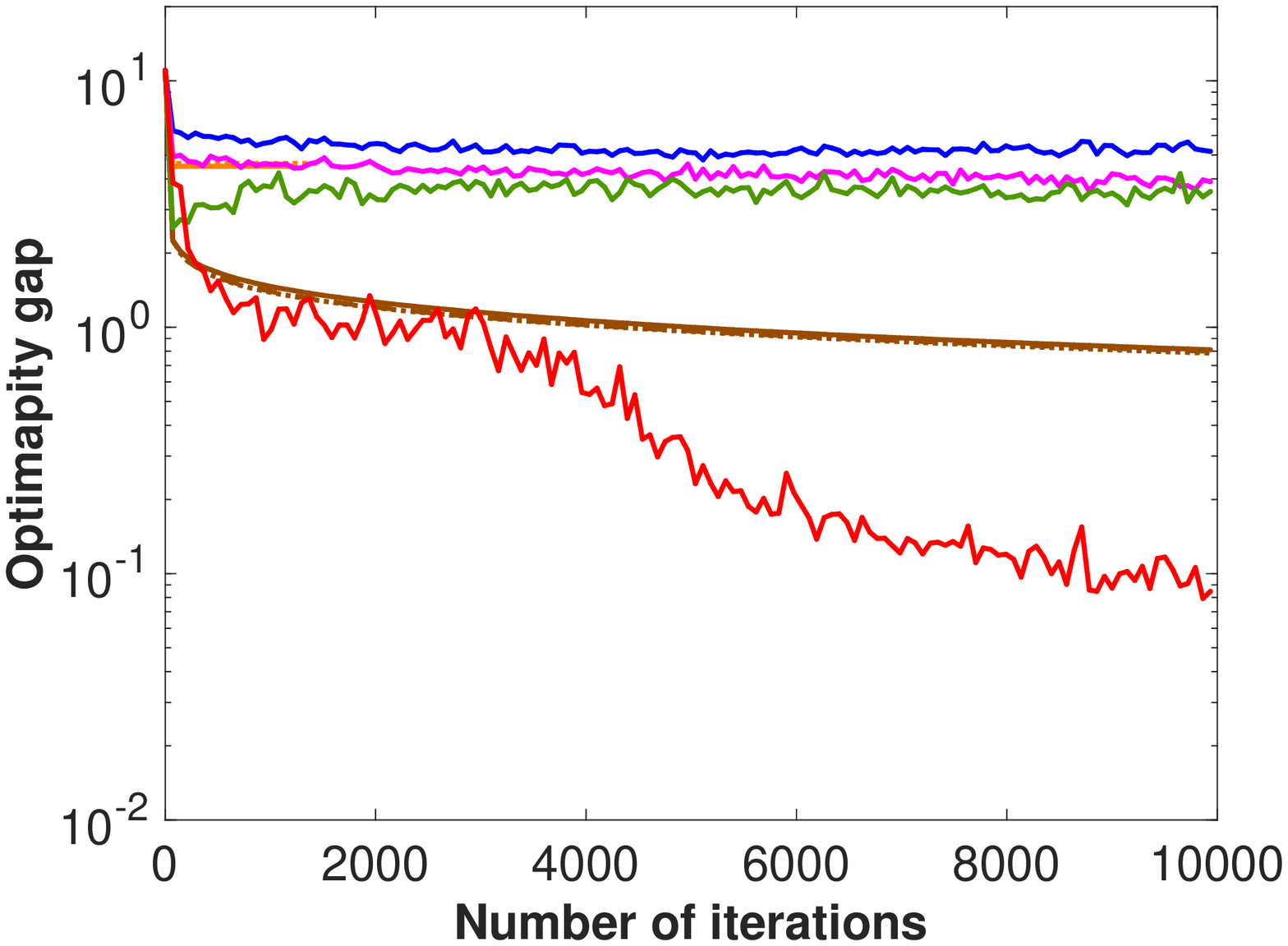}\\
		
		{\small  (d) $\alpha_0=0.5$.}
	\end{center}		
	\end{minipage}	
	\hspace*{-0.1cm}
	\begin{minipage}[t]{.2\textwidth}
	\begin{center}
		\includegraphics[width=\textwidth]{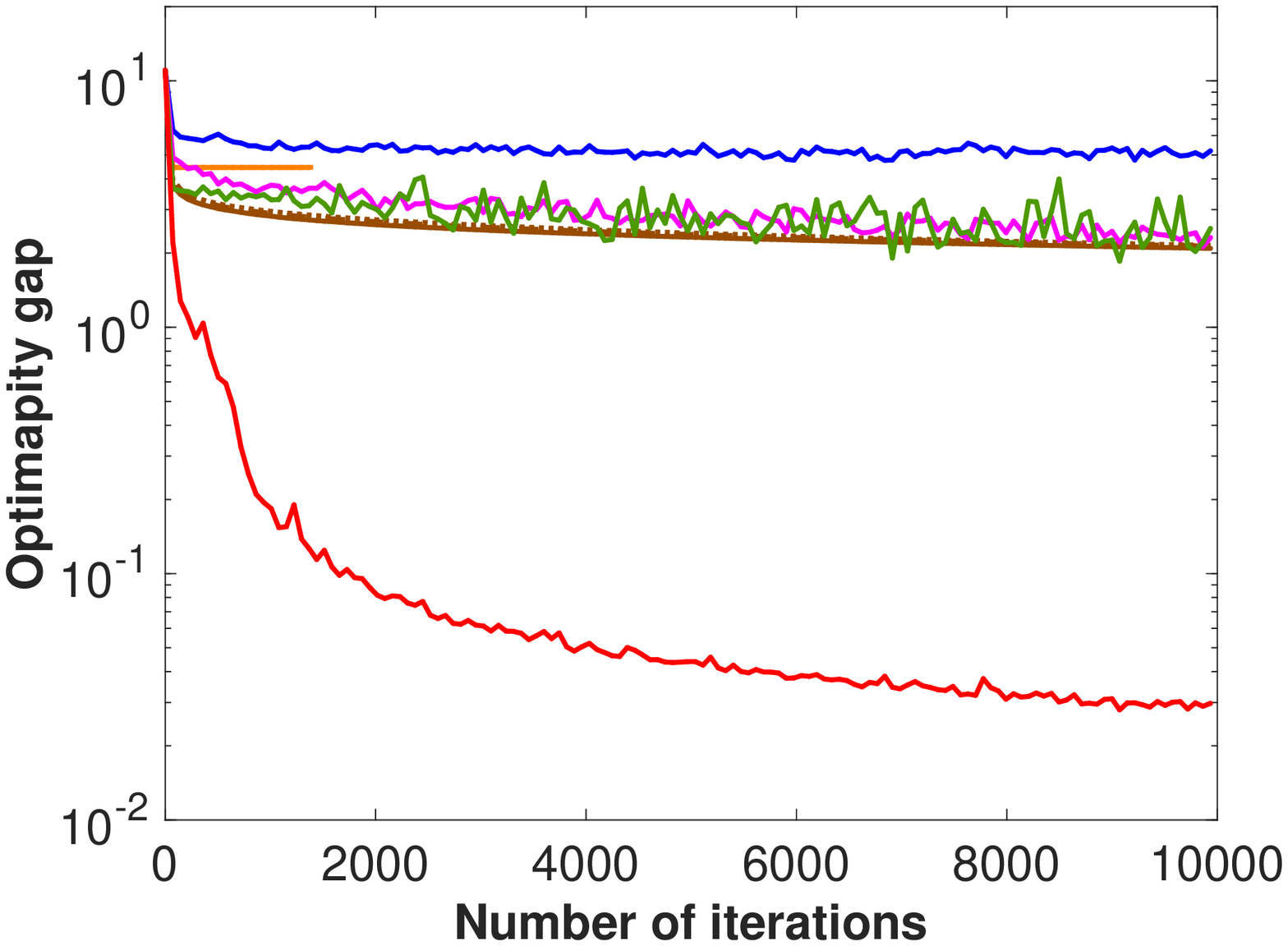}\\
		
		{\small  (e) $\alpha_0=0.1$.}
		
	\end{center} 
	\end{minipage}
	\hspace*{-0.1cm}
	\begin{minipage}[t]{.2\textwidth}
	\begin{center}
		\includegraphics[width=\textwidth]{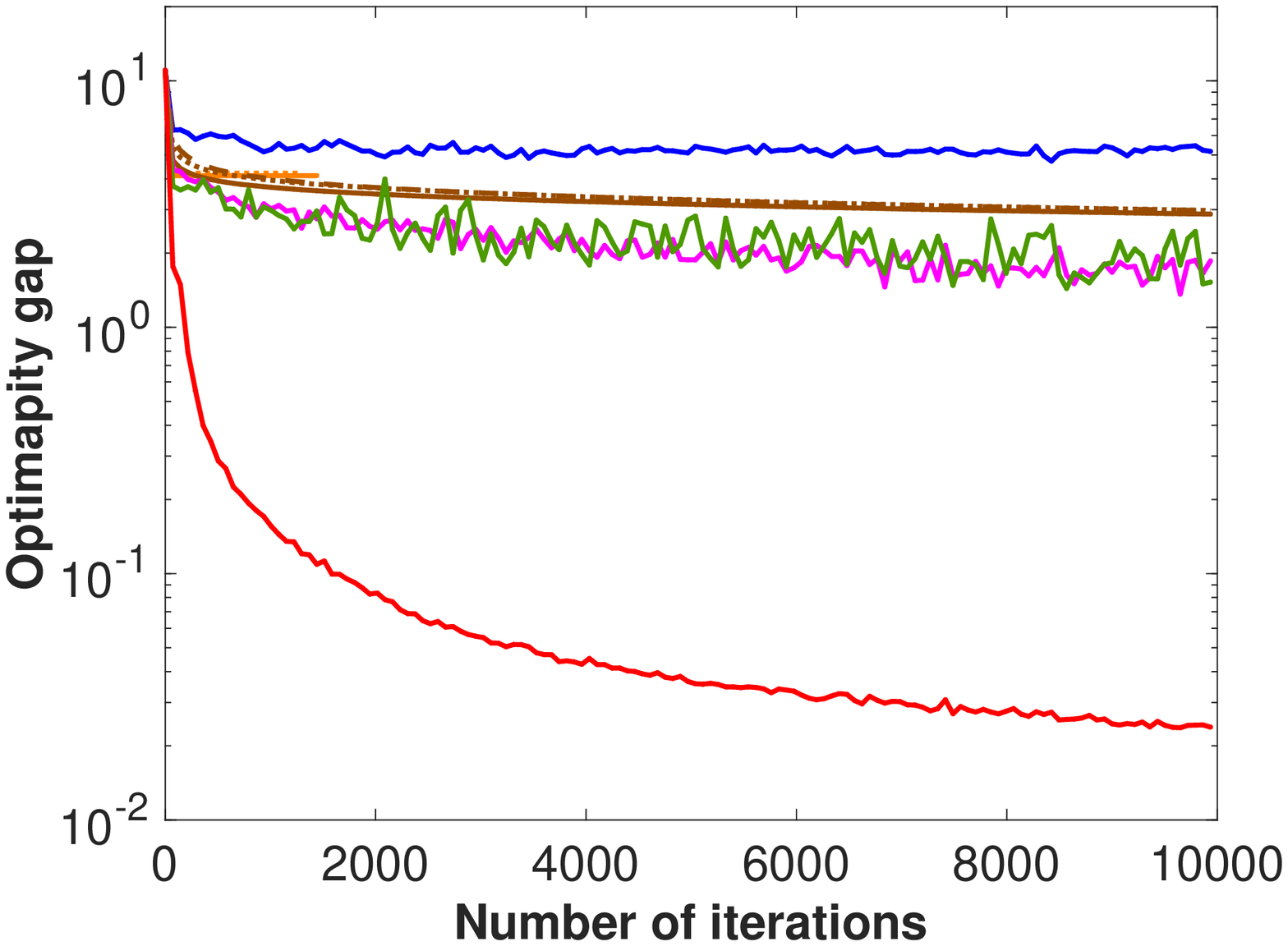}\\
		
		{\small  (f) $\alpha_0=0.05$.}
		
	\end{center} 
	\end{minipage}
	\hspace*{-0.1cm}
	\begin{minipage}[t]{.2\textwidth}
	\begin{center}
		\includegraphics[width=\textwidth]{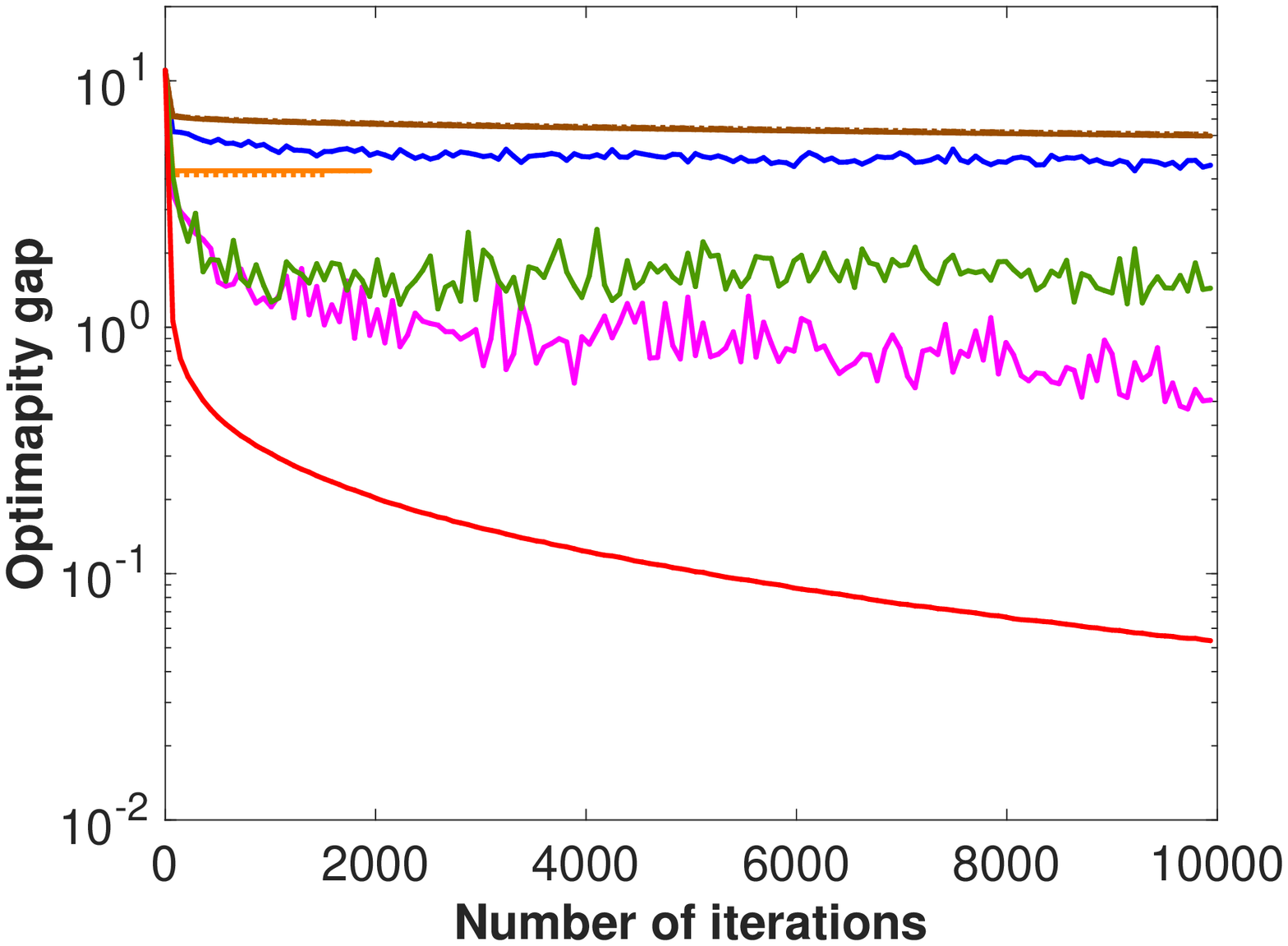}\\
		
		{\small  (g) $\alpha_0=0.01$.}
		
	\end{center} 
	\end{minipage}

	\vspace*{0.2cm}

	\begin{minipage}[t]{.2\textwidth}
	\begin{center}
		\includegraphics[width=\textwidth]{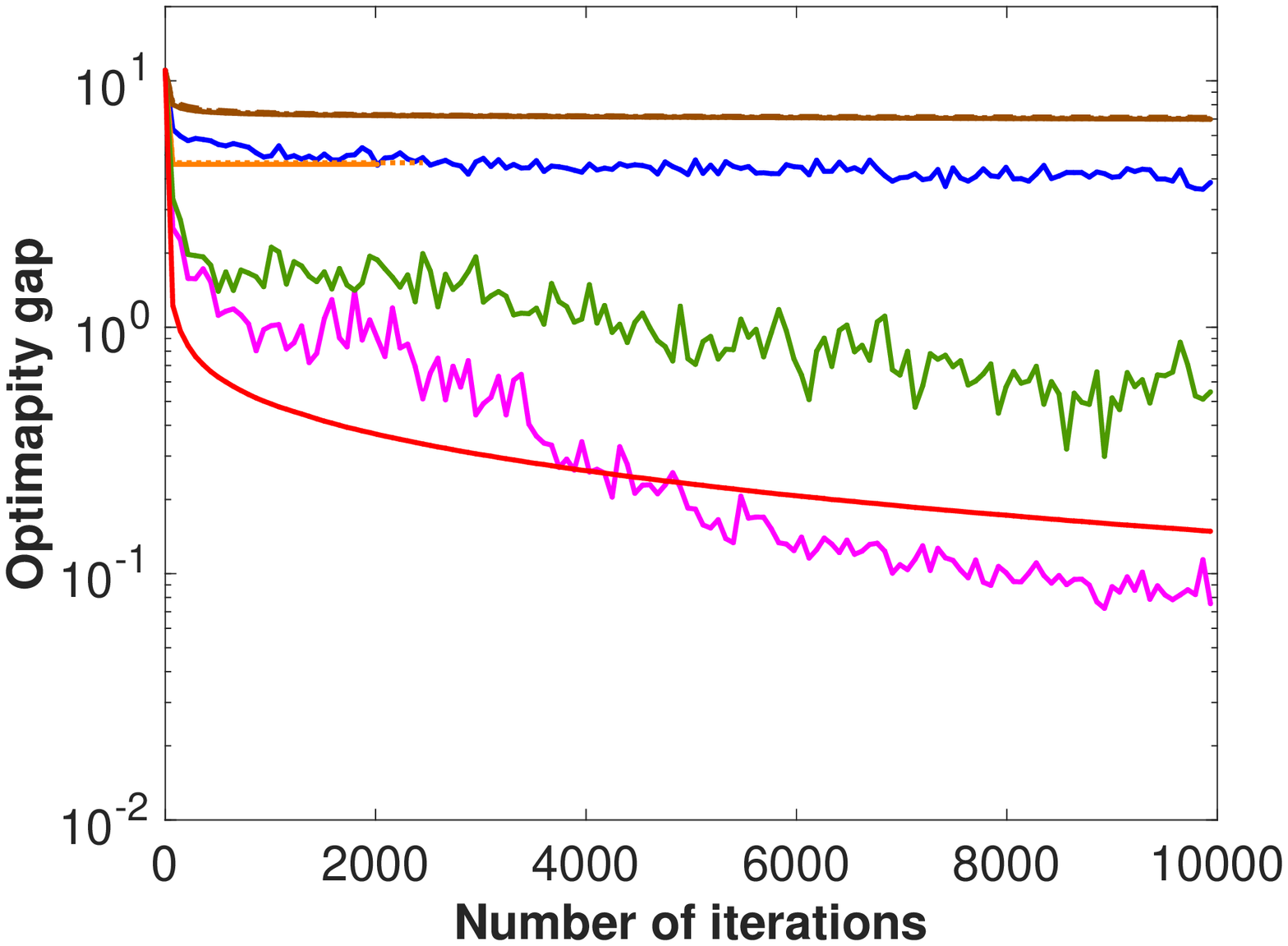}\\
		
		{\small  (h) $\alpha_0=0.005$.}
	\end{center}		
	\end{minipage}	
	\hspace*{-0.1cm}
	\begin{minipage}[t]{.2\textwidth}
	\begin{center}
		\includegraphics[width=\textwidth]{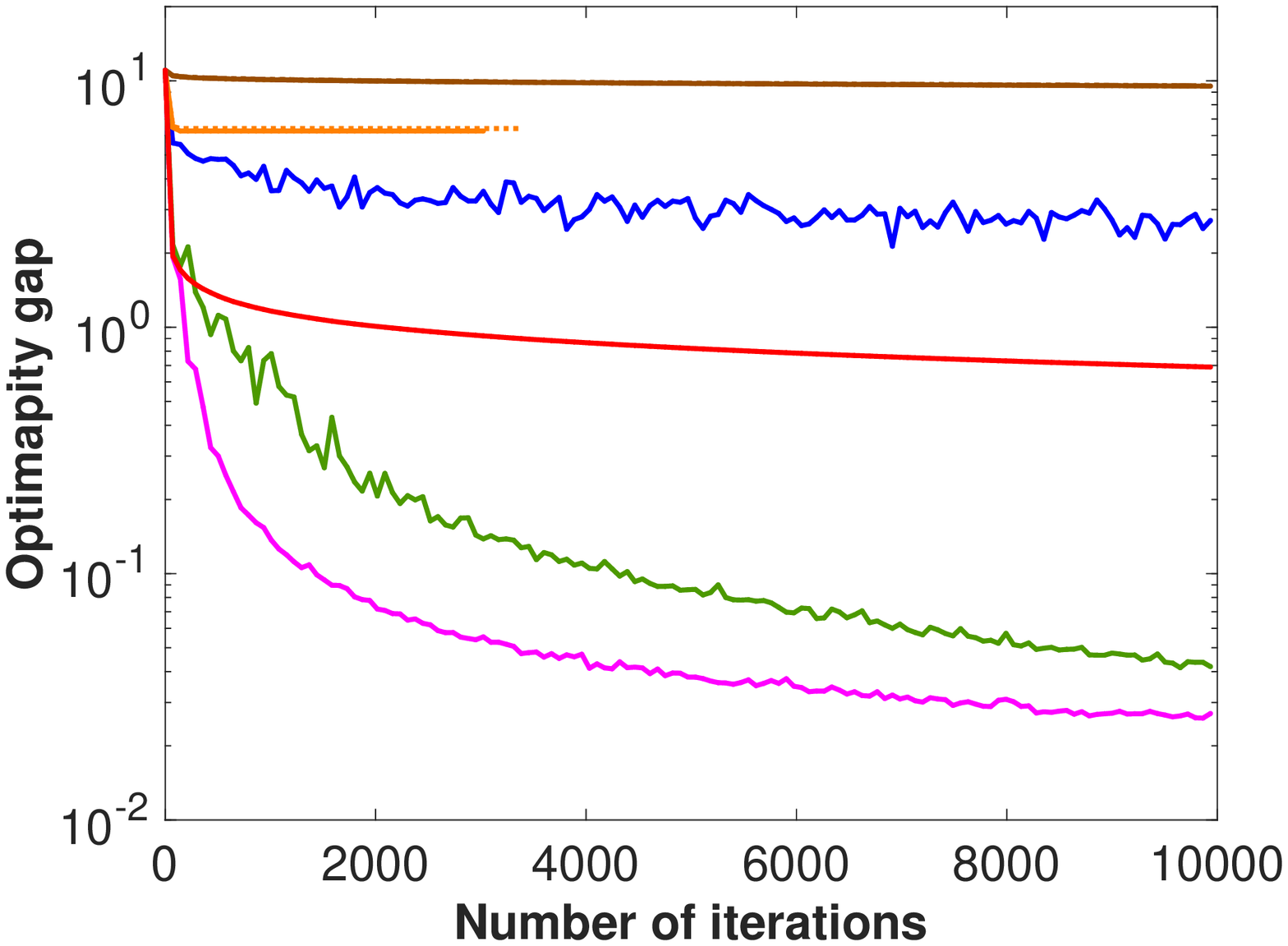}\\
		
		{\small  (i) $\alpha_0=0.001$.}
		
	\end{center} 
	\end{minipage}
	\hspace*{-0.1cm}
	\begin{minipage}[t]{.2\textwidth}
	\begin{center}
		\includegraphics[width=\textwidth]{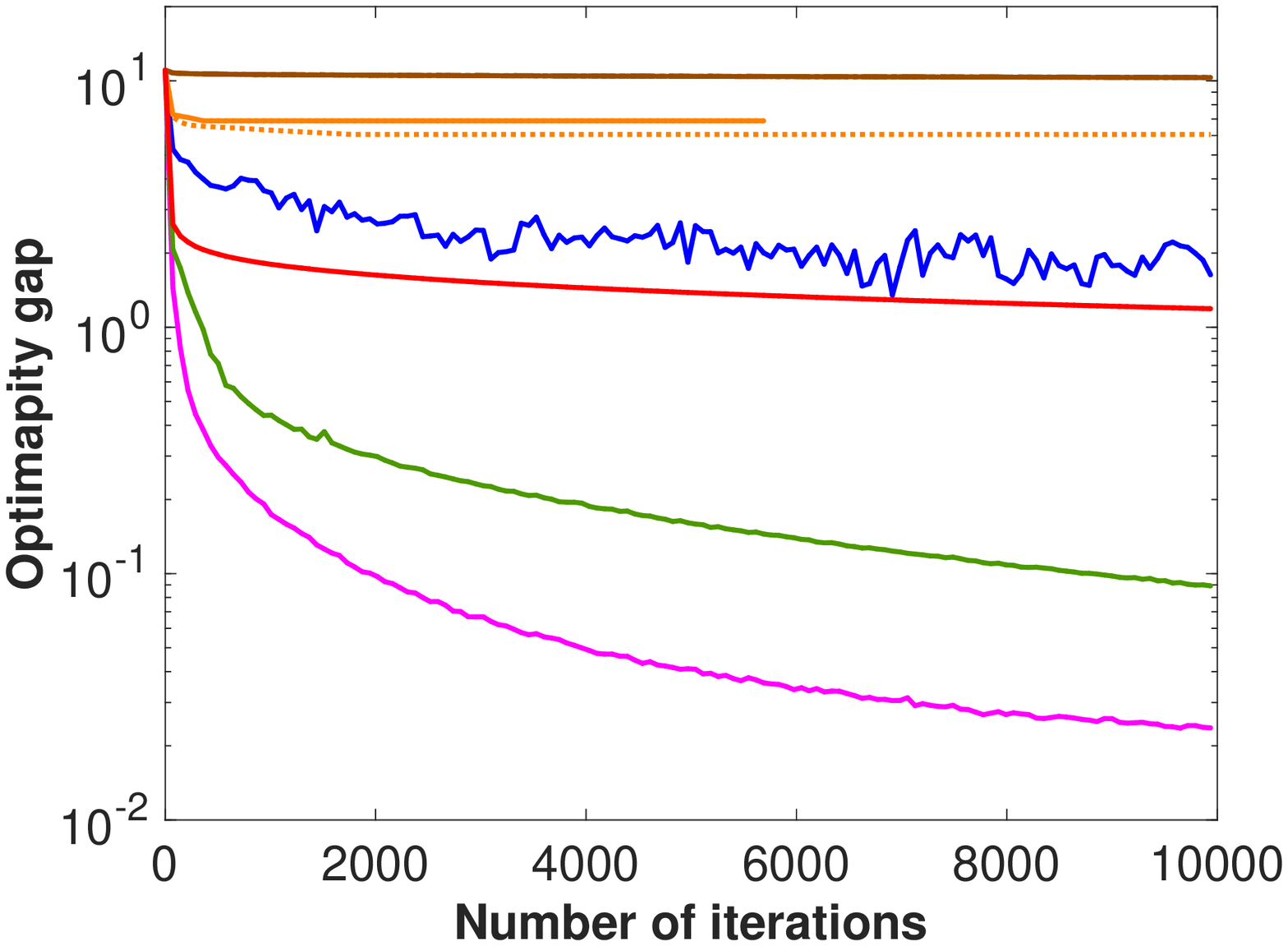}\\
		
		{\small  (j) $\alpha_0=0.0005$.}
		
	\end{center} 
	\end{minipage}
	\hspace*{-0.1cm}
	\begin{minipage}[t]{.2\textwidth}
	\begin{center}
		\includegraphics[width=\textwidth]{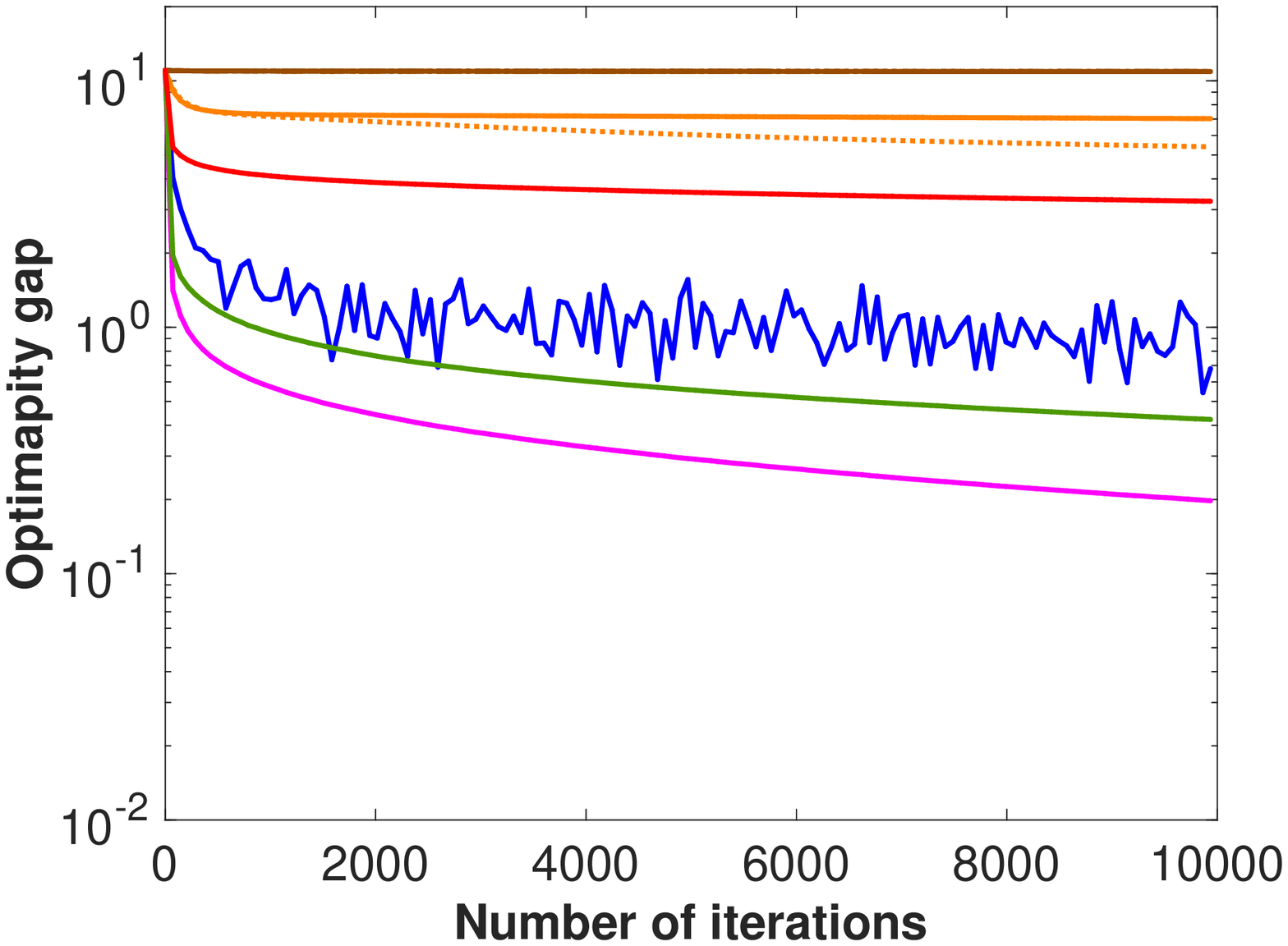}\\
		
		{\small  (k) $\alpha_0=0.0001$.}
		
	\end{center} 
	\end{minipage}

	\vspace*{0.2cm}

	\begin{minipage}[t]{.2\textwidth}
	\begin{center}
		\includegraphics[width=\textwidth]{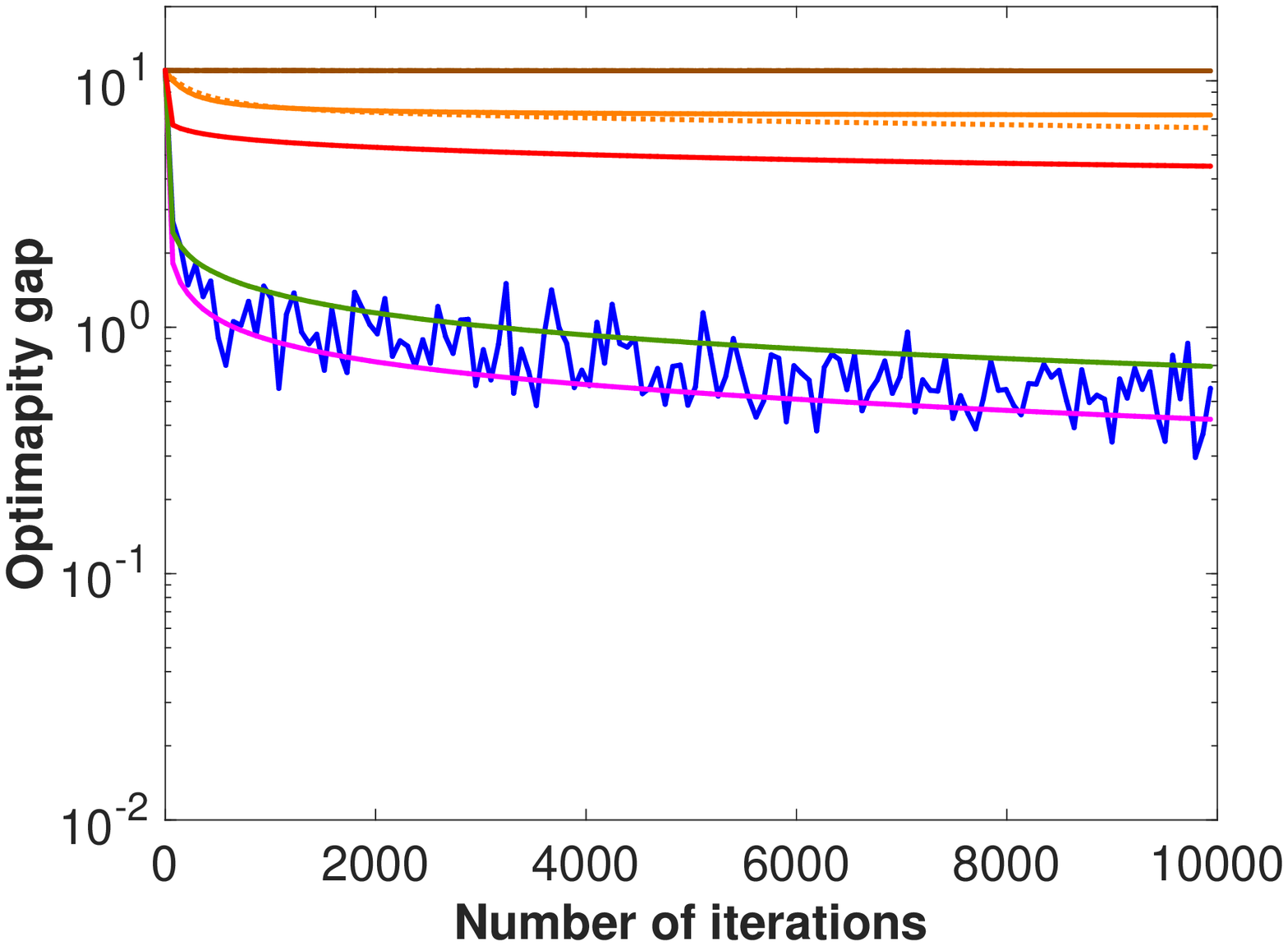}\\
		
		{\small  (l) $\alpha_0=0.00005$.}
	\end{center}		
	\end{minipage}	
	\hspace*{-0.1cm}
	\begin{minipage}[t]{.2\textwidth}
	\begin{center}
		\includegraphics[width=\textwidth]{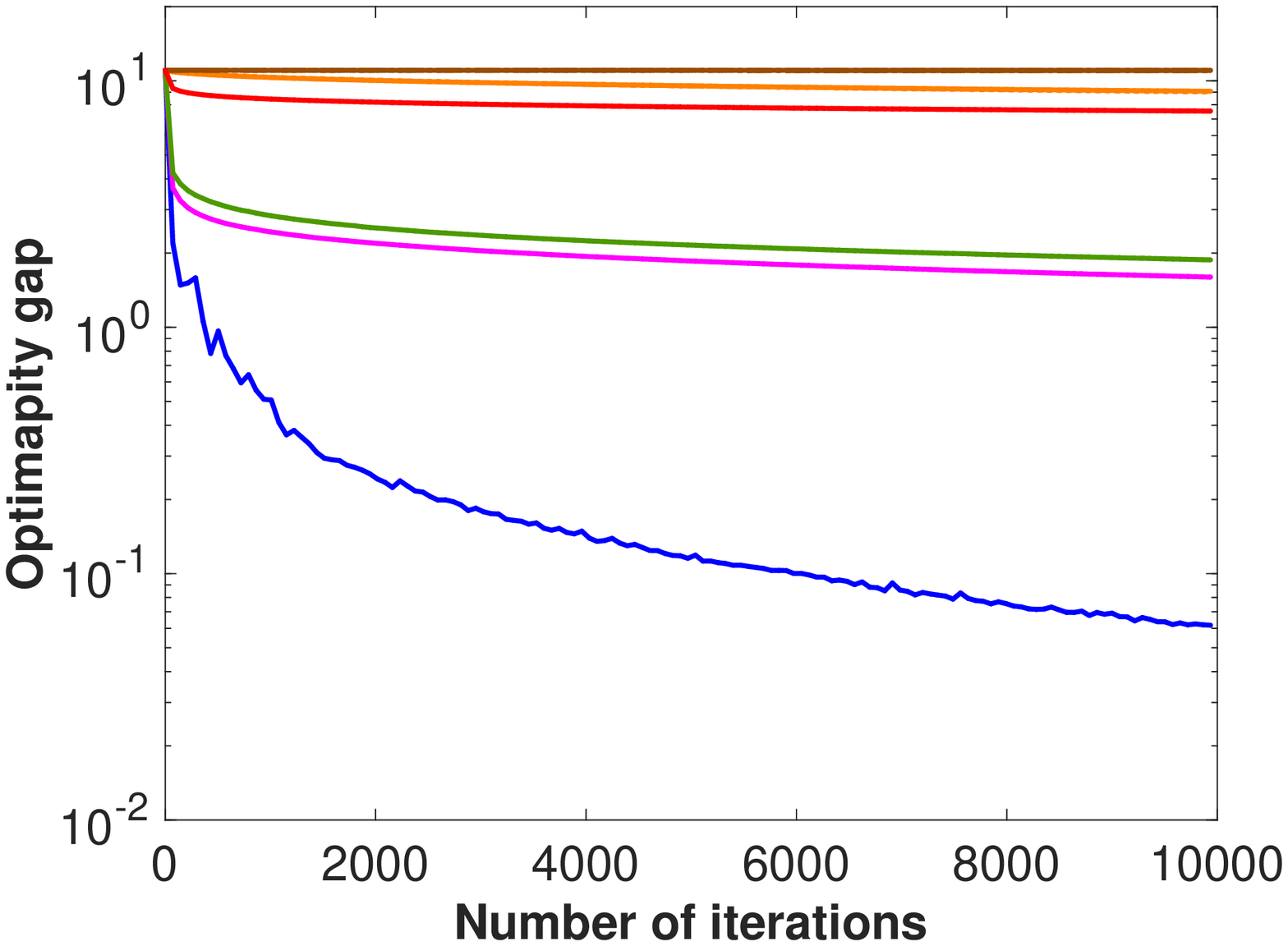}\\
		
		{\small  (m) $\alpha_0=0.00001$.}
		
	\end{center} 
	\end{minipage}

	\vspace*{0.2cm}

	\caption{{\tt COIL100} dataset on the PCA problem ({\bf Case P3}).}

\label{appfig:PCA_results_COIL100}
\end{center}
\end{figure*}

\begin{figure*}[t]
\begin{center}

	\begin{minipage}[t]{.2\textwidth}
	\begin{center}
		\includegraphics[width=\textwidth]{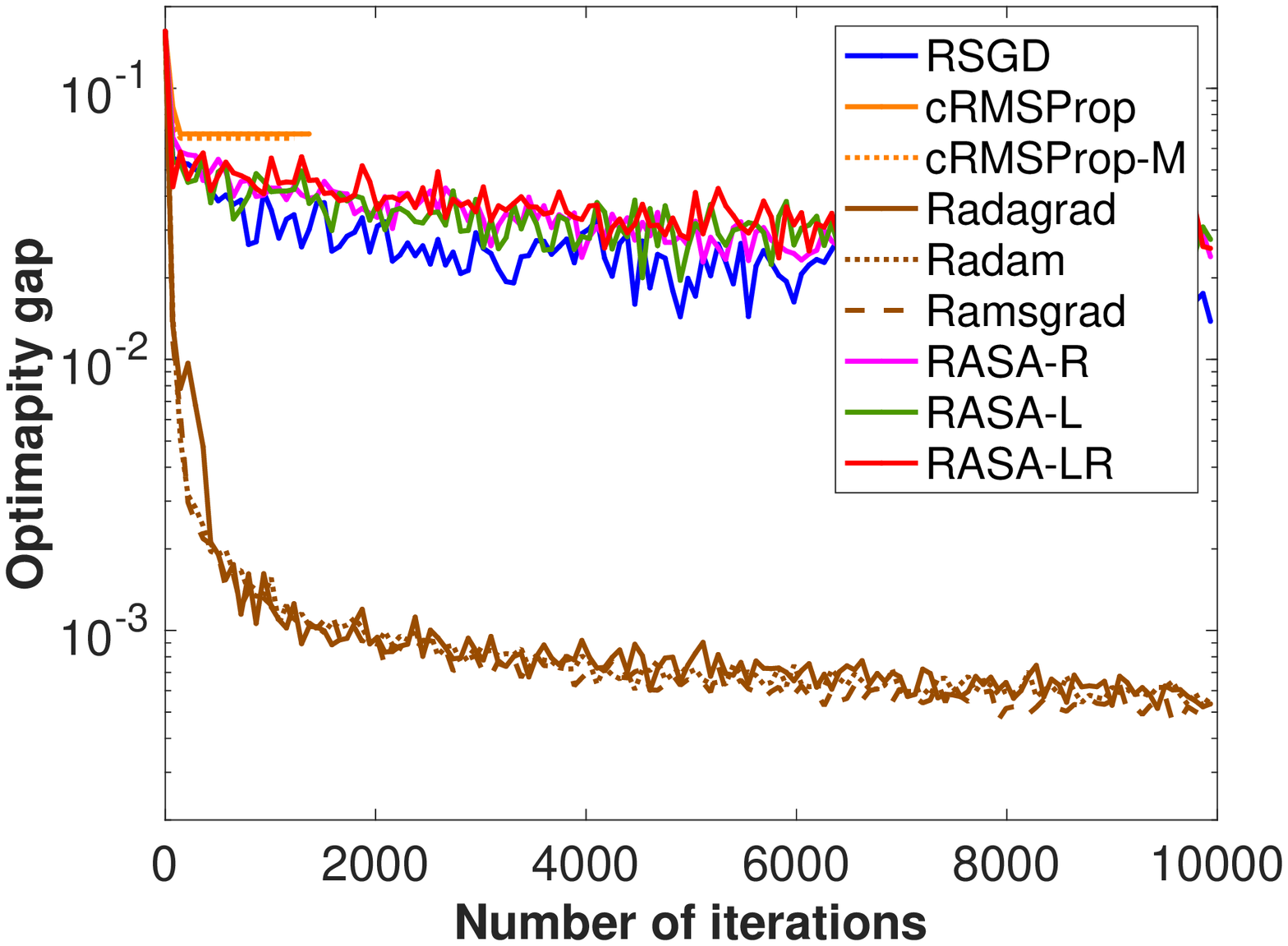}\\
		
		{\small  (a) $\alpha_0=10$.}
	\end{center}		
	\end{minipage}	
	\hspace*{-0.1cm}
	\begin{minipage}[t]{.2\textwidth}
	\begin{center}
		\includegraphics[width=\textwidth]{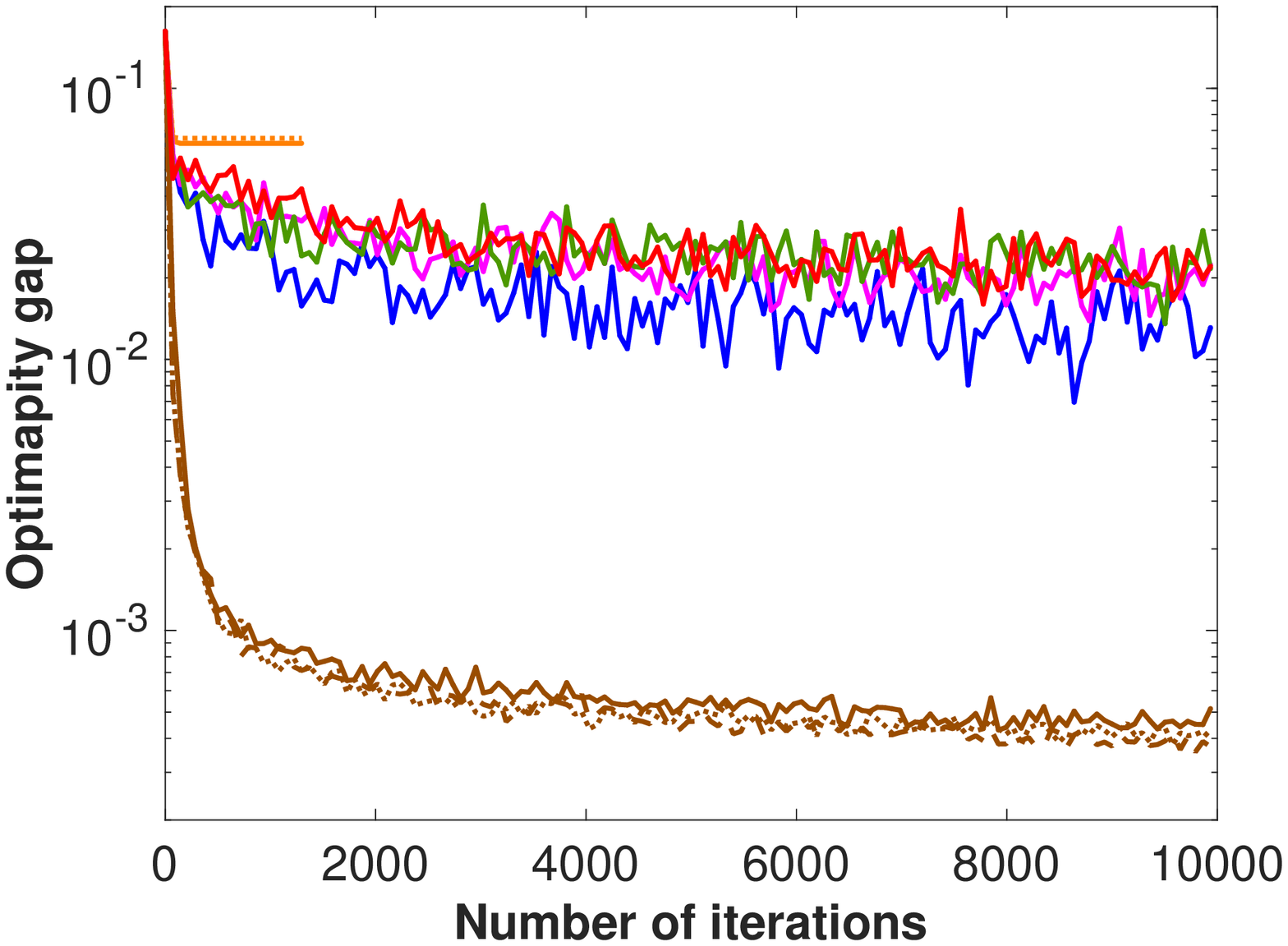}\\
		
		{\small  (b) $\alpha_0=5$.}
		
	\end{center} 
	\end{minipage}
	\hspace*{-0.1cm}
	\begin{minipage}[t]{.2\textwidth}
	\begin{center}
		\includegraphics[width=\textwidth]{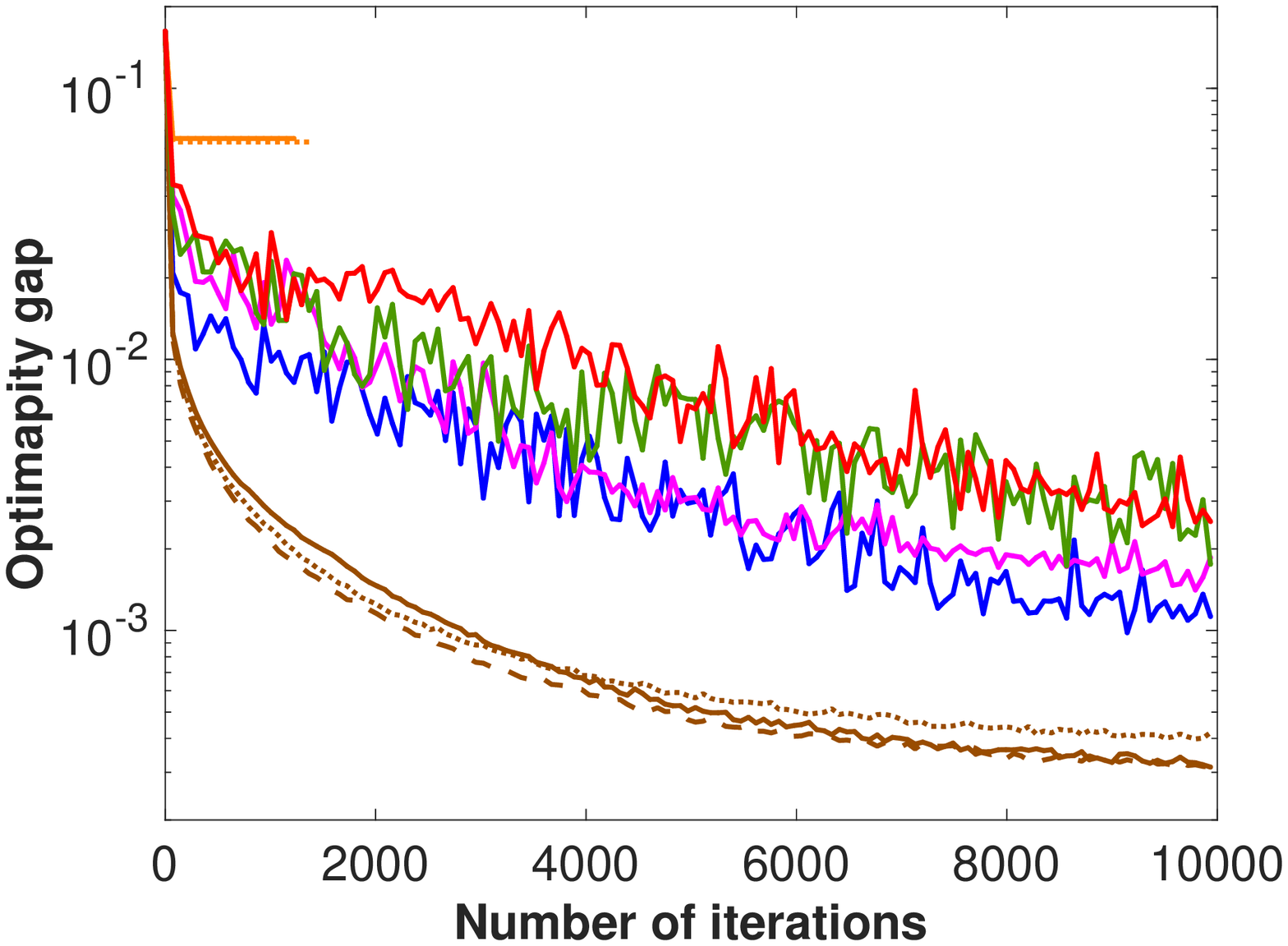}\\
		
		{\small  (c) $\alpha_0=1$.}
		
	\end{center} 
	\end{minipage}
	\hspace*{-0.1cm}
	\begin{minipage}[t]{.2\textwidth}
	\begin{center}
		\includegraphics[width=\textwidth]{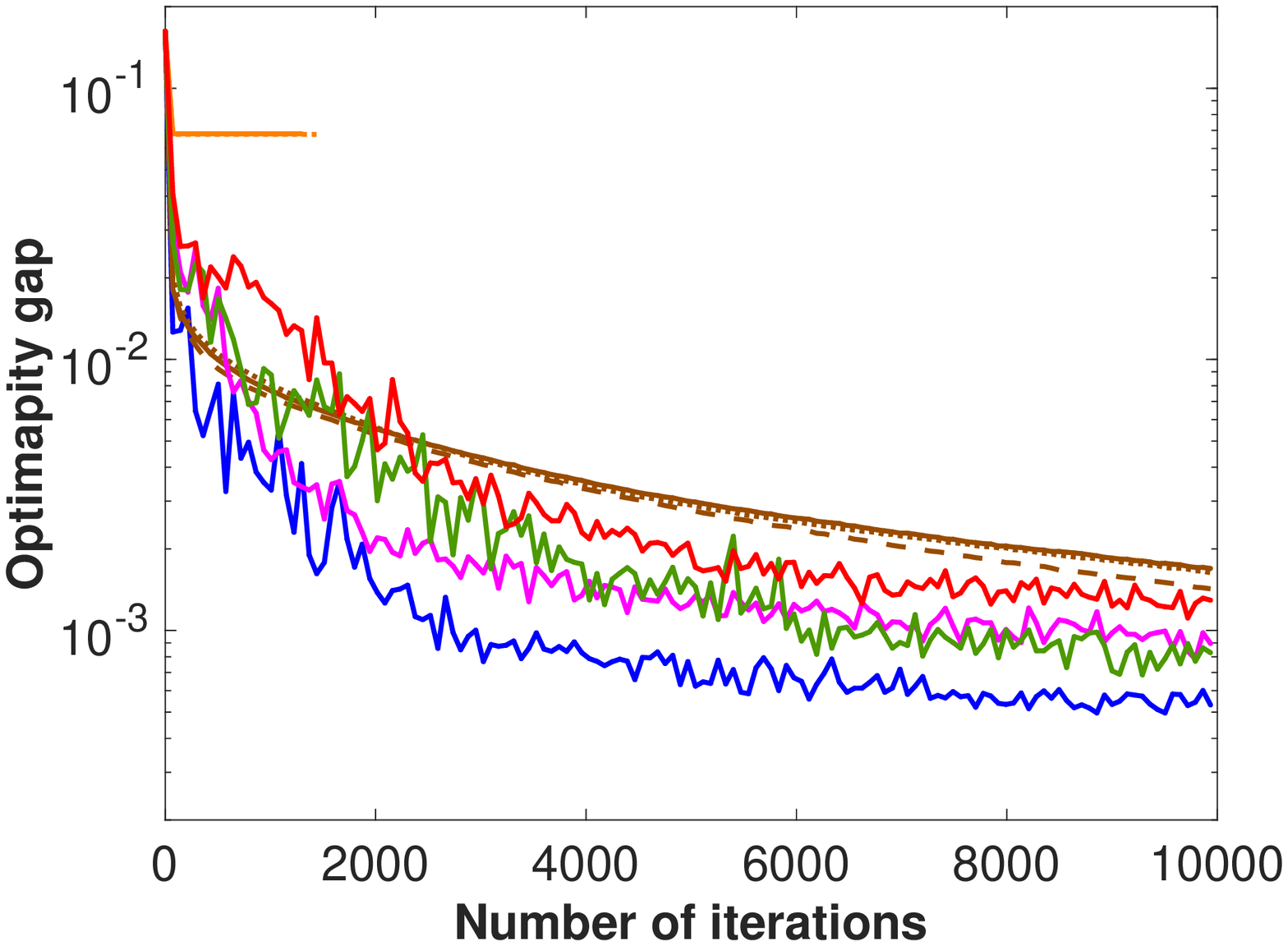}\\
		
		{\small  (d) $\alpha_0=0.5$.}
		
	\end{center} 
	\end{minipage}

	\vspace*{0.2cm}

	\begin{minipage}[t]{.2\textwidth}
	\begin{center}
		\includegraphics[width=\textwidth]{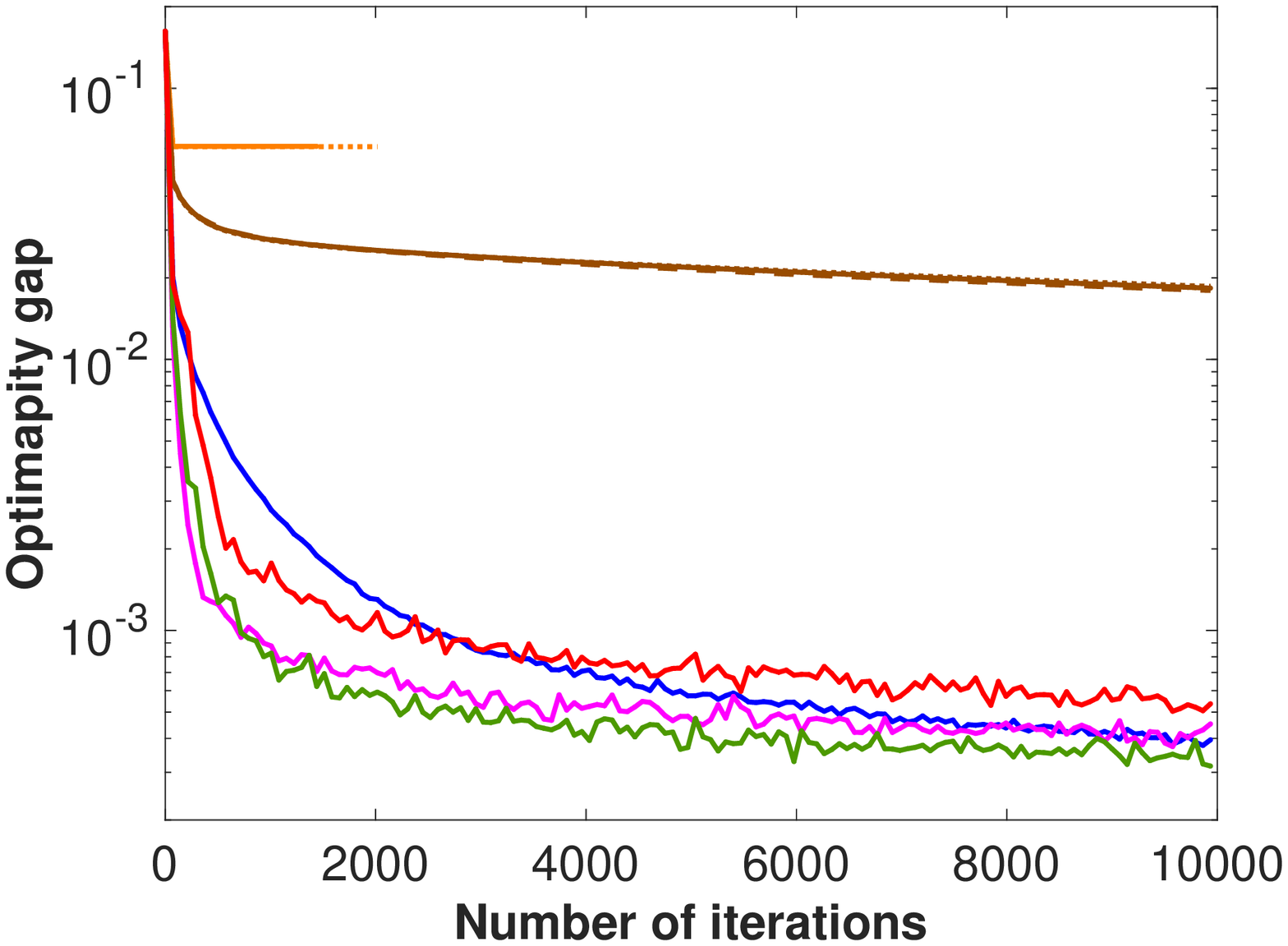}\\
		
		{\small  (e) $\alpha_0=0.1$.}
	\end{center}		
	\end{minipage}	
	\hspace*{-0.1cm}
	\begin{minipage}[t]{.2\textwidth}
	\begin{center}
		\includegraphics[width=\textwidth]{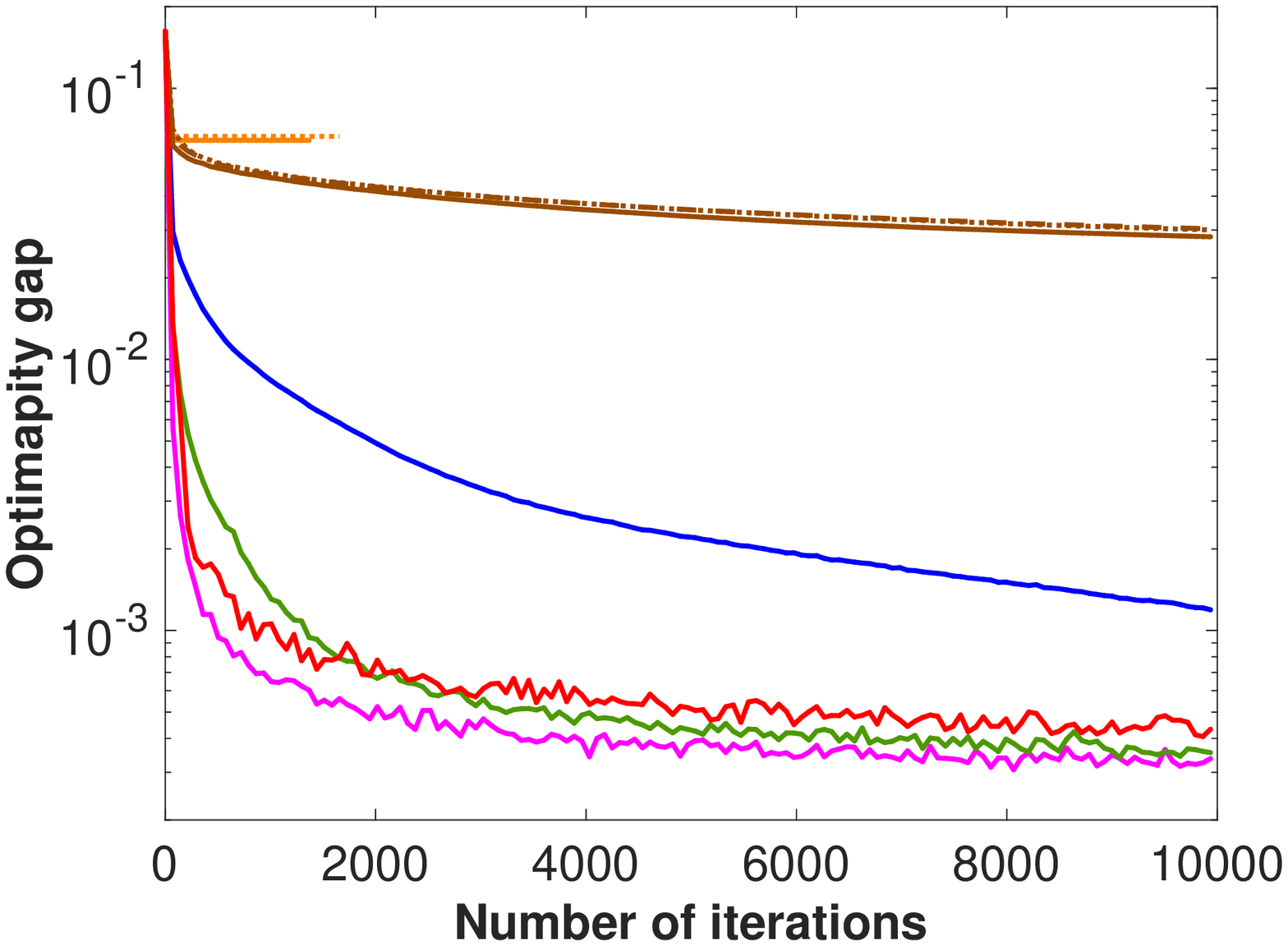}\\
		
		{\small  (f) $\alpha_0=0.05$.}
		
	\end{center} 
	\end{minipage}
	\hspace*{-0.1cm}
	\begin{minipage}[t]{.2\textwidth}
	\begin{center}
		\includegraphics[width=\textwidth]{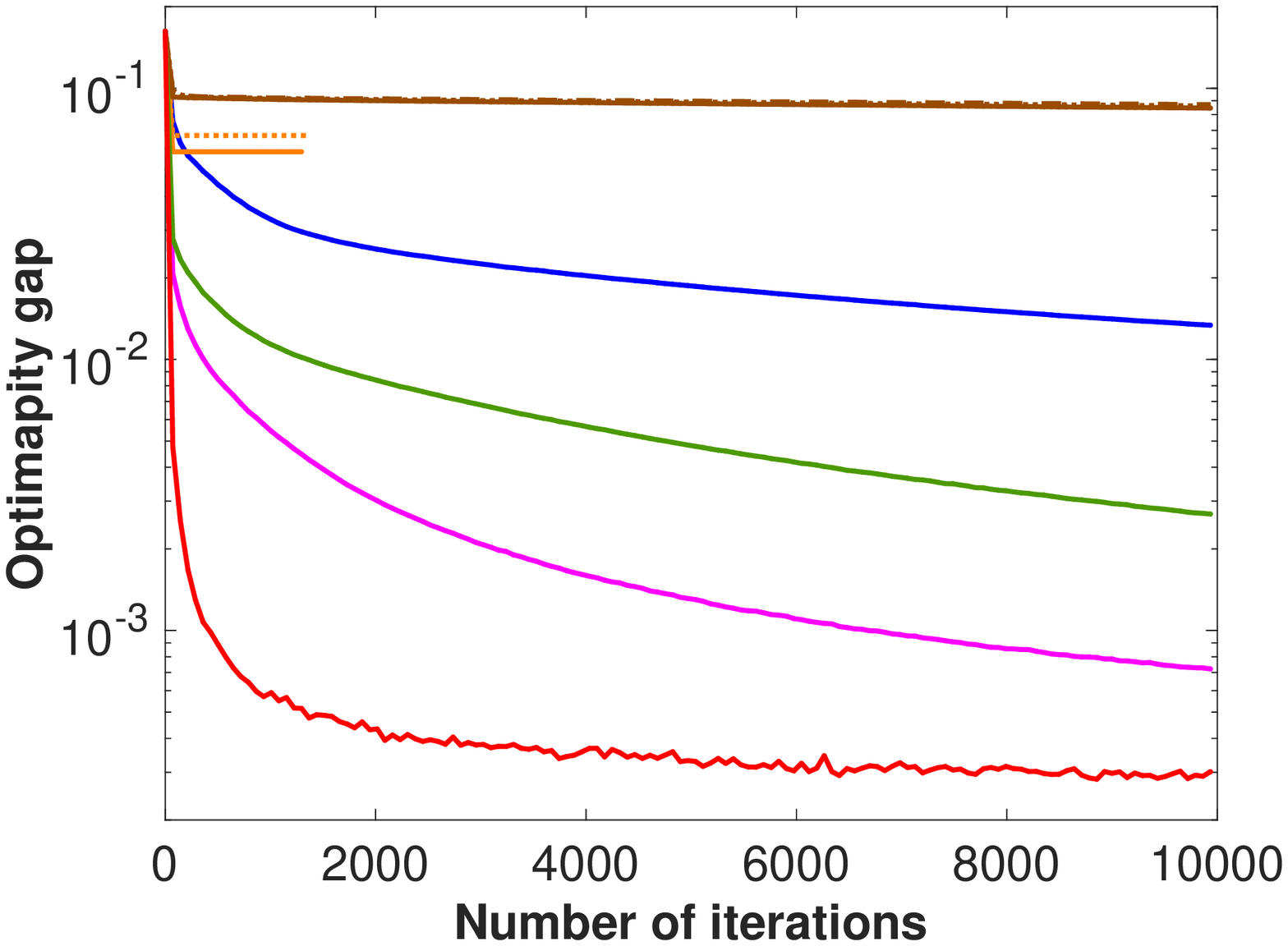}\\
		
		{\small  (g) $\alpha_0=0.01$.}
		
	\end{center} 
	\end{minipage}
	\hspace*{-0.1cm}
	\begin{minipage}[t]{.2\textwidth}
	\begin{center}
		\includegraphics[width=\textwidth]{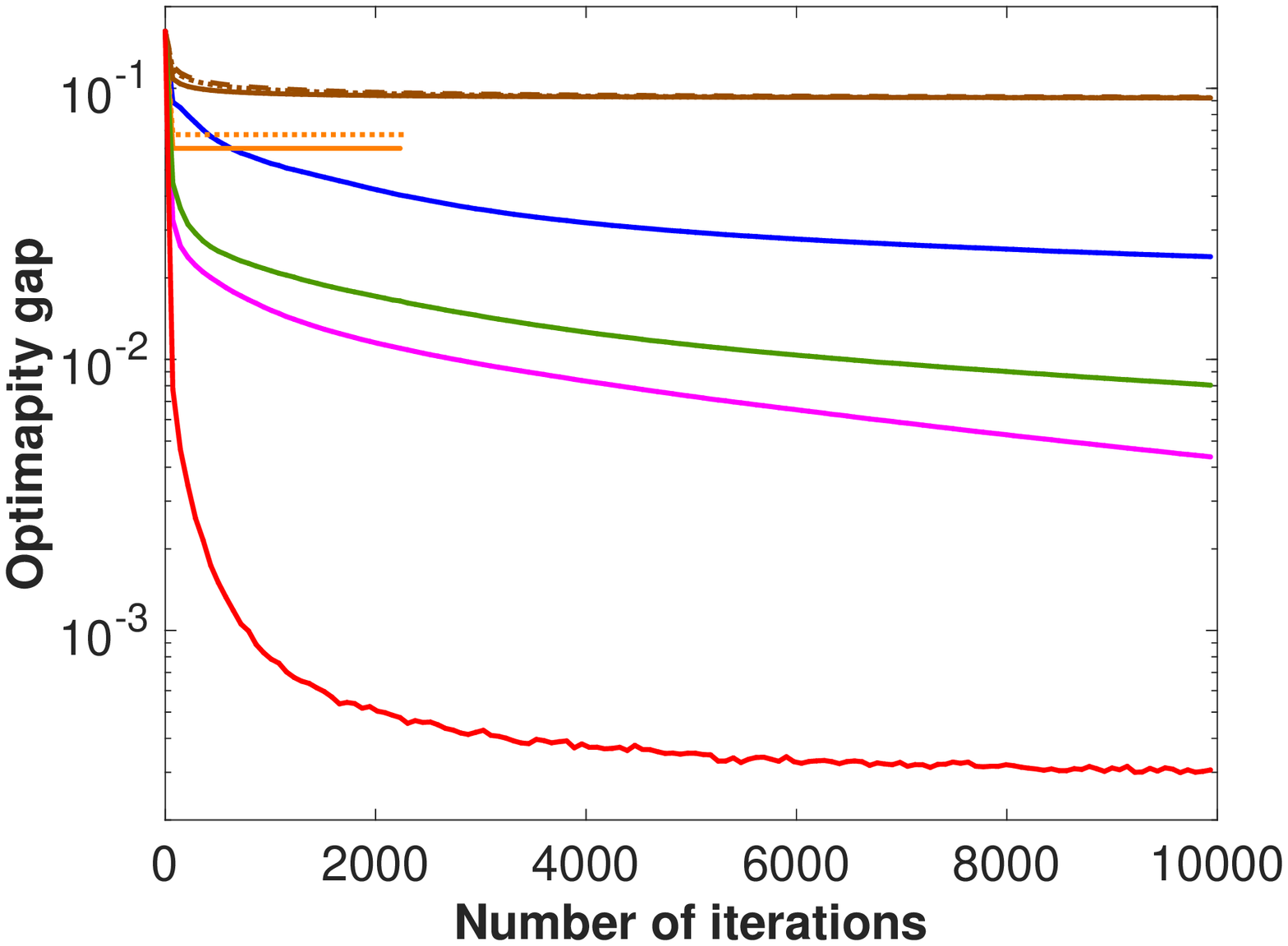}\\
		
		{\small  (h) $\alpha_0=0.005$.}
		
	\end{center} 
	\end{minipage}

	\vspace*{0.2cm}

	\caption{{\tt COIL20} dataset on the PCA problem.}

\label{appfig:PCA_results_COIL20}
\end{center}
\end{figure*}

\subsection{\changeHK{Results on the real-world datasets for the ICA problem ({\bf Case I1} and {\bf Case I2})}}

\changeHK{This section gives additional results of performance of the algorithms across different values of $\alpha_0$ for {\bf Case I1} and {\bf Case I2}, which are shown in Figures \ref{appfig:ICA_results_YaleB} and \ref{appfig:ICA_results_COIL100}, respectively. 
}

\begin{figure*}[h]
\begin{center}

	\begin{minipage}[t]{.2\textwidth}
	\begin{center}
		\includegraphics[width=\textwidth]{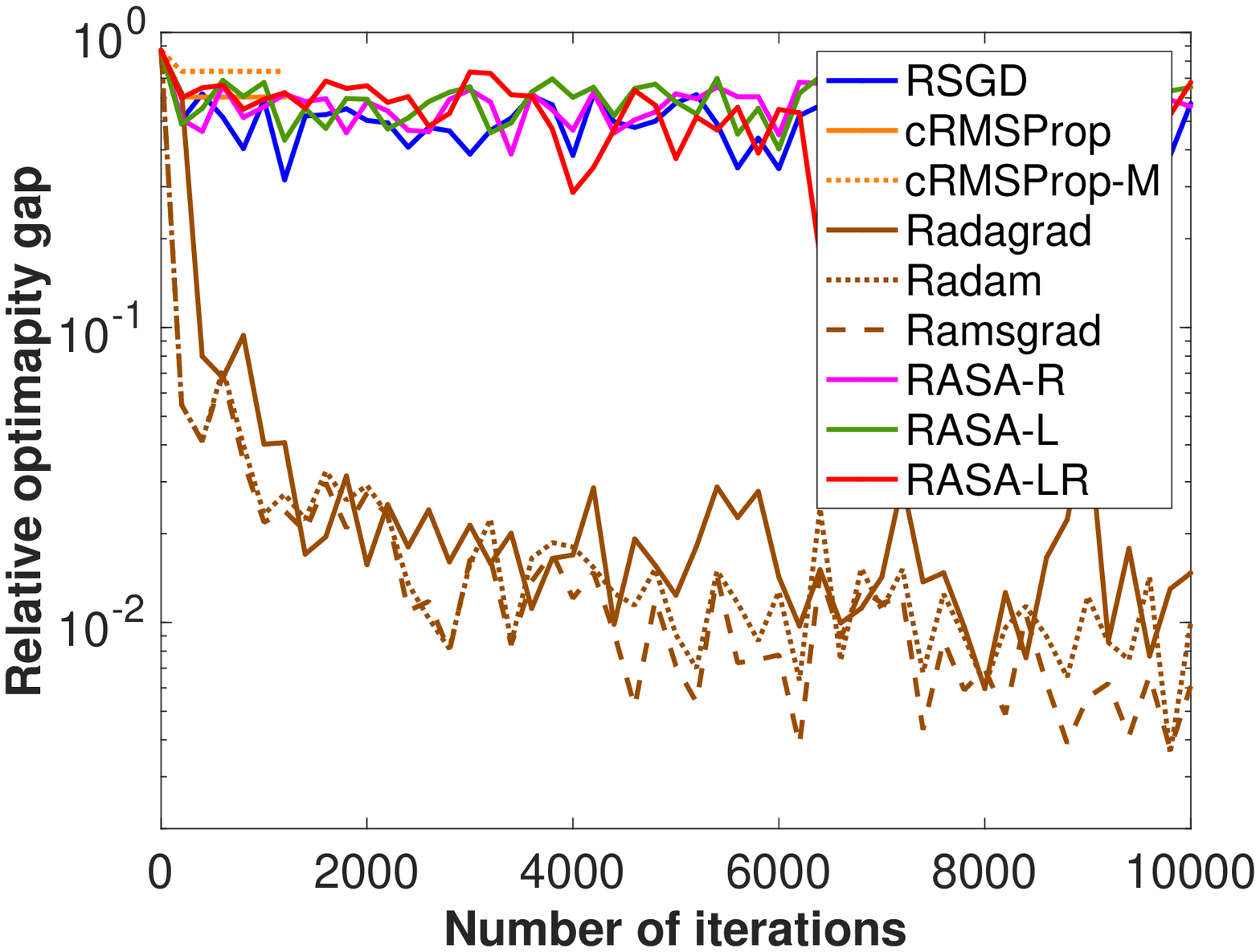}\\
		
		{\small  (a) $\alpha_0=10$.}
	\end{center}		
	\end{minipage}	
	\hspace*{-0.1cm}
	\begin{minipage}[t]{.2\textwidth}
	\begin{center}
		\includegraphics[width=\textwidth]{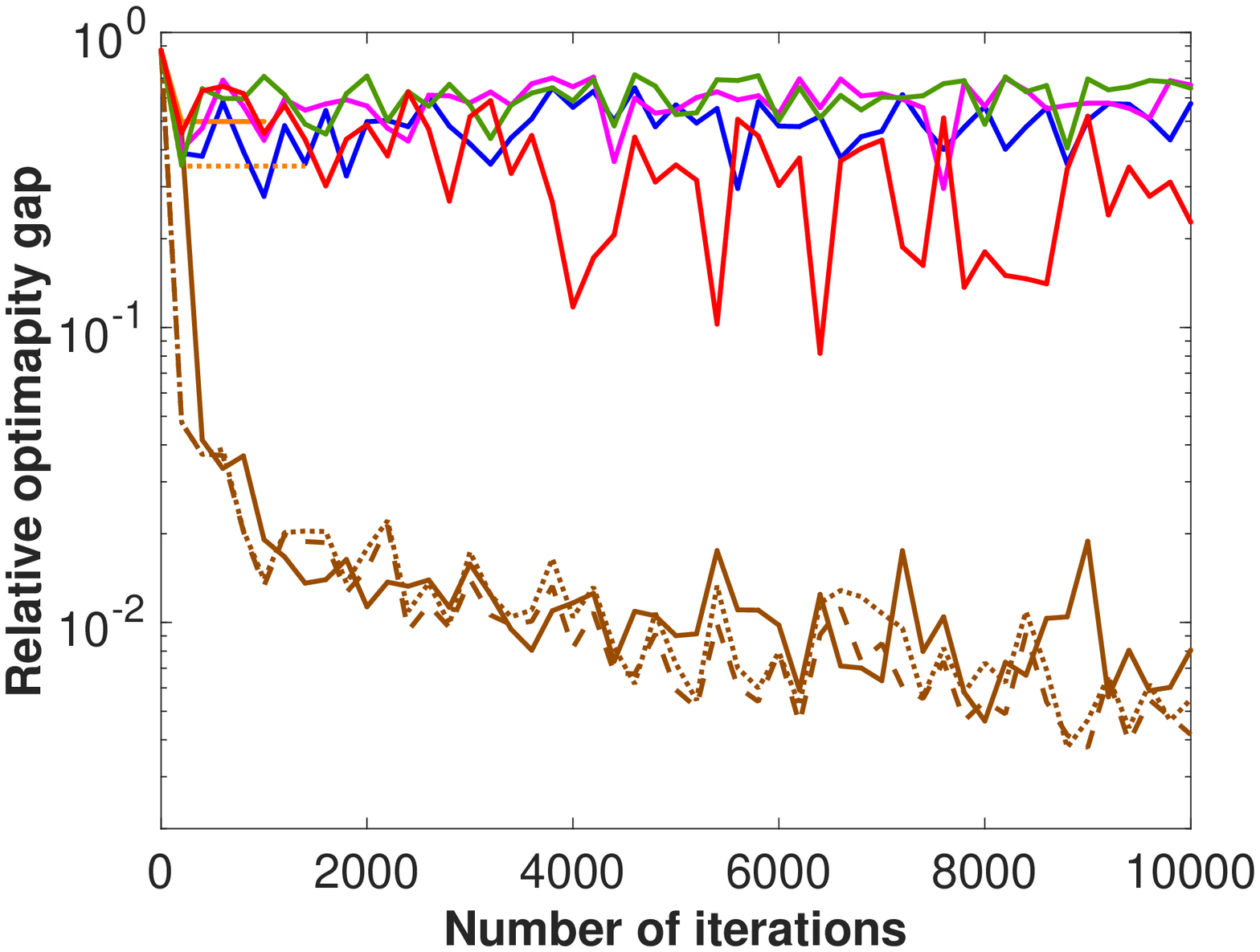}\\
		
		{\small  (b) $\alpha_0=5$.}
		
	\end{center} 
	\end{minipage}
	\hspace*{-0.1cm}
	\begin{minipage}[t]{.2\textwidth}
	\begin{center}
		\includegraphics[width=\textwidth]{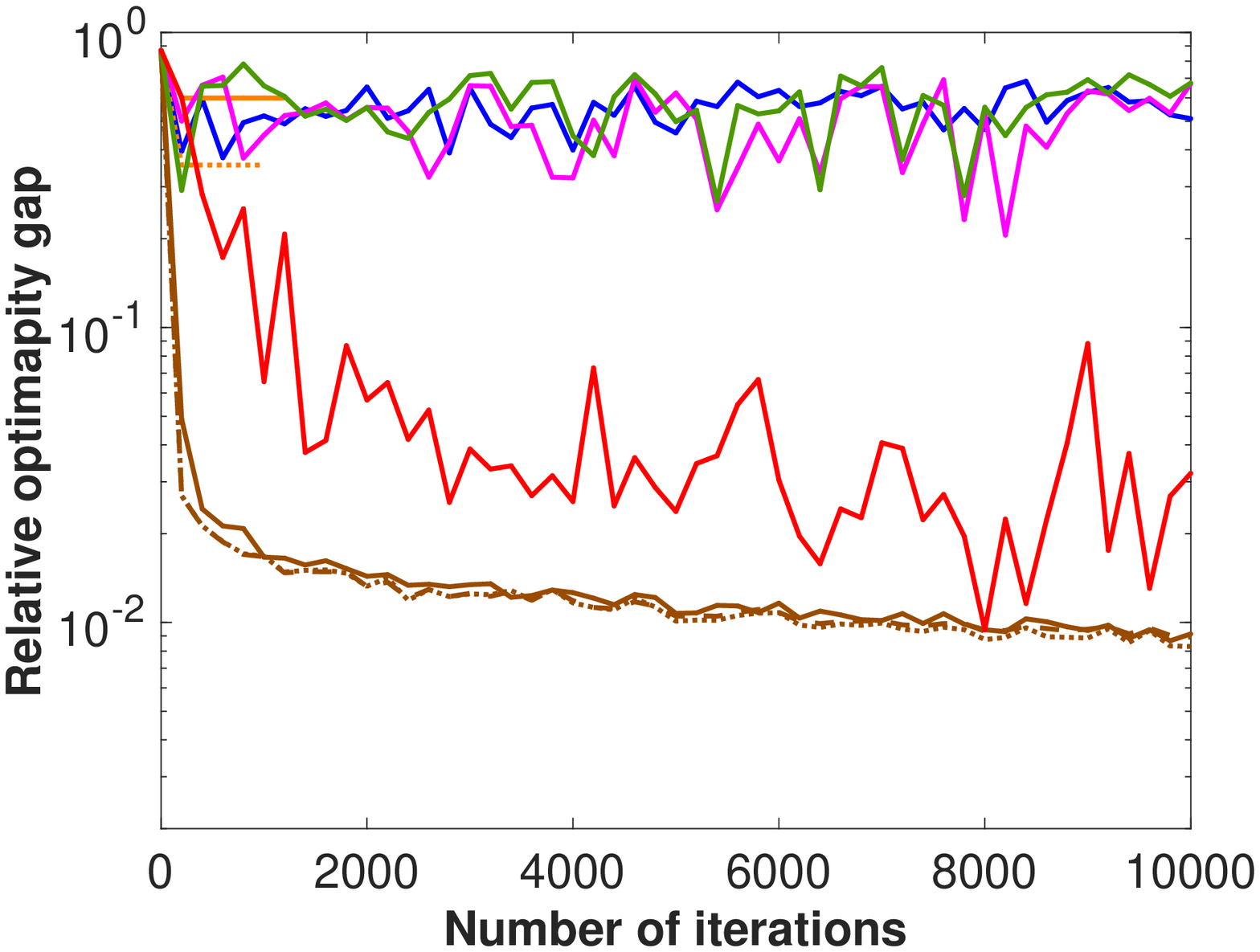}\\
		
		{\small  (c) $\alpha_0=1$.}
		
	\end{center} 
	\end{minipage}
	\hspace*{-0.1cm}
	\begin{minipage}[t]{.2\textwidth}
	\begin{center}
		\includegraphics[width=\textwidth]{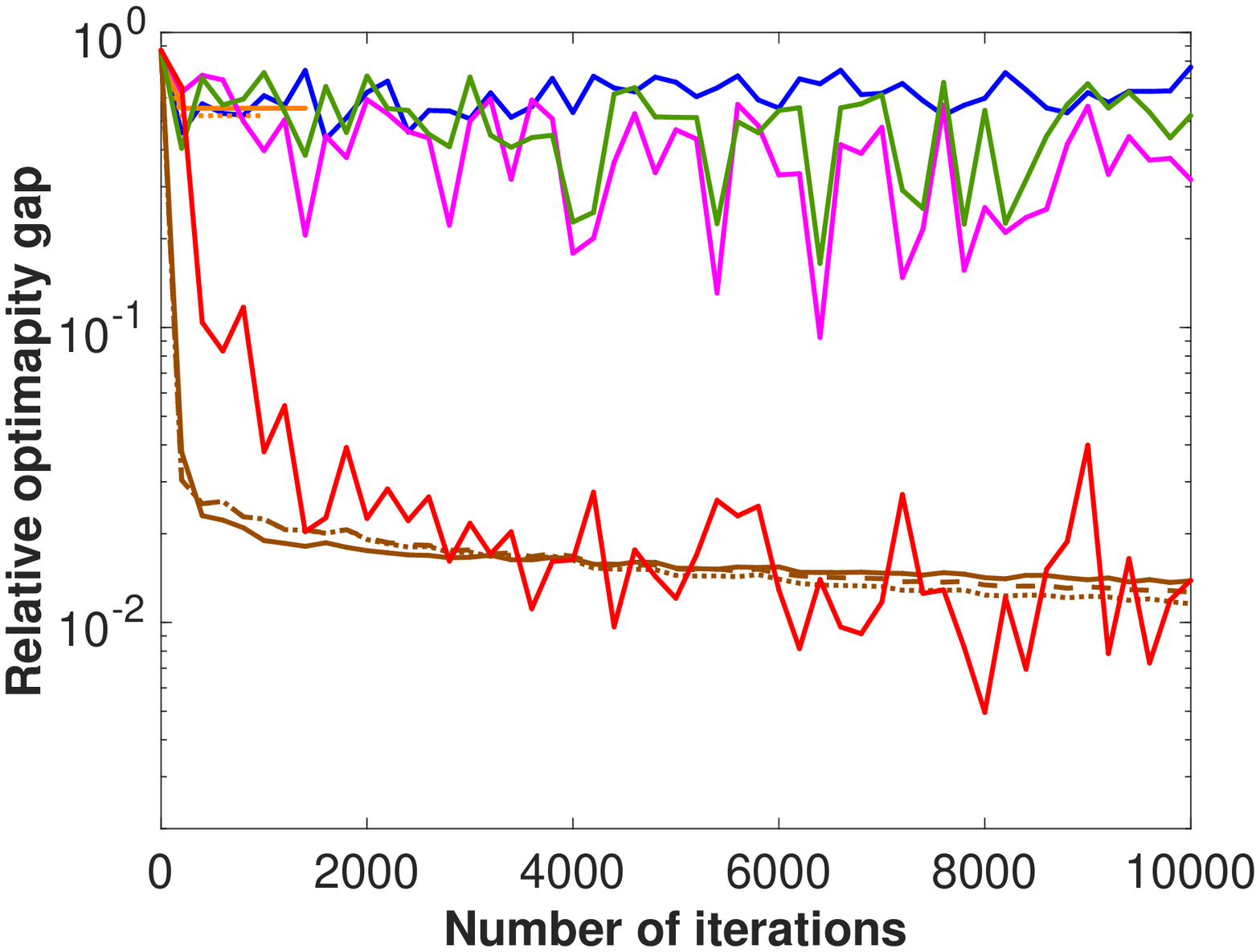}\\
		
		{\small  (d) $\alpha_0=0.5$.}
		
	\end{center} 
	\end{minipage}

	\vspace*{0.2cm}

	\begin{minipage}[t]{.2\textwidth}
	\begin{center}
		\includegraphics[width=\textwidth]{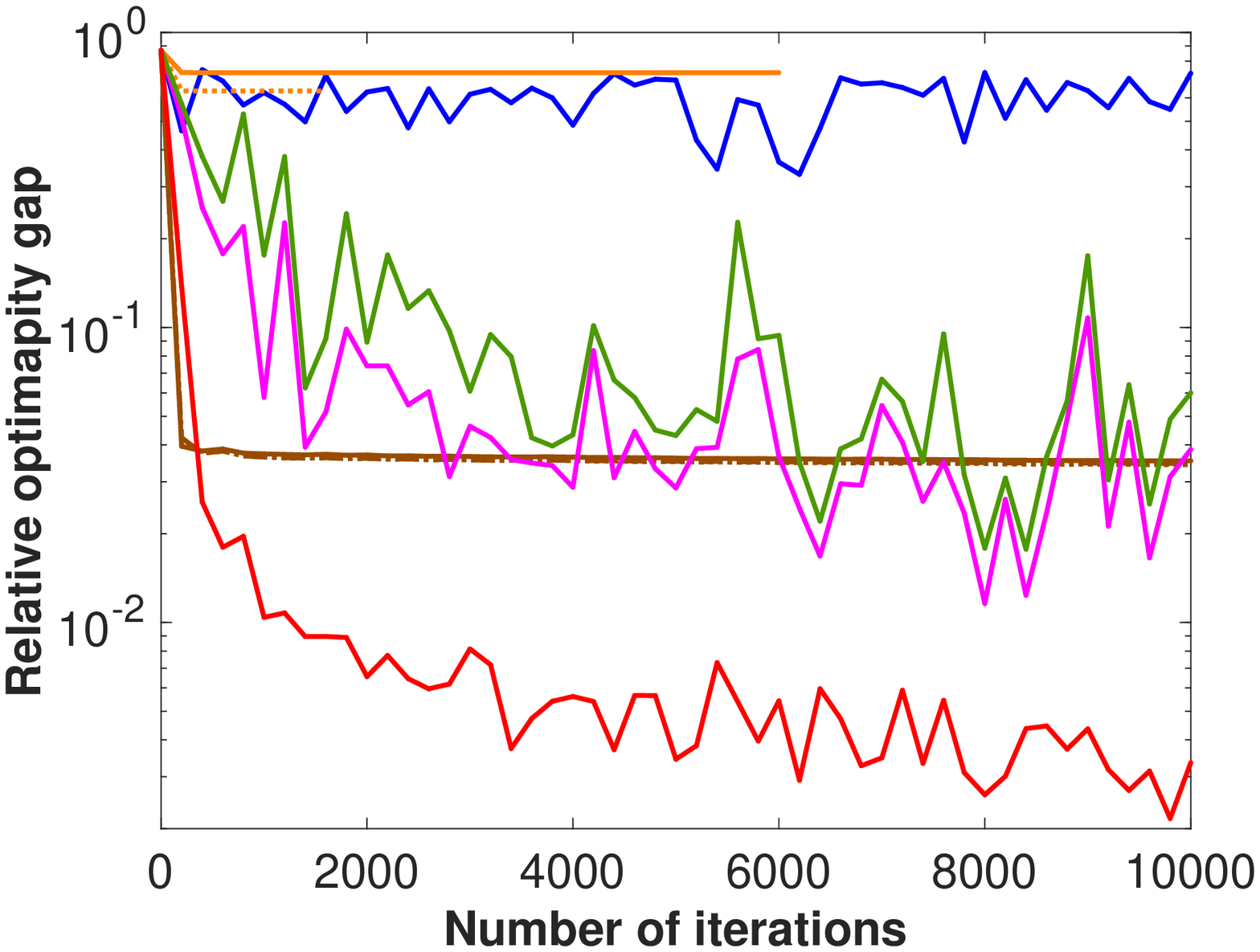}\\
		
		{\small  (e) $\alpha_0=0.1$.}
	\end{center}		
	\end{minipage}	
	\hspace*{-0.1cm}
	\begin{minipage}[t]{.2\textwidth}
	\begin{center}
		\includegraphics[width=\textwidth]{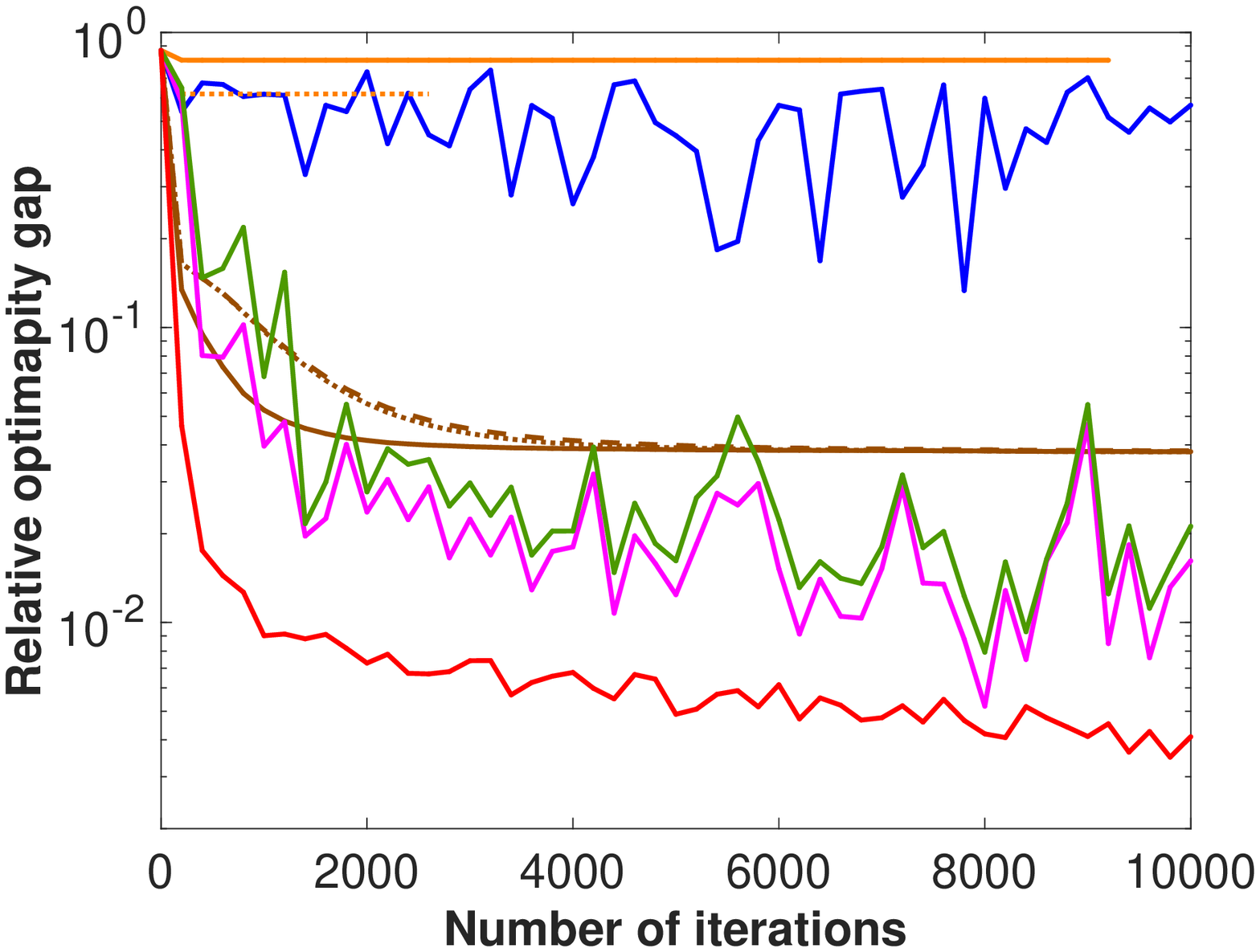}\\
		
		{\small  (f) $\alpha_0=0.05$.}
		
	\end{center} 
	\end{minipage}
	\hspace*{-0.1cm}
	\begin{minipage}[t]{.2\textwidth}
	\begin{center}
	\includegraphics[width=\textwidth]{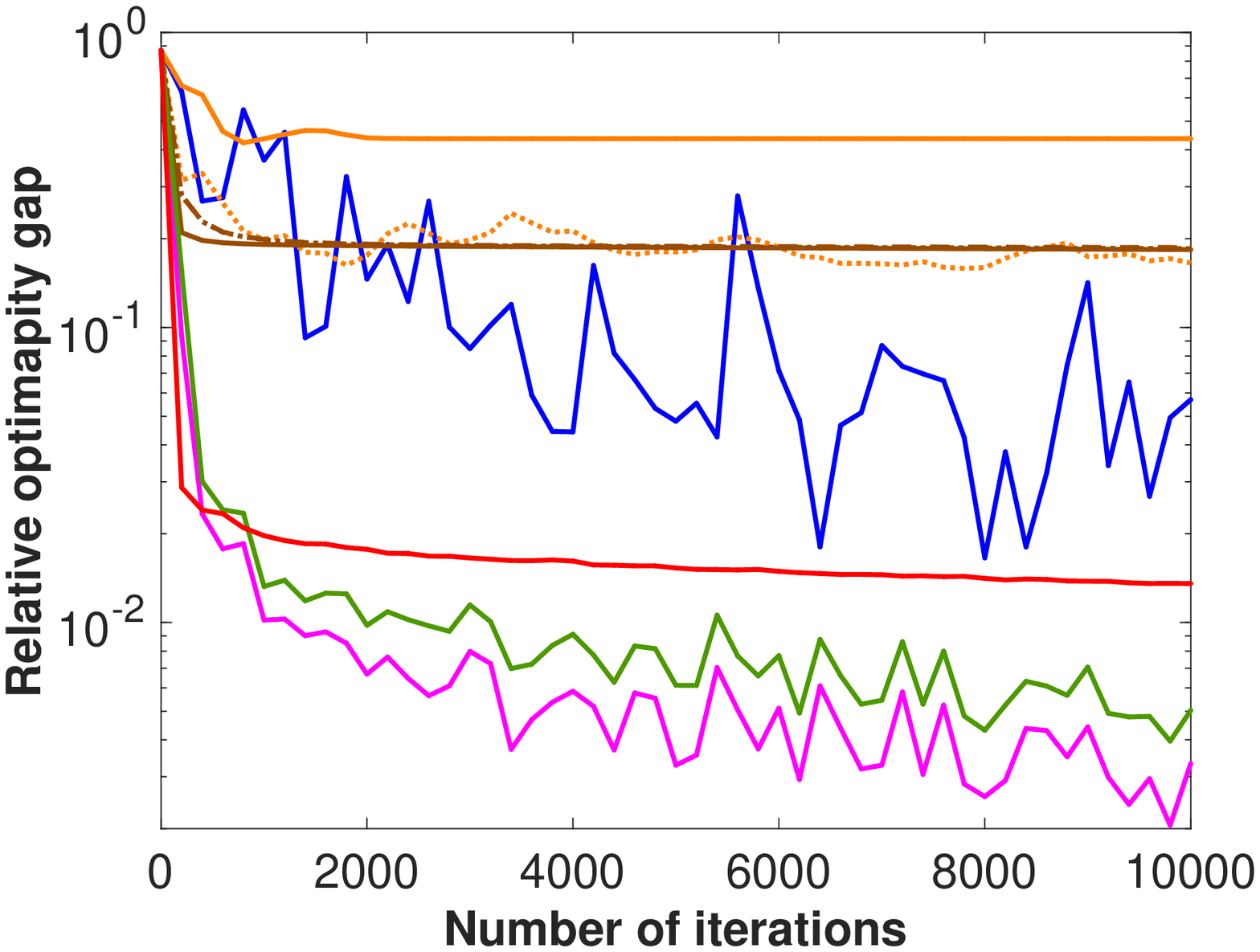}\\
		
		{\small  (g) $\alpha_0=0.01$.}
		
	\end{center} 
	\end{minipage}
	\hspace*{-0.1cm}
	\begin{minipage}[t]{.2\textwidth}
	\begin{center}
		\includegraphics[width=\textwidth]{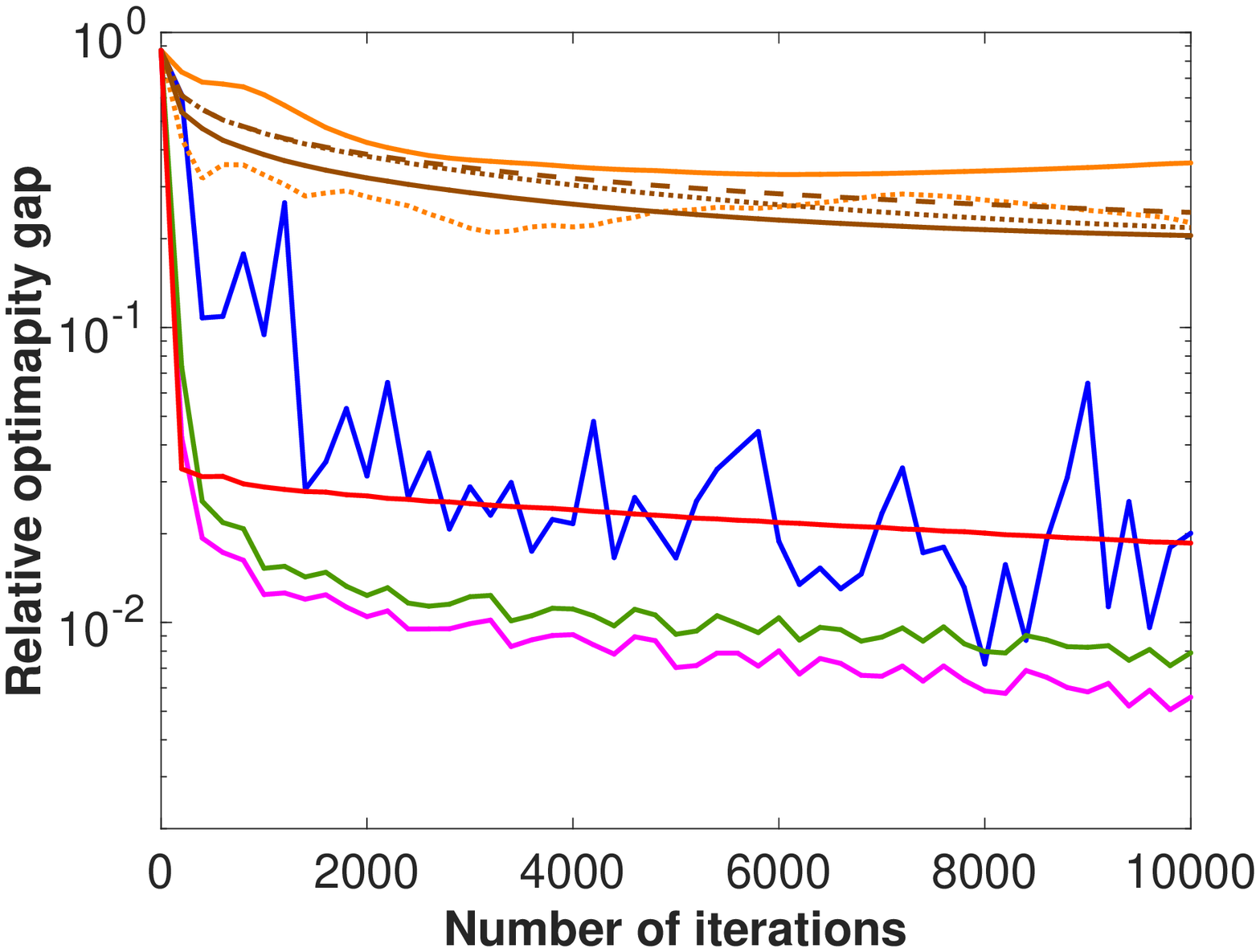}\\
		
		{\small  (h) $\alpha_0=0.005$.}
		
	\end{center} 
	\end{minipage}

	\vspace*{0.2cm}

	\begin{minipage}[t]{.2\textwidth}
	\begin{center}
		\includegraphics[width=\textwidth]{{results/ica/yaleb/ica_comp-yaleb-2015-40-30-0.005}.eps}\\
		
		{\small  (i) $\alpha_0=0.0001$.}
	\end{center}		
	\end{minipage}	
	\hspace*{-0.1cm}
	\begin{minipage}[t]{.2\textwidth}
	\begin{center}
		\includegraphics[width=\textwidth]{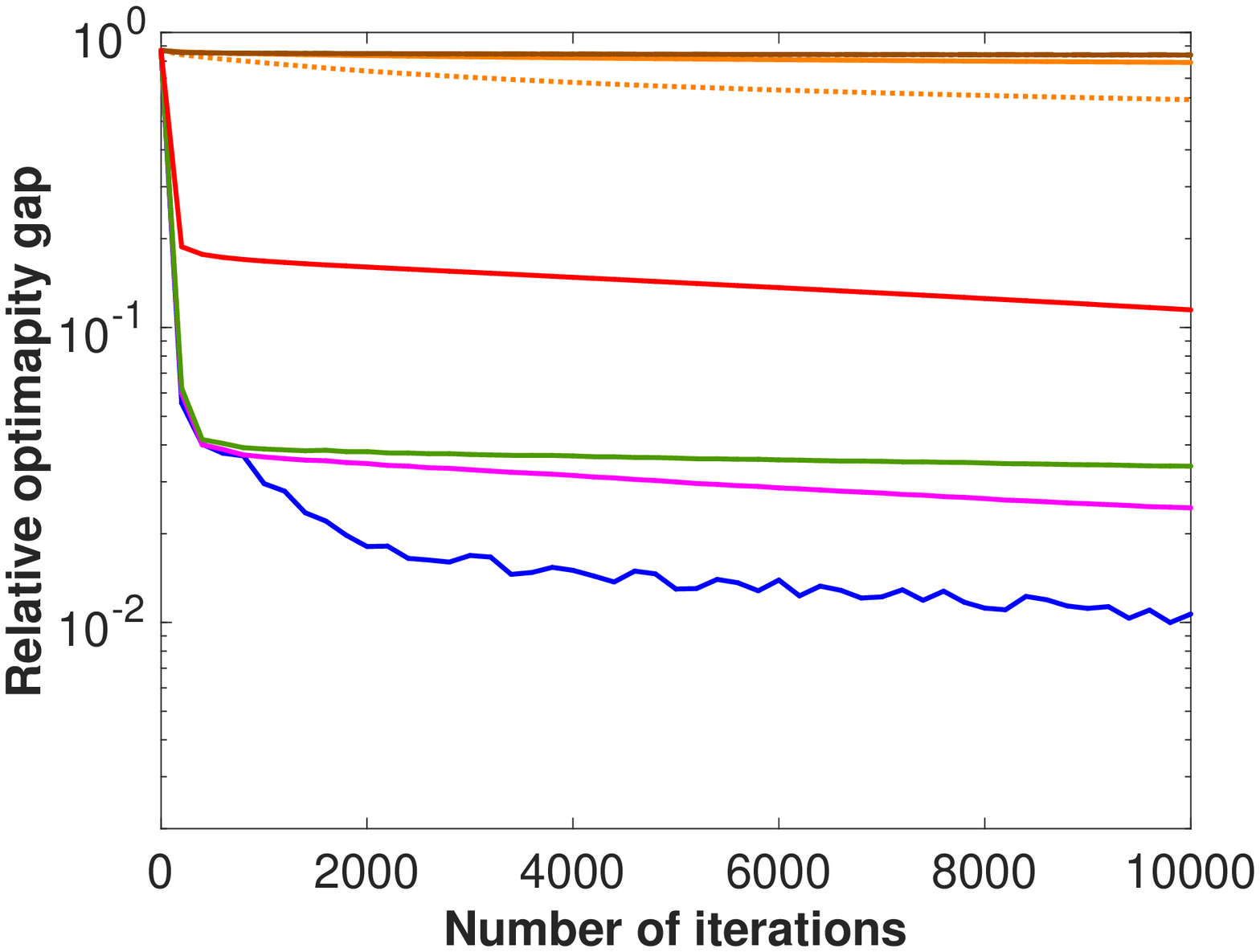}\\
		
		{\small  (j) $\alpha_0=0.0005$.}
		
	\end{center} 
	\end{minipage}
	\hspace*{-0.1cm}
	\begin{minipage}[t]{.2\textwidth}
	\begin{center}
	\end{center} 
	\end{minipage}
	\begin{minipage}[t]{.2\textwidth}
	\begin{center}
	\end{center} 
	\end{minipage}

	\vspace*{0.2cm}

	\caption{{\tt YaleB} dataset on ICA problem ({\bf Case I1}).}

\label{appfig:ICA_results_YaleB}
\end{center}
\end{figure*}

\begin{figure*}[t]
\begin{center}

	\begin{minipage}[t]{.2\textwidth}
	\begin{center}
		\includegraphics[width=\textwidth]{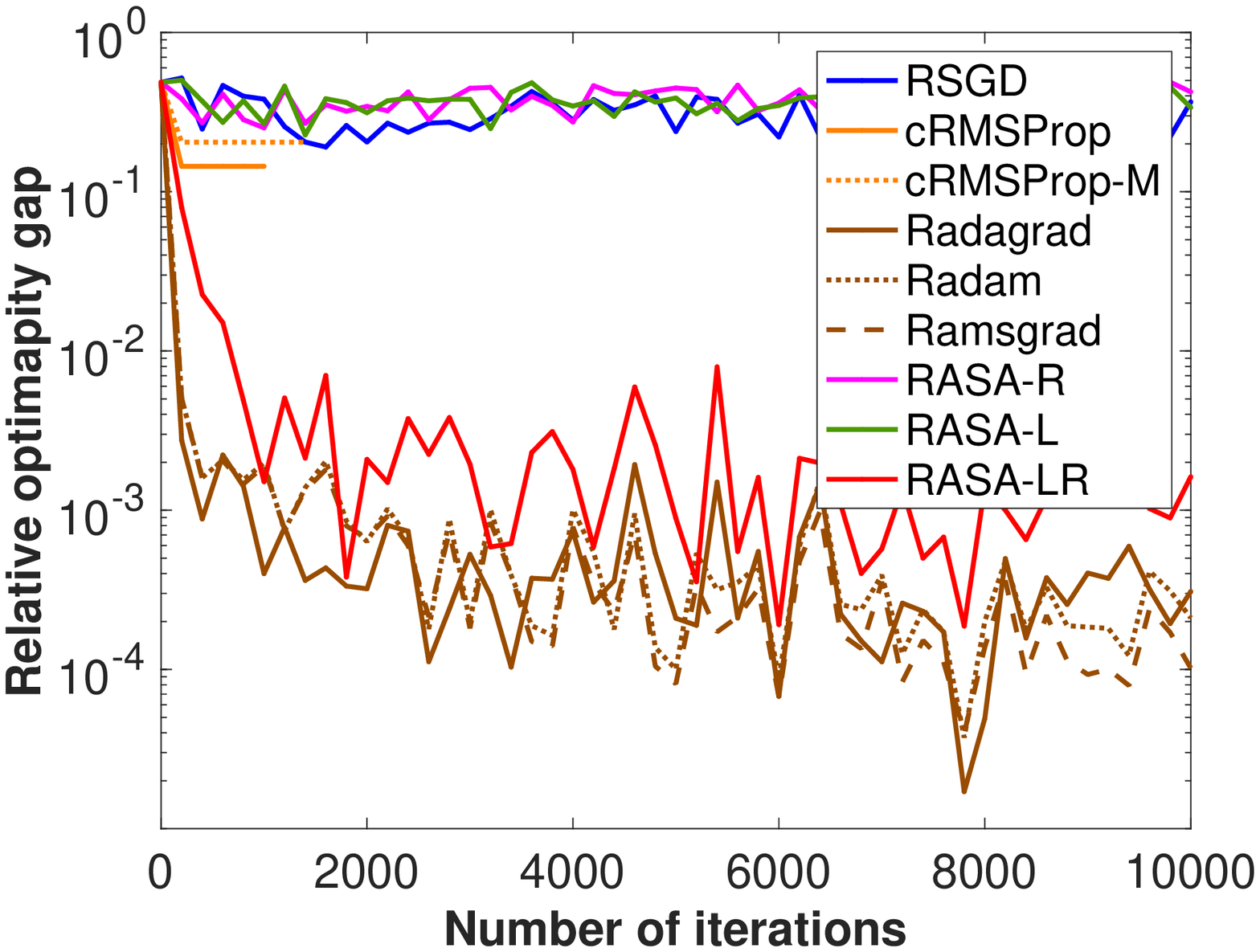}\\
		
		{\small  (a) $\alpha_0=1$.}
	\end{center}		
	\end{minipage}	
	\hspace*{-0.1cm}
	\begin{minipage}[t]{.2\textwidth}
	\begin{center}
		\includegraphics[width=\textwidth]{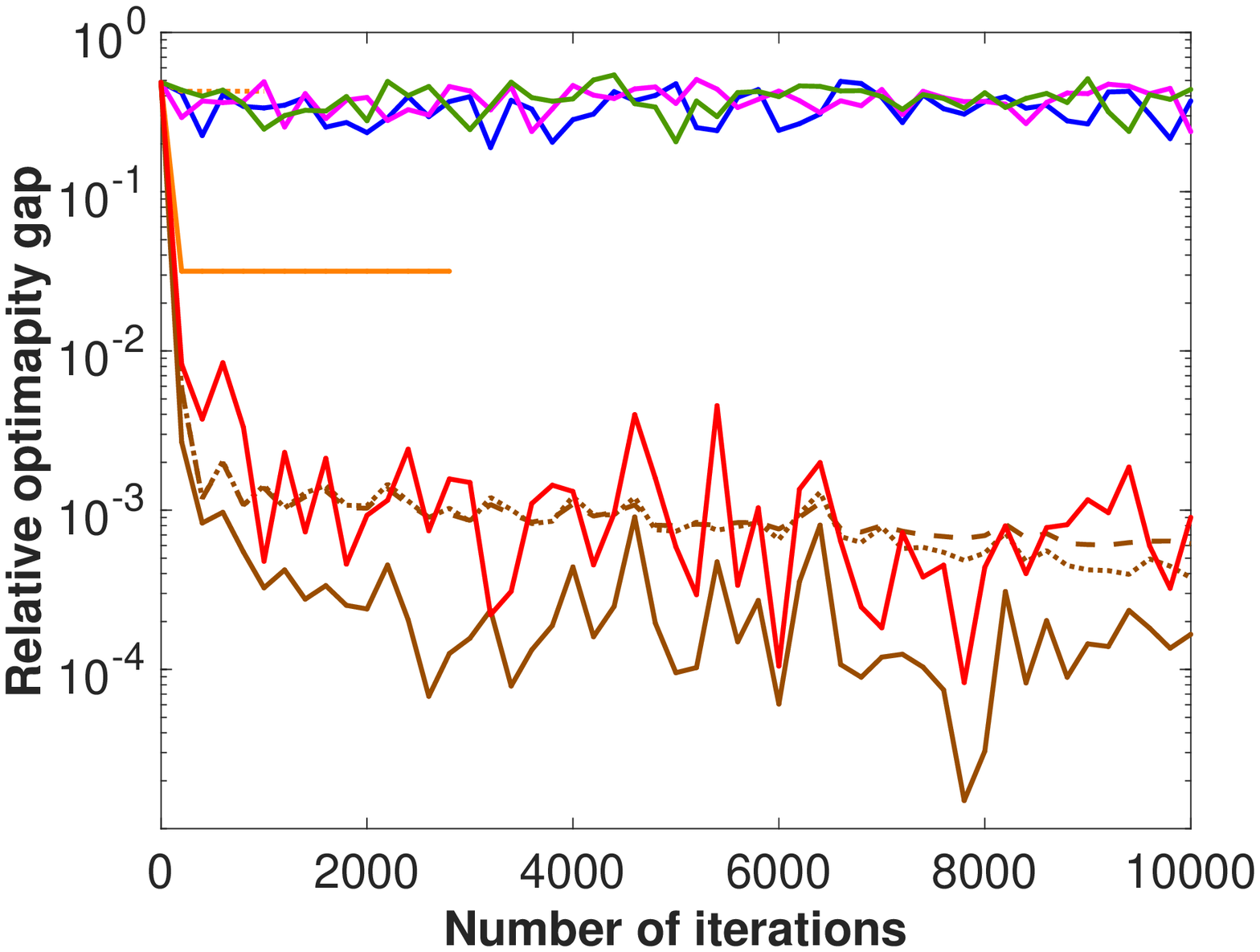}\\
		
		{\small  (b) $\alpha_0=0.5$.}
		
	\end{center} 
	\end{minipage}
	\hspace*{-0.1cm}
	\begin{minipage}[t]{.2\textwidth}
	\begin{center}
		\includegraphics[width=\textwidth]{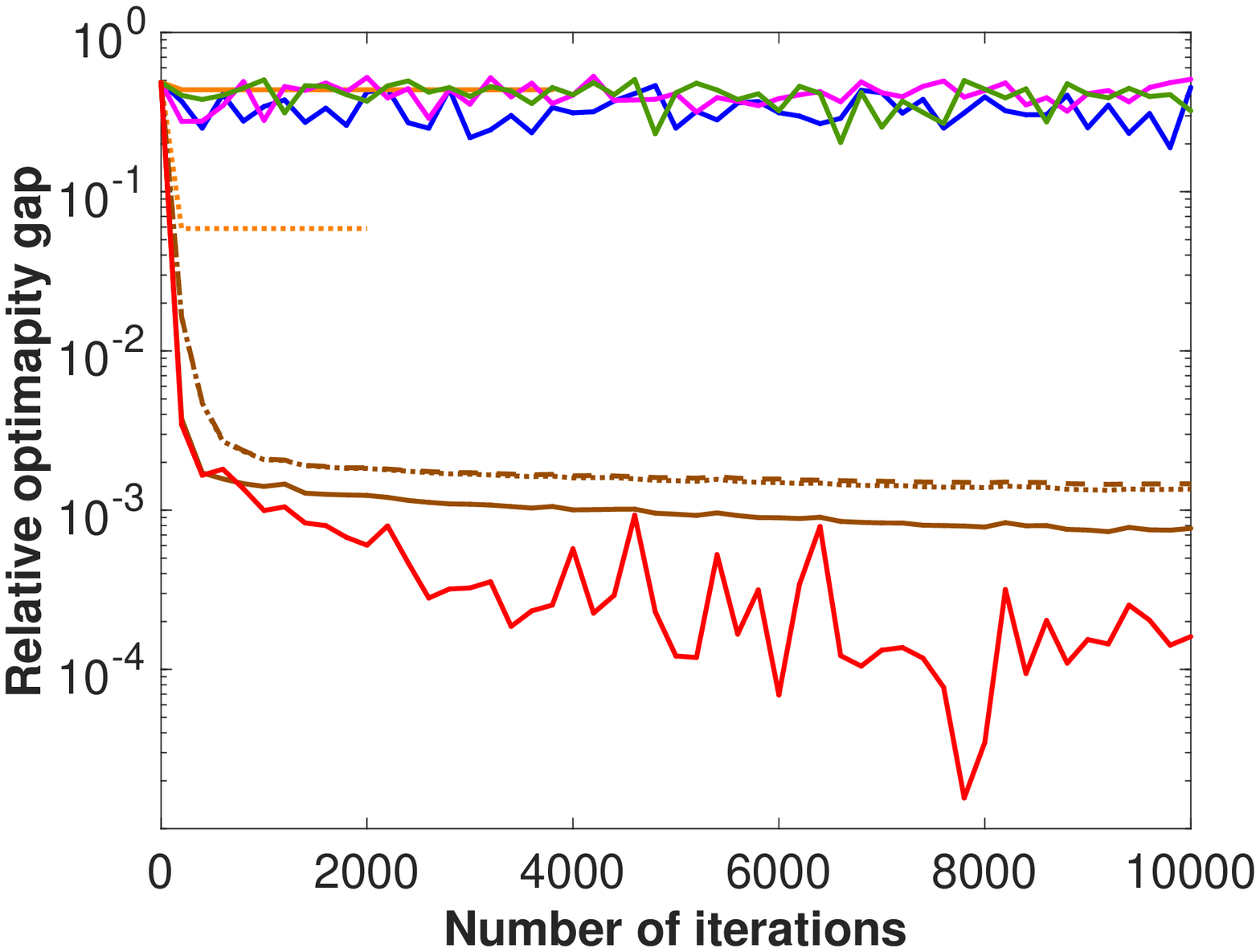}\\
		
		{\small  (c) $\alpha_0=0.1$.}
		
	\end{center} 
	\end{minipage}
	\hspace*{-0.1cm}
	\begin{minipage}[t]{.2\textwidth}
	\begin{center}
		\includegraphics[width=\textwidth]{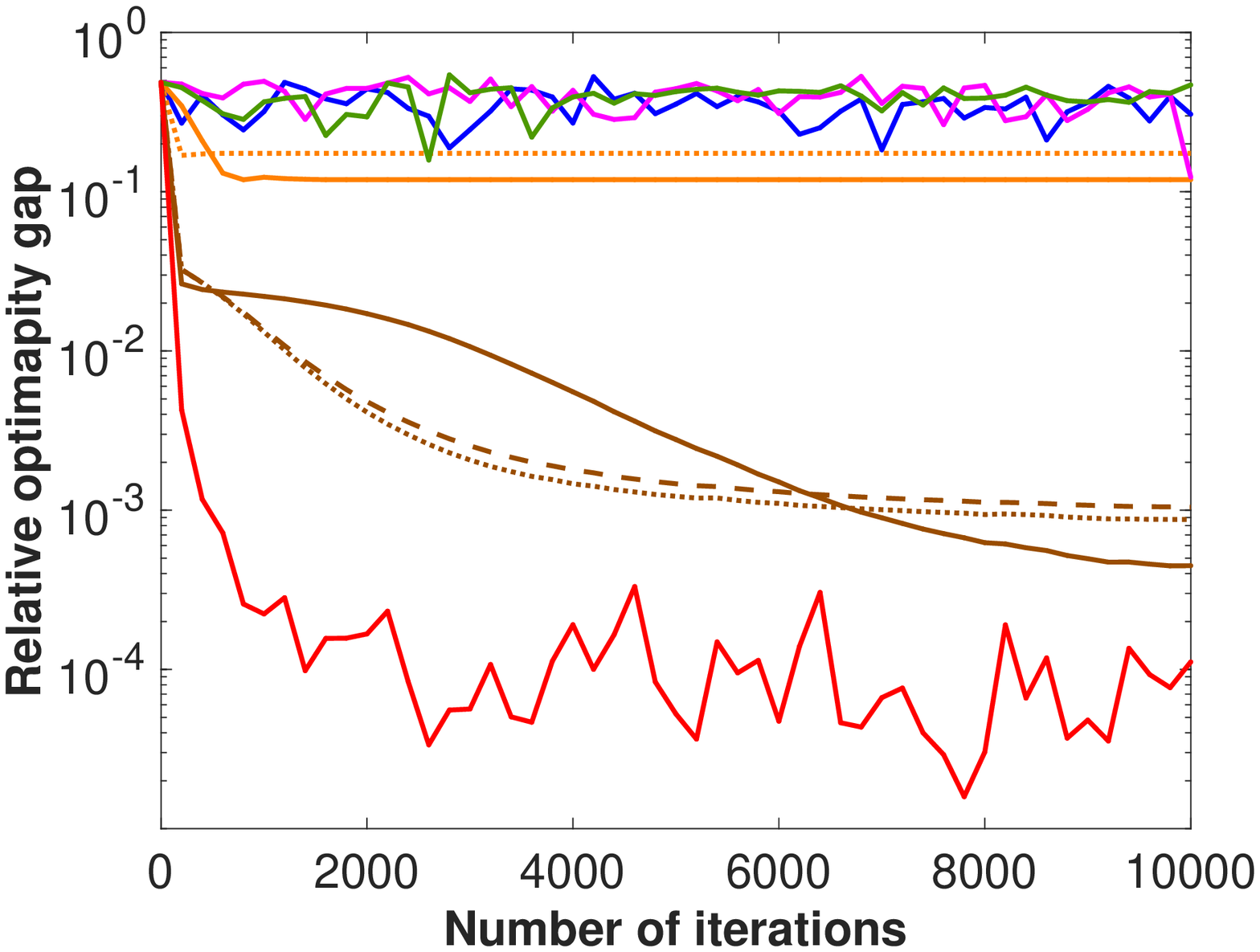}\\
		
		{\small  (d) $\alpha_0=0.05$.}
		
	\end{center} 
	\end{minipage}

	\vspace*{0.2cm}

	\begin{minipage}[t]{.2\textwidth}
	\begin{center}
		\includegraphics[width=\textwidth]{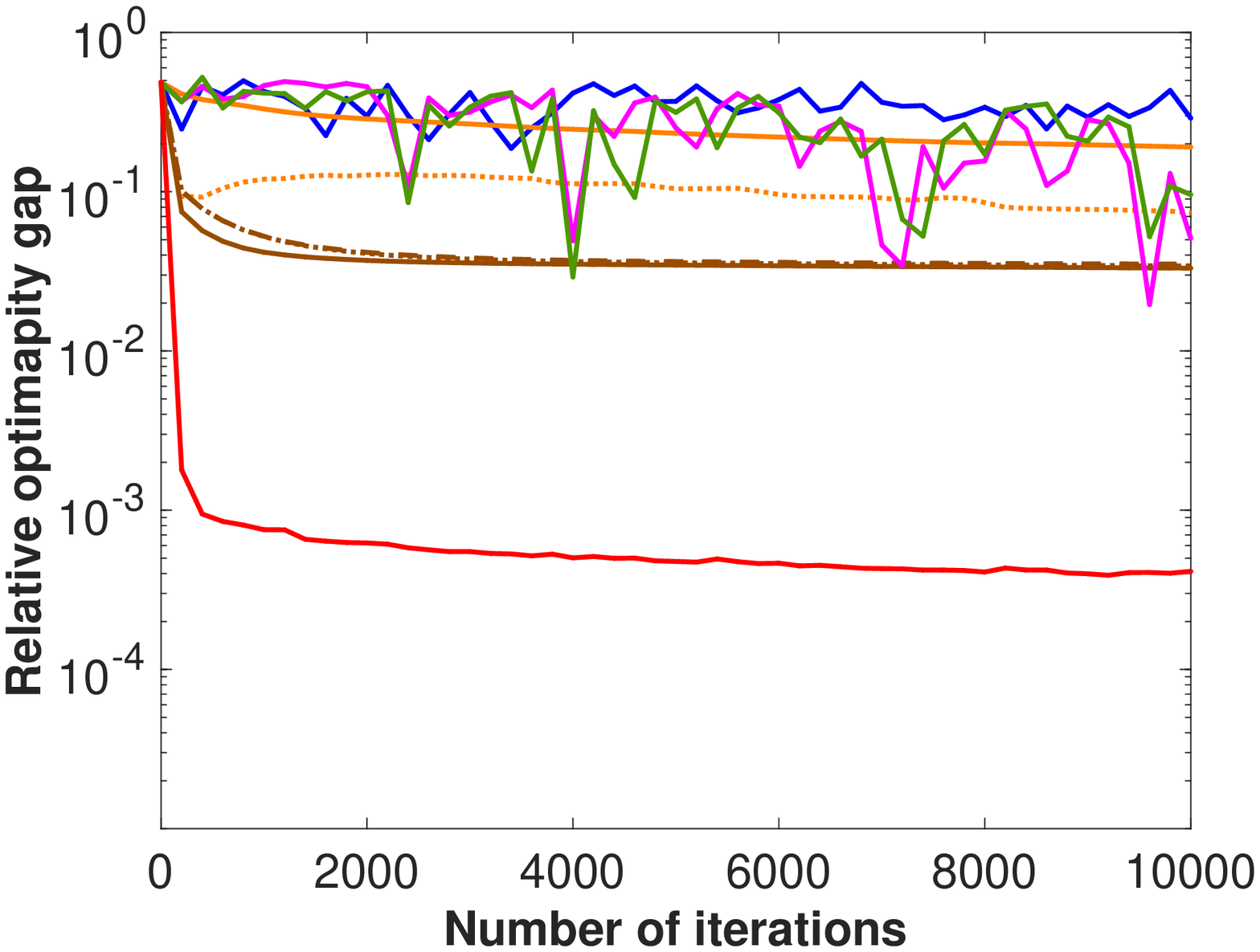}\\
		
		{\small  (e) $\alpha_0=0.01$.}
	\end{center}		
	\end{minipage}	
	\hspace*{-0.1cm}
	\begin{minipage}[t]{.2\textwidth}
	\begin{center}
		\includegraphics[width=\textwidth]{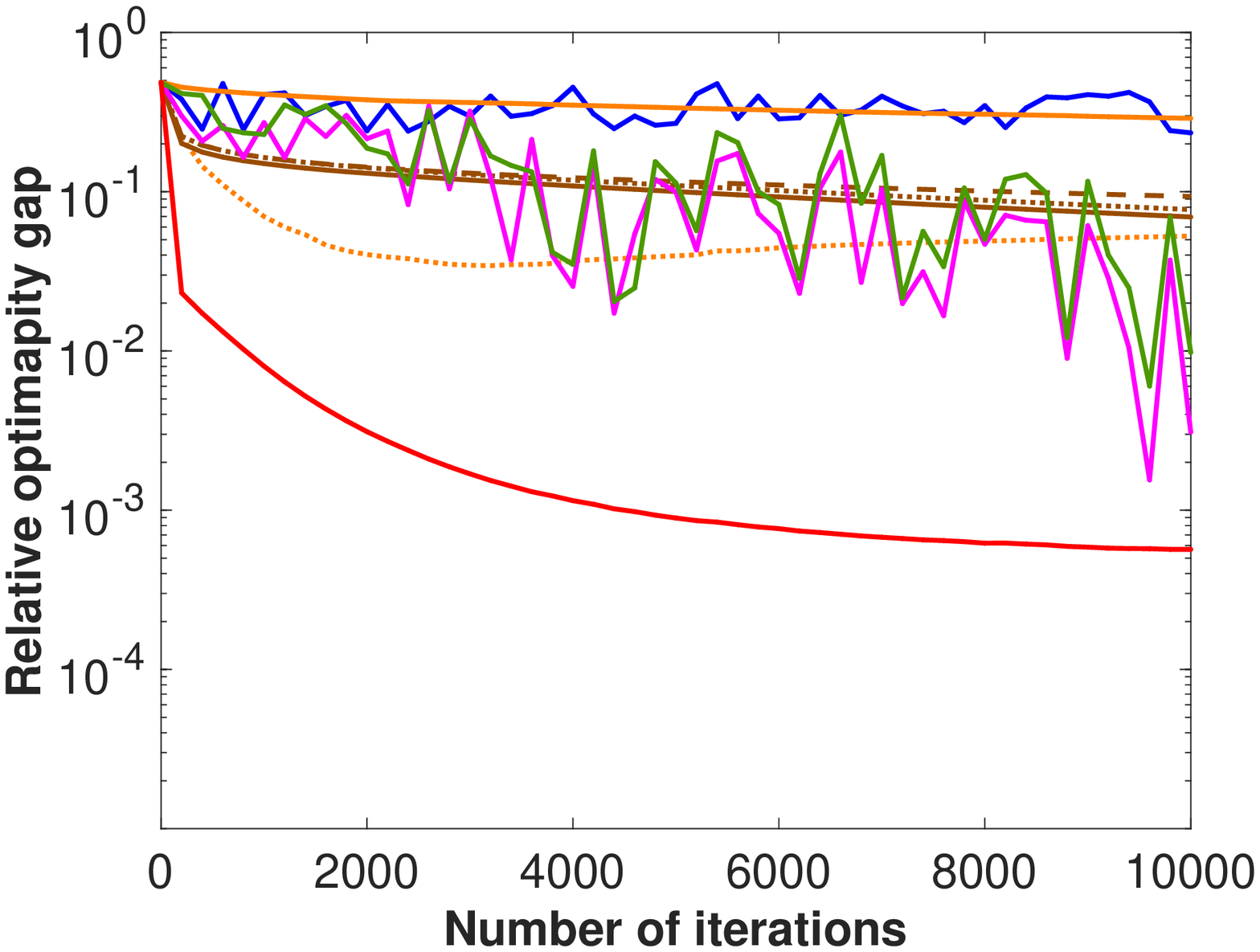}\\
		
		{\small  (f) $\alpha_0=0.005$.}
		
	\end{center} 
	\end{minipage}
	\hspace*{-0.1cm}
	\begin{minipage}[t]{.2\textwidth}
	\begin{center}
	\includegraphics[width=\textwidth]{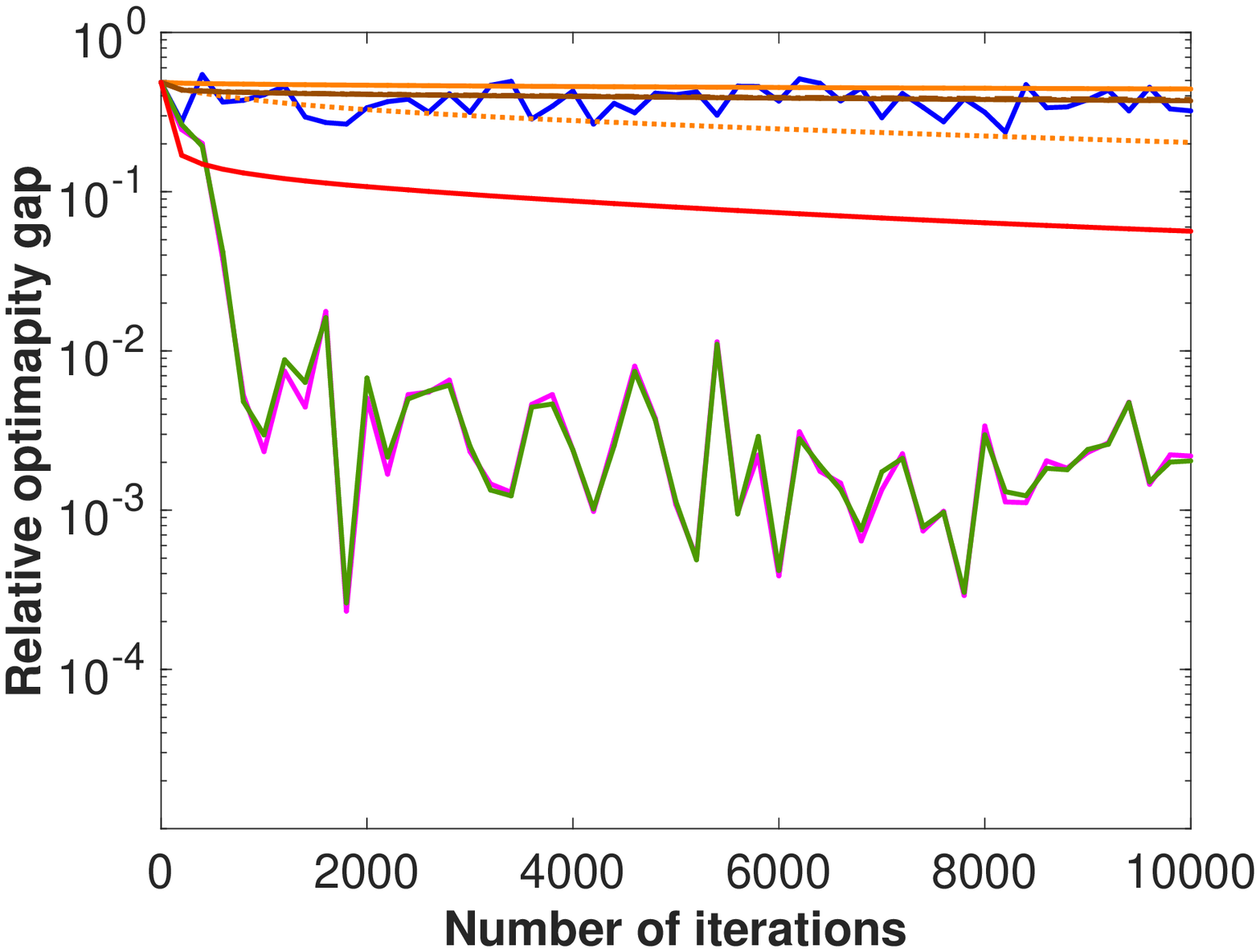}\\
		
		{\small  (g) $\alpha_0=0.001$.}
		
	\end{center} 
	\end{minipage}
	\hspace*{-0.1cm}
	\begin{minipage}[t]{.2\textwidth}
	\begin{center}
		\includegraphics[width=\textwidth]{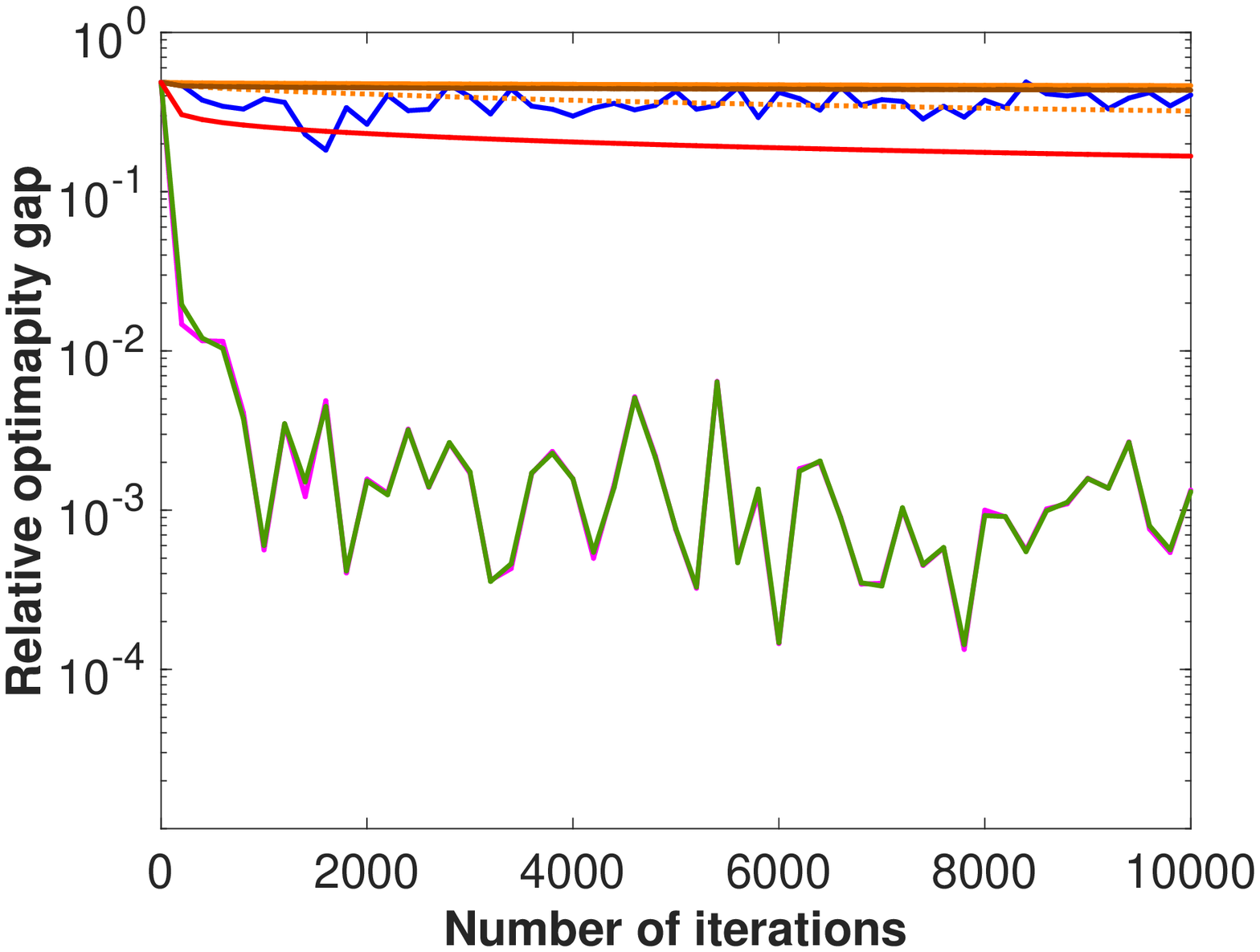}\\
		
		{\small  (h) $\alpha_0=0.0005$.}
		
	\end{center} 
	\end{minipage}

	\vspace*{0.2cm}

	\begin{minipage}[t]{.2\textwidth}
	\begin{center}
		\includegraphics[width=\textwidth]{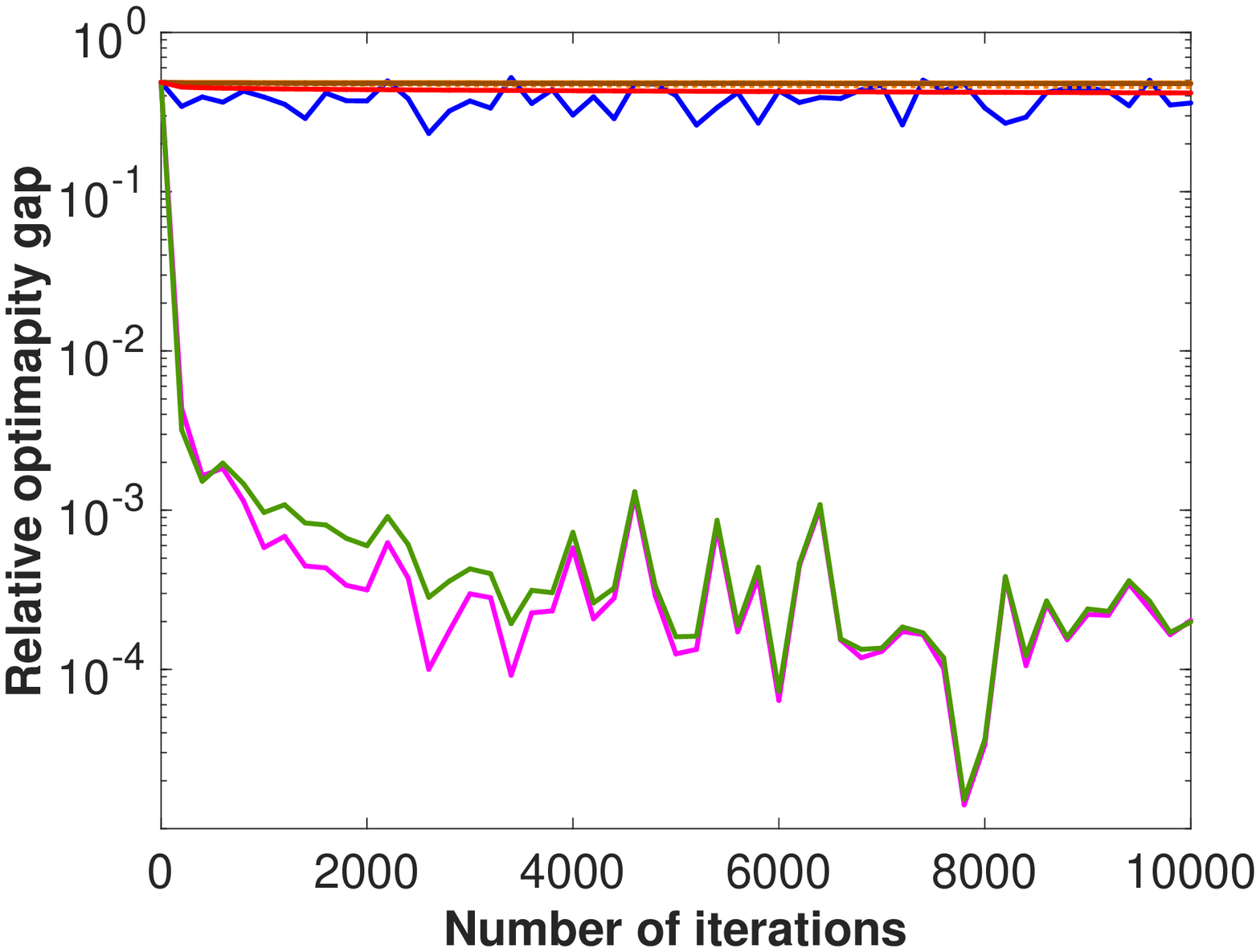}\\
		
		{\small  (i) $\alpha_0=0.0001$.}
	\end{center}		
	\end{minipage}	
	\hspace*{-0.1cm}
	\begin{minipage}[t]{.2\textwidth}
	\begin{center}
	\end{center} 
	\end{minipage}
	\hspace*{-0.1cm}
	\begin{minipage}[t]{.2\textwidth}
	\begin{center}
	\end{center} 
	\end{minipage}
	\hspace*{-0.1cm}
	\begin{minipage}[t]{.2\textwidth}
	\begin{center}
	\end{center} 
	\end{minipage}

	\vspace*{0.2cm}

	\caption{{\tt COIL100} dataset for the ICA problem ({\bf Case I2}).}

\label{appfig:ICA_results_COIL100}
\end{center}
\end{figure*}

\subsection{Results on the real-world datasets for the matrix completion (MC) problem}

This section shows the results of the MC problems on real-world datasets. 

Figures \ref{appfig:MC_results_movielens}(a) and (b) show the train and test root MSEs on MovieLens-1M dataset and MovieLens-10M dataset, respectively, under best-tuned step sizes.  The comparison algorithms are RSGD, Radagrad, Radam, Ramsgrad and RASA-LR. From the figures, we see that RASA-LR yields best performances on all settings.

\begin{figure*}[htbp]
\begin{center}
	\begin{minipage}[t]{.2\textwidth}
	\begin{center}
		\includegraphics[width=\textwidth]{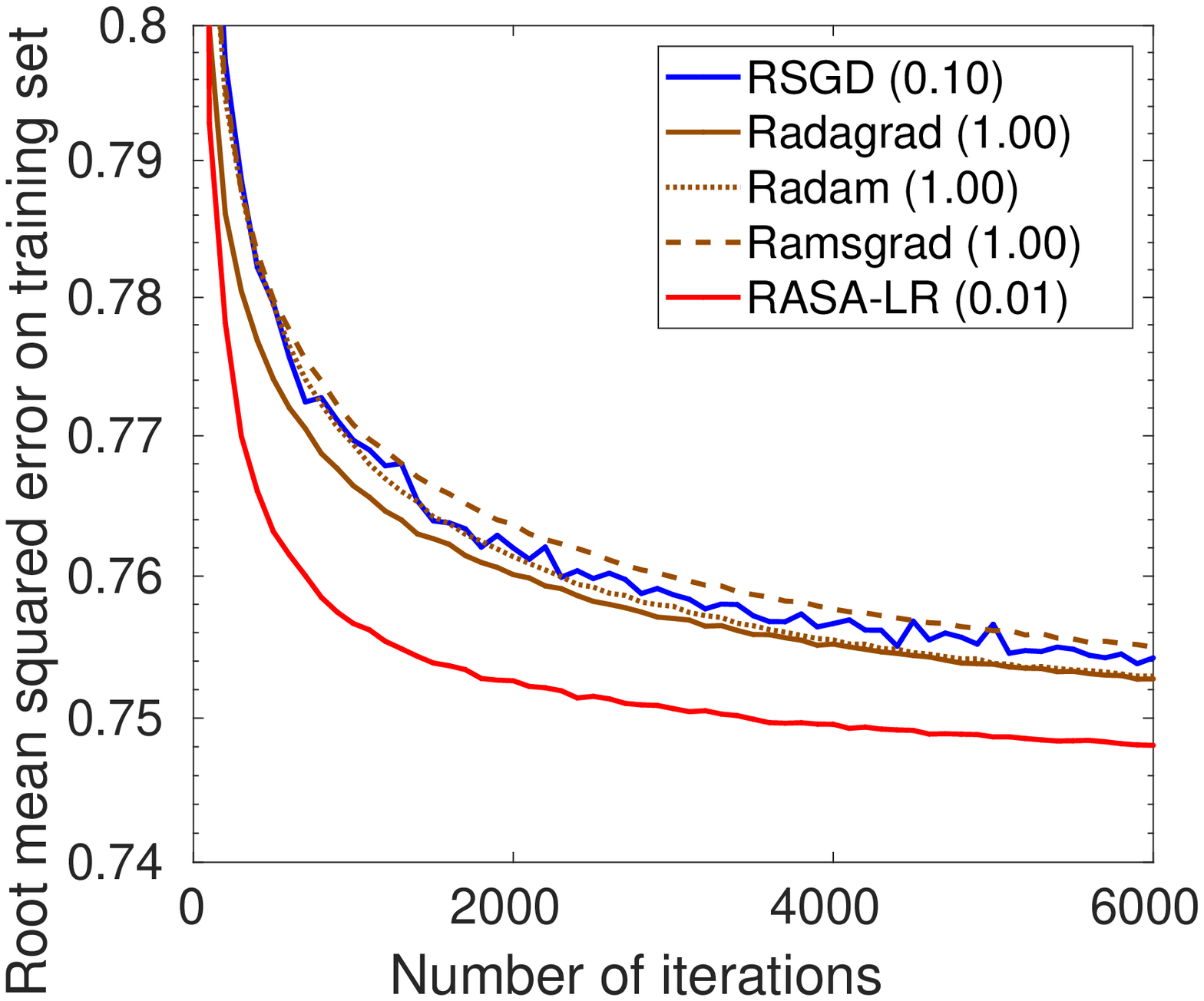}\\
		
		{\small  (a-1) Train root MSE.}
	\end{center}		
	\end{minipage}	
	\hspace*{-0.1cm}
	\begin{minipage}[t]{.2\textwidth}
	\begin{center}
		\includegraphics[width=\textwidth]{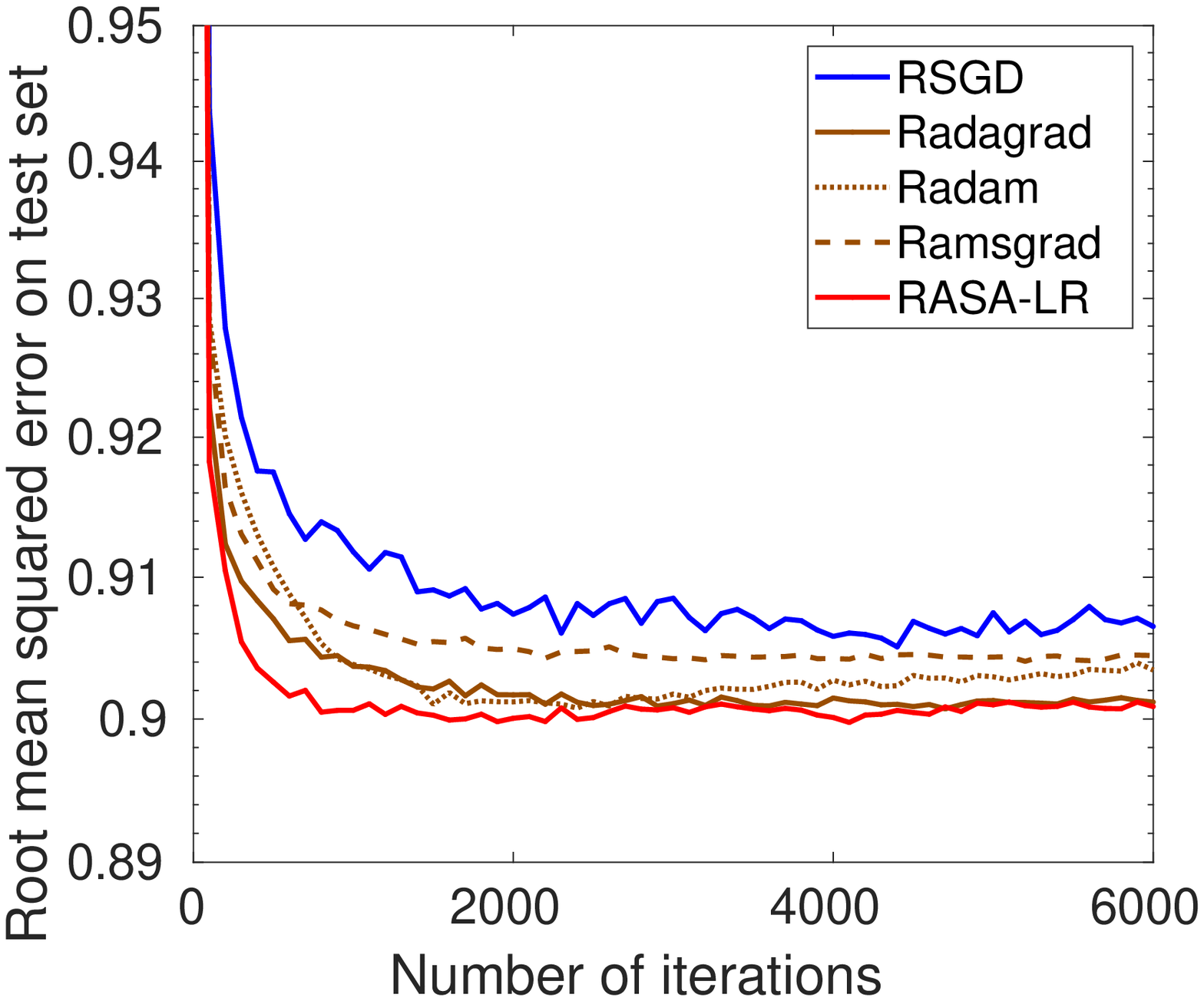}\\
		
		{\small  (a-2) Test root MSE.}
		
	\end{center} 
	\end{minipage}
%
%
%
%
	\hspace*{-0.1cm}
	\begin{minipage}[t]{.2\textwidth}
	\begin{center}
		\includegraphics[width=\textwidth]{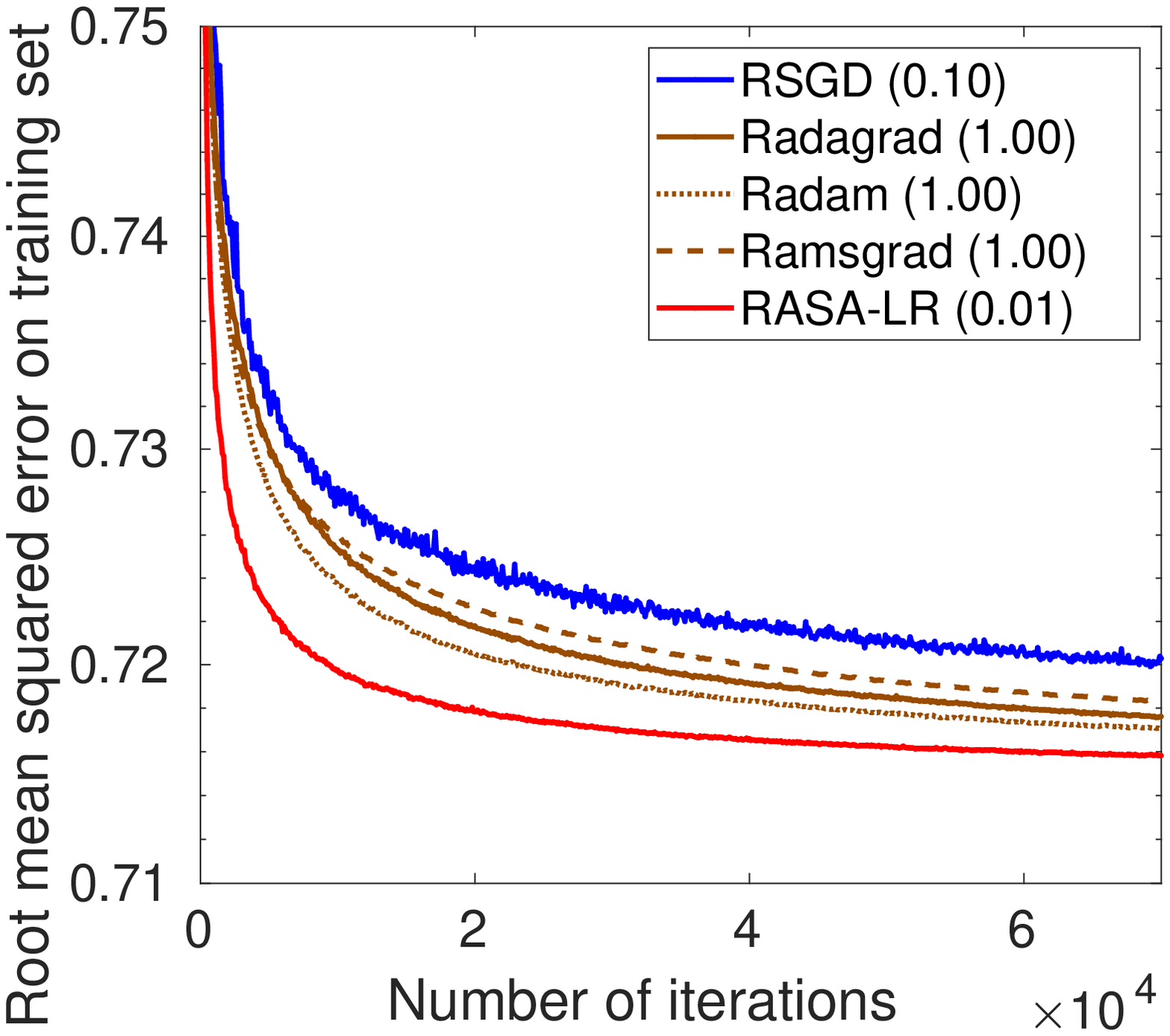}\\
		
		{\small  (b-1) Train root MSE.}
	\end{center}		
	\end{minipage}	
	\hspace*{-0.1cm}
	\begin{minipage}[t]{.2\textwidth}
	\begin{center}
		\includegraphics[width=\textwidth]{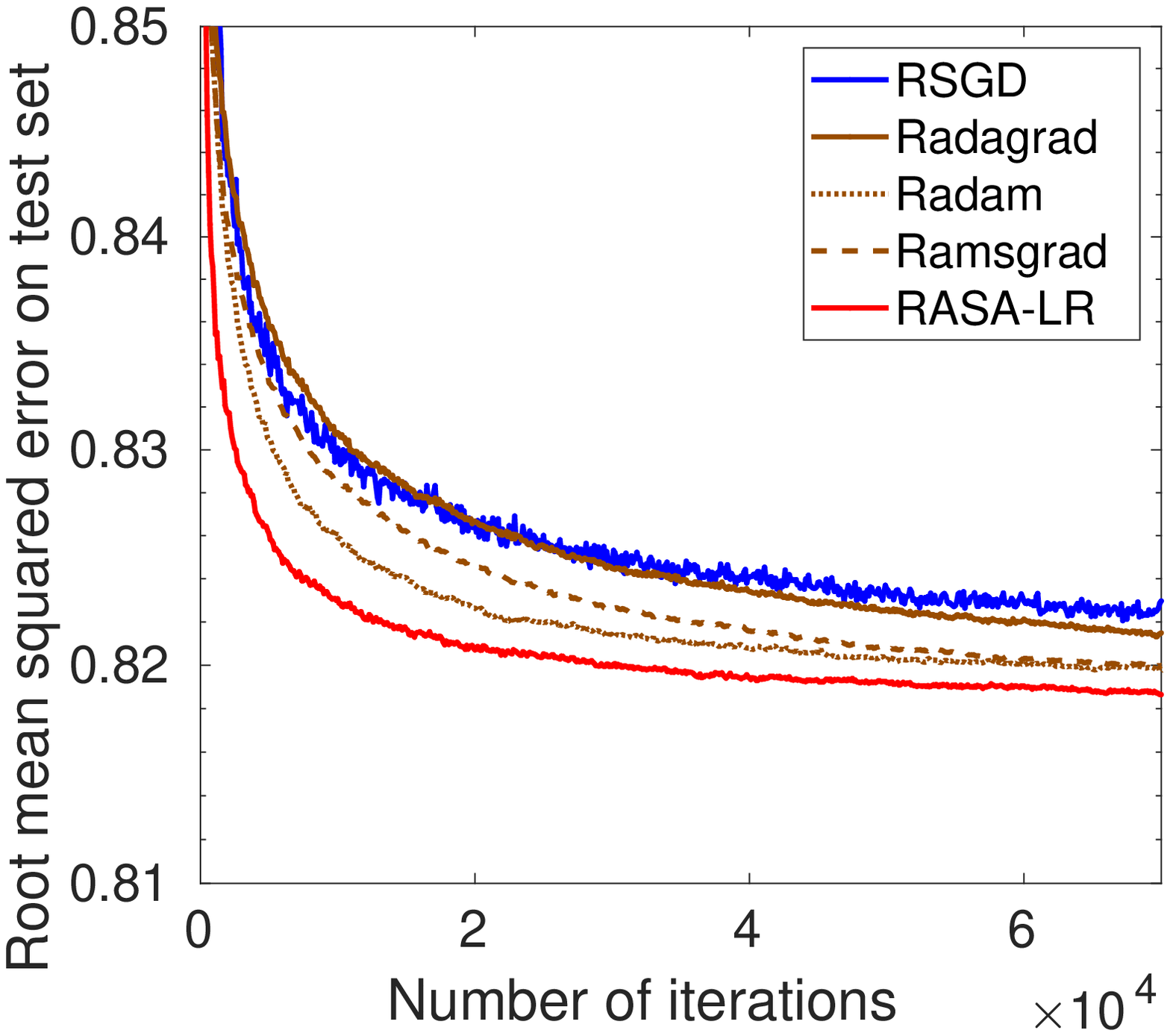}\\
		
		{\small  (b-2) Test root MSE.}
		
	\end{center} 
	\end{minipage}
%
%
%

	\vspace{0.4cm}	
	
	\begin{minipage}[t]{.49\textwidth}
	\begin{center}
	{\small \bf (a) MovieLens-1M dataset.}
	\end{center} 	
	\end{minipage}
	\begin{minipage}[t]{.49\textwidth}
	\begin{center}
	{\small \bf (b) MovieLens-10M dataset.}
	\end{center} 		
	\end{minipage}
	
	\caption{Best-tuned results of MovieLens dataset on MC problem.}

\label{appfig:MC_results_movielens}
\end{center}
\end{figure*}

Finally, Figures \ref{appfig:MC_results_jester} (a) and (b) show the MSEs of the training set and the test set on the {\tt Jester} dataset \citep{Goldberg_IR_2001_s}, respectively. The dataset consists of ratings of $100$ jokes by $24983$ users. Each rating is a real number between $-10$ and $10$. The step sizes values are $\alpha_0=\{1, \changeHK{0.1, 0.01, 0.001}\}$. We also show the best-tuned results in Figures \ref{appfig:MC_results_jester_best}. From the figures, RASA-LR yields slightly better performance than others on the training MSE. All the algorithms converge to the same test MSE.

\begin{figure*}[htbp]
\begin{center}
	\begin{minipage}[t]{.2\textwidth}
	\begin{center}
		\includegraphics[width=\textwidth]{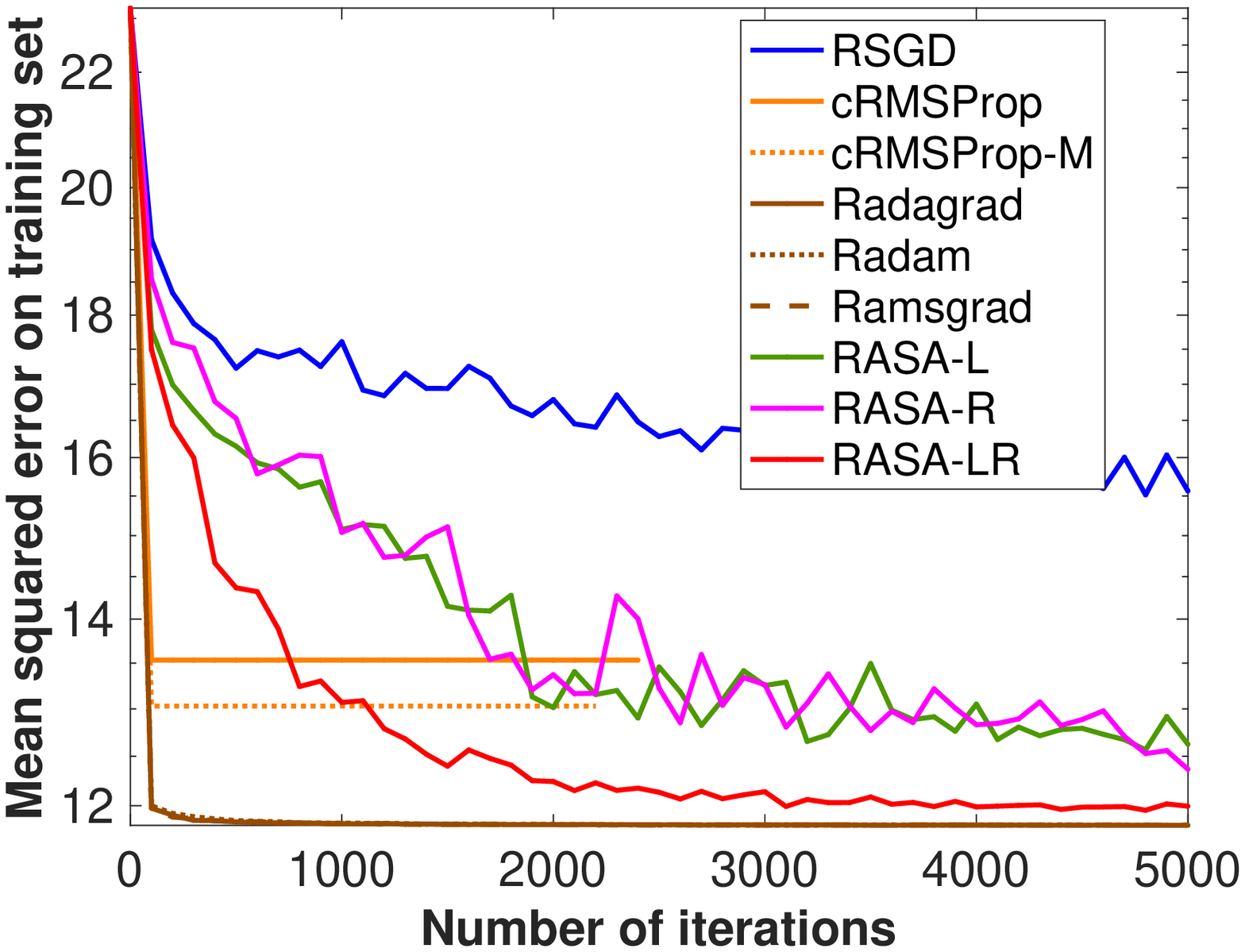}\\
		
		{\small  (a-1) $\alpha_0=1$.}
	\end{center}		
	\end{minipage}	
	\hspace*{-0.1cm}
	\begin{minipage}[t]{.2\textwidth}
	\begin{center}
		\includegraphics[width=\textwidth]{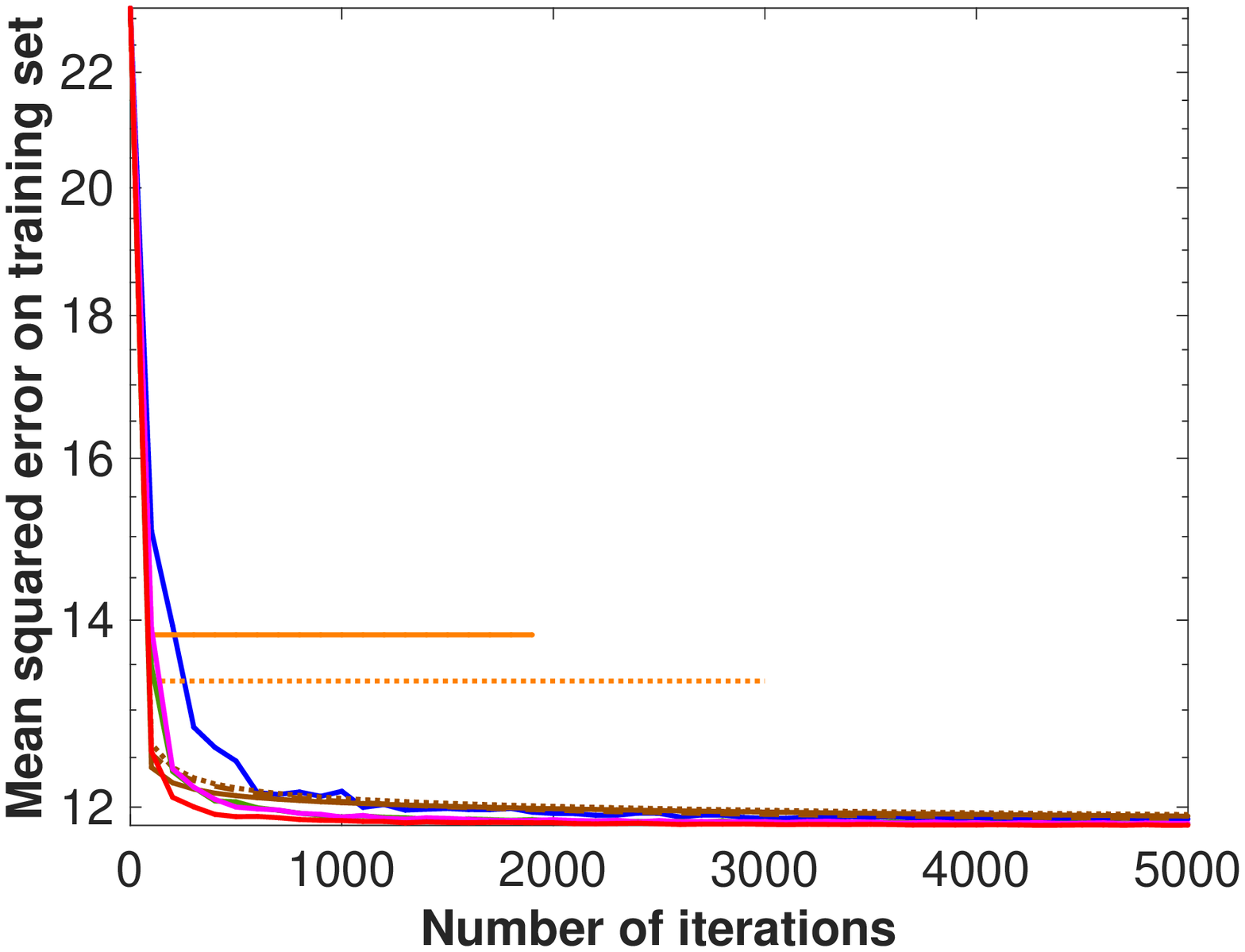}\\
		
		{\small  (a-2) $\alpha_0=0.1$.}
		
	\end{center} 
	\end{minipage}
	\hspace*{-0.1cm}
	\begin{minipage}[t]{.2\textwidth}
	\begin{center}
		\includegraphics[width=\textwidth]{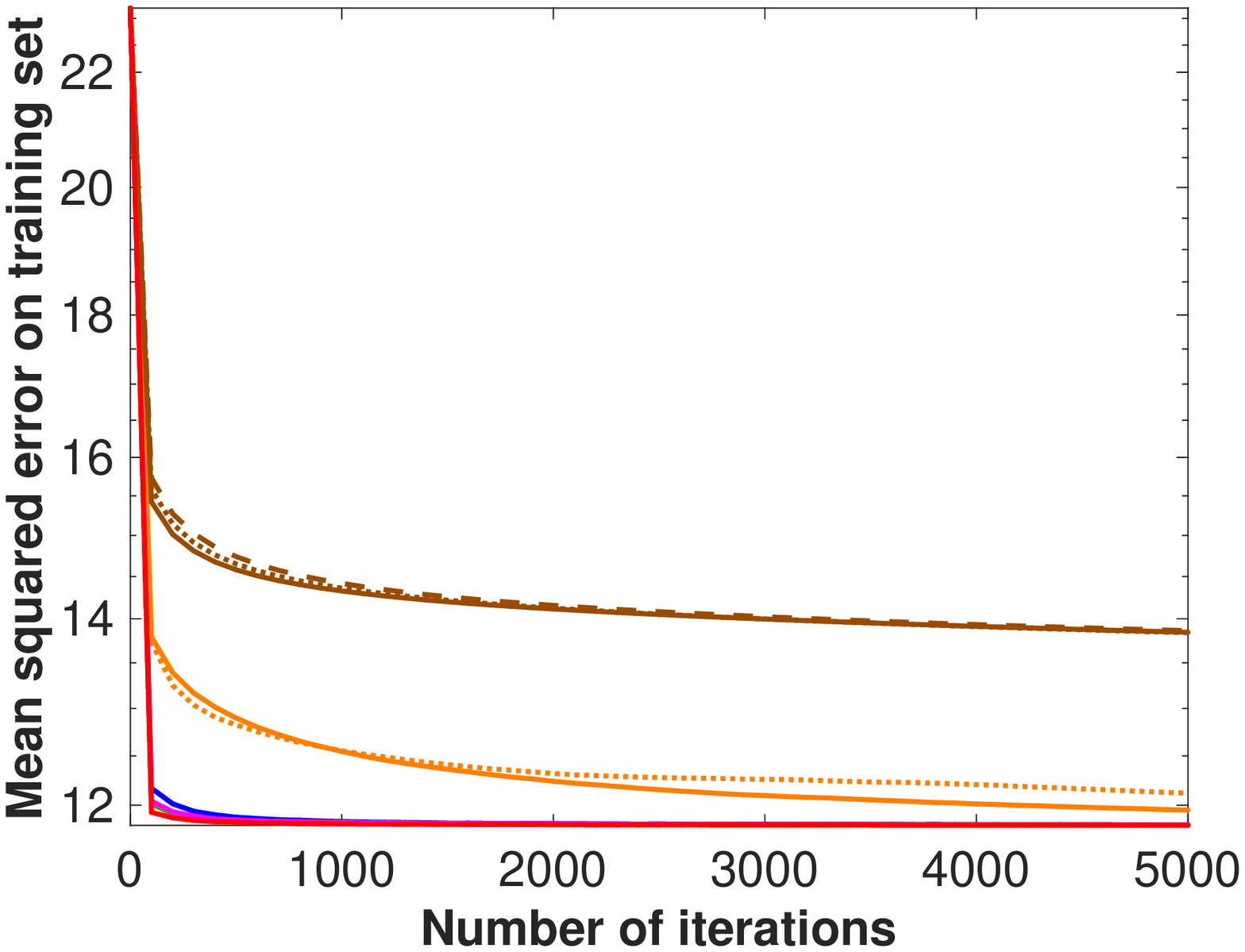}\\
		
		{\small  (a-3) $\alpha_0=0.01$.}
		
	\end{center} 
	\end{minipage}
	\hspace*{-0.1cm}
	\begin{minipage}[t]{.2\textwidth}
	\begin{center}
		\includegraphics[width=\textwidth]{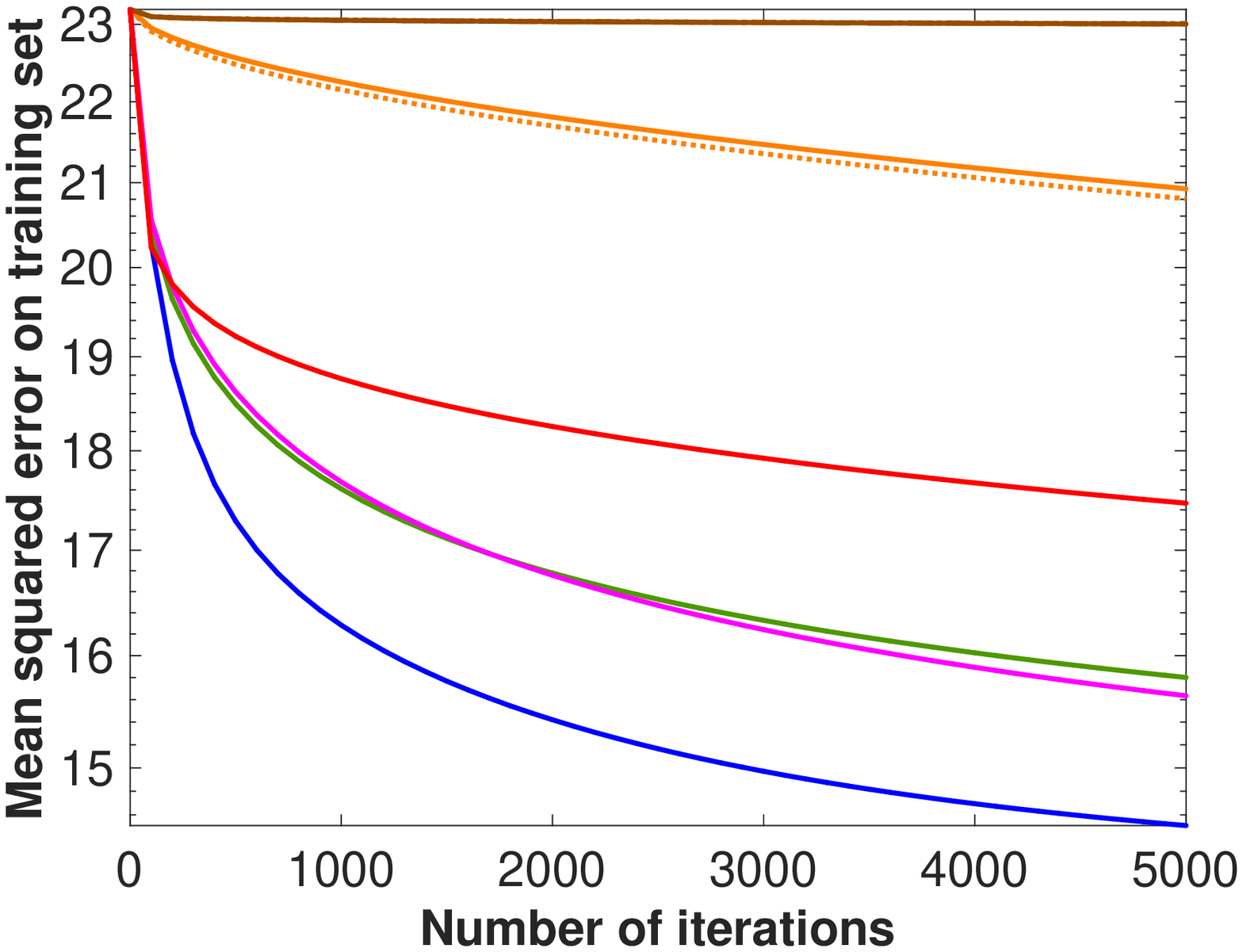}\\
		
		{\small  (a-4) $\alpha_0=0.0001$.}
		
	\end{center} 
	\end{minipage}	
	
	\vspace{0.3cm}

	{\small\bf (a) MSE on training set.}

	\vspace{0.5cm}

	\begin{minipage}[t]{.2\textwidth}
	\begin{center}
		\includegraphics[width=\textwidth]{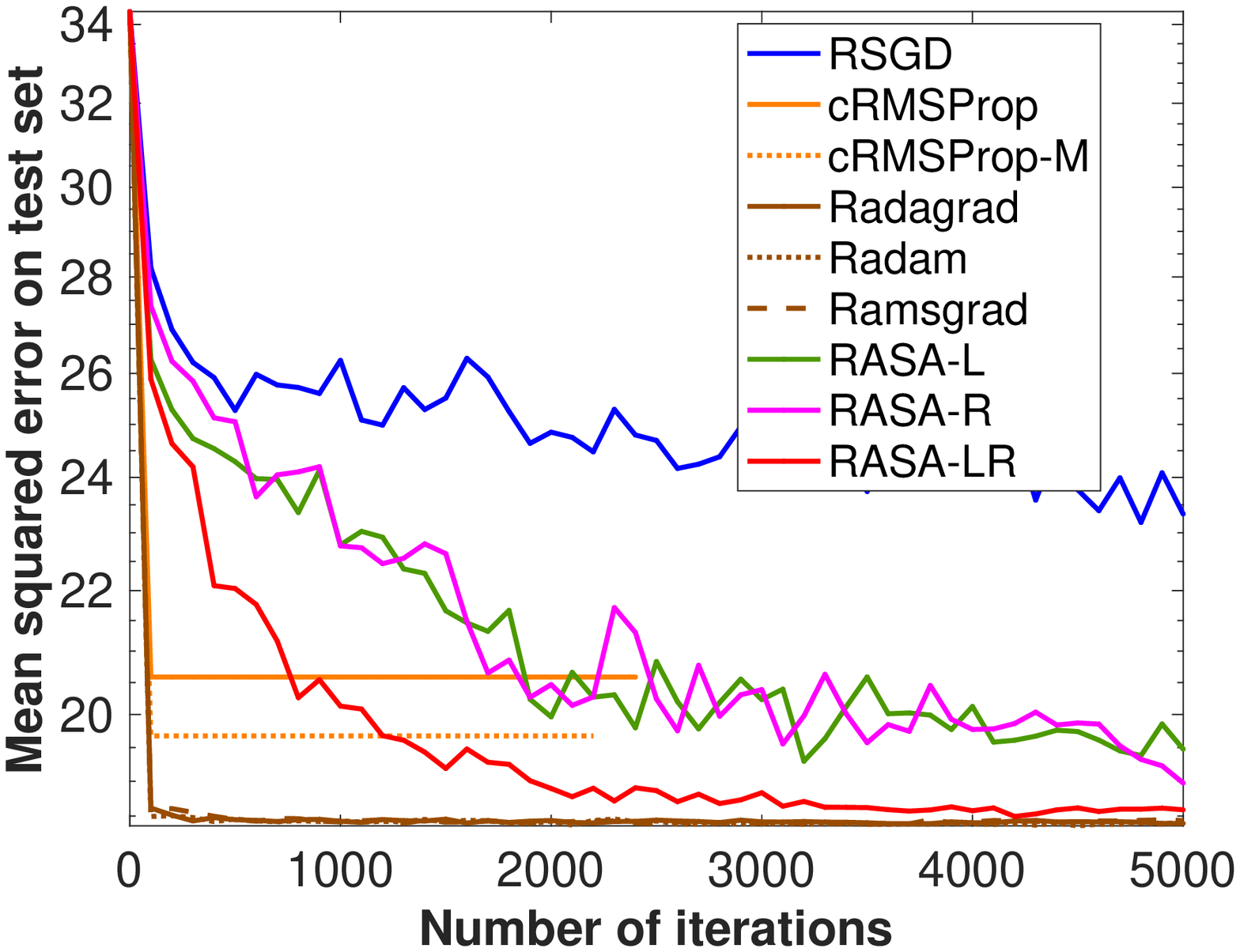}\\
		
		{\small  (b-1) $\alpha_0=1$.}
	\end{center}		
	\end{minipage}	
	\hspace*{-0.1cm}
	\begin{minipage}[t]{.2\textwidth}
	\begin{center}
		\includegraphics[width=\textwidth]{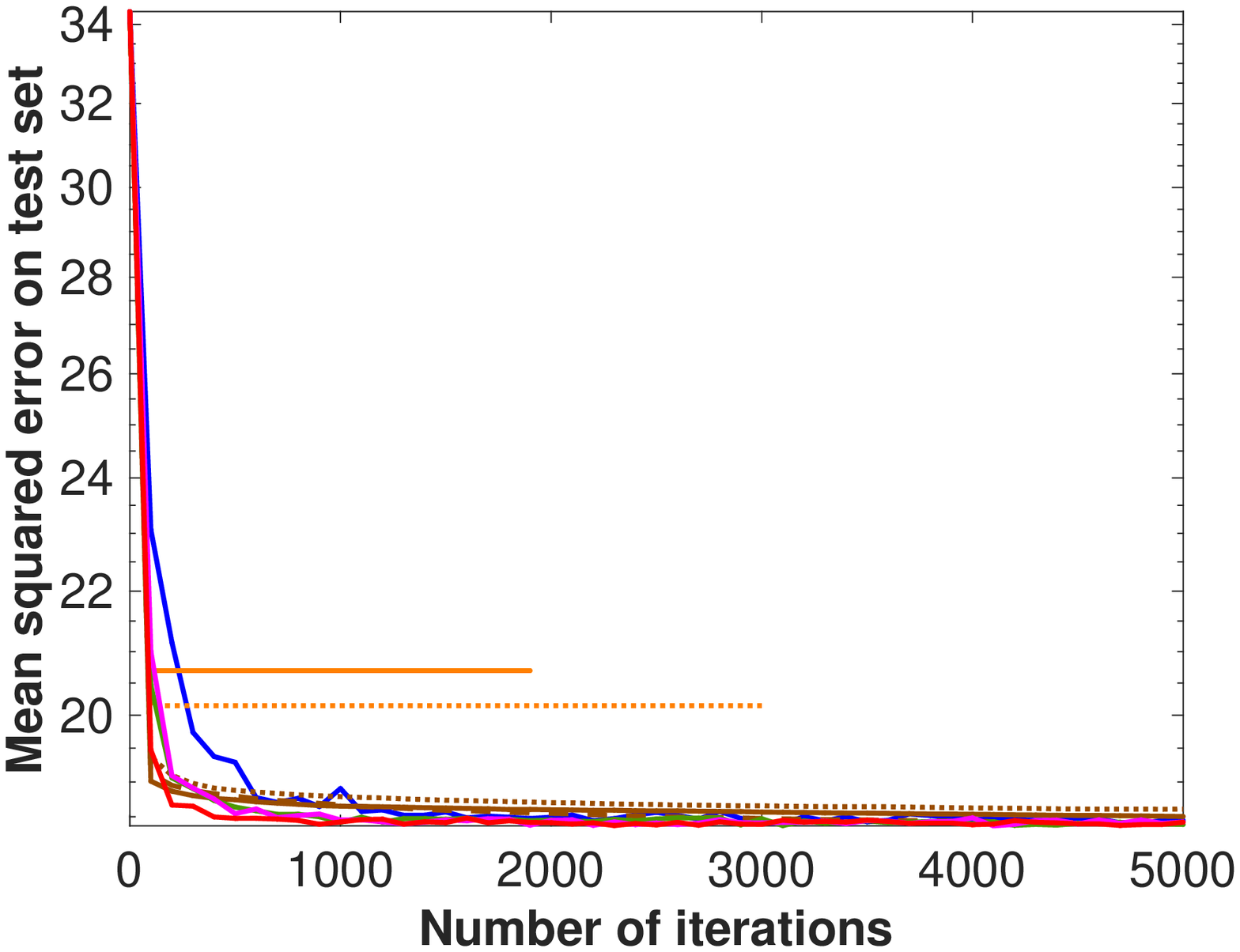}\\
		
		{\small  (b-2) $\alpha_0=0.1$.}
		
	\end{center} 
	\end{minipage}]
	\hspace*{-0.1cm}
	\begin{minipage}[t]{.2\textwidth}
	\begin{center}
		\includegraphics[width=\textwidth]{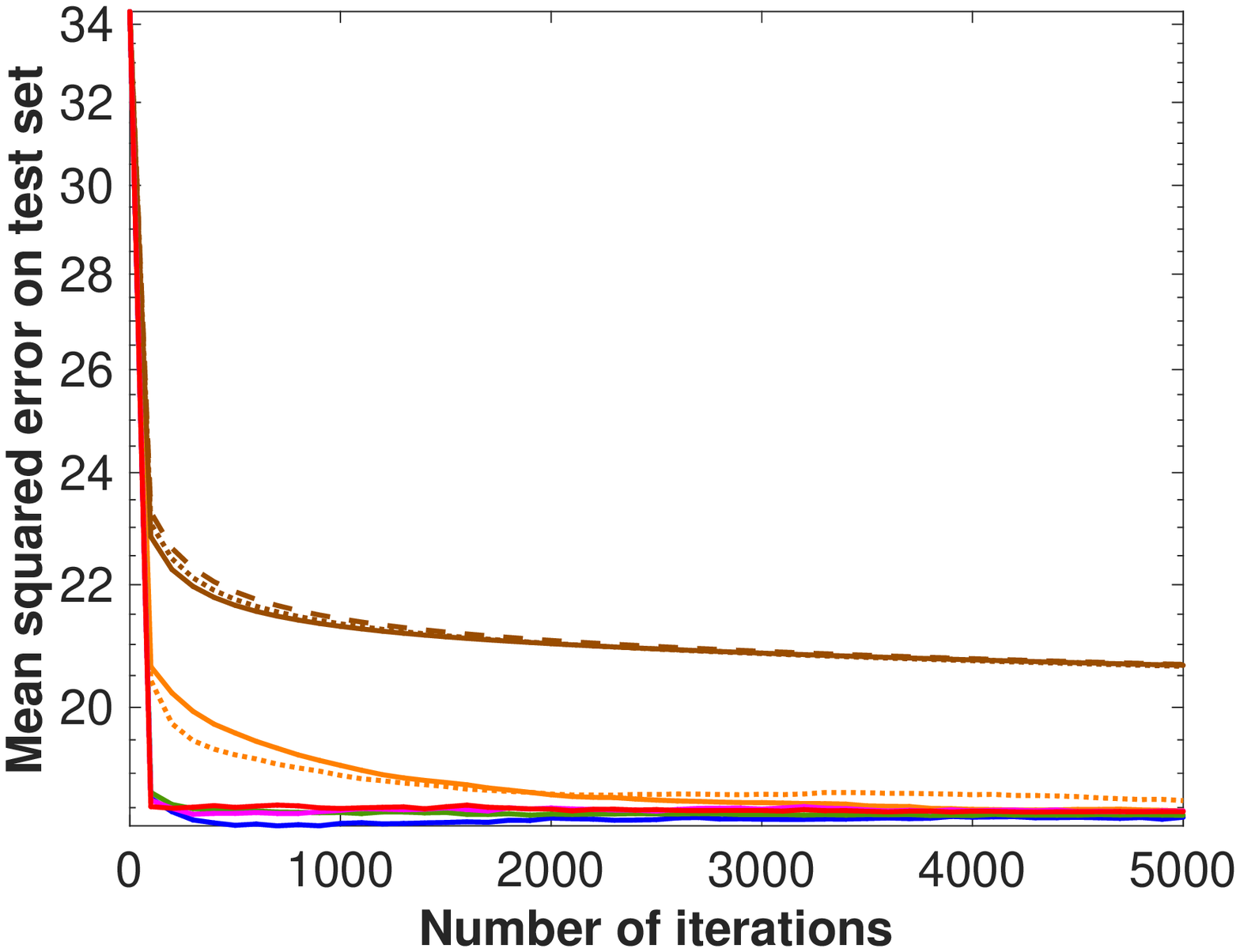}\\
		
		{\small  (b-3) $\alpha_0=0.01$.}
		
	\end{center} 
	\end{minipage}
	\hspace*{-0.1cm}
	\begin{minipage}[t]{.2\textwidth}
	\begin{center}
		\includegraphics[width=\textwidth]{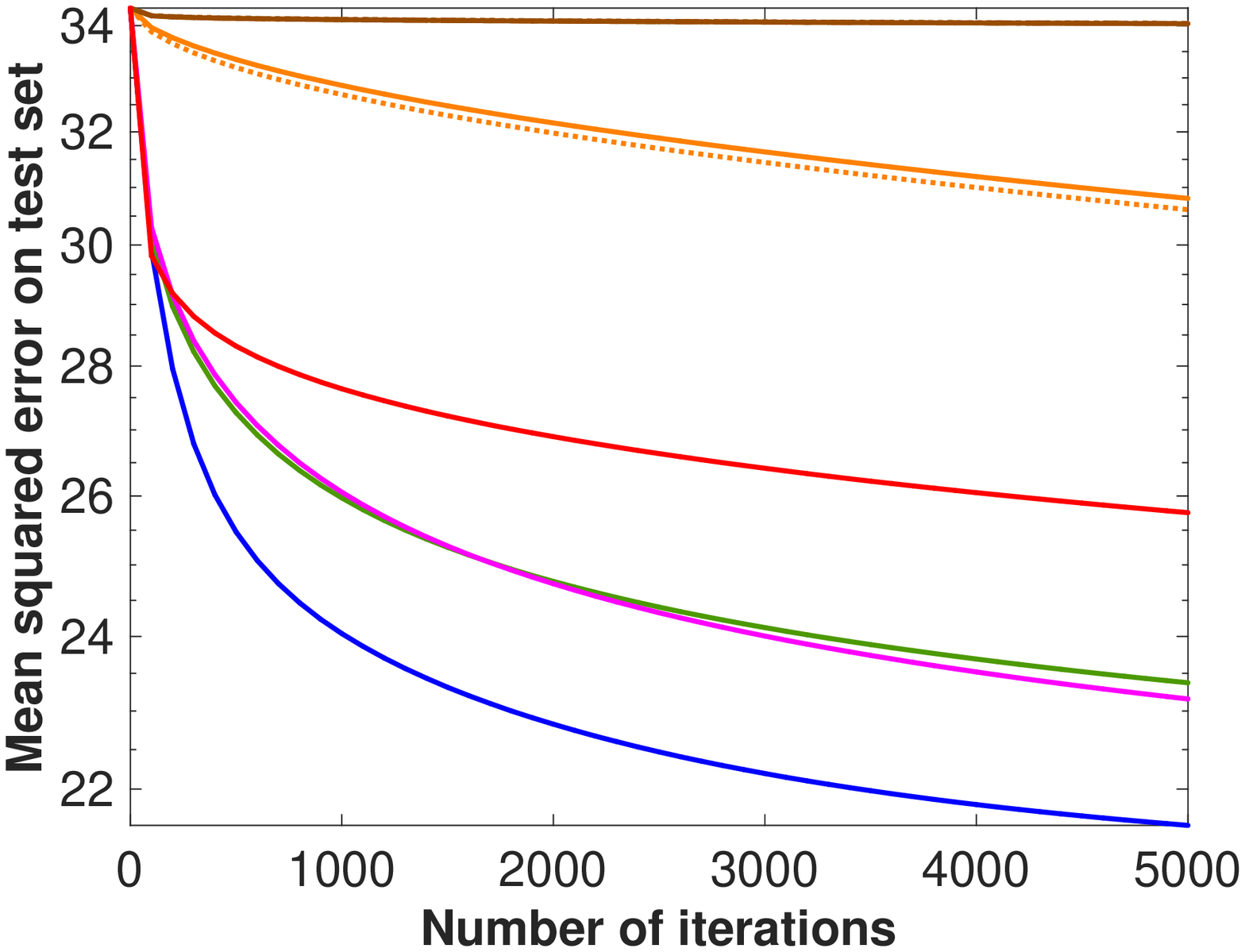}\\
		
		{\small  (b-4) $\alpha_0=0.0001$.}
		
	\end{center} 
	\end{minipage}	
	
	\vspace{0.3cm}

	{\small\bf (b) MSE on test set.}
	
	\caption{Jester dataset for the MC problem.}

\label{appfig:MC_results_jester}
\end{center}
\end{figure*}

\begin{figure*}[htbp]
\begin{center}
	\begin{minipage}[t]{.33\textwidth}
	\begin{center}
		\includegraphics[width=\textwidth]{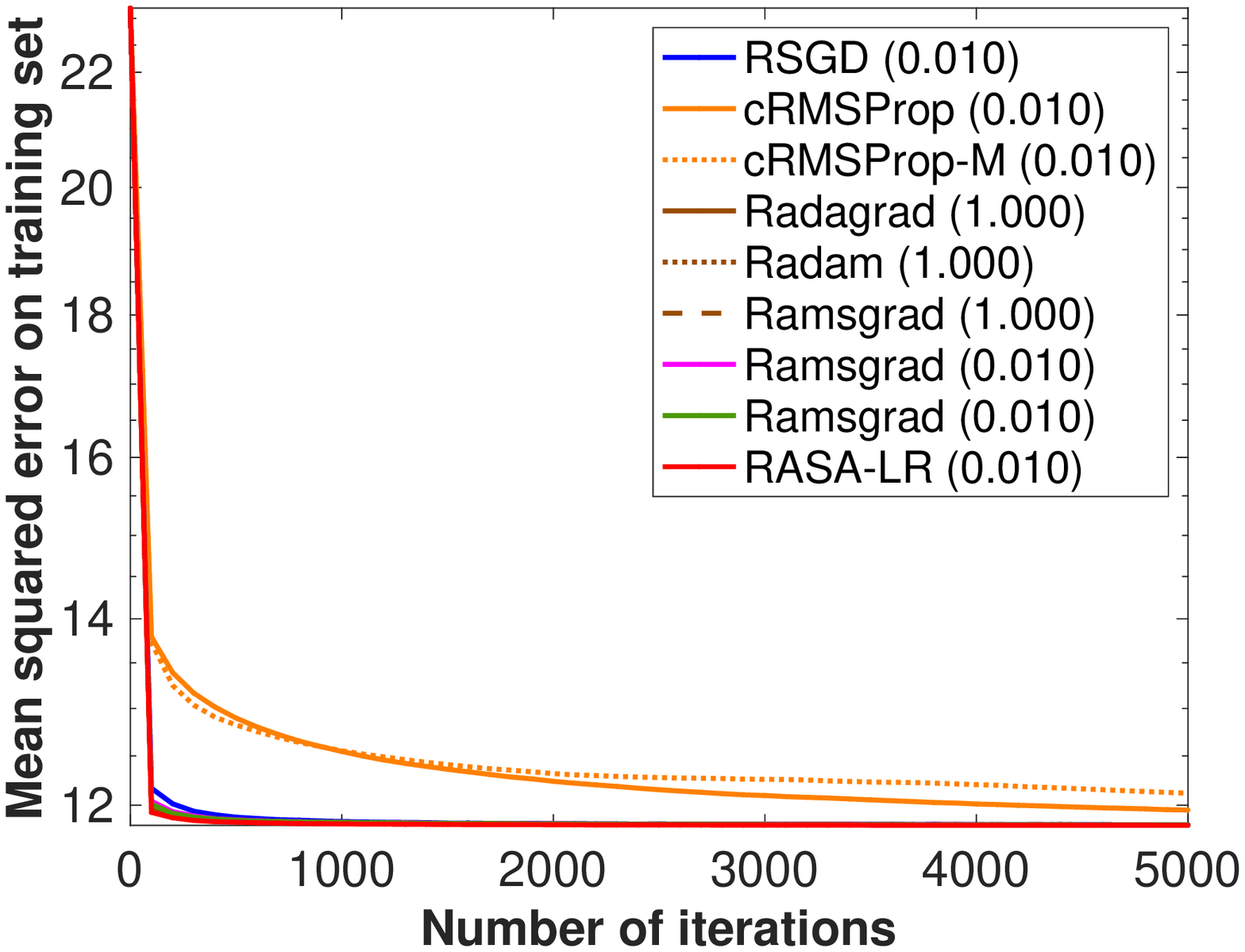}\\
		
		{\small  (a) MSE on training set.}
		
	\end{center} 
	\end{minipage}
	\hspace*{1cm}
	\begin{minipage}[t]{.33\textwidth}
	\begin{center}
		\includegraphics[width=\textwidth]{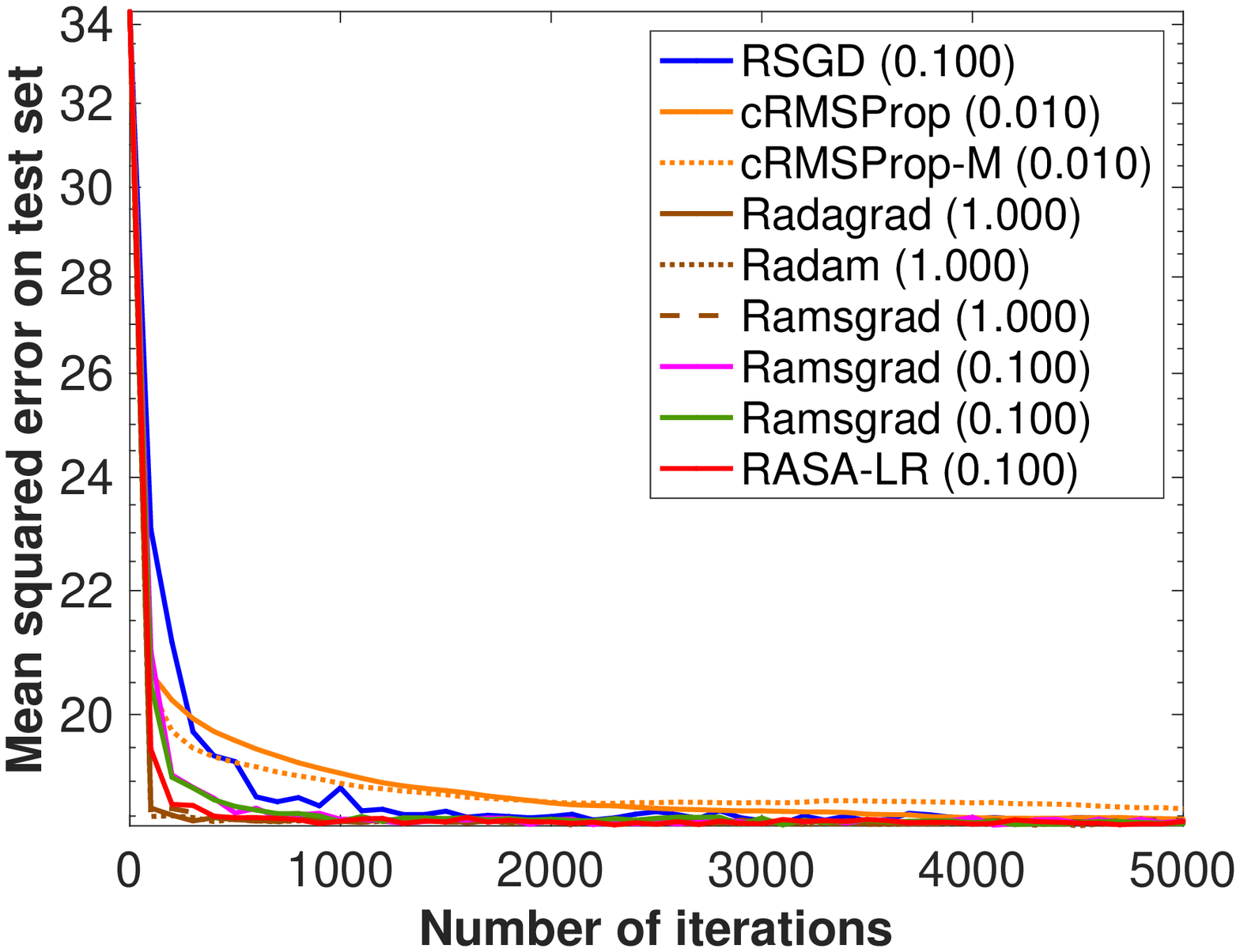}\\
		
		{\small  (b) MSE on test set.}
		
	\end{center} 
	\end{minipage}	
	
	\caption{Best-tuned results on the Jester dataset for the MC problem.}

\label{appfig:MC_results_jester_best}
\end{center}
\end{figure*}

\subsection{\changeHK{Results for Algorithm \ref{Alg:R-AGD-diag-variable-beta} (variable $\beta$)}}

\changeHK{This section evaluates the performances of Algorithm \ref{Alg:R-AGD-diag-variable-beta} by comparing it with RASA-LR of Algorithm \ref{Alg:R-AGD-diag} on the PCA problem with synthetic datasets. Figures \ref{appfig:VariBeta_PCA_results_Syn} (a) and (b) show all the results for different initial step size choices $\alpha_0=\{0.5, 0.1, 0.05, 0.01\}$ under well-conditioned and ill-conditioned cases, which correspond to Figures \ref{appfig:PCA_results_Syn_well_cond} and \ref{appfig:PCA_results_Syn_ill_cond}, respectively. From the figures, in both the cases, Algorithm \ref{Alg:R-AGD-diag-variable-beta} is comparable to RASA-LR. However, the best result of Algorithm \ref{Alg:R-AGD-diag-variable-beta} is slightly inferior to that of RASA-LR. In addition, it should be  emphasized that Algorithm \ref{Alg:R-AGD-diag-variable-beta} is computationally more inefficient than RASA as explained in Section \ref{appSec:VariBeataAlgorithm}.}

\begin{figure*}[htbp]
\begin{center}
	\begin{minipage}[t]{.2\textwidth}
	\begin{center}
		\includegraphics[width=\textwidth]{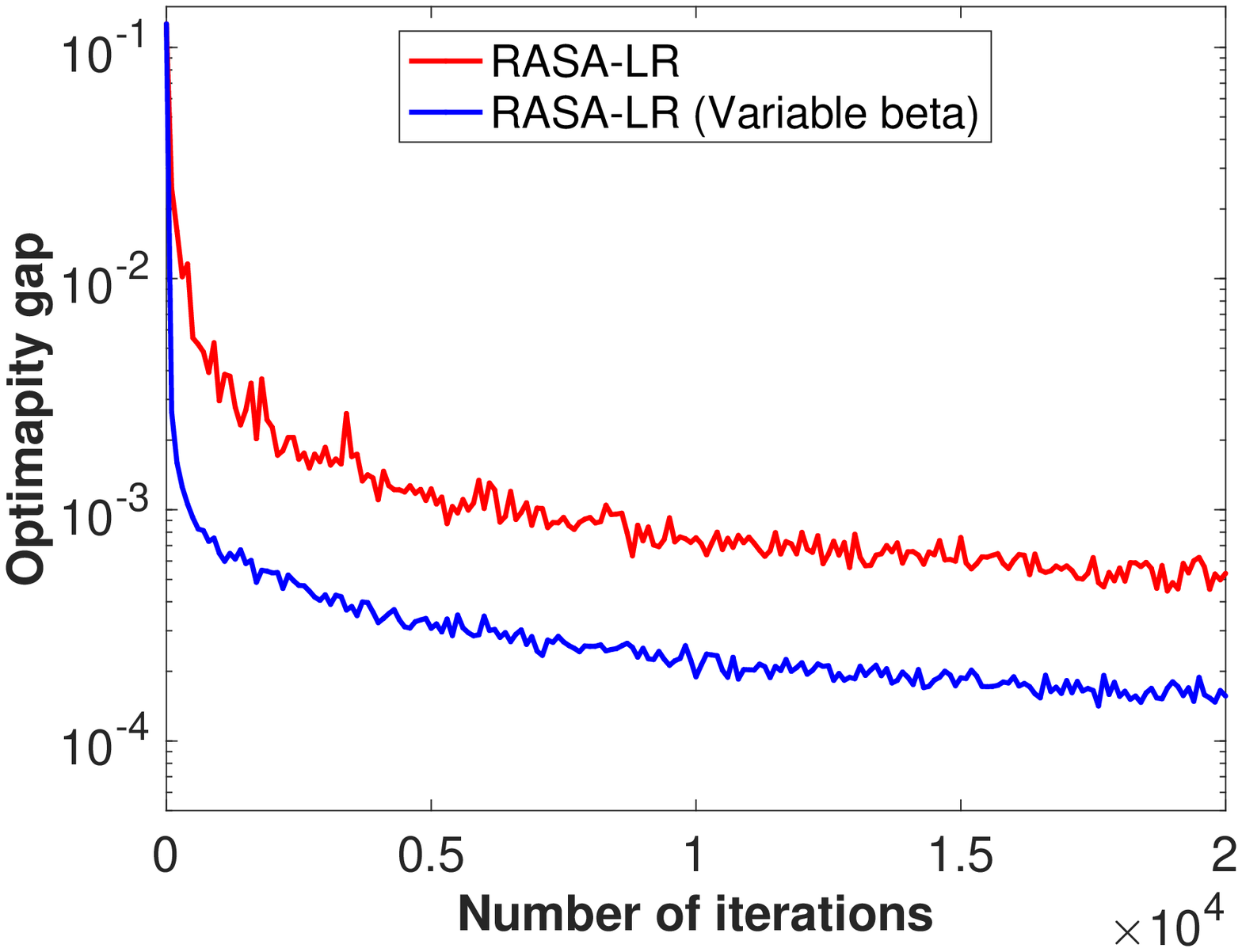}\\
		
		{\small  (a-1) $\alpha_0=0.5$.}
	\end{center}		
	\end{minipage}	
	\hspace*{-0.1cm}
	\begin{minipage}[t]{.2\textwidth}
	\begin{center}
		\includegraphics[width=\textwidth]{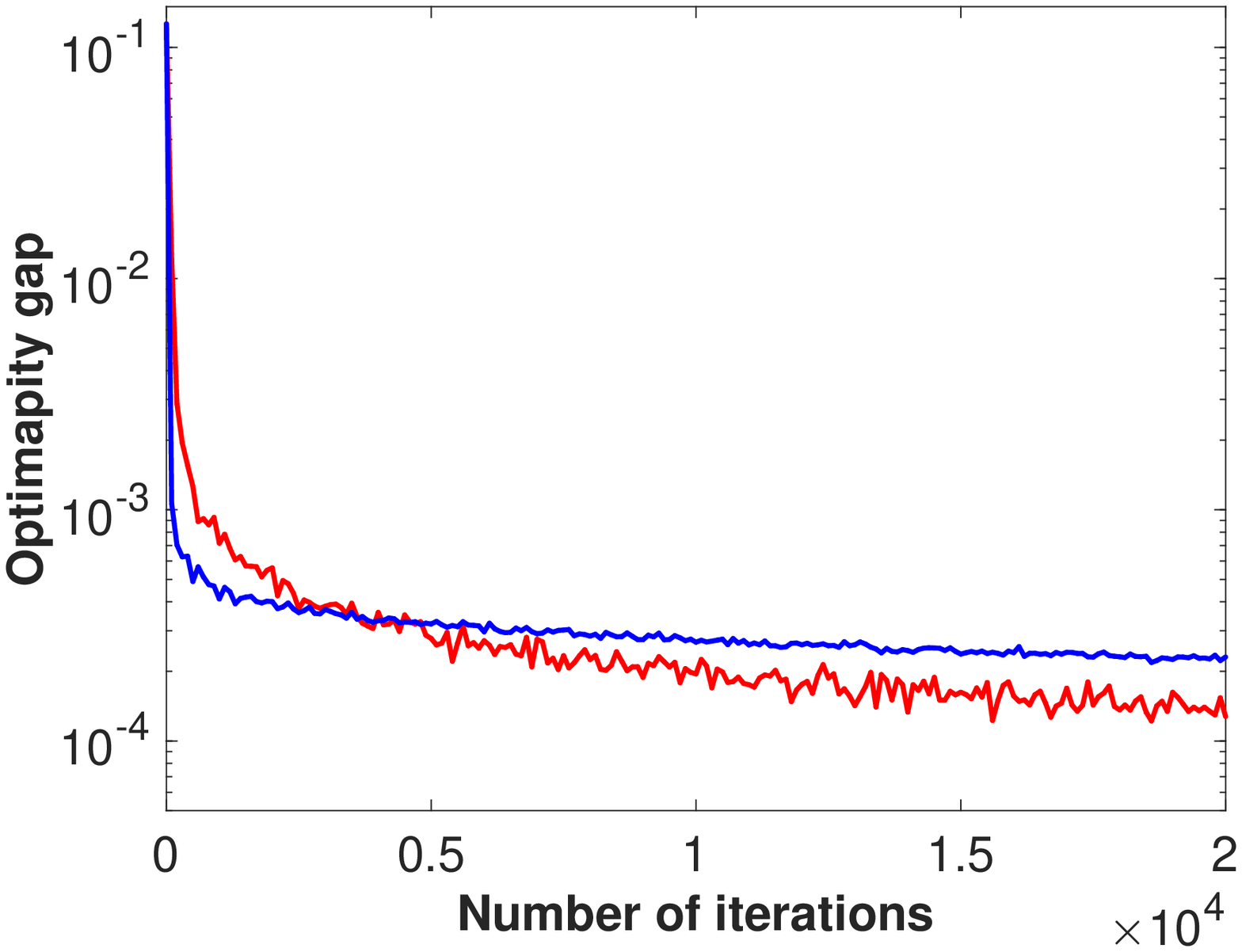}\\
		
		{\small  (a-2) $\alpha_0=0.1$.}
		
	\end{center} 
	\end{minipage}
	\hspace*{-0.1cm}
	\begin{minipage}[t]{.2\textwidth}
	\begin{center}
		\includegraphics[width=\textwidth]{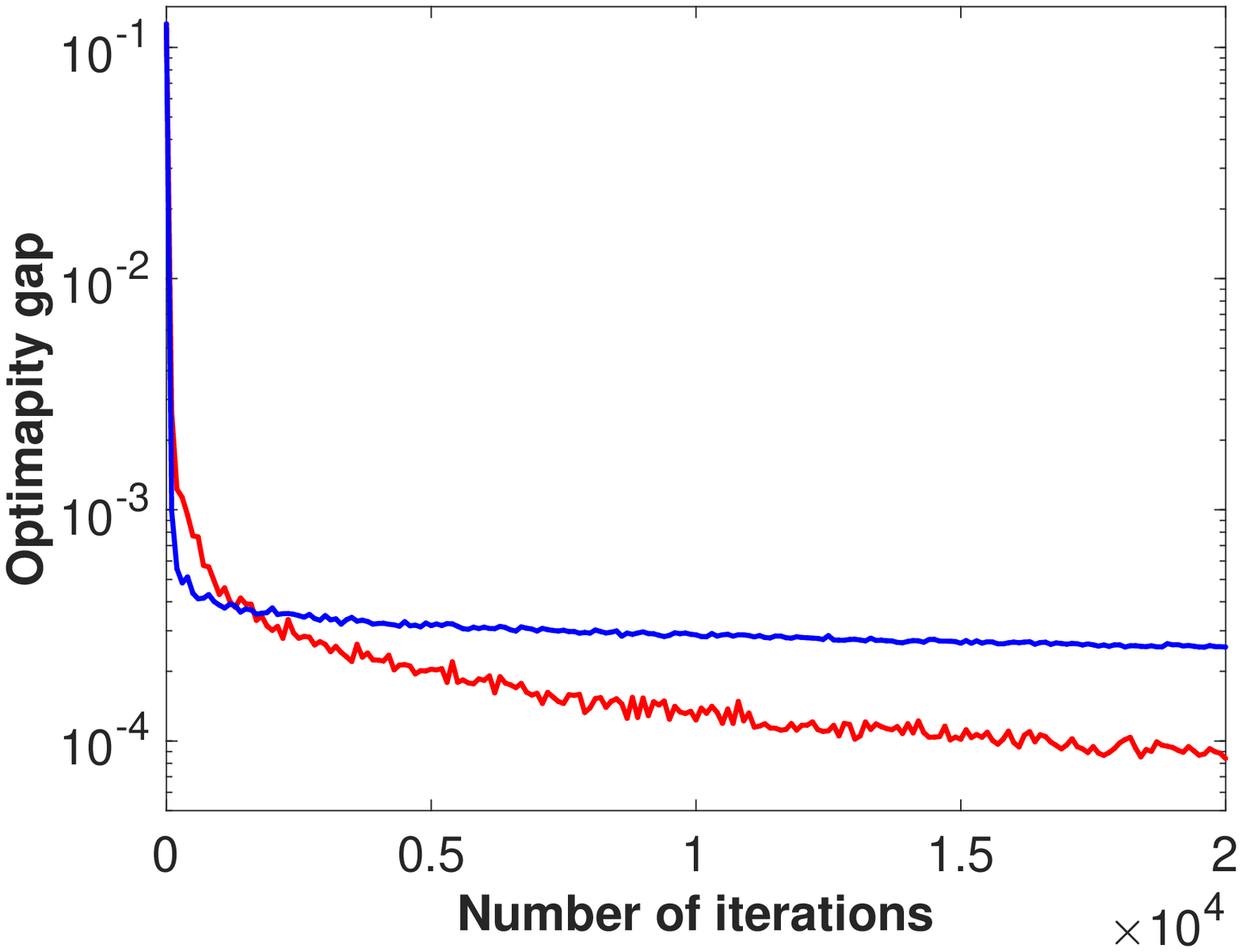}\\
		
		{\small  (a-3) $\alpha_0=0.05$.}
	\end{center}		
	\end{minipage}	
	\hspace*{-0.1cm}
	\begin{minipage}[t]{.2\textwidth}
	\begin{center}
		\includegraphics[width=\textwidth]{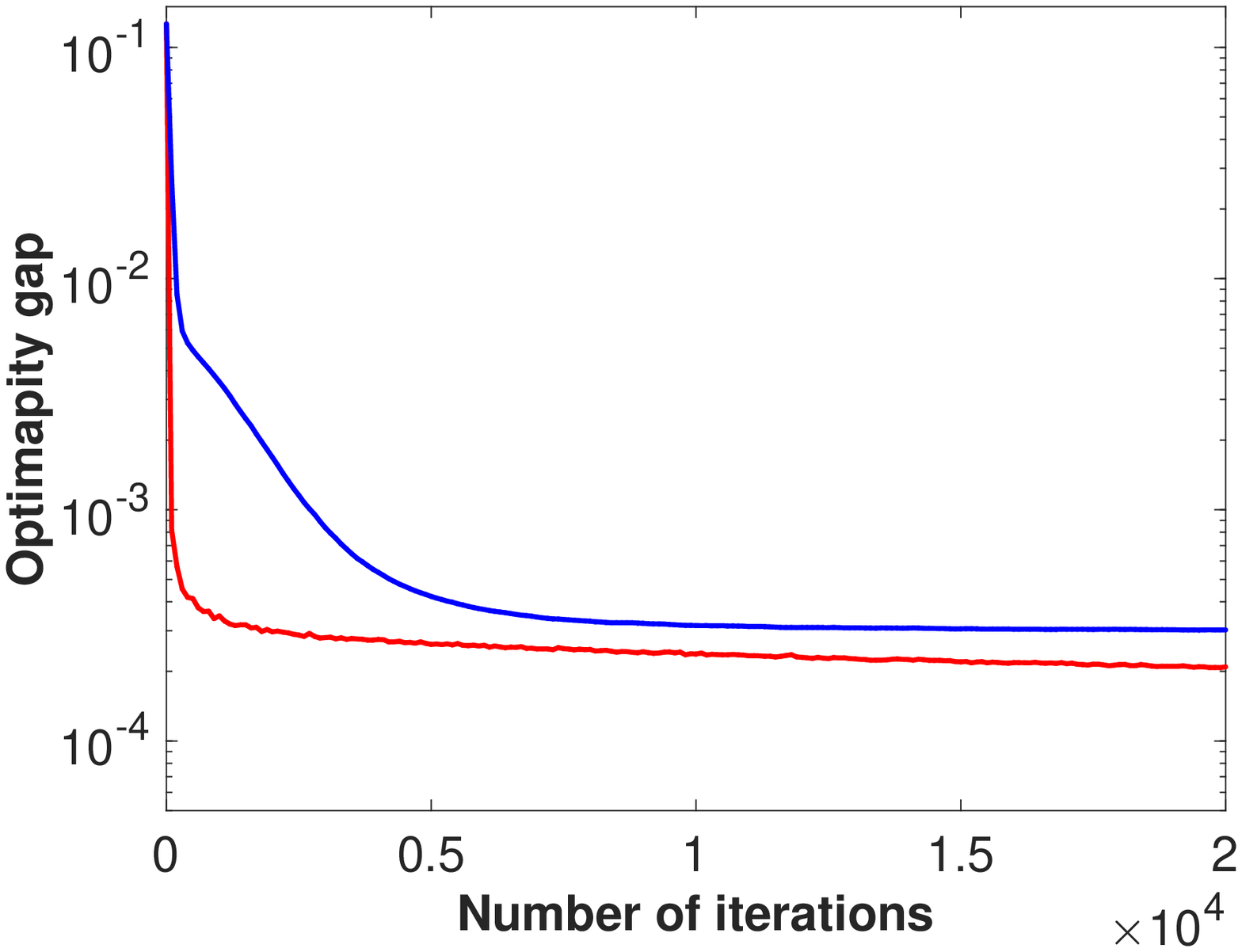}\\
		
		{\small  (a-4) $\alpha_0=0.01$.}
		
	\end{center} 
	\end{minipage}
	\vspace*{0.3cm}

	{\small\bf  (a) well-conditioned case.}
	\vspace*{0.5cm}

	\begin{minipage}[t]{.2\textwidth}
	\begin{center}
		\includegraphics[width=\textwidth]{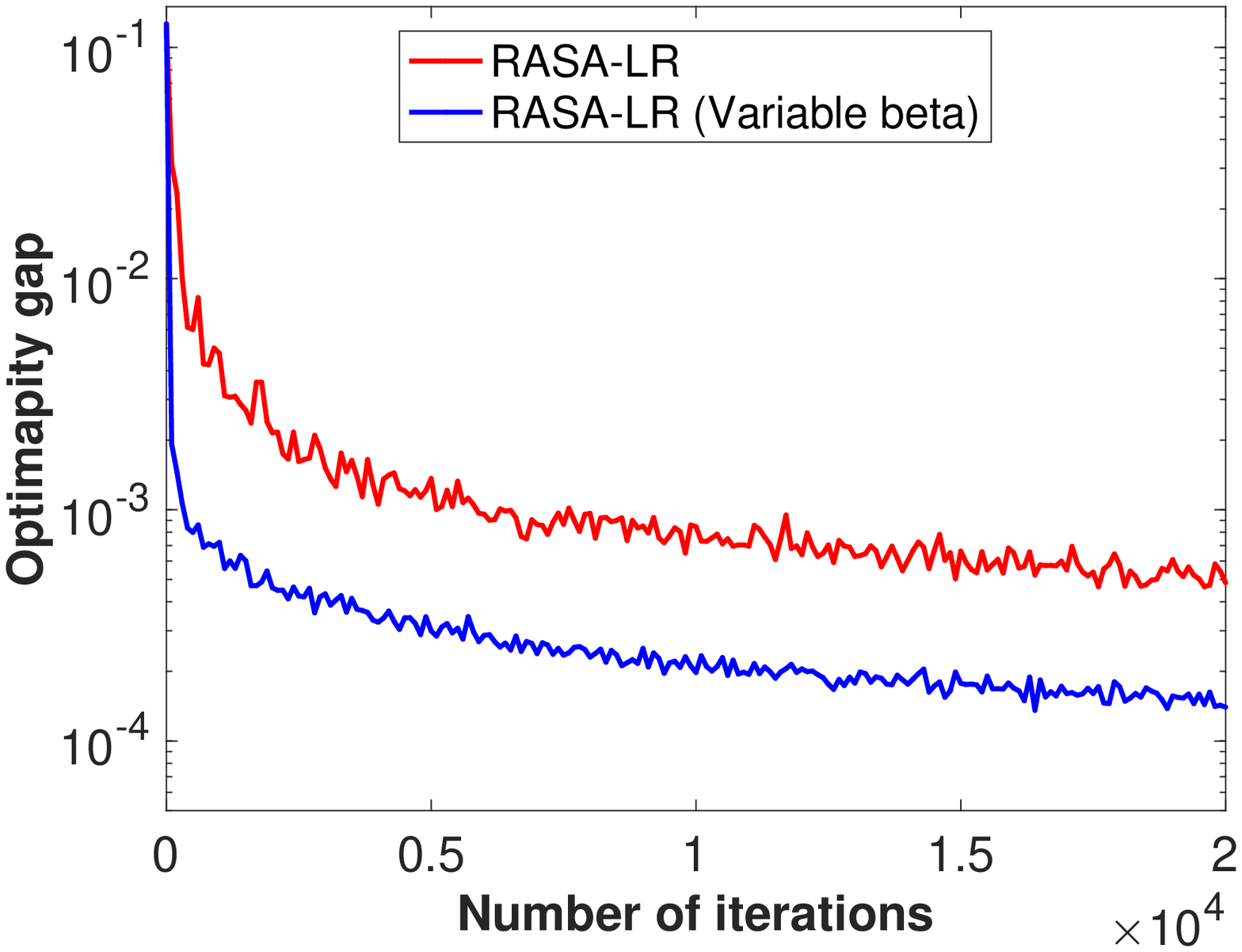}\\
		
		{\small  (b-1) $\alpha_0=0.5$.}
	\end{center}		
	\end{minipage}	
	\hspace*{-0.1cm}
	\begin{minipage}[t]{.2\textwidth}
	\begin{center}
		\includegraphics[width=\textwidth]{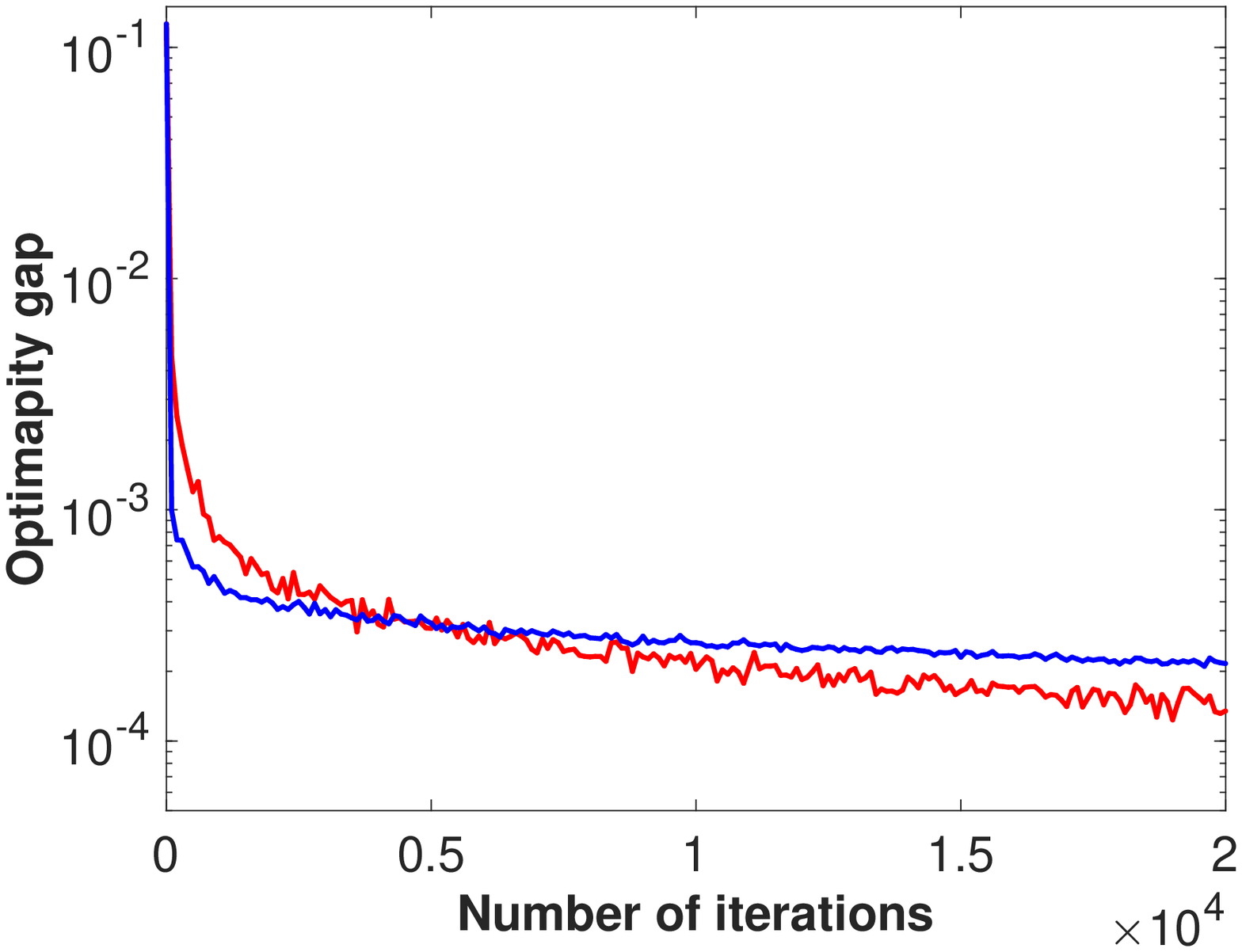}\\
		
		{\small  (b-2) $\alpha_0=0.1$.}
		
	\end{center} 
	\end{minipage}
	\hspace*{-0.1cm}	
	\begin{minipage}[t]{.2\textwidth}
	\begin{center}
		\includegraphics[width=\textwidth]{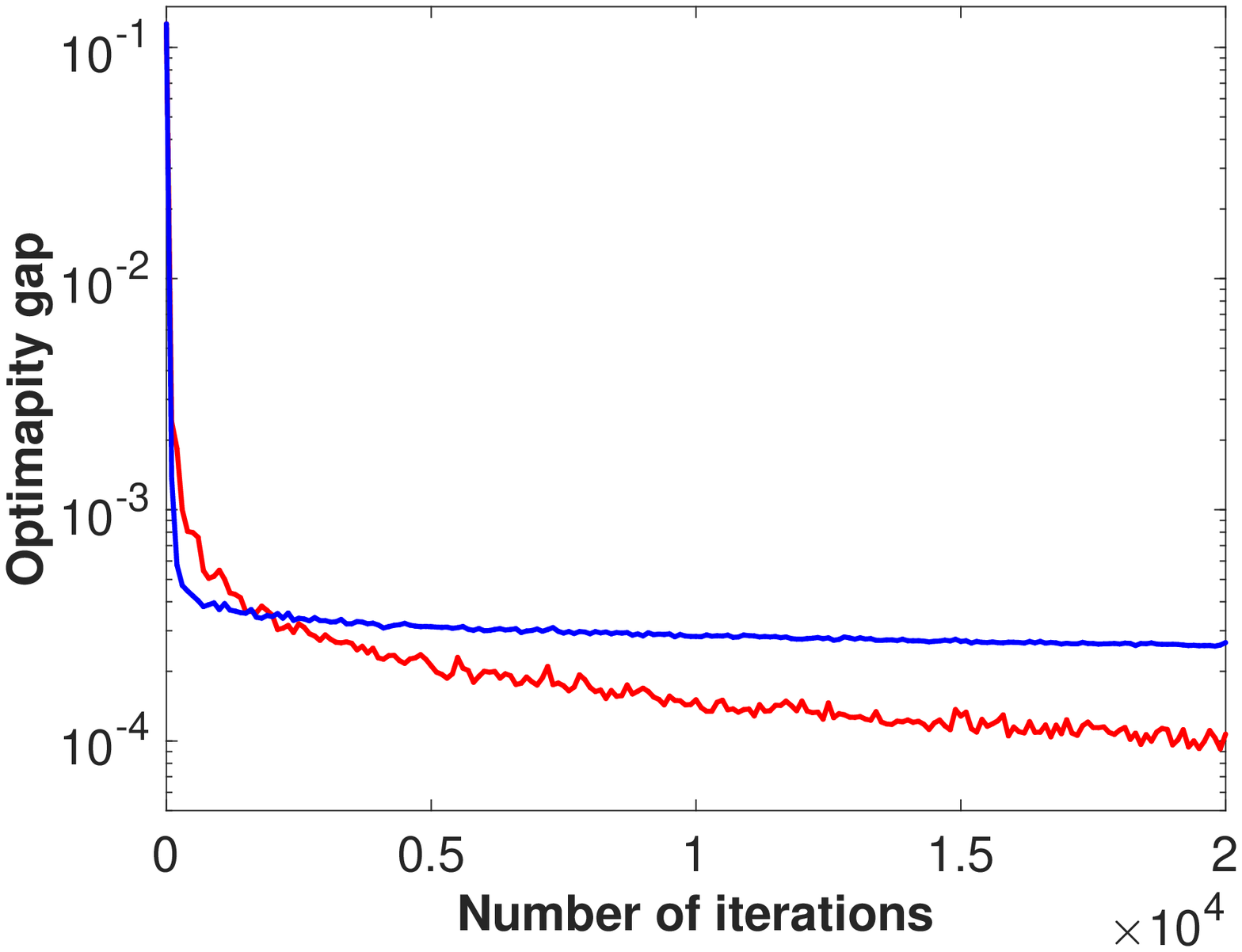}\\
		
		{\small  (b-3) $\alpha_0=0.05$.}
	\end{center}		
	\end{minipage}	
	\hspace*{-0.1cm}
	\begin{minipage}[t]{.2\textwidth}
	\begin{center}
		\includegraphics[width=\textwidth]{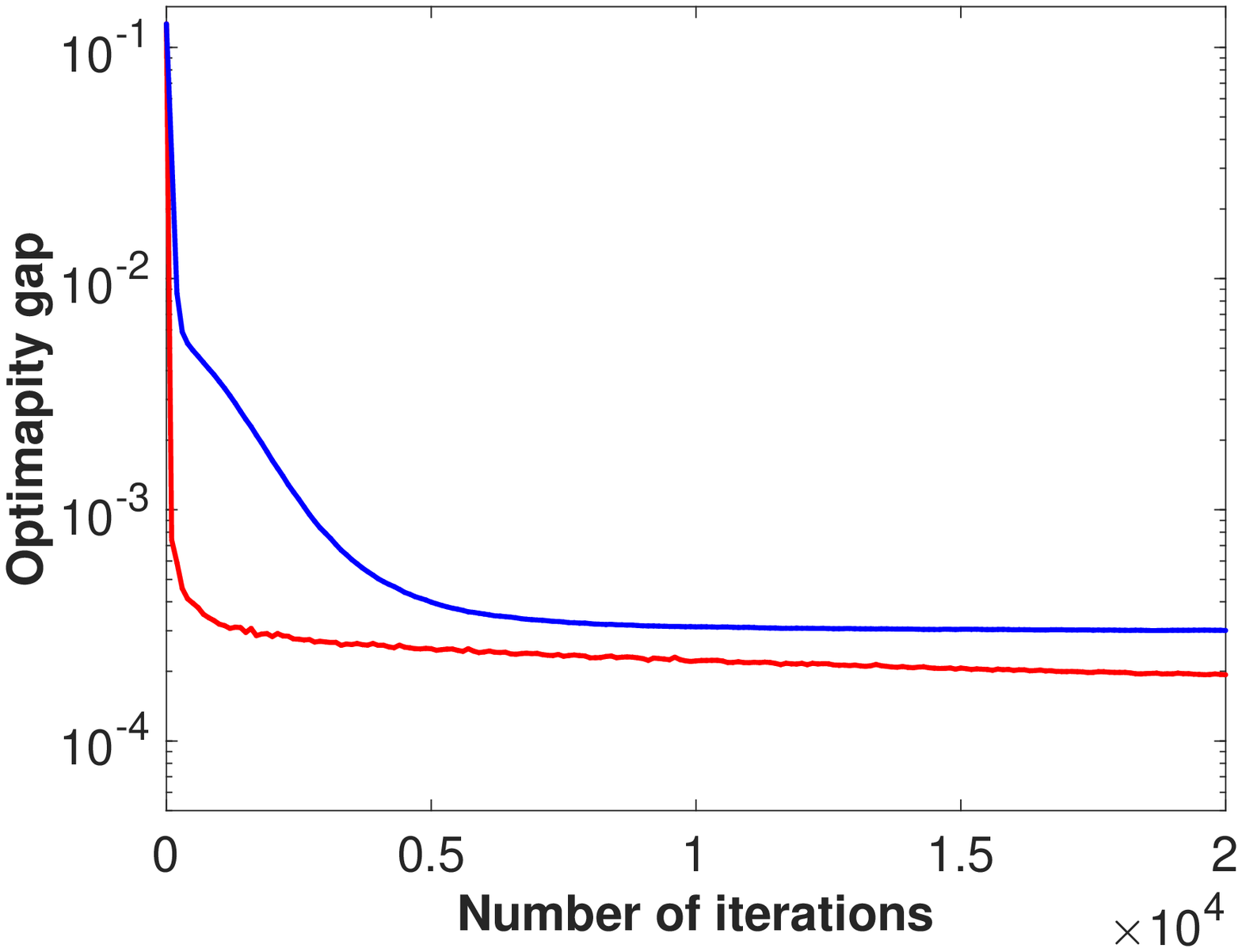}\\
		
		{\small  (b-4) $\alpha_0=0.01$.}
		
	\end{center} 
	\end{minipage}
	
	\vspace*{0.3cm}
	{\small\bf  (b) ill-conditioned case.}
	\vspace*{0.2cm}

	\caption{\changeHK{Comparison between Algorithms \ref{Alg:R-AGD-diag} and \ref{Alg:R-AGD-diag-variable-beta} on synthetic datasets for the PCA problem ({\bf Case P1}).}}

\label{appfig:VariBeta_PCA_results_Syn}
\end{center}
\end{figure*}

\section{Manifolds}\label{app:sec:manifolds}


{\bf Stiefel manifold ${\rm St}(r,n)$:} The Stiefel manifold is the set of orthogonal $r$-frames in $\mathbb{R}^n$ for some $r \leq n$, and it is an embedded submanifold of $\mathbb{R}^{n \times r}$. The orthogonal group ${\rm O}(n)$ is a special case of the Stiefel manifold, i.e., ${\rm O}(n) = {\rm St}(n,n)$. Because ${\rm St}(r,n)$ is a submanifold embedded in $\mathbb{R}^{n \times r}$, we can endow the canonical inner product in $\mathbb{R}^{n \times r}$ as a Riemannian metric $\langle \xi, \eta \rangle_{\scriptsize \mat{U}}= {\rm tr}(\xi^\top \eta)$ for $\xi, \eta \in T_{\scriptsize \mat{U}}{\rm St}(r,n)$. With this Riemannian metric, the projection onto the tangent space $T_{\scriptsize \mat{U}}{\rm St}(r,n)$ is defined as an orthogonal projection ${\rm P}_{\scriptsize \mat{U}}(\mat{W})=\mat{W} - \mat{U} {\rm sym}(\mat{U}^\top \mat{W})$ for $\mat{U} \in {\rm St}(r,n)$ and $\mat{W} \in \mathbb{R}^{n \times r}$.
A popular retraction is 
$
R_{\scriptsize \mat{U}}(\xi) = {\rm qf}( \mat{U}+\xi)
$
for $ \mat{U} \in {\rm St}(r,n)$ and $\xi \in T_{\scriptsize \mat{U}}{\rm St}(r,n)$, where ${\rm qf}( \cdot)$ extracts the orthonormal factor based on QR decomposition. Other details about optimization-related notions on the Stiefel manifold are in \citep{Absil_OptAlgMatManifold_2008}.

{\bf Grassmann manifold ${\rm Gr}(r,n)$: }
A point on the Grassmann manifold is an equivalence class represented by a $n \times r$ orthogonal matrix $\mat{U}$ with orthonormal columns, i.e., $\mat{U}^\top\mat{U}=\mat{I}$. Two orthogonal matrices express the same element on the Grassmann manifold if they are related by right multiplication of an $r\times r$ orthogonal matrix $\mat{O} \in {\rm O}(r)$. Equivalently, an element of ${\rm Gr}(r,n)$ is identified with a set of $n \times r$ orthogonal matrices $[\mat{U}]: =\{\mat{U}\mat{O} :\mat{O} \in {\rm O}(r)\}$. That is, ${\rm Gr}(r,n) :={\rm St}(r,n)/ {\rm O}(r)$, where ${\rm St}(r,n)$ is the {\it Stiefel manifold} that is the set of matrices of size $n \times r$ with orthonormal columns. The Grassmann manifold has the structure of a Riemannian quotient manifold 
\citep[Section~3.4]{Absil_OptAlgMatManifold_2008}. 
A popular retraction on the Grassmann manifold is $R_{\scriptsize \mat{U}}(\xi)={\rm qf}(\mat{U}+\xi)$. Other details about optimization-related notions on the Grassmann manifold are in \citep{Absil_OptAlgMatManifold_2008}.

\section{Problems and derivations of the Riemannian gradient}\label{app:sec:problems}

{\bf ICA problem} \citep{theis09a}: 
{A particular variant to solve the independent components analysis (ICA) problem is through joint diagonalization on the Stiefel manifold, i.e.,}
\begin{eqnarray*}
\min_{{\scriptsize \mat{U}} \in \mathbb{R}^{n \times r}} f_{\rm ica}(\mat{U}) :=
- \frac{1}{N}\sum_{i=1}^N \| {\rm diag}(\mat{U}^\top \mat{C}_i \mat{U}) \|^2_F,
\end{eqnarray*}
where $\| {\rm diag}(\mat{A}) \|^2_F$ defines the sum of the squared diagonal elements of $\mat{A}$. $\mat{C}_i$ can, for example, be cumulant matrices or time-lagged covariance matrices {of size $n\times n$}. 
The Riemannian gradient $\gradf_{\rm ica}(\mat{U})$ of the cost function 
$f_{\rm ica}(\mat{U})$ is 
\begin{eqnarray*}
\gradf_{\rm ica}(\mat{U})={\rm P}_{\scriptsize\mat{U}}~({\rm egrad}f_{\rm ica}(\mat{U}) )
= {\rm P}_{\scriptsize\mat{U}} \left(- {\frac{1}{N}}\sum_{i=1}^N 4 \mat{C}_i \mat{U}~{\rm ddiag}(\mat{U}^\top \mat{C}_i \mat{U})\right),
\end{eqnarray*}
where ${\rm egrad}f_{\rm ica}(\mat{U})$ is the Euclidean gradient of $f_{\rm ica}(\mat{U})$, {${\rm ddiag}$ is the diagonal matrix,} and ${\rm P}_{\scriptsize\mat{U}}$ denotes the orthogonal projection onto the tangent space of $\mat{U}$, i.e., $T_{\scriptsize\mat{U}}{\rm St}(r, n)$, which is defined as ${\rm P}_{\scriptsize\mat{U}}(\mat{W}) = \mat{W} - \mat{U} {\rm sym}(\mat{U}^\top \mat{W})$, where ${\rm sym}(\mat{A})$ represents the symmetric matrix $(\mat{A} + \mat{A}^\top)/2$.

{\bf PCA problem}: Given an orthonormal matrix projector $\mat{U} \in {\rm St}(r,n)$, which is the Stiefel manifold that is the set of matrices of size $n \times r$ with orthonormal columns, the principal components analysis (PCA) problem is to minimize the sum of squared residual errors between projected data points and the original data as
\begin{eqnarray*}
\label{Eq:PCA}
\min_{{\scriptsize \mat{U} \in {\rm St}(r,n)}} \frac{1}{N}  \sum_{i=1}^N \| \vec{z}_i -  \mat{U}\mat{U}^\top \vec{z}_i \|_2^2,
\end{eqnarray*}
where $\vec{z}_i$ is a data vector of size $n\times 1$. This problem is equivalent to 
\begin{eqnarray*}
\min_{{\scriptsize \mat{U} \in {\rm St}(r,n)}}  {f_{\rm pca}(\mat{U})} :=-\frac{1}{N} \sum_{i=1}^N \vec{z}_i^\top\mat{U}\mat{U}^\top\vec{z}_i.
\end{eqnarray*}

{Similar to the arguments in the ICA problem above, the expressions of the Riemannian gradient for the PCA problem on the Grassmann manifold is as follows:
\begin{eqnarray*}
\gradf_{\rm pca}(\mat{U})&= &{\rm P}_{\scriptsize\mat{U}}~({\rm egrad}f_{\rm pca}(\mat{U}) )
= {\rm P}_{\scriptsize\mat{U}} \left(- {\frac{1}{N}}\sum_{i=1}^N  2 \vec{z}_i \vec{z}_i^\top \mat{U} \right),
\end{eqnarray*}
where the orthogonal projector ${\rm P}_{\scriptsize\mat{U}}(\mat{W}) = \mat{W} - \mat{U}\mat{U}^\top\mat{W}$.}

{\bf MC problem}: The matrix completion (MC) problem amounts to completing an incomplete matrix $\mat{Z}$, say of size $n \times N$, from a small number of entries by assuming a low-rank model for the matrix. If $\Omega$ is the set of the indices for which we know the entries in $\mat{Z}$, the rank-$r$ MC problem amounts to solving the problem
\begin{equation*}
\label{Eq:MC_batch}
\begin{array}{ll}
\min_{{\scriptsize \mat{U}} \in \mathbb{R}^{n \times r}, {\scriptsize \mat{A}} \in \mathbb{R}^{r \times N}} \|(\mat{UA})_{\Omega} - {\mat Z}_{\Omega} \|_F^2,
\end{array}
\end{equation*}
where $\Omega$ is the set of indices whose entries are known. Partitioning $\mat{Z} = [\vec{z}_1, \vec{z}_2, \ldots, \vec{z}_i] $, the problem is equivalent to the problem
\begin{eqnarray*}
\label{Eq:MC}
\min_{{\scriptsize \mat{U}} \in \mathbb{R}^{n \times r}, \vec{a}_i \in \mathbb{R}^{r}} 
\frac{1}{N} \sum_{i=1}^N \|  {(\mat{U} \vec{a}_i)}_{\Omega_i} - {\vec{z}_i}_{\Omega_i} \|_2^2,
\end{eqnarray*}
where $\vec{z}_i \in \mathbb{R}^n$ and ${\Omega_i}$ is the set of indices (of known entries) for the $i$-th column. Given $\mat{U}$, $\vec{a}_i$ admits the closed-form solution $a_i =  \mat{U}_{\Omega_1}^{\dagger}{\vec{z}_i}_{\Omega_i}$, where $\dagger$ is the pseudo inverse and $\mat{U}_{\Omega_i}$ and ${\vec{z}_i}_{\Omega_i}$ are respectively the rows of $\mat{U}$ and $\vec{z}_i$ corresponding to the row indices in $\Omega_i$. Consequently, the problem only depends on the column space of $\mat{U}$ 
and is on the Grassmann manifold \citep{boumal15a}, i.e.,
\begin{eqnarray*}
\min_{{\scriptsize \mat{U}} \in {\rm Gr}(r,n)} f_{\rm mc} (\mat{U}) := \min_{\vec{a}_i \in \mathbb{R}^{r}} 
\frac{1}{N} \sum_{i=1}^N \|  (\mat{U} \vec{a}_i)_{\Omega_i} - {\vec{z}_i}_{\Omega_i} \|_2^2.
\end{eqnarray*}

{The expressions of the Riemannian gradient for the MC problem on the Grassmann manifold is as follows:
\begin{eqnarray*}
\gradf_{\rm mc}(\mat{U})&= &{\rm P}_{\scriptsize\mat{U}}~({\rm egrad}f_{\rm mc}(\mat{U})) 
= {\rm P}_{\scriptsize\mat{U}} \left( {\frac{1}{N}}\sum_{i=1}^N  2 ((\mat{U} \vec{a}_i)_{\Omega_i} - {\vec{z}_i}_{\Omega_i}) \vec{a}_i^\top  \right), \\
\end{eqnarray*}
where the orthogonal projector ${\rm P}_{\scriptsize\mat{U}}(\mat{W}) = \mat{W} - \mat{U}\mat{U}^\top\mat{W}$. 

\end{document}